\documentclass[english]{article}
\usepackage[T1]{fontenc}
\usepackage[latin9]{inputenc}
\usepackage{geometry}
\usepackage[active]{srcltx}
\usepackage{color}
\usepackage{babel}
\usepackage{array}
\usepackage{verbatim}
\usepackage{float}
\usepackage{mathtools}
\usepackage{bm}
\usepackage{multirow}
\usepackage{amsmath}
\usepackage{amssymb}
\usepackage{graphicx}
\usepackage{booktabs}
\usepackage[unicode=true,
 bookmarks=false,
 breaklinks=false,pdfborder={0 0 1},colorlinks=false]
 {hyperref}
\hypersetup{
 colorlinks,linkcolor=red,anchorcolor=blue,citecolor=blue}

\makeatletter

\providecommand{\tabularnewline}{\\}
\floatstyle{ruled}
\newfloat{algorithm}{tbp}{loa}
\providecommand{\algorithmname}{Algorithm}
\floatname{algorithm}{\protect\algorithmname}

\usepackage{babel}

\usepackage[T1]{fontenc}
\usepackage{cite}\usepackage{amsthm}\usepackage{dsfont}\usepackage{array}\usepackage{mathrsfs}\usepackage{comment}\onecolumn
\usepackage{natbib}
\usepackage{color}\usepackage{babel}

\allowdisplaybreaks

\usepackage{enumitem}
\setlist[itemize]{leftmargin=1.5em}
\setlist[enumerate]{leftmargin=1.5em}

\usepackage{babel}
\usepackage{algorithm}
\usepackage{algorithmic}
\usepackage{arydshln}
\usepackage{caption}

\DeclareMathOperator{\ind}{\mathds{1}}  

\numberwithin{equation}{section}

\definecolor{yxc}{RGB}{255,0,0}
\definecolor{yjc}{RGB}{125,0,0}
\definecolor{cm}{RGB}{0,0,200}
\definecolor{kzw}{RGB}{0,150,0}
\definecolor{byw}{RGB}{255,0,0}

\makeatother

\addto\captionsbritish{\renewcommand{\algorithmname}{Algorithm}}

\begin{document}
\theoremstyle{plain} \newtheorem{lemma}{\textbf{Lemma}} \newtheorem{proposition}{\textbf{Proposition}}\newtheorem{theorem}{\textbf{Theorem}}\setcounter{theorem}{0}
\newtheorem{corollary}{\textbf{Corollary}} \newtheorem{assumption}{\textbf{Assumption}}
\newtheorem{example}{\textbf{Example}} \newtheorem{definition}{\textbf{Definition}}
\newtheorem{condition}{\textbf{Condition}}
\newtheorem{fact}{\textbf{Fact}}\newtheorem{property}{Property}
\theoremstyle{definition}

\theoremstyle{remark}\newtheorem{remark}{\textbf{Remark}}\newtheorem{claim}{Claim}\newtheorem{conjecture}{Conjecture}
\title{Convex and Nonconvex Optimization Are Both Minimax-Optimal for Noisy
Blind Deconvolution under Random Designs\footnotetext{Author
	names are sorted alphabetically.}}
\author{Yuxin Chen\thanks{Department of Electrical and Computer Engineering, Princeton University, Princeton,
		NJ 08544, USA; Email: \texttt{yuxin.chen@princeton.edu}.} \and Jianqing Fan\thanks{Department of Operations Research and Financial Engineering, Princeton
		University, Princeton, NJ 08544, USA; Email: \texttt{\{jqfan, bingyanw,
			yulingy\}@princeton.edu}.} \and Bingyan Wang\footnotemark[2] \and Yuling Yan\footnotemark[2]}
		
\maketitle
\begin{abstract}
We investigate the effectiveness of convex relaxation and nonconvex
optimization in solving bilinear systems of equations under two different
designs (i.e.~a sort of random Fourier design and Gaussian design). Despite
the wide applicability, the theoretical understanding about these
two paradigms remains largely inadequate in the presence of random noise.
The current paper makes two contributions by demonstrating that: 
(1) a two-stage nonconvex algorithm attains minimax-optimal
accuracy within a logarithmic number of iterations.
(2) convex relaxation also achieves minimax-optimal statistical accuracy vis-à-vis
random noise.  Both results significantly improve upon the state-of-the-art theoretical guarantees.

\end{abstract}


\tableofcontents{}

\section{Introduction and motivation}

Suppose we are interested in a pair of unknown objects $\bm{h}^{\star},\bm{x}^{\star}\in\mathbb{C}^{K}$
and are given a collection of $m$ nonlinear measurements taking the
following form
\begin{equation}
y_{j}=\boldsymbol{b}_{j}^{\mathsf{H}}\bm{h}^{\star}\boldsymbol{x}^{\star\mathsf{H}}\boldsymbol{a}_{j}+\xi_{j},\qquad1\leq j\leq m.\label{eq:measurement-BD}
\end{equation}
Here, $\bm{z}^{\mathsf{H}}$ denotes the conjugate transpose of a
vector $\bm{z}$, $\left\{ \xi_{j}\right\} $ stands for the additive
noise, whereas $\left\{ \bm{a}_{j}\right\} $ and $\left\{ \bm{b}_{j}\right\} $
are design vectors (or sampling vectors). The aim is to faithfully
reconstruct both $\bm{h}^{\star}$ and $\bm{x}^{\star}$ from the
above set of bilinear measurements.\footnote{This formulation is reminiscent of the problem of phase retrieval (or solving quadratic systems of equations). But the two problems turn out to be quite different due to the common assumptions imposed on the design vectors, as we shall elucidate in Section \ref{sec:Priorart}. } 

This problem of solving bilinear systems of equations spans multiple
domains in science and engineering, including but not limited to astronomy,
medical imaging, optics, and communication engineering 
\citep{campisi2016blind,jefferies1993restoration,wang1998blind,wunder2015sparse,tong1994blind,chan1998total}. 
Particularly worth emphasizing is the application of blind deconvolution
\citep{ahmed2013blind,kundur1996blind,ling2015self,ma2017implicit},
which involves recovering two unknown signals from their circular
convolution. As has been made apparent in the seminal work \citet{ahmed2013blind},
deconvolving two signals can be reduced to solving bilinear equations,
provided that the unknown signals lie within some \emph{a priori }known
subspaces; the interested reader is referred to \citet{ahmed2013blind}
for details. A variety of approaches have since been put forward for
blind deconvolution, most notable of which are convex relaxation and
nonconvex optimization \citep{ahmed2013blind,ling2017blind,li2019rapid,ma2017implicit,huang2018blind,ling2019regularized}.
Despite a large body of prior work tackling this problem, however,
where these algorithms stand vis-à-vis random noise remains unsettled,
which we seek to address in the current paper.

\subsection{Convex and nonconvex algorithms}

Among various algorithms that have been proposed for blind deconvolution,
two paradigms have received much attention: (1) convex relaxation
and (2) nonconvex optimization, both of which can be explained rather
simply. The starting point for both paradigms is a natural least-squares
formulation
\begin{equation}
\underset{\ensuremath{\boldsymbol{h},}\ensuremath{\boldsymbol{x}\in\mathbb{C}^{K}}}{\text{minimize}}\quad\sum_{j=1}^{m}\left|\boldsymbol{b}_{j}^{\mathsf{H}}\boldsymbol{hx}^{\mathsf{H}}\boldsymbol{a}_{j}-y_{j}\right|^{2},\label{eq:objncvx-LS}
\end{equation}
which is, unfortunately, highly nonconvex due to the bilinear structure
of the sampling mechanism. It then boils down to how to guarantee
a reliable solution despite the intrinsic nonconvexity.

\paragraph{Convex relaxation.} In order to tame nonconvexity, a
popular strategy is to lift the problem into higher dimension followed
by convex relaxation (namely, representing $\boldsymbol{hx}^{\mathsf{H}}$
by a matrix variable $\bm{Z}$ and then dropping the rank-1 constraint)
\citep{ahmed2013blind,ling2015self,ling2017blind}. More concretely,
we consider the following convex program:\footnote{As we shall see shortly, we keep a factor 2 here so as to better connect
the convex and nonconvex algorithms; it does not affect our main theoretical
guarantees at all. }
\begin{equation}
\underset{\ensuremath{\boldsymbol{Z}\in\mathbb{C}^{K\times K}}}{\mathrm{minimize}}\quad g\left(\boldsymbol{Z}\right)=\sum_{j=1}^{m}\left|\boldsymbol{b}_{j}^{\mathsf{H}}\boldsymbol{Z}\boldsymbol{a}_{j}-y_{j}\right|^{2}+2\lambda\left\Vert \boldsymbol{Z}\right\Vert _{*},\label{eq:objcvx}
\end{equation}
where $\lambda>0$ denotes the regularization parameter, and $\Vert\boldsymbol{Z}\Vert_{*}$
is the nuclear norm of $\boldsymbol{Z}$ (i.e.$\ $the sum of singular
values of $\bm{Z}$) and is known to be the convex surrogate for the
rank function. The rationale is rather simple: given that we seek
to recover a rank-1 matrix $\bm{Z}^{\star}=\bm{h}^{\star}\bm{x}^{\star\mathsf{H}}$,
it is common to enforce nuclear norm penalization to encourage the
rank-1 structure. In truth, this comes down to solving a nuclear-norm
regularized least squares problem in the matrix domain $\mathbb{C}^{K\times K}$. 

\paragraph{Nonconvex optimization.} Another popular paradigm maintains
all iterates in the original vector space (i.e.~$\mathbb{C}^{K}$)
and attempts solving the above nonconvex formulation or its variants
directly. The crucial ingredient is to ensure fast and reliable convergence
in spite of nonconvexity. While multiple variants of the nonconvex
formulation (\ref{eq:objncvx-LS}) have been studied in the literature
(e.g.~\citet{li2019rapid,ma2017implicit,charisopoulos2019composite,charisopoulos2019low,huang2018blind}),
the present paper focuses attention on the following ridge-regularized
least-squares problem:
\begin{align}
\underset{\ensuremath{\boldsymbol{h},}\ensuremath{\boldsymbol{x}\in\mathbb{C}^{K}}}{\text{minimize}}\quad f\left(\boldsymbol{h},\boldsymbol{x}\right) & =\sum_{j=1}^{m}\left|\boldsymbol{b}_{j}^{\mathsf{H}}\boldsymbol{hx}^{\mathsf{H}}\boldsymbol{a}_{j}-y_{j}\right|^{2}+\lambda\left\Vert \boldsymbol{h}\right\Vert _{2}^{2}+\lambda\left\Vert \boldsymbol{x}\right\Vert _{2}^{2},\label{eq:objncvx}
\end{align}
with $\lambda>0$ the regularization parameter. This choice of objective
function is crucial to the establishment of our main theorems as can
be seen later. Owing to the nonconvexity of (\ref{eq:objncvx}), one
needs to also specify which algorithm to employ in attempt to solve
this nonconvex problem. Our focal point is a two-stage optimization
algorithm: it starts with a rough initial guess $(\boldsymbol{h}^{0},\bm{x}^{0})$
computed by means of a spectral method, followed by Wirtinger gradient descent (GD)
that iteratively refines the estimates (to be made precise in (\ref{eq:gradient-descent-BD})).
At the end of each gradient iteration, we further rescale the sizes
of the two iterates $\boldsymbol{h}^{t}$ and $\bm{x}^{t}$, so as
to ensure that they have identical $\ell_{2}$ norm (see (\ref{eq:gradient-descent-BD-proj})). 
In truth, this balancing step helps stabilize the algorithm, while facilitating
analysis. The whole algorithm is summarized in Algorithm~\ref{alg:gd-BD-ncvx}. 

\begin{algorithm}[h]
	\caption{Nonconvex gradient descent with spectral initialization}

\label{alg:gd-BD-ncvx}\begin{algorithmic}

\STATE \textbf{{Input}}: $\left\{ y_{j}\right\} _{1\leq j\leq m}$,
$\left\{ \bm{a}_{j}\right\} _{1\leq j\leq m}$ and $\left\{ \bm{b}_{j}\right\} _{1\leq j\leq m}$.

\STATE \textbf{{Spectral initialization}}: let $\sigma_{1}\left(\bm{M}\right)$,
$\check{\bm{h}}^{0}$ and $\check{\bm{x}}^{0}$ denote respectively
the leading singular value, the leading left and the right singular
vectors of 
\begin{equation}
\bm{M}\coloneqq\sum_{j=1}^{m}y_{j}\bm{b}_{j}\bm{a}_{j}^{\mathsf{H}}.\label{eq:spectral-method-matrix}
\end{equation}
Set $\bm{h}^{0}=\sqrt{\sigma_{1}\left(\bm{M}\right)}\,\check{\bm{h}}^{0}$
and $\bm{x}^{0}=\sqrt{\sigma_{1}\left(\bm{M}\right)}\,\check{\bm{x}}^{0}$.

\STATE \textbf{{Gradient updates}}: \textbf{for }$t=0,1,\ldots,t_{0}-1$
\textbf{do}

\STATE \vspace{-1em}
 \begin{subequations}\label{subeq:gradient_update_ncvx}
\begin{align}
\left[\begin{array}{c}
\boldsymbol{h}^{t+1/2}\\
\boldsymbol{x}^{t+1/2}
\end{array}\right]= & \left[\begin{array}{c}
\boldsymbol{h}^{t}\\
\boldsymbol{x}^{t}
\end{array}\right]-\eta\left[\begin{array}{c}
\nabla_{\boldsymbol{h}}f\left(\boldsymbol{h}^{t},\boldsymbol{x}^{t}\right)\\
\nabla_{\boldsymbol{x}}f\left(\boldsymbol{h}^{t},\boldsymbol{x}^{t}\right)
\end{array}\right],\label{eq:gradient-descent-BD}\\
\left[\begin{array}{c}
\boldsymbol{h}^{t+1}\\
\boldsymbol{x}^{t+1}
\end{array}\right]= & \left[\begin{array}{c}
\sqrt{\frac{\left\Vert \bm{x}^{t+1/2}\right\Vert _{2}}{\left\Vert \bm{h}^{t+1/2}\right\Vert _{2}}}\boldsymbol{h}^{t+1/2}\\
\sqrt{\frac{\left\Vert \bm{h}^{t+1/2}\right\Vert _{2}}{\left\Vert \bm{x}^{t+1/2}\right\Vert _{2}}}\boldsymbol{x}^{t+1/2}
\end{array}\right],\label{eq:gradient-descent-BD-proj}
\end{align}
\end{subequations}where $\nabla_{\boldsymbol{h}}f(\cdot)$ and $\nabla_{\boldsymbol{x}}f(\cdot)$
represent the Wirtinger gradient (see \citet[Section 3.3]{li2019rapid}
and Appendix~\ref{subsec:Wirtinger-calculus}) of $f(\cdot)$ w.r.t.~$\bm{h}$
and $\bm{x}$, respectively. 

\end{algorithmic}
\end{algorithm}

\begin{table}[t]
\caption{Comparison of our theoretical guarantees of blind deconvolution under Fourier design
to prior theory, where we hide all logarithmic factors. Here, the
Euclidean estimation error refers to $\|\bm{Z}_{\mathsf{cvx}}-\bm{h}^{\star}\bm{x}^{\star\mathsf{H}}\|_{\mathrm{F}}$
for the convex case and $\|\bm{h}_{\mathsf{ncvx}}\bm{x}_{\mathsf{ncvx}}^{\mathsf{H}}-\bm{h}^{\star}\bm{x}^{\star\mathsf{H}}\|_{\mathrm{F}}$
for the nonconvex case, respectively.}
\label{table:comparison} \vspace{0.8em}
 \centering %
\begin{tabular}{c|c|c|c|c}
\hline 
$\vphantom{2_{2_{2_{2}}}^{2^{2^{2}}}}$ & Sample & \multirow{2}{*}{Algorithm} & Euclidean error & Computational\tabularnewline
 & complexity &  & in the noisy case & complexity\tabularnewline
\hline 
\hline 
\citet{ahmed2013blind} & $\mu^{2}K$ & convex relaxation & $\sigma\sqrt{Km}$\vphantom{$\frac{1^{7^{7^{7}}}}{1^{7^{7^{7}}}}$}\hspace{-0.4em} & ---\tabularnewline
\hline 
\citet{ling2017blind} & $\mu^{2}K$ & convex relaxation & $\sigma\sqrt{Km}$\vphantom{$\frac{1^{7^{7^{7}}}}{1^{7^{7^{7}}}}$}\hspace{-0.4em} & ---\tabularnewline
\hline 
\textbf{This paper} & $\mu^{2}K\vphantom{2_{2_{2_{2}}}^{2^{2^{2}}}}$ & convex relaxation & $\sigma\sqrt{K}$\vphantom{$\frac{1^{7^{7^{7}}}}{1^{7^{7^{7}}}}$}\hspace{-0.4em} & ---\tabularnewline
\hline 
\hline 
\citet{li2019rapid} & $\mu^{2}K\vphantom{2_{2_{2_{2}}}^{2^{2^{2}}}}$ & nonconvex regularized GD & $\sigma\sqrt{K}$\vphantom{$\frac{1^{7^{7^{7}}}}{1^{7^{7^{7}}}}$}\hspace{-0.4em} & $mK^{2}$\tabularnewline
\hline 
\citet{huang2018blind} & $\mu^{2}K\vphantom{2_{2_{2_{2}}}^{2^{2^{2}}}}$ & Riemannian steepest descent & $\sigma\sqrt{K}$\vphantom{$\frac{1^{7^{7^{7}}}}{1^{7^{7^{7}}}}$}\hspace{-0.4em} & $mK^{2}$\tabularnewline
\hline 
\citet{ma2017implicit} & $\mu^{2}K\vphantom{2_{2_{2_{2}}}^{2^{2^{2}}}}$ & nonconvex vanilla GD & --- & $mK$ (noiseless)\tabularnewline
\hline 
\multirow{2}{*}{\textbf{This paper}} & \multirow{2}{*}{$\mu^{2}K\vphantom{2_{2_{2_{2}}}^{2^{2^{2}}}}$} & nonconvex GD & \multirow{2}{*}{$\sigma\sqrt{K}$\vphantom{$\frac{1^{7^{7^{7}}}}{1^{7^{7^{7}}}}$}\hspace{-0.4em}} & \multirow{2}{*}{$mK$}\tabularnewline
 &  & (with balancing operations) &  & \tabularnewline
\hline 
\end{tabular}
\end{table}

\subsection{Inadequacy of prior theory}

The aforementioned two algorithms have found
solid theoretical support under certain randomized sampling mechanisms.
Informally, imagine that the $\bm{a}_{j}$'s and the $\bm{b}_{j}$'s
follow standard Gaussian and partial Fourier designs, respectively,
and that each noise component $\xi_{j}$ is a zero-mean sub-Gaussian
random variable with variance at most $\sigma^{2}$ (more precise
descriptions are deferred to Assumption \ref{assumptions:models}).
The following performance guarantees have been established in prior theory.  
\begin{itemize}
\item
Convex relaxation is guaranteed to return
an estimate of $\bm{h}^{\star}\bm{x}^{\star\mathsf{H}}$ with an Euclidean
estimation error bounded by $\sigma\sqrt{Km}$ (modulo some log factor)
\citep{ahmed2013blind,ling2017blind}. This, however, exceeds the
minimax lower bound (to be presented in Theorem \ref{thm:LB}) by
at least a factor of $\sqrt{m}$. 

\item In comparison, nonconvex algorithms
are capable of achieving nearly minimax optimal statistical
accuracy, with a computational complexity on the order of $mK^2$ (up
to some log factor) \citep{li2019rapid,huang2018blind}. 
Here, the computational complexity  encompasses the cost of spectral
initialization in Algorithm \ref{alg:gd-BD-ncvx} if implemented
by power methods \citep{golub2013matrix}. This computational cost, however, could be an order of $K$ times larger than the cost taken to read the data. 
\end{itemize}
See Table
\ref{table:comparison} for a more complete summary of existing theoretical results for this scenario. 

These prior results, while offering rigorous theoretical underpinnings
for the two popular algorithms, lead to several natural questions:
\begin{enumerate}
	\item (Improving statistical guarantees) \emph{Is the statistical accuracy of convex relaxation inherently suboptimal when coping with
random noise?}
\item (Improving computational complexity) \emph{Is it possible to further accelerate the nonconvex algorithm
without compromising statistical accuracy? }
\end{enumerate}
The present paper is devoted to addressing these two questions. Informally, we aim to demonstrate that (1) convex relaxation achieves minimax-optimal statistical accuracy in the face of random noise, and (2) nonconvex optimization converges to a nearly minimax-optimal solution in time proportional to that taken to read the data.

\subsection{Paper organization and notation\label{subsec:Notation}}
The outline of the paper is as follows. Section~\ref{sec:Main-results} gives
the formal statement of the model assumptions and presents our main results for
two different designs. Section~\ref{sec:Priorart} reviews previous
literature on blind deconvolution. Section~\ref{sec:Numerical-experiments}
presents numerical experiments that corroborate our theoretical results. 
We conclude the paper in Section \ref{sec:Discussion}
by pointing out several future directions. All the proof details
are deferred to the Appendix. 

Throughout the paper, we shall often use the vector notation $\bm{y}:=[y_{1},\cdots,y_{m}]^{\top}$
and $\bm{\xi}:=[\xi_{1},\cdots,\xi_{m}]^{\top}\in\mathbb{C}^{m}$.
For any vector $\bm{v}$ and any matrix $\bm{M}$, we denote by $\bm{v}^{\mathsf{H}}$
and $\bm{M}^{\mathsf{H}}$ their conjugate transpose, respectively.
The notation $\|\bm{v}\|_{2}$ represents the $\ell_{2}$ norm of
an vector $\bm{v}$, and we let $\left\Vert \bm{M}\right\Vert $,
$\left\Vert \bm{M}\right\Vert _{\mathrm{F}}$ and $\left\Vert \bm{M}\right\Vert _{*}$
represent the spectral norm, the Frobenius norm and the nuclear norm
of $\bm{M}$, respectively. For a function $f(\bm{h},\bm{x})$, we
use $\nabla_{\bm{h}}f(\bm{h},\bm{x})$ (resp.~$\nabla_{\bm{x}}f(\bm{h},\bm{x})$)
to denote its Wirtinger gradient (see \citet[Section 3.3]{li2019rapid}
for detailed introduction) of $f(\cdot)$ with respect to $\bm{h}$
(resp.~$\bm{x}$). Further, we define $\nabla f(\bm{h},\bm{x})=[\nabla_{\bm{h}}f(\bm{h},\bm{x})^{\top},\nabla_{\bm{x}}f(\bm{h},\bm{x})^{\top}]^{\top}$.
For any subspace $T$, we use $T^{\bot}$ to denote its orthogonal
complement, and $\mathcal{P}_{T}(\bm{M})$ the Euclidean projection
of a matrix $\bm{M}$ onto $T$. 
Moreover, we adopt $f_{1}(m,K)\lesssim f_{2}(m,K)$
or $f_{1}(m,K)=O(f_{2}(m,K))$ to indicate that there exists some
constant $C_{1}>0$ such that $f_{1}(m,K)\leq Cf_{2}(m,K)$ holds
for all $(m,K)$ that are sufficiently large, and use $f_{1}(m,K)\gtrsim f_{2}(m,K)$
to indicate that $f_{1}(m,K)\geq C_{2}f_{2}(m,K)$ holds for some
constant $C>0$ whenever $(m,K)$ are sufficiently large. The notation
$f_{1}(m,K)\asymp f_{2}(m,K)$ means that $f_{1}(m,K)\lesssim f_{2}(m,K)$
and $f_{1}(m,K)\gtrsim f_{2}(m,K)$ hold simultaneously.  In our proof, $C$ serves as a universal constant
whose value might change from line to line.

\section{Main results\label{sec:Main-results}}

In this section, we present our theoretical guarantees for the above
two algorithms for two types of random designs commonly studied in the blind deconvolution literature.

\subsection{Blind deconvolution under random Fourier designs \label{subsec:Fourier-design}}


\paragraph{Model and assumptions.} 

We start by introducing a sort of random Fourier designs motivated by practical engineering applications (see \citet{ahmed2013blind,li2019rapid}). 


\begin{assumption}
	\label{assumptions:models}
Let $\bm{A}:=\left[\boldsymbol{a}_{1},\boldsymbol{a}_{2},\cdots,\boldsymbol{a}_{m}\right]^{\mathsf{H}}\in\mathbb{C}^{m\times K}$ 
and $\boldsymbol{B}:=\left[\boldsymbol{b}_{1},\boldsymbol{b}_{2},\cdots,\boldsymbol{b}_{m}\right]^{\mathsf{H}}\in\mathbb{C}^{m\times K}$
be matrices obtained by concatenating the design vectors.
\begin{itemize}
\item The entries of $\bm{A}$ are independently drawn from standard complex
Gaussian distributions, namely, $\boldsymbol{a}_{j}\overset{\mathrm{i.i.d.}}{\sim}\mathcal{N}\left(\boldsymbol{0},\frac{1}{2}\bm{I}_{K}\right)+i\mathcal{N}\left(\boldsymbol{0},\frac{1}{2}\bm{I}_{K}\right)$
with $i$ the imaginary unit;
\item The design matrix $\bm{B}$ consists of the first $K$ columns of
the unitary discrete Fourier transform (DFT) matrix $\boldsymbol{F}\in\mathbb{C}^{m\times m}$
obeying $\bm{F}\bm{F}^{\mathsf{H}}=\boldsymbol{I}_{m}$;
\item The noise components $\{\xi_{i}\}$ are independent zero-mean sub-Gaussian
random variables with sub-Gaussian norm obeying $\|\xi_{i}\|_{\psi_{2}}\leq\sigma$
($1\leq i\leq m)$. See \citet[Definition 5.7]{vershynin2010introduction}
for the definition of $\|\cdot\|_{\psi_{2}}$.
\end{itemize}
\end{assumption}
\begin{remark} As can be easily verified, we have $\left\Vert \boldsymbol{b}_{j} \right\Vert _{2} = \sqrt{K/m}$ ($1\leq j\leq m$) under this model. \end{remark}
	It is worth noting that the Fourier design is largely motivated by the duality relation between convolution in the time domain and multiplication in the frequency domain, which is closely related to practical scenarios; see  \citet{ahmed2013blind} for details. In fact, the model described in Assumption \ref{assumptions:models} has been the focus of a number of recent papers including \cite{ahmed2013blind, li2019rapid, ma2017implicit, huang2018blind, ling2019regularized, ling2016simultaneous, ling2017blind}, to name a few. 

In addition, as pointed out by prior works \citet{ahmed2013blind,li2019rapid,ma2017implicit},
the following incoherence condition --- which captures the interplay
between the truth and the measurement mechanism --- plays a crucial
role in enabling tractable estimation schemes. 

\begin{definition}[Incoherence]Define the incoherence parameter
$\mu$ as the smallest number obeying
\begin{equation}
\left|\boldsymbol{b}_{j}^{\mathsf{H}}\boldsymbol{h}^{\star}\right|
\leq \frac{\mu}{\sqrt{K}} \left\Vert \boldsymbol{b}_{j} \right\Vert _{2} \left\Vert \boldsymbol{h}^{\star}\right\Vert _{2} = \frac{\mu}{\sqrt{m}}\left\Vert \boldsymbol{h}^{\star}\right\Vert _{2},\qquad1\leq j\leq m.
	\label{eq:incoherence-condition}
\end{equation}
%
\end{definition}

\begin{remark}
Comparing the Cauchy-Schwarz inequality $\left|\boldsymbol{b}_{j}^{\mathsf{H}}\boldsymbol{h}^{\star}\right|
	\leq \left\Vert \boldsymbol{b}_{j} \right\Vert _{2} \left\Vert \boldsymbol{h}^{\star}\right\Vert _{2}$ with \eqref{eq:incoherence-condition} reveals that $\mu\leq \sqrt{K}$.   
It is noteworthy that our theory does not require $\mu$ to be small constant; in fact, all of our theoretical findings allow $\mu$ to grow with the problem dimension.
\end{remark}

\noindent Informally, a small incoherence parameter indicates that
the truth is not quite aligned with the sampling basis. As a concrete example, 
when $\bm{h}^{\star}$ is randomly generated (i.e.~$\bm{h}^{\star}\sim\mathcal{N}(\bm{0},\bm{I}_{K})$),
it can be easily verified that the incoherence parameter $\mu$ is, with high probability, at most
$O(\sqrt{\log m})$.  In fact, this type of
condition is widely proposed in statistical literature on various
problem besides blind deconvolution, such as \citet{candes2009exact,ma2017implicit,chen2019noisy}
on matrix completion and \citet{candes2011robust, chandrasekaran2011rank, chen2020bridging}
on robust principal component analysis. The important role of this incoherence
parameter will also be confirmed by our numerical simulations momentarily (cf.~Figure \ref{fig:incoherence}).


\paragraph{Main theory.} We are now positioned to state our main
theory for this setting, followed by discussing the implications of our theory.
Towards this end, we begin with the statistical guarantees for the
convex formulation. Denote the minimizer of (\ref{eq:objcvx}) by
$\bm{Z}_{\mathsf{cvx}}$. Then our result is this:

\begin{theorem}[Convex relaxation]\label{theorem:cvx} Set $\lambda=C_{\lambda}\sigma\sqrt{K\log m}$
for some large enough constant $C_{\lambda}>0$. Assume

\begin{equation}
m\geq C\mu^{2}K\log^{9}m\qquad\text{and}\qquad\sigma\sqrt{K\log^{5}m}\leq c\left\Vert \boldsymbol{h}^{\star}\boldsymbol{x}^{\mathsf{\star H}}\right\Vert _{\mathrm{F}}\label{eq:assumptions-sample-noise}
\end{equation}
for some sufficiently large (resp.~small) constant $C>0$ (resp.~$c>0$).
Then under Assumption \ref{assumptions:models} and the incoherence
condition (\ref{eq:incoherence-condition}), one has with probability
exceeding $1-O\left(m^{-3}+me^{-K}\right)$ that
\begin{align}
\left\Vert \boldsymbol{Z}_{\mathsf{cvx}}-\boldsymbol{h}^{\star}\boldsymbol{x}^{\mathsf{\star H}}\right\Vert \leq\left\Vert \boldsymbol{Z}_{\mathsf{cvx}}-\boldsymbol{h}^{\star}\boldsymbol{x}^{\mathsf{\star H}}\right\Vert _{\mathrm{F}} & \lesssim\sigma\sqrt{K\log m}.\label{eq:main1}
\end{align}
In addition, the bounds in (\ref{eq:main1}) continue to hold if $\boldsymbol{Z}_{\mathsf{cvx}}$
is replaced by $\boldsymbol{Z}_{\mathsf{cvx,}1}\coloneqq\arg\min_{\boldsymbol{Z}:\mathsf{rank}\left(\boldsymbol{Z}\right)\leq1}\left\Vert \boldsymbol{Z}-\boldsymbol{Z}_{\mathsf{cvx}}\right\Vert _{\mathrm{F}}$
(i.e.~the best rank-1 approximation of $\boldsymbol{Z}_{\mathsf{cvx}}$).\end{theorem}

\begin{remark}
In \eqref{eq:assumptions-sample-noise}, $\log^9 m$ and $\log^5 m$ appear due to our decoupling arguments. We believe it would be difficult to get rid of the logarithmic factors completely using the current analyis framework, although it might be possible to reduce the power of the logarithmic factors slightly by means of more refined analysis. 
\end{remark}
Our proof  for this theorem, whose details are postponed to  Appendix \ref{sec:Proof-outline-for-theorem-cvx},  is largely inspired by the idea of connecting convex and nonconvecx optimization as proposed by \citet{chen2019noisy, chen2020bridging} for noisy matrix completion and robust principal component analysis respectively. 
Note, however, that implementing this high-level idea requires drastically different analysis from \citet{chen2019noisy, chen2020bridging},
primarily due to the absence of randomness in the highly structured Fourier design matrix $\bm{B}$. For instance, in contrast to prior works that were built upon a ``leave-one-out'' analysis framework to decouple statistical dependency,  simply ``leaving out'' one row of $\bm{B}$ in the blind deconvolution analysis does not lead to immediate statistical benefits due to the deterministic nature of $\bm{B}$. Consequently, considerably more delicate analyses are needed in order to enable fine-grained statistical analysis. 

%

Next, we turn to theoretical guarantees for the nonconvex algorithm
described in Algorithm \ref{alg:gd-BD-ncvx}. For notational convenience,
we define
\begin{equation}
\boldsymbol{z}^{t}:=\left[\begin{array}{c}
\boldsymbol{h}^{t}\\
\boldsymbol{x}^{t}
\end{array}\right]\qquad\text{and}\qquad\boldsymbol{z}^{\star}:=\left[\begin{array}{c}
\boldsymbol{h}^{\star}\\
\boldsymbol{x}^{\star}
\end{array}\right]\label{eq:defn-zt-zstar}
\end{equation}
throughout this paper. Before presenting the results, we make note of
 an unavoidable scaling ambiguity issue underlying this model.
Given that $\boldsymbol{h}^{\star}$ and $\boldsymbol{x}^{\star}$
are only identifiable up to global scaling (meaning that one cannot
hope to distinguish $(\alpha\bm{h}^{\star},\frac{1}{\overline{\alpha}}\bm{x}^{\star})$
from $(\bm{h}^{\star},\bm{x}^{\star})$ given only bilinear measurements),
we shall measure the discrepancy between $\boldsymbol{z}^{\star}$
and any point $\boldsymbol{z}:=\footnotesize\left[\begin{array}{c}
\boldsymbol{h}\\
\boldsymbol{x}
\end{array}\right]$ through the following metric:
\begin{equation}
\mathsf{dist}\left(\boldsymbol{z},\boldsymbol{z}^{\star}\right):=\min_{\alpha\in\mathbb{C}}\sqrt{\left\Vert \frac{1}{\overline{\alpha}}\boldsymbol{h}-\boldsymbol{h}^{\star}\right\Vert _{2}^{2}+\left\Vert \alpha\boldsymbol{x}-\boldsymbol{x}^{\star}\right\Vert _{2}^{2}}.\label{eq:dist-z-zstar}
\end{equation}
In words, this metric is an extension of the $\ell_{2}$ distance
modulo global scaling. 
Our result is this: 

\begin{theorem}[Nonconvex optimization]\label{thm:nonconvex}Set
$\lambda=C_{\lambda}\sigma\sqrt{K\log m}$ for some large enough constant
$C_{\lambda}>0$. Take $\eta=c_{\eta}$ for some sufficiently small
constant $c_{\eta}>0$. Suppose that Assumption \ref{assumptions:models},
the incoherence condition (\ref{eq:incoherence-condition}) and the
condition (\ref{eq:assumptions-sample-noise}) hold. Then with probability
at least $1-O\left(m^{-5}+me^{-K}\right)$, the iterates $\left\{ \boldsymbol{h}^{t},\boldsymbol{x}^{t}\right\} _{0\leq t\leq t_{0}}$
of the spectrally initialized nonconvex algorithm (see Algorithm \ref{alg:gd-BD-ncvx}) obey\begin{subequations}
\begin{align}
\mathsf{dist}\left(\boldsymbol{z}^{0},\boldsymbol{z}^{\star}\right) & \lesssim\sqrt{\frac{\mu^{2}K\log m}{m}}\left\Vert \bm{z}^{\star}\right\Vert _{2}+\frac{\sigma\sqrt{K\log m}}{\left\Vert \boldsymbol{h}^{\star}\boldsymbol{x}^{\mathsf{\star H}}\right\Vert _{\mathrm{F}}^{1/2}},\label{eq:L2-error-ncvx}\\
\mathsf{dist}\left(\boldsymbol{z}^{t},\boldsymbol{z}^{\star}\right) & \le\rho^{t}\mathsf{dist}\left(\bm{z}^{0},\bm{z}^{\star}\right)+\frac{C_{1}\left(\lambda+\sigma\sqrt{K\log m}\right)}{c_{\rho}\left\Vert \boldsymbol{h}^{\star}\boldsymbol{x}^{\mathsf{\star H}}\right\Vert _{\mathrm{F}}^{1/2}}, \label{eq:L2-error-ncvx-t}\\
\big\|\bm{h}^{t}\big(\bm{x}^{t}\big)^{\mathsf{H}}-\boldsymbol{h}^{\star}\boldsymbol{x}^{\mathsf{\star H}}\big\|_{\mathrm{F}} & \leq2\rho^{t}\mathsf{dist}\left(\bm{z}^{0},\bm{z}^{\star}\right)\left\Vert \bm{z}^{\star}\right\Vert _{2}+\frac{2C_{1}\left(\lambda+\sigma\sqrt{K\log m}\right)}{c_{\rho}},\label{eq:Fro-error-ncvx}
\end{align}
\end{subequations}simultaneously for all $0\leq t\leq t_{0}\leq m^{20}$.
Here, we take $C_{1}>0$ to be some sufficiently large constant and
$0<\rho=1-c_{\rho}\eta<1$ for some sufficiently small constant $c_{\rho}>0$.
\end{theorem}
\begin{remark}
	It is noteworthy that the quantity $m^{-5}$ in the probability term $1-O\left(m^{-5}+me^{-K}\right)$ in this theorem can actually be replaced by $m^{-C}$ for any positive integer $C$. 
\end{remark}

	Informally, this theorem guarantees that the estimation error of the iterates
	$\left\{ \boldsymbol{h}^{t},\boldsymbol{x}^{t}\right\} _{0\leq t\leq t_{0}}$
	generated by Algorithm \ref{alg:gd-BD-ncvx} decays geometrically fast until some error floor is hit. As we shall demonstrate momentarily in Theorem \ref{thm:LB}, this error floor matches the minimax-optimal statistical error up to some logarithmic term. 
	
	Compared with one of the most relevant papers to us --- \citet{ma2017implicit} --- on blind deconvolution under Fourier designs, this theorem generalizes the noiseless case studied in \citet{ma2017implicit} to the noisy case. This generalization actually needs a lot of efforts since it calls for delicate and careful control of the noise effect, as detailed in the proof in Appendix \ref{appendix:noncvx}.

\subsection{Blind deconvolution under Gaussian designs\label{subsec:Gaussian-design}}

In addition to the above-mentioned random Fourier design, our results also extend to the scenario under Gaussian design, as formalized below. 

\paragraph{Model and assumptions.} Let us describe the model and assumptions of this scenario as follows.

\begin{assumption}\label{assumptions:models-gausssian}Let $\bm{A}:=\left[\boldsymbol{a}_{1},\boldsymbol{a}_{2},\cdots,\boldsymbol{a}_{m}\right]^{\mathsf{H}}\in\mathbb{C}^{m\times K}$
and $\boldsymbol{B}:=\left[\boldsymbol{b}_{1},\boldsymbol{b}_{2},\cdots,\boldsymbol{b}_{m}\right]^{\mathsf{H}}\in\mathbb{C}^{m\times K}$
be matrices obtained by concatenating the design vectors.
\begin{itemize}
\item The entries of $\bm{A}$ and $\bm{B}$ are independently drawn from
standard complex Gaussian distributions, namely, $\boldsymbol{a}_{j},\bm{b}_{j}\overset{\mathrm{i.i.d.}}{\sim}\mathcal{N}\left(\boldsymbol{0},\frac{1}{2}\bm{I}_{K}\right)+i\mathcal{N}\left(\boldsymbol{0},\frac{1}{2}\bm{I}_{K}\right)$
with $i$ the imaginary unit;
\item The noise components $\{\xi_{i}\}$ are independent zero-mean sub-Gaussian
random variables with sub-Gaussian norm obeying $\|\xi_{i}\|_{\psi_{2}}\leq\sigma$
($1\leq i\leq m)$. See \citet[Definition 5.7]{vershynin2010introduction}
for the definition of $\|\cdot\|_{\psi_{2}}$.
\end{itemize}
\end{assumption}

Akin to Theorems~\ref{theorem:cvx} and \ref{thm:nonconvex}, we consider the loss functions (\ref{eq:objcvx}) and (\ref{eq:objncvx}). 
The main results under the Gaussian design are summarized in the following theorems. 

\begin{theorem}[Convex relaxation]\label{theorem:gaussian-cvx}Let $\lambda=C_{\lambda}\sigma\sqrt{mK\log m}$
for some sufficiently large constant $C_{\lambda}>0$. Assume the
sample complexity and the noise level satisfy
\begin{equation}
m\geq CK\log^{6}m\qquad\text{and}\qquad\sigma\sqrt{\frac{K\log^{5}m}{m}}\leq c\left\Vert \boldsymbol{h}^{\star}\boldsymbol{x}^{\mathsf{\star H}}\right\Vert _{\mathrm{F}}
	\label{eq:gaussian-assumptions-sample-noise}
\end{equation}
for some sufficiently large (resp.~small) constant $C>0$ (resp.~$c>0$).
Then 
\begin{equation}
	\left\Vert \boldsymbol{Z}_{\mathsf{cvx}}-\boldsymbol{h}^{\star}\boldsymbol{x}^{\mathsf{\star H}}\right\Vert \leq\left\Vert \boldsymbol{Z}_{\mathsf{cvx}}-\boldsymbol{h}^{\star}\boldsymbol{x}^{\mathsf{\star H}}\right\Vert _{\mathrm{F}}  \lesssim\sigma\sqrt{\frac{K\log m}{m}} \label{eq:thm:gaussian-cvx}
\end{equation}
holds with probability at least $1-O(m^{-5}+m\exp(-c_1K))$ for some
constant $c_1>0$. In addition, the bounds in (\ref{eq:thm:gaussian-cvx}) continue to hold if $\boldsymbol{Z}_{\mathsf{cvx}}$ is replaced by $\boldsymbol{Z}_{\mathsf{cvx,}1}\coloneqq\arg\min_{\boldsymbol{Z}:\mathsf{rank}\left(\boldsymbol{Z}\right)\leq1}\left\Vert \boldsymbol{Z}-\boldsymbol{Z}_{\mathsf{cvx}}\right\Vert _{\mathrm{F}}$
(i.e.~the best rank-1 approximation of $\boldsymbol{Z}_{\mathsf{cvx}}$).
\end{theorem}
This theorem, which is in parallel to Theorem \ref{theorem:cvx} for Fourier designs, 
confirms the appealing statistical guarantees of convex relaxation under Gaussian designs. The minimax optimality of this result will be discussed in Section \ref{subsec:Insights} in detail. 

\begin{theorem}[Nonconvex optimization]\label{thm:gaussian-ncvx}Set
$\lambda=C_{\lambda}\sigma\sqrt{mK\log m}$ for some large enough constant
$C_{\lambda}>0$. Take $\eta=c_{\eta}/m$ for some sufficiently small
constant $c_{\eta}>0$. Suppose that Assumption \ref{assumptions:models-gausssian} and Condition \eqref{eq:gaussian-assumptions-sample-noise}
hold. Then with probability at least $1-O\left(m^{-5}+me^{-K}\right)$,
the iterates $\left\{ \boldsymbol{h}^{t},\boldsymbol{x}^{t}\right\} _{0\leq t\leq t_{0}}$
of Algorithm \ref{alg:gd-BD-ncvx} obey\begin{subequations}
\begin{align}
\mathsf{dist}\left(\boldsymbol{z}^{0},\boldsymbol{z}^{\star}\right) & \lesssim\sqrt{\frac{K\log^2 m}{m}}\left\Vert \bm{z}^{\star}\right\Vert _{2}+\sigma\sqrt{\frac{K\log m}{m\left\Vert \boldsymbol{h}^{\star}\boldsymbol{x}^{\mathsf{\star H}}\right\Vert _{\mathrm{F}}}},\label{eq:L2-error-ncvx-1}\\
\mathsf{dist}\left(\boldsymbol{z}^{t},\boldsymbol{z}^{\star}\right) & \le\rho^{t}\mathsf{dist}\left(\bm{z}^{0},\bm{z}^{\star}\right)+\frac{C_{11}\left(\lambda+\sigma\sqrt{mK\log m}\right)}{c_{\rho}m\left\Vert \boldsymbol{h}^{\star}\boldsymbol{x}^{\mathsf{\star H}}\right\Vert _{\mathrm{F}}^{1/2}}\\
\big\|\bm{h}^{t}\big(\bm{x}^{t}\big)^{\mathsf{H}}-\boldsymbol{h}^{\star}\boldsymbol{x}^{\mathsf{\star H}}\big\|_{\mathrm{F}} & \leq2\rho^{t}\mathsf{dist}\left(\bm{z}^{0},\bm{z}^{\star}\right)\left\Vert \bm{z}^{\star}\right\Vert _{2}+\frac{2C_{11}\left(\lambda+\sigma\sqrt{mK\log m}\right)}{c_{\rho}m}\label{eq:Fro-error-ncvx-1}
\end{align}
\end{subequations}simultaneously for all $0\leq t\leq t_{0}\leq m^{20}$.
Here, we take $C_{11}>0$ to be some sufficiently large constant and
$0<\rho=1-c_{\rho}c_{\eta}<1$ for some sufficiently small constant
$c_{\rho}>0$. \end{theorem}
Similar to the Fourier designs studied in Section~\ref{subsec:Fourier-design}, our theory asserts that the estimation error of $\left\{ \boldsymbol{h}^{t},\boldsymbol{x}^{t}\right\} _{0\leq t\leq t_{0}}$ produced by Algorithm \ref{alg:gd-BD-ncvx} decreases geometrically fast before reaching an error floor on the order of the minimax-optimal statistical limit modulo some logarithmic factor (cf.~Theorem \ref{thm:LB}).

\subsection{Insights\label{subsec:Insights}}
The above theorems 
strengthen our understanding about the performance of both convex and nonconvex algorithms in the presence
of random noise. In what follows, we elaborate on the tightness of
our results as well as other important algorithmic implications. 
\begin{itemize}
\item \textit{Minimax optimality of both convex relaxation and nonconvex
optimization. }\textit{\emph{Theorems }}\ref{theorem:cvx}-\ref{thm:nonconvex}
(resp.~Theorems \ref{theorem:gaussian-cvx}-\ref{thm:gaussian-ncvx})
reveal that both convex and nonconvex optimization estimate $\boldsymbol{h}^{\star}\boldsymbol{x}^{\mathsf{\star H}}$
to within an Euclidean error at most $\sigma\sqrt{K}$ (resp.~$\sigma\sqrt{mK}$)
up to some log factor for random Fourier design (resp.~Gaussian design), provided that the regularization parameter
is taken to be $\lambda\asymp\sigma\sqrt{K\log m}$ (resp.~$\lambda\asymp\sigma\sqrt{mK\log m}$).
This closes the gap between the statistical guarantees for convex
and nonconvex optimization, confirming that convex relaxation is no
less statistically efficient than nonconvex optimization. Further,
in order to assess the statistical optimality of our results, it is
instrumental to understand the statistical limit one can hope for.
This is provided in the following claim, whose proof is postponed to Appendix \ref{appendix:LB}.
\begin{theorem}\label{thm:LB} Suppose that the noise components obey $\xi_{j}\overset{\text{i.i.d.}}{\sim}\mathcal{N}(0,\sigma^{2}/2)+i\mathcal{N}(0,\sigma^{2}/2)$.
	Define
	\[
	\mathcal{M}^{\star}\coloneqq\left\{ \bm{Z}=\bm{h}\bm{x}^{\mathsf{H}}\,\big|\,\bm{h},\bm{x}\in\mathbb{C}^{K}\right\} .
	\]
	Then under Assumption \ref{assumptions:models}, there exists some
	universal constant $c_{\mathsf{lb}}^{(1)}>0$ such that, with probability
	exceeding $1-O(K^{-10})$,
	\begin{align}
	\inf_{\widehat{\bm{Z}}}\sup_{\bm{Z}^{\star}\in\mathcal{M}^{\star}}\mathbb{E}\left[\big\|\widehat{\bm{Z}}-\bm{Z}^{\star}\big\|_{\mathrm{F}}^{2}\mid\bm{A}\right]\geq c_{\mathsf{lb}}^{(1)}\frac{\sigma^{2}K}{\log m}, \label{eq:bd-lower-bound}
	\end{align}
	where the infimum is taken over all estimator $\widehat{\bm{Z}}$.
	Furthermore, under Assumption \ref{assumptions:models-gausssian},
	there exists another universal constant $c_{\mathsf{lb}}^{(2)}>0$
	such that
	\begin{align}
	\inf_{\widehat{\bm{Z}}}\sup_{\bm{Z}^{\star}\in\mathcal{M}^{\star}}\mathbb{E}\left[\big\|\widehat{\bm{Z}}-\bm{Z}^{\star}\big\|_{\mathrm{F}}^{2}\mid\bm{A},\bm{B}\right]\geq c_{\mathsf{lb}}^{(2)}\frac{\sigma^{2}K}{m\log m}  \label{eq:gaussian-lower-bound}
	\end{align}
	holds with probability exceeding $1-O(K^{-10})$. \end{theorem}

Encouragingly, the minimax lower bound (\ref{eq:bd-lower-bound})
(resp.~(\ref{eq:gaussian-lower-bound})) matches the statistical error
bounds in Theorems \ref{theorem:cvx}-\ref{thm:nonconvex} (resp.~Theorems \ref{theorem:gaussian-cvx}-\ref{thm:gaussian-ncvx}) up
to some logarithmic factor, thus confirming the near minimaxity of
both convex relaxation and nonconvex optimization for blind deconvolution under both designs. 
\item \textit{Fast convergence of nonconvex algorithms. }From the computational
	perspective, Theorem \ref{thm:nonconvex} guarantees linear convergence (or geometric convergence)
of the nonconvex algorithm with a contraction rate $\rho$. Given
that $1-\rho$ is a constant bounded away from 1 (as long as the stepsize
is taken to be a sufficiently small constant), the iteration complexity
of the algorithm scales at most logarithmically with the model parameters.
As a result, the total computational complexity is proportional to
the per-iteration cost $O(mK)$ (up to some log factor), which scales
nearly linearly with the time taken to read the data. Compared with
past work on nonconvex algorithms \citep{li2019rapid,huang2018blind},
our theory reveals considerably faster convergence and hence improved
computational cost, without compromising statistical efficiency. 
A key enabler of the improved theory lies in fine-grained understanding of the part of optimization lanscape visited by the nonconvex algorithm, thus allowing for the use of more aggressive constant step
sizes instead of diminishing step sizes. See Table \ref{table:comparison} for details. 
\end{itemize}
The careful reader might immediately remark that the validity of the
above results requires the assumptions (\ref{eq:assumptions-sample-noise})
on both the sample size and the noise level. Fortunately, a closer
inspection of these conditions reveals the broad applicability of
these conditions. 
\begin{itemize}
\item \textit{Sample complexity. }\textit{\emph{The sample }}size requirement
in our theory of blind deconvolution under Fourier design (resp.~Gaussian design),
as stated in Condition (\ref{eq:assumptions-sample-noise}) (resp.~Condition (\ref{eq:gaussian-assumptions-sample-noise})), scales as
\[
m\gtrsim K\mathrm{poly}\log(m),
\]
which matches the information-theoretical lower limit even in the
absence of noise (modulo some logarithmic factor) as proved in \citet{kech2017optimal}
(resp.~\citet{cai2015rop}).
\item \textit{Signal-to-noise ratio (SNR). }The noise level required for
our theory to work under Fourier design (see Condition (\ref{eq:assumptions-sample-noise}))
is given by $\sigma\sqrt{K\log^{5}m}\lesssim\left\Vert \bm{h}^{\star}\bm{x}^{\star\mathsf{H}}\right\Vert _{\mathrm{F}}$.
If we define the sample-wise signal-to-noise ratio as follows
\begin{equation}
\mathsf{SNR}:=\frac{\frac{1}{m}\sum_{k=1}^{m}\mathbb{E}\big[\big|\bm{b}_{k}^{\mathsf{H}}\bm{h}^{\star}\bm{x}^{\star\mathsf{H}}\bm{a}\big|^{2}\big]}{\sigma^{2}},\label{eq:SNR-definition}
\end{equation}
then our noise requirement can be equivalently phrased as
\[
\mathsf{SNR}=\frac{\|\bm{h}^{\star}\|_{2}^{2}\|\bm{x}^{\star}\|_{2}^{2}}{m\sigma^{2}}\gtrsim\frac{K\log^{5}m}{m},
\]
where the right-hand side of the above relation is vanishingly small in light of our sample
complexity constraint $m\gtrsim\mu^{2}K\log^{9}m$. In other words,
our theory works even in the low-SNR regime. Furthermore, for the
Gaussian design, the noise level required in our theory is $\sigma\sqrt{K\log^{5}m/m}\lesssim\left\Vert \bm{h}^{\star}\bm{x}^{\star\mathsf{H}}\right\Vert _{\mathrm{F}}$.
We can introduce the following SNR that allows us to rewrite this requirement as
\[
\mathsf{SNR}=\frac{\frac{1}{m}\sum_{k=1}^{m}\mathbb{E}\big[\big|\bm{b}_{k}^{\mathsf{H}}\bm{h}^{\star}\bm{x}^{\star\mathsf{H}}\bm{a}\big|^{2}\big]}{\sigma^{2}}=\frac{\|\bm{h}^{\star}\|_{2}^{2}\|\bm{x}^{\star}\|_{2}^{2}}{\sigma^{2}}\gtrsim\frac{K\log^{5}m}{m}, 
\]
which resembles the one for Fourier designs. 
\end{itemize}

\section{Prior art\label{sec:Priorart}}

Before embarking on our discussion on the prior art for blind deconvolution, 
it is noteworthy that the model (\ref{eq:measurement-BD}) 
might remind  readers of the famous problem of phase retrieval \citep{candes2013phaselift,shechtman2015phase,chi2019nonconvex},
which is concerned with solving random quadratic systems of equations and clearly related to the problem of solving bilinear systems.
Despite the similarity between these two problems at first glance, 
the majority of prior phase retrieval theory focuses on either i.i.d.~Gaussian designs or randomized coded diffraction patterns, 
which are drastically different from the kind of random Fourier designs commonly assumed in blind deconvolution. 
In fact, the presence of Fourier designs in blind deconvolution is a consequence of the duality relation between convolution in the time domain and multiplication in the frequency domain \citep{ahmed2013blind,li2019rapid}. 
The deterministic nature of the Fourier design matrix $\bm{B}$ under the Fourier model, however, presents a substantial challenge in the analysis of both convex and nonconvex optimization algorithms; in contrast, the Gaussian design matrix in prior phase retrieval theory is assumed to be highly random, which remarkably simplifies analysis. 



We now turn attention to the blind deconvolution literature. 
As mentioned previously, recent years have witnessed much progress
towards understanding convex and nonconvex optimization for solving
bilinear systems of equations. First, we give a brief review on previous literature of blind deconvolution under Fourier design. Regarding the convex programming approach,
\citet{ahmed2013blind} was the first to apply the lifting idea to
transform bilinear system of equations into linear measurements about
a rank-one matrix --- an idea that has proved effective in a number
of nonconvex problems \citep{candes2013phaselift,waldspurger2015phase,chen2014robust,tang2013compressed,chi2016guaranteed,chen2014near,goemans1994879,shechtman2014gespar,oymak2015simultaneously}.
Focusing on convex relaxing in the lifted domain, \citet{ahmed2013blind}
showed that exact recovery is possible from a near-optimal number
of measurements in the noiseless case, and developed the first statistical
guarantees for the noisy case (which are, as alluded to previously,
highly suboptimal). Several other works have also been devoted to understanding
convex relaxation under possibly different assumptions. Another paper
\citet{aghasi2019branchhull} proposed an effective convex algorithm
for bilinear inversion, assuming that the signs of the signals are
known \emph{a priori}. Moving beyond blind deconvolution, the convex
approach has been extended to accommodate the blind demixing problem
\citep{ling2017blind,jung2017blind}, which is more general than blind
deconvolution. 
\[
\underset{\ensuremath{\boldsymbol{Z}\in\mathbb{C}^{K\times K}}}{\mathrm{minimize}}\quad\left\Vert \boldsymbol{Z}\right\Vert _{*}\qquad\text{subject\ to}\quad\bm{y}=\mathcal{A}\left(\bm{Z}\right).
\]

Another line of works has focused on the development of fast nonconvex
algorithms \citep{li2019rapid,lee2018fast,ma2017implicit,huang2018blind,ling2019regularized,charisopoulos2019composite,charisopoulos2019low},
which was largely motivated by recent advances in efficient nonconvex
optimization for tackling statistical estimation problems \citep{candes2015phase,chen2017solving,charisopoulos2019low,keshavan2009matrix,jain2013low,zhang2016provable,chen2015fast,sun2016guaranteed,zheng2016convergence,wang2017solving,cai2019nonconvex,wang2017sparse,qu2017convolutional,duchi2019solving,ma2019optimization}
(see \citet{chi2019nonconvex} for an overview). \citet{li2019rapid}
proposed a feasible nonconvex recipe by attempting to optimize a regularized
squared loss (which includes extra penalty term to promote incoherence),
and showed that in conjunction with proper initialization, nonconvex
gradient descent converges to the ground truth in the absence of noise.
Another work \citet{huang2018blind} proposed a Riemannian steepest
descent method by exploiting the quotient structure, which is also
guaranteed to work in the noise-free setting with nearly minimal sample
complexity. Further, \citet{ling2019regularized,dong2018nonconvex}
extended the nonconvex paradigm to accommodate the blind demixing
problem, which subsumes blind deconvolution a special case. 

Going beyond algorithm designs, the past works \citet{li2016identifiability,li2015,kech2017optimal}
investigated how many samples are needed to ensure the identifiability
of blind deconvolution under the subspace model. Furthermore, it is
worth noting that another line of recent works \citet{wang2016blind,lee2016blind,zhang2017global,zhang2019structured,zhang2020symmetry,li2019multichannel,shi2020manifold,qu2019nonconvex}
studied a different yet fundamentally important model of blind deconvolution,
assuming that one of the two signals is sparse instead of lying within
a known subspace. These are, however, beyond the scope of the current
paper. 

In addition, as far as we know, previous works on blind deconvolution under Gaussian design is not as extensive as the case with Fourier designs, the latter of which is closer to practical blind deconvolution applications. Among the most relevant works: 
\citet{cai2015rop} proposed a constrained convex optimization problem
under the same setting as Assumption \ref{assumptions:models-gausssian}
and establishes that the estimation error is bounded by $\sigma\min\{K\sqrt{\log m}/m+\sqrt{K/m},1\}$, which is on the same order (up to logarithmic
factors) as our bound in Theorem \ref{theorem:gaussian-cvx} when $m\gg K\log m$ and matches the minimax optimal estimation error lower bound; \citet{zhong2015efficient} studied the noiseless case in
terms of both convex and nonconvex formulations; \citet{charisopoulos2019composite}
analyzed the nonsmooth nonconvex formulation of the problem for bilinear
measurements with corruption frequency less than $1/2$, and proved
that the subgradient algorithms proposed there  converges linearly,
while the specific prox-linear method converges quadratically albeit with
higher  per-iteration cost. Compared with these works, our paper studies the unconstrained version of convex relaxation
 and establishes an estimation error upper bound
that nearly matches the minimax lower bound. When it comes to nonconvex formulation,
the current paper is, as far as we know, the first to justify the optimality of its estimation
accuracy in the noisy setting. 

At the technical level, the pivotal idea of our paper lies in bridging
convex and nonconvex estimators, which is motivated by prior works
\citet{chen2019noisy,chen2019inference,chen2020bridging} on matrix
completion and robust principal component analysis. Such crucial connections
have been established with the assistance of the leave-one-out analysis
framework, which has already proved effective in analyzing a variety
of nonconvex statistical problems \citep{el2018impact,chen2019gradient,chen2019spectral,ding2020leave,cai2020uncertainty,dong2018nonconvex,xu2019consistent,cai2019subspace,chen2020partial,zhong2018near}.

\section{Numerical experiments\label{sec:Numerical-experiments}}

In this subsection, we carry out a series of numerical experiments
to confirm the validity of our theory. Throughout the experiments,
the signals of interest $\bm{h}^{\star}$, $\bm{x}^{\star}\in\mathbb{C}^{K}$
are drawn from $\mathcal{N}\left(\bm{0},\frac{1}{2K}\bm{I}_{K}\right)+\text{i}\mathcal{N}\left(\bm{0},\frac{1}{2K}\bm{I}_{K}\right)$
(so that they have approximately unit $\ell_{2}$ norm). Under the
Assumption \ref{assumptions:models} (resp. Assumption \ref{assumptions:models-gausssian}),
the stepsize $\eta$ is set to be $0.05$ (resp. $0.05/m$), whereas
the regularization parameter is taken to be $\lambda=5\sigma\sqrt{K\log m}$
(resp. $\lambda=5\sigma\sqrt{mK\log m}$). The convex problem is solved
by means of the proximal gradient method \citep{parikh2014proximal}. 

In the first series of experiments, we report the statistical estimation
errors of both convex and nonconvex approaches as the noise level
$\sigma$ varies from $10^{-6}$ to $10^{-3}$ for blind deconvolution under Fourier design,
while the noise level for blind deconvolution under Gaussian design is from $10^{-5}$ to
$10^{-2}$; here, we set $K=100$ and $m=10K$. Let $\bm{Z}_{\mathsf{ncvx}}=\bm{h}_{\mathsf{ncvx}}\bm{x}_{\mathsf{ncvx}}^{\mathsf{H}}$
be the nonconvex solution and $\bm{Z}_{\mathsf{cvx}}$ be the convex
solution. Figure \ref{fig:dist_cvx_ncvx} depicts the relative Euclidean
estimation errors ($\left\Vert \bm{Z}_{\mathsf{ncvx}}-\bm{Z}^{\star}\right\Vert _{\text{F}}/\left\Vert \bm{Z}^{\star}\right\Vert _{\text{F}}$
and $\left\Vert \bm{Z}_{\mathsf{cvx}}-\bm{Z}^{\star}\right\Vert _{\text{F}}/\left\Vert \bm{Z}^{\star}\right\Vert _{\text{F}}$)
vs.~the noise level, where the results are averaged from 20 independent
trials. Clearly, both approaches enjoy almost identical statistical
accuracy, thus confirming the optimality of convex relaxation as well.
Another interesting observation revealed by Figure \ref{fig:dist_cvx_ncvx}
is the closeness of the solutions of these two approaches, which,
as we shall elucidate momentarily, forms the basis of our analysis
idea. 

\begin{figure}
\hfill{}\includegraphics[width=0.4\columnwidth]{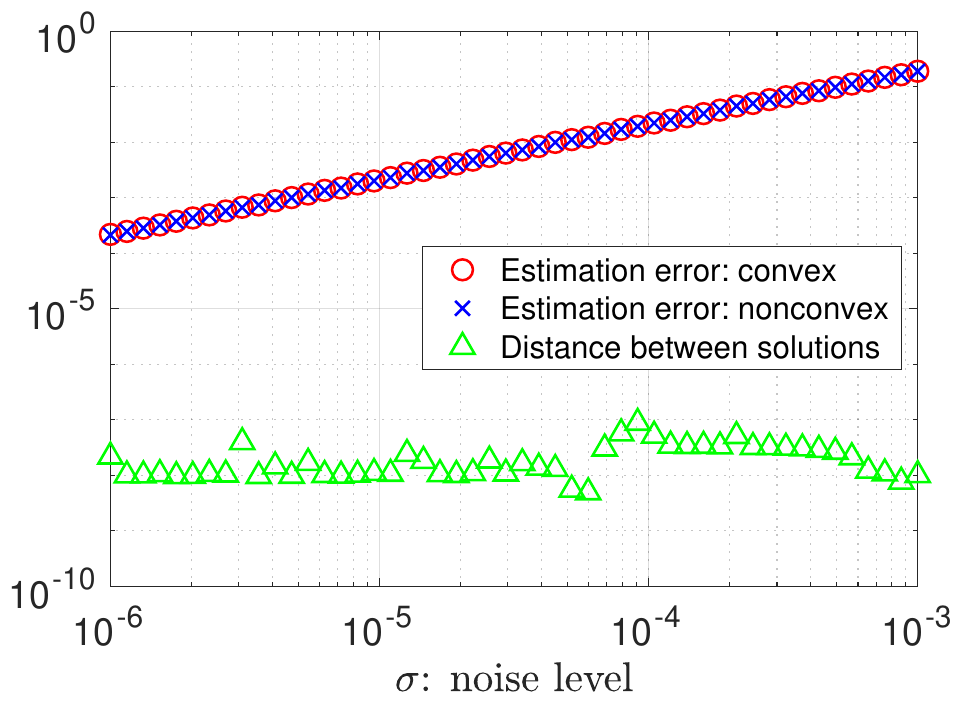}\hfill{}\includegraphics[width=0.4\columnwidth]{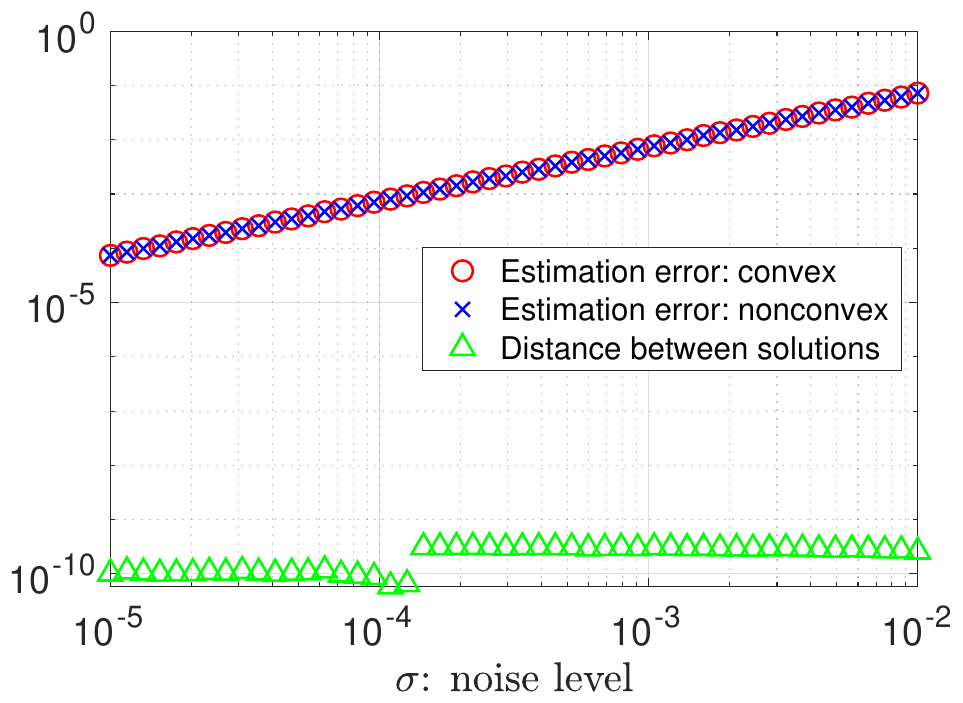}\hfill{}

\caption{Left: blind deconvolution under Fourier design. Right: blind deconvolution under Gaussian design. Relative estimation
errors of both $\bm{Z}_{\mathsf{cvx}}$ and $\bm{Z}_{\mathsf{ncvx}}$
and the relative distance between them vs. the noise level $\sigma$.
The results are averaged over 20 independent trials. \label{fig:dist_cvx_ncvx}}
\end{figure}

In the second series of experiments, we report the numerical convergence
of gradient descent (cf.~Algorithm \ref{alg:gd-BD-ncvx}). We choose
$K\in\{30,100,300,1000\}$ and let $m=10K$, with the noise level
fixed at $\sigma=10^{-4}$. Figure \ref{fig:ncvx_conv} plots the
relative Euclidean estimation error $\left\Vert \bm{h}^{t}\bm{x}^{t\mathsf{H}}-\bm{h}^{\star}\bm{x}^{\star\mathsf{H}}\right\Vert _{\text{F}}/\left\Vert \bm{h}^{\star}\bm{x}^{\star\mathsf{H}}\right\Vert _{\text{F}}$
vs.~the iteration count. As can be seen from the plots, the nonconvex
gradient algorithm studied here converges linearly (in fact, within
around 200-300 iterations) before it hits an error floor. In addition,
the relative error of blind deconvolution under Fourier design increases as the dimension $K$ increases, which is consistent with Theorem \ref{thm:nonconvex}.
While the relative error of blind deconvolution under Gaussian design remains generally
the same across different choices of $K$, this can be explained by
Theorem \ref{thm:gaussian-ncvx} since the ratio between $m$ and
$K$ is kept to be $10$. 

\begin{figure}
\hfill{}\includegraphics[width=0.4\columnwidth]{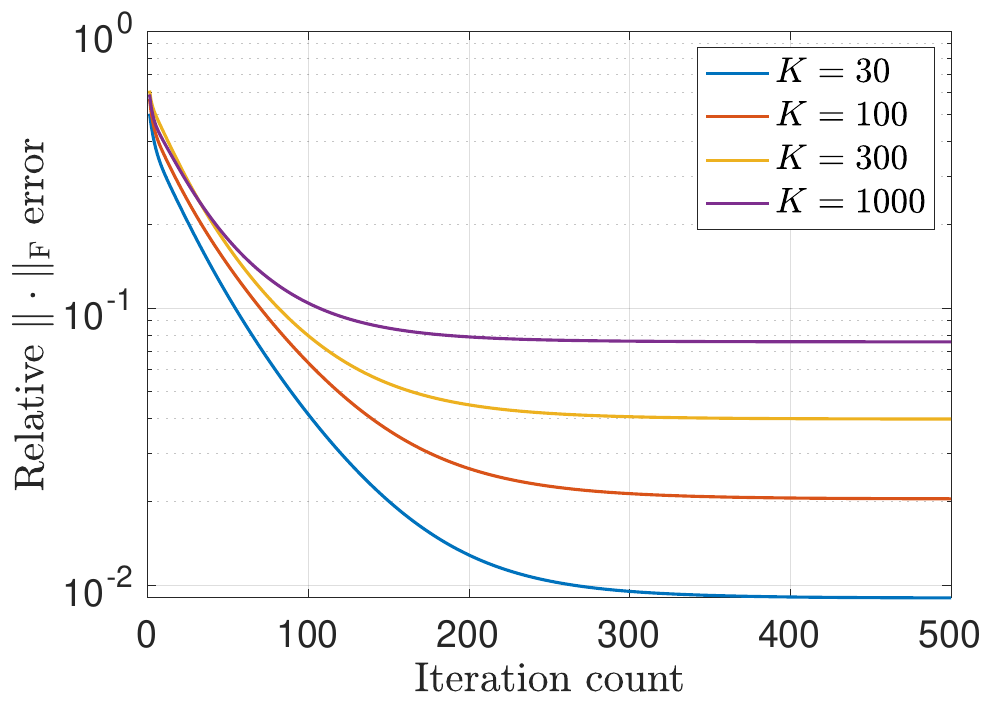}\hfill{}\includegraphics[width=0.4\columnwidth]{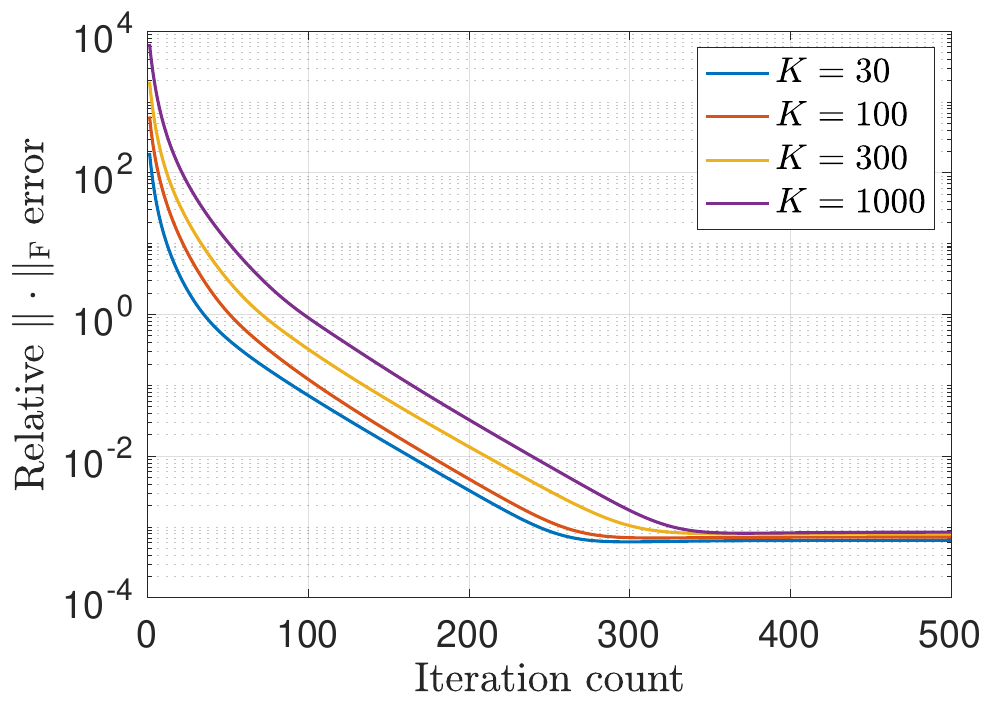}\hfill{}

\caption{Left: blind deconvolution. Right: Gaussian design. Relative Euclidean
error $\big\|\bm{h}^{t}\bm{x}^{t\mathsf{H}}-\bm{h}^{\star}\bm{x}^{\star\mathsf{H}}\big\|_{\mathrm{F}}$
vs.~iteration count. \label{fig:ncvx_conv}}
\end{figure}
In the last series of experiments, we 
examine the necessity of the incoherence condition \eqref{eq:incoherence-condition}
empirically. The experiments are conducted with $\mu^{2}$ taking
on 10 equidistant values from 3 to 30. For each choice of $\mu$,
$\bm{h}^{\star}$ is generated by first setting the first $\mu^{2}$
entries to be 1 and the others 0 , and then normalizing it to have unit
norm; $\bm{x}^{\star}$ is generated randomly from Gaussian distribution
$\mathcal{N}(\bm{0},\bm{I}_{K})$ and then normalized to have unit
norm. This way we guarantee that $\max_{1\leq j\leq m}|\boldsymbol{b}_{j}^{\mathsf{H}}\boldsymbol{h}^{\star}|=\mu/\sqrt{m}$.
We fix $K=100$ and the noise level $\sigma=10^{-4}$ throughout.
For each $\mu^{2}$ and $m$, 20 random trials are conducted. In each
trial, we run convex and nonconvex algorithms until convergence or
the maximum number of iterations is reached, and then report the
relative Euclidean error $\Vert\bm{h}^{t}\bm{x}^{t\mathsf{H}}-\bm{h}^{\star}\bm{x}^{\star\mathsf{H}}\big\|_{\mathrm{F}}$.
If the relative error is less than $0.1$, the trial is declared as
successful. The proportion of successful recovery for convex and nonconvex
problems are plotted in Figure \ref{fig:incoherence}, which suggests
that sample complexity $m$ does scale linearly with $\mu^{2}$ for
both problems and hence corroborates the theoretical results provided in Theorems
\ref{theorem:cvx} and \ref{thm:nonconvex}. 

\begin{figure}
\hfill{}\includegraphics[width=0.4\columnwidth]{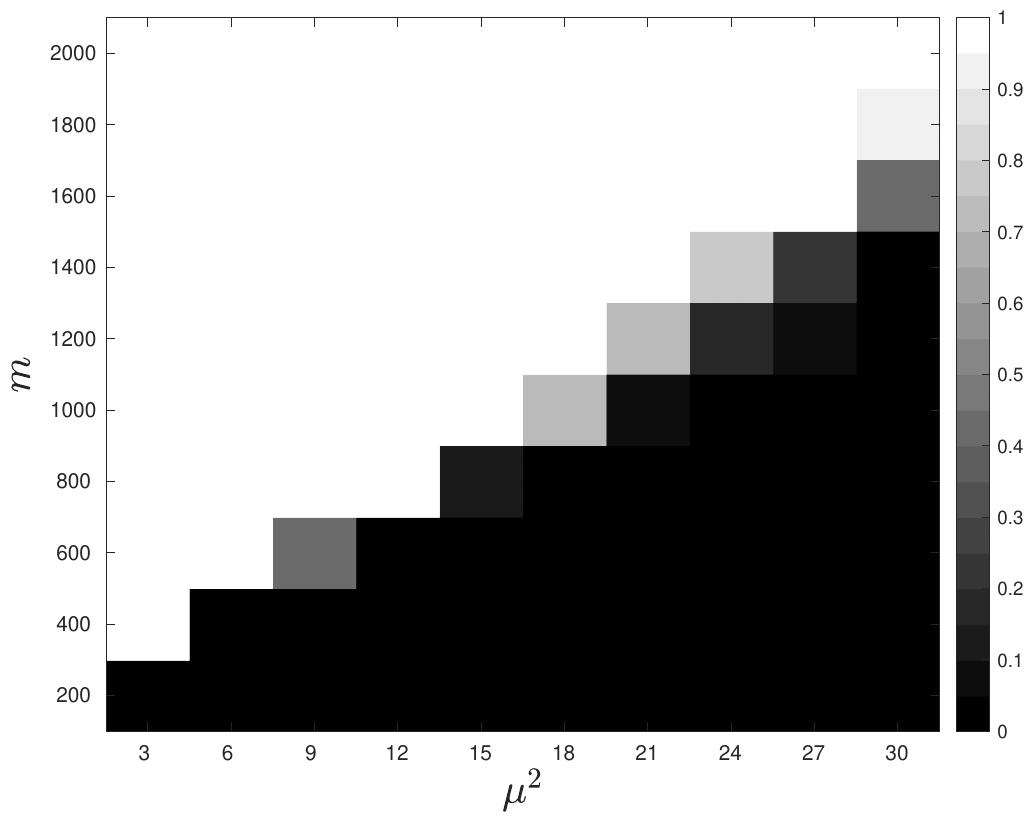}\hfill{}\includegraphics[width=0.4\columnwidth]{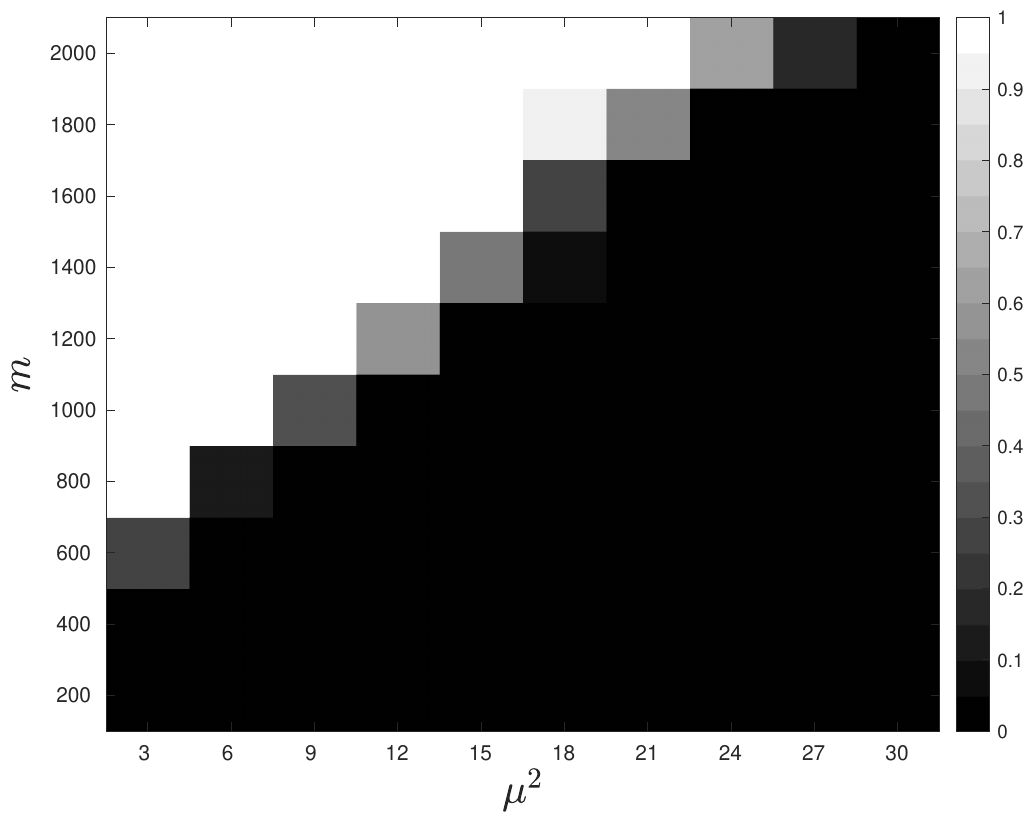}\hfill{}

\caption{Left: nonconvex problem. Right: convex problem. Sample size $m$ vs.~squared
incoherence $\mu^{2}$. The scaled colormap represents the proportion
of successful recovery out of 20 random trials. \label{fig:incoherence}}
\end{figure}

\section{Discussion\label{sec:Discussion}}

This paper has investigated the effectiveness of both convex relaxation
and nonconvex optimization in solving bilinear systems of equations
in the presence of random noise. We have demonstrated that a simple two-stage
nonconvex algorithm solves the problem to optimal statistical accuracy
within nearly linear time. Further, by establishing an intimate connection
between convex programming and nonconvex optimization, we have established
--- for the first time --- optimal statistical guarantees of convex
relaxation when applied to blind deconvolution. 
Our results are established for two different types of design mechanisms: the random Fourier design and the Gaussian design. 
%
%
Our results considerably improve upon the state-of-the-art theory for blind deconvolution, and contribute towards demystifying the
efficacy of optimization-based methods in solving this fundamental
nonconvex problem.


Moving forward, the findings of this paper suggest multiple directions
that merit further investigations. For instance, while the current
paper adopts a balancing operation in each iteration of the nonconvex
algorithm (cf.~Algorithm \ref{alg:gd-BD-ncvx}), it might not be
necessary in practice; in fact, numerical experiments suggest that
the size of the scaling parameter $\vert\alpha^{t}\vert$ stays close
to $1$ even without proper balancing. It would be interesting to
investigate whether vanilla GD without rescaling is able to achieve
comparable performance. In addition, the estimation guarantees provided
in this paper might serve as a starting point for conducting uncertainty
quantification for noisy blind deconvolution --- namely, how to use
it to construct valid and short confidence intervals for the unknowns.
Going beyond blind deconvolution, it would be of interest to extend
the current analysis to handle blind demixing --- a problem that
can be viewed as an extension of blind deconvolution beyond the rank-one
setting \citep{ling2017blind,ling2019regularized,dong2018nonconvex}.
As can be expected, existing statistical guarantees for convex programming
remain highly suboptimal for noisy blind demixing, and the analysis
developed in the current paper suggests a feasible path towards closing
the gap.

\section*{Acknowledgements}
Y.~Chen is supported in part by the AFOSR YIP award FA9550-19-1-0030,
by the ONR grant N00014-19-1-2120, by the ARO grants W911NF-20-1-0097
and W911NF-18-1-0303, by the NSF grants CCF-1907661, IIS-1900140,
IIS-2100158 and DMS-2014279, and by the Princeton SEAS innovation
award. J.~Fan is supported in part by the ONR grant N00014-19-1-2120 and the NSF grants DMS-1662139,
DMS-1712591, DMS-2052926, DMS-2053832, and the NIH grant 2R01-GM072611-15. 
B.~Wang is supported in part by Gordon Y.~S.~Wu Fellowships in Engineering.

\bibliographystyle{abbrvnat}
\bibliography{bibfile}

\begin{thebibliography}{81}
\providecommand{\natexlab}[1]{#1}
\providecommand{\url}[1]{\texttt{#1}}
\expandafter\ifx\csname urlstyle\endcsname\relax
  \providecommand{\doi}[1]{doi: #1}\else
  \providecommand{\doi}{doi: \begingroup \urlstyle{rm}\Url}\fi

\bibitem[Aghasi et~al.(2019)Aghasi, Ahmed, Hand, and
  Joshi]{aghasi2019branchhull}
A.~Aghasi, A.~Ahmed, P.~Hand, and B.~Joshi.
\newblock Branchhull: Convex bilinear inversion from the entrywise product of
  signals with known signs.
\newblock \emph{Applied and Computational Harmonic Analysis}, 2019.

\bibitem[Ahmed et~al.(2013)Ahmed, Recht, and Romberg]{ahmed2013blind}
A.~Ahmed, B.~Recht, and J.~Romberg.
\newblock Blind deconvolution using convex programming.
\newblock \emph{IEEE Transactions on Information Theory}, 60\penalty0
  (3):\penalty0 1711--1732, 2013.

\bibitem[Cai et~al.(2020)Cai, Poor, and Chen]{cai2020uncertainty}
C.~Cai, H.~V. Poor, and Y.~Chen.
\newblock Uncertainty quantification for nonconvex tensor completion:
  Confidence intervals, heteroscedasticity and optimality.
\newblock In \emph{International Conference on Machine Learning}, pages
  1271--1282. PMLR, 2020.

\bibitem[Cai et~al.(2021{\natexlab{a}})Cai, Li, Chi, Poor, and
  Chen]{cai2019subspace}
C.~Cai, G.~Li, Y.~Chi, H.~V. Poor, and Y.~Chen.
\newblock Subspace estimation from unbalanced and incomplete data matrices:
  $\ell_{2,\infty}$ statistical guarantees.
\newblock \emph{The Annals of Statistics}, 49\penalty0 (2):\penalty0 944--967,
  2021{\natexlab{a}}.

\bibitem[Cai et~al.(2021{\natexlab{b}})Cai, Li, Poor, and
  Chen]{cai2019nonconvex}
C.~Cai, G.~Li, H.~V. Poor, and Y.~Chen.
\newblock Nonconvex low-rank tensor completion from noisy data.
\newblock \emph{Operations Research}, 2021{\natexlab{b}}.

\bibitem[Cai et~al.(2015)Cai, Zhang, et~al.]{cai2015rop}
T.~T. Cai, A.~Zhang, et~al.
\newblock Rop: Matrix recovery via rank-one projections.
\newblock \emph{Annals of Statistics}, 43\penalty0 (1):\penalty0 102--138,
  2015.

\bibitem[Campisi and Egiazarian(2016)]{campisi2016blind}
P.~Campisi and K.~Egiazarian.
\newblock \emph{Blind image deconvolution: theory and applications}.
\newblock CRC press, 2016.

\bibitem[Candes and Plan(2011)]{candes2011tight}
E.~J. Candes and Y.~Plan.
\newblock Tight oracle inequalities for low-rank matrix recovery from a minimal
  number of noisy random measurements.
\newblock \emph{IEEE Transactions on Information Theory}, 57\penalty0
  (4):\penalty0 2342--2359, 2011.

\bibitem[Cand{\`e}s and Recht(2009)]{candes2009exact}
E.~J. Cand{\`e}s and B.~Recht.
\newblock Exact matrix completion via convex optimization.
\newblock \emph{Foundations of Computational mathematics}, 9\penalty0
  (6):\penalty0 717--772, 2009.

\bibitem[Cand{\`e}s et~al.(2011)Cand{\`e}s, Li, Ma, and
  Wright]{candes2011robust}
E.~J. Cand{\`e}s, X.~Li, Y.~Ma, and J.~Wright.
\newblock Robust principal component analysis?
\newblock \emph{Journal of the ACM (JACM)}, 58\penalty0 (3):\penalty0 1--37,
  2011.

\bibitem[Candes et~al.(2013)Candes, Strohmer, and
  Voroninski]{candes2013phaselift}
E.~J. Candes, T.~Strohmer, and V.~Voroninski.
\newblock Phaselift: Exact and stable signal recovery from magnitude
  measurements via convex programming.
\newblock \emph{Communications on Pure and Applied Mathematics}, 66\penalty0
  (8):\penalty0 1241--1274, 2013.

\bibitem[Candes et~al.(2015)Candes, Li, and Soltanolkotabi]{candes2015phase}
E.~J. Candes, X.~Li, and M.~Soltanolkotabi.
\newblock Phase retrieval via {W}irtinger flow: Theory and algorithms.
\newblock \emph{IEEE Transactions on Information Theory}, 61\penalty0
  (4):\penalty0 1985--2007, 2015.

\bibitem[Chan and Wong(1998)]{chan1998total}
T.~F. Chan and C.-K. Wong.
\newblock Total variation blind deconvolution.
\newblock \emph{IEEE transactions on Image Processing}, 7\penalty0
  (3):\penalty0 370--375, 1998.

\bibitem[Chandrasekaran et~al.(2011)Chandrasekaran, Sanghavi, Parrilo, and
  Willsky]{chandrasekaran2011rank}
V.~Chandrasekaran, S.~Sanghavi, P.~A. Parrilo, and A.~S. Willsky.
\newblock Rank-sparsity incoherence for matrix decomposition.
\newblock \emph{SIAM Journal on Optimization}, 21\penalty0 (2):\penalty0
  572--596, 2011.

\bibitem[Charisopoulos et~al.(2019)Charisopoulos, Davis, D{\'\i}az, and
  Drusvyatskiy]{charisopoulos2019composite}
V.~Charisopoulos, D.~Davis, M.~D{\'\i}az, and D.~Drusvyatskiy.
\newblock Composite optimization for robust blind deconvolution.
\newblock \emph{arXiv preprint arXiv:1901.01624}, 2019.

\bibitem[Charisopoulos et~al.(2021)Charisopoulos, Chen, Davis, D{\'\i}az, Ding,
  and Drusvyatskiy]{charisopoulos2019low}
V.~Charisopoulos, Y.~Chen, D.~Davis, M.~D{\'\i}az, L.~Ding, and
  D.~Drusvyatskiy.
\newblock Low-rank matrix recovery with composite optimization: good
  conditioning and rapid convergence.
\newblock \emph{Foundations of Computational Mathematics}, pages 1--89, 2021.

\bibitem[Chen et~al.(2020{\natexlab{a}})Chen, Gao, and Zhang]{chen2020partial}
P.~Chen, C.~Gao, and A.~Y. Zhang.
\newblock Partial recovery for top-k ranking: Optimality of mle and
  sub-optimality of spectral method.
\newblock \emph{arXiv preprint arXiv:2006.16485}, 2020{\natexlab{a}}.

\bibitem[Chen and Cand{\`e}s(2017)]{chen2017solving}
Y.~Chen and E.~J. Cand{\`e}s.
\newblock Solving random quadratic systems of equations is nearly as easy as
  solving linear systems.
\newblock \emph{Communications on Pure and Applied Mathematics}, 70\penalty0
  (5):\penalty0 822--883, 2017.

\bibitem[Chen and Chi(2014)]{chen2014robust}
Y.~Chen and Y.~Chi.
\newblock Robust spectral compressed sensing via structured matrix completion.
\newblock \emph{IEEE Transactions on Information Theory}, 10\penalty0
  (60):\penalty0 6576--6601, 2014.

\bibitem[Chen and Wainwright(2015)]{chen2015fast}
Y.~Chen and M.~J. Wainwright.
\newblock Fast low-rank estimation by projected gradient descent: General
  statistical and algorithmic guarantees.
\newblock \emph{arXiv preprint arXiv:1509.03025}, 2015.

\bibitem[Chen et~al.(2014)Chen, Guibas, and Huang]{chen2014near}
Y.~Chen, L.~Guibas, and Q.~Huang.
\newblock Near-optimal joint object matching via convex relaxation.
\newblock In \emph{International Conference on International Conference on
  Machine Learning}, 2014.

\bibitem[Chen et~al.(2019{\natexlab{a}})Chen, Chi, Fan, and
  Ma]{chen2019gradient}
Y.~Chen, Y.~Chi, J.~Fan, and C.~Ma.
\newblock Gradient descent with random initialization: Fast global convergence
  for nonconvex phase retrieval.
\newblock \emph{Mathematical Programming}, 176\penalty0 (1-2):\penalty0 5--37,
  2019{\natexlab{a}}.

\bibitem[Chen et~al.(2019{\natexlab{b}})Chen, Fan, Ma, and
  Wang]{chen2019spectral}
Y.~Chen, J.~Fan, C.~Ma, and K.~Wang.
\newblock Spectral method and regularized mle are both optimal for top-k
  ranking.
\newblock \emph{Annals of statistics}, 47\penalty0 (4):\penalty0 2204,
  2019{\natexlab{b}}.

\bibitem[Chen et~al.(2019{\natexlab{c}})Chen, Fan, Ma, and
  Yan]{chen2019inference}
Y.~Chen, J.~Fan, C.~Ma, and Y.~Yan.
\newblock Inference and uncertainty quantification for noisy matrix completion.
\newblock \emph{Proceedings of the National Academy of Sciences}, 116\penalty0
  (46):\penalty0 22931--22937, 2019{\natexlab{c}}.

\bibitem[Chen et~al.(2020{\natexlab{b}})Chen, Chi, Fan, Ma, and
  Yan]{chen2019noisy}
Y.~Chen, Y.~Chi, J.~Fan, C.~Ma, and Y.~Yan.
\newblock Noisy matrix completion: Understanding statistical guarantees for
  convex relaxation via nonconvex optimization.
\newblock \emph{SIAM Journal on Optimization}, 30\penalty0 (4):\penalty0
  3098--3121, 2020{\natexlab{b}}.

\bibitem[Chen et~al.(2020{\natexlab{c}})Chen, Fan, Ma, and
  Yan]{chen2020bridging}
Y.~Chen, J.~Fan, C.~Ma, and Y.~Yan.
\newblock Bridging convex and nonconvex optimization in robust pca: Noise,
  outliers, and missing data.
\newblock \emph{arXiv preprint arXiv:2001.05484, accepted to Annals of
  Statistics}, 2020{\natexlab{c}}.

\bibitem[Chi(2016)]{chi2016guaranteed}
Y.~Chi.
\newblock Guaranteed blind sparse spikes deconvolution via lifting and convex
  optimization.
\newblock \emph{IEEE Journal of Selected Topics in Signal Processing},
  10\penalty0 (4):\penalty0 782--794, 2016.

\bibitem[Chi et~al.(2019)Chi, Lu, and Chen]{chi2019nonconvex}
Y.~Chi, Y.~M. Lu, and Y.~Chen.
\newblock Nonconvex optimization meets low-rank matrix factorization: An
  overview.
\newblock \emph{IEEE Transactions on Signal Processing}, 67\penalty0
  (20):\penalty0 5239--5269, 2019.

\bibitem[Ding and Chen(2020)]{ding2020leave}
L.~Ding and Y.~Chen.
\newblock Leave-one-out approach for matrix completion: Primal and dual
  analysis.
\newblock \emph{IEEE Transactions on Information Theory}, 2020.

\bibitem[Dong and Shi(2018)]{dong2018nonconvex}
J.~Dong and Y.~Shi.
\newblock Nonconvex demixing from bilinear measurements.
\newblock \emph{IEEE Transactions on Signal Processing}, 66\penalty0
  (19):\penalty0 5152--5166, 2018.

\bibitem[Dopico(2000)]{dopico2000note}
F.~M. Dopico.
\newblock A note on sin $\theta$ theorems for singular subspace variations.
\newblock \emph{BIT Numerical Mathematics}, 40\penalty0 (2):\penalty0 395--403,
  2000.

\bibitem[Duchi and Ruan(2019)]{duchi2019solving}
J.~C. Duchi and F.~Ruan.
\newblock Solving (most) of a set of quadratic equalities: Composite
  optimization for robust phase retrieval.
\newblock \emph{Information and Inference: A Journal of the IMA}, 8\penalty0
  (3):\penalty0 471--529, 2019.

\bibitem[El~Karoui(2018)]{el2018impact}
N.~El~Karoui.
\newblock On the impact of predictor geometry on the performance on
  high-dimensional ridge-regularized generalized robust regression estimators.
\newblock \emph{Probability Theory and Related Fields}, 170\penalty0
  (1-2):\penalty0 95--175, 2018.

\bibitem[Goemans and Williamson(1994)]{goemans1994879}
M.~X. Goemans and D.~P. Williamson.
\newblock . 879-approximation algorithms for max cut and max 2sat.
\newblock In \emph{Proceedings of the twenty-sixth annual ACM symposium on
  Theory of computing}, pages 422--431, 1994.

\bibitem[Golub and Van~Loan(2013)]{golub2013matrix}
G.~H. Golub and C.~F. Van~Loan.
\newblock \emph{Matrix computations}, volume~3.
\newblock JHU press, 2013.

\bibitem[Huang and Hand(2018)]{huang2018blind}
W.~Huang and P.~Hand.
\newblock Blind deconvolution by a steepest descent algorithm on a quotient
  manifold.
\newblock \emph{SIAM Journal on Imaging Sciences}, 11\penalty0 (4):\penalty0
  2757--2785, 2018.

\bibitem[Jain et~al.(2013)Jain, Netrapalli, and Sanghavi]{jain2013low}
P.~Jain, P.~Netrapalli, and S.~Sanghavi.
\newblock Low-rank matrix completion using alternating minimization.
\newblock In \emph{Proceedings of the forty-fifth annual ACM symposium on
  Theory of computing}, pages 665--674, 2013.

\bibitem[Jefferies and Christou(1993)]{jefferies1993restoration}
S.~M. Jefferies and J.~C. Christou.
\newblock Restoration of astronomical images by iterative blind deconvolution.
\newblock \emph{The Astrophysical Journal}, 415:\penalty0 862, 1993.

\bibitem[Jung et~al.(2017)Jung, Krahmer, and St{\"o}ger]{jung2017blind}
P.~Jung, F.~Krahmer, and D.~St{\"o}ger.
\newblock Blind demixing and deconvolution at near-optimal rate.
\newblock \emph{IEEE Transactions on Information Theory}, 64\penalty0
  (2):\penalty0 704--727, 2017.

\bibitem[Kech and Krahmer(2017)]{kech2017optimal}
M.~Kech and F.~Krahmer.
\newblock Optimal injectivity conditions for bilinear inverse problems with
  applications to identifiability of deconvolution problems.
\newblock \emph{SIAM Journal on Applied Algebra and Geometry}, 1\penalty0
  (1):\penalty0 20--37, 2017.

\bibitem[Keshavan et~al.(2009)Keshavan, Montanari, and Oh]{keshavan2009matrix}
R.~Keshavan, A.~Montanari, and S.~Oh.
\newblock Matrix completion from noisy entries.
\newblock In \emph{Advances in neural information processing systems}, pages
  952--960, 2009.

\bibitem[Koltchinskii et~al.(2011)Koltchinskii, Lounici, Tsybakov,
  et~al.]{koltchinskii2011nuclear}
V.~Koltchinskii, K.~Lounici, A.~B. Tsybakov, et~al.
\newblock Nuclear-norm penalization and optimal rates for noisy low-rank matrix
  completion.
\newblock \emph{The Annals of Statistics}, 39\penalty0 (5):\penalty0
  2302--2329, 2011.

\bibitem[Kundur and Hatzinakos(1996)]{kundur1996blind}
D.~Kundur and D.~Hatzinakos.
\newblock Blind image deconvolution.
\newblock \emph{IEEE signal processing magazine}, 13\penalty0 (3):\penalty0
  43--64, 1996.

\bibitem[Lee et~al.(2016)Lee, Li, Junge, and Bresler]{lee2016blind}
K.~Lee, Y.~Li, M.~Junge, and Y.~Bresler.
\newblock Blind recovery of sparse signals from subsampled convolution.
\newblock \emph{IEEE Transactions on Information Theory}, 63\penalty0
  (2):\penalty0 802--821, 2016.

\bibitem[Lee et~al.(2018)Lee, Tian, and Romberg]{lee2018fast}
K.~Lee, N.~Tian, and J.~Romberg.
\newblock Fast and guaranteed blind multichannel deconvolution under a bilinear
  system model.
\newblock \emph{IEEE Transactions on Information Theory}, 64\penalty0
  (7):\penalty0 4792--4818, 2018.

\bibitem[Li et~al.(2019)Li, Ling, Strohmer, and Wei]{li2019rapid}
X.~Li, S.~Ling, T.~Strohmer, and K.~Wei.
\newblock Rapid, robust, and reliable blind deconvolution via nonconvex
  optimization.
\newblock \emph{Applied and computational harmonic analysis}, 47\penalty0
  (3):\penalty0 893--934, 2019.

\bibitem[Li and Bresler(2019)]{li2019multichannel}
Y.~Li and Y.~Bresler.
\newblock Multichannel sparse blind deconvolution on the sphere.
\newblock \emph{IEEE Transactions on Information Theory}, 65\penalty0
  (11):\penalty0 7415--7436, 2019.

\bibitem[Li et~al.(2015)Li, Lee, and Bresler]{li2015}
Y.~Li, K.~Lee, and Y.~Bresler.
\newblock A unified framework for identifiability analysis in bilinear inverse
  problems with applications to subspace and sparsity models.
\newblock \emph{arXiv preprint arXiv:1501.06120}, 2015.

\bibitem[Li et~al.(2016)Li, Lee, and Bresler]{li2016identifiability}
Y.~Li, K.~Lee, and Y.~Bresler.
\newblock Identifiability in blind deconvolution with subspace or sparsity
  constraints.
\newblock \emph{IEEE Transactions on information Theory}, 62\penalty0
  (7):\penalty0 4266--4275, 2016.

\bibitem[Ling and Strohmer(2015)]{ling2015self}
S.~Ling and T.~Strohmer.
\newblock Self-calibration and biconvex compressive sensing.
\newblock \emph{Inverse Problems}, 31\penalty0 (11):\penalty0 115002, 2015.

\bibitem[Ling and Strohmer(2016)]{ling2016simultaneous}
S.~Ling and T.~Strohmer.
\newblock Simultaneous blind deconvolution and blind demixing via convex
  programming.
\newblock In \emph{2016 50th Asilomar Conference on Signals, Systems and
  Computers}, pages 1223--1227. IEEE, 2016.

\bibitem[Ling and Strohmer(2017)]{ling2017blind}
S.~Ling and T.~Strohmer.
\newblock Blind deconvolution meets blind demixing: Algorithms and performance
  bounds.
\newblock \emph{IEEE Transactions on Information Theory}, 63\penalty0
  (7):\penalty0 4497--4520, 2017.

\bibitem[Ling and Strohmer(2019)]{ling2019regularized}
S.~Ling and T.~Strohmer.
\newblock Regularized gradient descent: a non-convex recipe for fast joint
  blind deconvolution and demixing.
\newblock \emph{Information and Inference: A Journal of the IMA}, 8\penalty0
  (1):\penalty0 1--49, 2019.

\bibitem[Ma et~al.(2018)Ma, Wang, Chi, and Chen]{ma2017implicit}
C.~Ma, K.~Wang, Y.~Chi, and Y.~Chen.
\newblock Implicit regularization in nonconvex statistical estimation: Gradient
  descent converges linearly for phase retrieval and matrix completion.
\newblock In \emph{International Conference on Machine Learning}, pages
  3345--3354. PMLR, 2018.

\bibitem[Ma et~al.(2019)Ma, Xu, and Maleki]{ma2019optimization}
J.~Ma, J.~Xu, and A.~Maleki.
\newblock Optimization-based amp for phase retrieval: The impact of
  initialization and $\ell_2$ regularization.
\newblock \emph{IEEE Transactions on Information Theory}, 65\penalty0
  (6):\penalty0 3600--3629, 2019.

\bibitem[Oymak et~al.(2015)Oymak, Jalali, Fazel, Eldar, and
  Hassibi]{oymak2015simultaneously}
S.~Oymak, A.~Jalali, M.~Fazel, Y.~C. Eldar, and B.~Hassibi.
\newblock Simultaneously structured models with application to sparse and
  low-rank matrices.
\newblock \emph{IEEE Transactions on Information Theory}, 61\penalty0
  (5):\penalty0 2886--2908, 2015.

\bibitem[Parikh and Boyd(2014)]{parikh2014proximal}
N.~Parikh and S.~Boyd.
\newblock Proximal algorithms.
\newblock \emph{Foundations and Trends in optimization}, 1\penalty0
  (3):\penalty0 127--239, 2014.

\bibitem[Qu et~al.(2017)Qu, Zhang, Eldar, and Wright]{qu2017convolutional}
Q.~Qu, Y.~Zhang, Y.~Eldar, and J.~Wright.
\newblock Convolutional phase retrieval.
\newblock In \emph{Advances in Neural Information Processing Systems}, pages
  6086--6096, 2017.

\bibitem[Qu et~al.(2019)Qu, Li, and Zhu]{qu2019nonconvex}
Q.~Qu, X.~Li, and Z.~Zhu.
\newblock A nonconvex approach for exact and efficient multichannel sparse
  blind deconvolution.
\newblock In \emph{Advances in Neural Information Processing Systems}, pages
  4015--4026, 2019.

\bibitem[Shechtman et~al.(2014)Shechtman, Beck, and Eldar]{shechtman2014gespar}
Y.~Shechtman, A.~Beck, and Y.~C. Eldar.
\newblock Gespar: Efficient phase retrieval of sparse signals.
\newblock \emph{IEEE transactions on signal processing}, 62\penalty0
  (4):\penalty0 928--938, 2014.

\bibitem[Shechtman et~al.(2015)Shechtman, Eldar, Cohen, Chapman, Miao, and
  Segev]{shechtman2015phase}
Y.~Shechtman, Y.~C. Eldar, O.~Cohen, H.~N. Chapman, J.~Miao, and M.~Segev.
\newblock Phase retrieval with application to optical imaging: a contemporary
  overview.
\newblock \emph{IEEE signal processing magazine}, 32\penalty0 (3):\penalty0
  87--109, 2015.

\bibitem[Shi and Chi(2021)]{shi2020manifold}
L.~Shi and Y.~Chi.
\newblock Manifold gradient descent solves multi-channel sparse blind
  deconvolution provably and efficiently.
\newblock \emph{IEEE Transactions on Information Theory}, 2021.

\bibitem[Sun and Luo(2016)]{sun2016guaranteed}
R.~Sun and Z.-Q. Luo.
\newblock Guaranteed matrix completion via non-convex factorization.
\newblock \emph{IEEE Transactions on Information Theory}, 62\penalty0
  (11):\penalty0 6535--6579, 2016.

\bibitem[Tang et~al.(2013)Tang, Bhaskar, Shah, and Recht]{tang2013compressed}
G.~Tang, B.~N. Bhaskar, P.~Shah, and B.~Recht.
\newblock Compressed sensing off the grid.
\newblock \emph{IEEE transactions on information theory}, 59\penalty0
  (11):\penalty0 7465--7490, 2013.

\bibitem[Tong et~al.(1994)Tong, Xu, and Kailath]{tong1994blind}
L.~Tong, G.~Xu, and T.~Kailath.
\newblock Blind identification and equalization based on second-order
  statistics: A time domain approach.
\newblock \emph{IEEE Transactions on information Theory}, 40\penalty0
  (2):\penalty0 340--349, 1994.

\bibitem[Vershynin(2010)]{vershynin2010introduction}
R.~Vershynin.
\newblock Introduction to the non-asymptotic analysis of random matrices.
\newblock \emph{arXiv preprint arXiv:1011.3027}, 2010.

\bibitem[Vershynin(2018)]{vershynin2018high}
R.~Vershynin.
\newblock \emph{High-dimensional probability: An introduction with applications
  in data science}, volume~47.
\newblock Cambridge University Press, 2018.

\bibitem[Waldspurger et~al.(2015)Waldspurger, d'Aspremont, and
  Mallat]{waldspurger2015phase}
I.~Waldspurger, A.~d'Aspremont, and S.~Mallat.
\newblock Phase recovery, maxcut and complex semidefinite programming.
\newblock \emph{Mathematical Programming}, 149\penalty0 (1-2):\penalty0 47--81,
  2015.

\bibitem[Wang et~al.(2017{\natexlab{a}})Wang, Giannakis, and
  Eldar]{wang2017solving}
G.~Wang, G.~B. Giannakis, and Y.~C. Eldar.
\newblock Solving systems of random quadratic equations via truncated amplitude
  flow.
\newblock \emph{IEEE Transactions on Information Theory}, 64\penalty0
  (2):\penalty0 773--794, 2017{\natexlab{a}}.

\bibitem[Wang et~al.(2017{\natexlab{b}})Wang, Zhang, Giannakis, Ak{\c{c}}akaya,
  and Chen]{wang2017sparse}
G.~Wang, L.~Zhang, G.~B. Giannakis, M.~Ak{\c{c}}akaya, and J.~Chen.
\newblock Sparse phase retrieval via truncated amplitude flow.
\newblock \emph{IEEE Transactions on Signal Processing}, 66\penalty0
  (2):\penalty0 479--491, 2017{\natexlab{b}}.

\bibitem[Wang and Chi(2016)]{wang2016blind}
L.~Wang and Y.~Chi.
\newblock Blind deconvolution from multiple sparse inputs.
\newblock \emph{IEEE Signal Processing Letters}, 23\penalty0 (10):\penalty0
  1384--1388, 2016.

\bibitem[Wang and Poor(1998)]{wang1998blind}
X.~Wang and H.~V. Poor.
\newblock Blind equalization and multiuser detection in dispersive cdma
  channels.
\newblock \emph{IEEE Transactions on Communications}, 46\penalty0 (1):\penalty0
  91--103, 1998.

\bibitem[Wunder et~al.(2015)Wunder, Boche, Strohmer, and
  Jung]{wunder2015sparse}
G.~Wunder, H.~Boche, T.~Strohmer, and P.~Jung.
\newblock Sparse signal processing concepts for efficient 5g system design.
\newblock \emph{IEEE Access}, 3:\penalty0 195--208, 2015.

\bibitem[Xu et~al.(2019)Xu, Maleki, and Rad]{xu2019consistent}
J.~Xu, A.~Maleki, and K.~R. Rad.
\newblock Consistent risk estimation in high-dimensional linear regression.
\newblock \emph{arXiv preprint arXiv:1902.01753}, 2019.

\bibitem[Zhang et~al.(2016)Zhang, Chi, and Liang]{zhang2016provable}
H.~Zhang, Y.~Chi, and Y.~Liang.
\newblock Provable non-convex phase retrieval with outliers: Median
  truncatedwirtinger flow.
\newblock In \emph{International conference on machine learning}, pages
  1022--1031, 2016.

\bibitem[Zhang et~al.(2017)Zhang, Lau, Kuo, Cheung, Pasupathy, and
  Wright]{zhang2017global}
Y.~Zhang, Y.~Lau, H.-w. Kuo, S.~Cheung, A.~Pasupathy, and J.~Wright.
\newblock On the global geometry of sphere-constrained sparse blind
  deconvolution.
\newblock In \emph{Proceedings of the IEEE Conference on Computer Vision and
  Pattern Recognition}, pages 4894--4902, 2017.

\bibitem[Zhang et~al.(2019)Zhang, Kuo, and Wright]{zhang2019structured}
Y.~Zhang, H.-W. Kuo, and J.~Wright.
\newblock Structured local optima in sparse blind deconvolution.
\newblock \emph{IEEE Transactions on Information Theory}, 66\penalty0
  (1):\penalty0 419--452, 2019.

\bibitem[Zhang et~al.(2020)Zhang, Qu, and Wright]{zhang2020symmetry}
Y.~Zhang, Q.~Qu, and J.~Wright.
\newblock From symmetry to geometry: Tractable nonconvex problems.
\newblock \emph{arXiv preprint arXiv:2007.06753}, 2020.

\bibitem[Zheng and Lafferty(2016)]{zheng2016convergence}
Q.~Zheng and J.~Lafferty.
\newblock Convergence analysis for rectangular matrix completion using
  burer-monteiro factorization and gradient descent.
\newblock \emph{arXiv preprint arXiv:1605.07051}, 2016.

\bibitem[Zhong et~al.(2015)Zhong, Jain, and Dhillon]{zhong2015efficient}
K.~Zhong, P.~Jain, and I.~S. Dhillon.
\newblock Efficient matrix sensing using rank-1 gaussian measurements.
\newblock In \emph{International conference on algorithmic learning theory},
  pages 3--18. Springer, 2015.

\bibitem[Zhong and Boumal(2018)]{zhong2018near}
Y.~Zhong and N.~Boumal.
\newblock Near-optimal bounds for phase synchronization.
\newblock \emph{SIAM Journal on Optimization}, 28\penalty0 (2):\penalty0
  989--1016, 2018.

\end{thebibliography}

\appendix

\section*{Appendix structure}

Appendix \ref{appendix:noncvx} and \ref{appendix:cvx} analyze the
Fourier designs. In Appendix \ref{appendix:noncvx}, we present the analysis
of the nonconvex gradient method and the proof of Theorem \ref{thm:nonconvex}.
Appendix \ref{appendix:cvx} gives the complete proof of Theorem \ref{theorem:cvx}. 
In addition, Appendix \ref{appendix:gaussian-ncvx} and \ref{appendix:Proof-of-Theorem-gaussian-cvx}
and provide proofs for the Gaussian
designs, while Appendix \ref{appendix:gaussian-ncvx}
proves Theorem \ref{thm:gaussian-ncvx} and Appendix \ref{appendix:Proof-of-Theorem-gaussian-cvx}
proves Theorem \ref{theorem:gaussian-cvx}. Appendix \ref{appendix:LB}
justifies two minimax lower bounds in Theorem \ref{thm:LB}.
Appendix \ref{appendix:auxiliary-lemmas} lists several useful lemmas
and their proofs. 

\section{Analysis: Nonconvex gradient method under Fourier design\label{appendix:noncvx}}

Since the proof of Theorem \ref{theorem:cvx} is built upon Theorem
\ref{thm:nonconvex}, we shall first present the proof of the nonconvex
part. Without loss of generality, we assume that
\begin{equation}
\left\Vert \boldsymbol{h}^{\star}\right\Vert _{2}=\left\Vert \boldsymbol{x}^{\star}\right\Vert _{2}=1\label{eq:assumption-h-x-norm}
\end{equation}
 throughout the proof. For the sake of notational convenience, for
each iterate $(\bm{h}^{t},\bm{x}^{t})$ we define the following alignment
parameters\begin{subequations}\label{subeq:defn-alphat-align}
\begin{align}
\alpha^{t} & \coloneqq\arg\min_{\alpha\in\mathbb{C}}\left\{ \left\Vert \tfrac{1}{\overline{\alpha}}\boldsymbol{h}^{t}-\boldsymbol{h}^{\star}\right\Vert _{2}^{2}+\left\Vert \alpha\boldsymbol{x}^{t}-\boldsymbol{x}^{\star}\right\Vert _{2}^{2}\right\} ,\label{eq:defn-alignment-alphat}\\
\alpha^{t+1/2} & \coloneqq\arg\min_{\alpha\in\mathbb{C}}\left\{ \left\Vert \tfrac{1}{\overline{\alpha}}\boldsymbol{h}^{t+1/2}-\boldsymbol{h}^{\star}\right\Vert _{2}^{2}+\big\|\alpha\boldsymbol{x}^{t+1/2}-\boldsymbol{x}^{\star}\big\|_{2}^{2}\right\} ,\label{eq:defn-alignment-alphat-half}
\end{align}
\end{subequations}which lead to the following simple relations
\begin{equation}
\alpha^{t+1}=\sqrt{\frac{\left\Vert \bm{x}^{t+1/2}\right\Vert _{2}}{\left\Vert \bm{h}^{t+1/2}\right\Vert _{2}}}\alpha^{t+1/2}\qquad\text{and}\qquad\mathsf{dist}\big(\bm{z}^{t+1/2},\bm{z}^{\star}\big)=\mathsf{dist}\big(\bm{z}^{t+1},\bm{z}^{\star}\big).\label{eq:balance-relation-1}
\end{equation}
With these in place, attention should be directed to the properly
rescaled iterate\begin{subequations}
\begin{align}
\widetilde{\bm{z}}^{t+1/2}=\big(\widetilde{\bm{h}}^{t+1/2},\widetilde{\bm{x}}^{t+1/2}\big) & :=\big(\tfrac{1}{\overline{\alpha^{t+1/2}}}\boldsymbol{h}^{t+1/2},\alpha^{t+1/2}\bm{x}^{t+1/2}\big),\label{eq:defn-htilde-xtilde-1/2}\\
\widetilde{\bm{z}}^{t}=\big(\widetilde{\bm{h}}^{t},\widetilde{\bm{x}}^{t}\big) & :=\big(\tfrac{1}{\overline{\alpha^{t}}}\boldsymbol{h}^{t},\alpha^{t}\bm{x}^{t}\big).\label{eq:defn-htilde-xtilde}
\end{align}
\end{subequations}Additionally, we shall also define\begin{subequations}
\begin{align}
\widehat{\bm{z}}^{t+1/2}=(\widehat{\bm{h}}^{t+1/2},\widehat{\bm{x}}^{t+1/2}) & :=\big(\tfrac{1}{\overline{\alpha^{t}}}\bm{h}^{t+1/2},\alpha^{t}\bm{x}^{t+1/2}\big)\label{eq:defn-hhat-xhat-1/2}\\
\widehat{\bm{z}}^{t+1}=(\widehat{\bm{h}}^{t+1},\widehat{\bm{x}}^{t+1}) & :=\big(\tfrac{1}{\overline{\alpha^{t}}}\bm{h}^{t+1},\alpha^{t}\bm{x}^{t+1}\big)\label{eq:defn-hhat-xhat}
\end{align}
\end{subequations}that are rescaled in a different way, which will
appear often in the analysis. 

\subsection{Induction hypotheses}

Our analysis is inductive in nature; more concretely, we aim to justify
the following set of hypotheses by induction: \begin{subequations}\label{sec:hypotheses-ncvx}
\begin{align}
\mathsf{dist}\left(\boldsymbol{z}^{t},\boldsymbol{z}^{\star}\right) & \leq\big\|\widehat{\boldsymbol{z}}^{t-1/2}-\boldsymbol{z}^{\star}\big\|_{2}\leq\rho\mathsf{dist}\left(\bm{z}^{t-1},\bm{z}^{\star}\right)+C_{1}\eta\left(\lambda+\sigma\sqrt{K\log m}\right),\label{eq:hypothesisdist}\\
\max_{1\leq l\leq m}\left|\boldsymbol{a}_{l}^{\mathsf{H}}\left(\widetilde{\bm{x}}^{t}-\boldsymbol{x}^{\star}\right)\right| & \leq C_{3}\left(\sqrt{\frac{\mu^{2}K\log^{2}m}{m}}+\sqrt{\log m}\left(\lambda+\sigma\sqrt{K\log m}\right)\right),\label{eq:hypothesisincoherence1}\\
\max_{1\leq l\leq m}\big|\boldsymbol{b}_{l}^{\mathsf{H}}\widetilde{\bm{h}}^{t}\big| & \leq C_{4}\left(\frac{\mu\log^{2}m}{\sqrt{m}}+\sigma\right),\label{eq:hypothesisincoherence2}
\end{align}
where $\rho=1-\eta/16$ and $C_{1},C_{3},C_{4}>0$ are some universal
constants. Here, the hypothesis \eqref{eq:hypothesisdist} is made
for all $0<t\leq t_{0}$, while the hypotheses \eqref{eq:hypothesisincoherence1}
and \eqref{eq:hypothesisincoherence2} are made for all $0\leq t\leq t_{0}$.
Clearly, if the hypotheses \eqref{eq:hypothesisdist} can be established,
then simple recursion yields
\begin{align}
\mathsf{dist}\left(\boldsymbol{z}^{t},\boldsymbol{z}^{\star}\right) & \lesssim\rho^{t}\mathsf{dist}\left(\bm{z}^{0},\bm{z}^{\star}\right)+\frac{C_{1}\eta\left(\lambda+\sigma\sqrt{K\log m}\right)}{1-\rho}\nonumber \\
 & =\rho^{t}\mathsf{dist}\left(\bm{z}^{0},\bm{z}^{\star}\right)+\frac{C_{1}\left(\lambda+\sigma\sqrt{K\log m}\right)}{c_{\rho}},\qquad0\leq t\leq t_{0}\label{eq:dist-geometric}
\end{align}
as claimed. Moreover, one might naturally wonder why we are in need
of the additional hypotheses (\ref{eq:hypothesisincoherence1}) and
(\ref{eq:hypothesisincoherence2}) that might seem irrelevant at first
glance. As it turns out, these two hypotheses --- which characterize
certain incoherence conditions of the iterates w.r.t.~the design
vectors --- play a pivotal role in the analysis, as they enable some
sort of ``restricted strong convexity'' that proves crucial for
guaranteeing linear convergence.

In addition, the analysis also relies upon the following important
properties of the initialization, which we shall establish momentarily:
\begin{align}
\mathsf{dist}\left(\boldsymbol{z}^{0},\boldsymbol{z}^{\star}\right) & \lesssim\sqrt{\frac{\mu^{2}K\log m}{m}}+\sigma\sqrt{K\log m},\label{eq:initialization-z0}\\
\max_{1\leq j\leq m}\left|\boldsymbol{a}_{j}^{\mathsf{H}}\big(\widetilde{\boldsymbol{x}}^{0}-\boldsymbol{x}^{\star}\big)\right| & \lesssim\sqrt{\frac{\mu^{2}K\log^{2}m}{m}}+\sigma\sqrt{K}\log m,\label{eq:initialization-incoh1}\\
\max_{1\leq l\le m}\big|\bm{b}_{l}^{\mathsf{H}}\widetilde{\bm{h}}^{0}\big| & \lesssim\frac{\mu\log^{2}m}{\sqrt{m}}+\sigma,\label{eq:initialization-incoh2}\\
\big||\alpha^{0}|-1\big| & \leq1/4.\label{eq:initialization-alpha0}
\end{align}
\end{subequations}

\subsection{Preliminaries}

Before proceeding to the proof, we gather several preliminary facts
that will be useful throughout.

\subsubsection{Wirtinger calculus and notation\label{subsec:Wirtinger-calculus}}

Given that this problem concerns complex-valued vectors/matrices,
we find it convenient to work with Wirtinger calculus; see \citet[Section 6]{candes2015phase}
and \citet[Section D.3.1]{ma2017implicit} for a brief introduction.
Here, we shall simply record below the expressions for the Wirtinger
gradient and the Wirtinger Hessian w.r.t.~the objective function
$f(\cdot)$ defined in (\ref{eq:objncvx}):\begin{subequations}\label{eq:whessian}
\begin{align}
\nabla_{\boldsymbol{h}}f\left(\boldsymbol{h},\boldsymbol{x}\right)= & \sum_{j=1}^{m}\left(\boldsymbol{b}_{j}^{\mathsf{H}}\boldsymbol{hx}^{\mathsf{H}}\boldsymbol{a}_{j}-y_{j}\right)\boldsymbol{b}_{j}\boldsymbol{a}_{j}^{\mathsf{H}}\boldsymbol{x}+\lambda\bm{h},\\
\nabla_{\boldsymbol{x}}f\left(\boldsymbol{h},\boldsymbol{x}\right)= & \sum_{j=1}^{m}\overline{\left(\boldsymbol{b}_{j}^{\mathsf{H}}\boldsymbol{hx}^{\mathsf{H}}\boldsymbol{a}_{j}-y_{j}\right)}\boldsymbol{a}_{j}\boldsymbol{b}_{j}^{\mathsf{H}}\boldsymbol{h}+\lambda\bm{x},\\
\nabla^{2}f\left(\boldsymbol{h},\boldsymbol{x}\right)= & \left[\begin{array}{cc}
\boldsymbol{A} & \boldsymbol{B}\\
\boldsymbol{B}^{\mathsf{H}} & \overline{\boldsymbol{A}}
\end{array}\right],
\end{align}
\end{subequations}where
\begin{align*}
\boldsymbol{A}:= & \left[\begin{array}{cc}
\sum_{j=1}^{m}\left|\boldsymbol{a}_{j}^{\mathsf{H}}\boldsymbol{x}\right|{}^{2}\boldsymbol{b}_{j}\boldsymbol{b}_{j}^{\mathsf{H}}+\lambda & \sum_{j=1}^{m}\left(\boldsymbol{b}_{j}^{\mathsf{H}}\boldsymbol{hx}^{\mathsf{H}}\boldsymbol{a}_{j}-y_{j}\right)\boldsymbol{b}_{j}\boldsymbol{a}_{j}^{\mathsf{H}}\\
\sum_{j=1}^{m}\left[\left(\boldsymbol{b}_{j}^{\mathsf{H}}\boldsymbol{hx}^{\mathsf{H}}\boldsymbol{a}_{j}-y_{j}\right)\boldsymbol{b}_{j}\boldsymbol{a}_{j}^{\mathsf{H}}\right]^{\mathsf{H}} & \sum_{j=1}^{m}\left|\boldsymbol{b}_{j}^{\mathsf{H}}\boldsymbol{h}\right|{}^{2}\boldsymbol{a}_{j}\boldsymbol{a}_{j}^{\mathsf{H}}+\lambda
\end{array}\right]\in\mathbb{C}^{2K\times2K},\\
\boldsymbol{B}:= & \left[\begin{array}{cc}
\boldsymbol{0} & \sum_{j=1}^{m}\boldsymbol{b}_{j}\boldsymbol{b}_{j}^{\mathsf{H}}\boldsymbol{h}\left(\boldsymbol{a}_{j}\boldsymbol{a}_{j}^{\mathsf{H}}\boldsymbol{x}\right)^{\mathsf{H}}\\
\sum_{j=1}^{m}\boldsymbol{a}_{j}\boldsymbol{a}_{j}^{\mathsf{H}}\boldsymbol{x}\left(\boldsymbol{b}_{j}\boldsymbol{b}_{j}^{\mathsf{H}}\boldsymbol{h}\right)^{\mathsf{H}} & \boldsymbol{0}
\end{array}\right]\in\mathbb{C}^{2K\times2K}.
\end{align*}
Throughout this paper, we shall often use $f\left(\boldsymbol{h},\boldsymbol{x}\right)$
and $f\left(\boldsymbol{z}\right)$ interchangeably for any $\bm{z}=\footnotesize\left[\begin{array}{c}
\bm{h}\\
\bm{x}
\end{array}\right]$, whenever it is clear from the context.

Before proceeding, we present two useful properties of the operator $\mathcal{A}$ and
the design vectors $\{\bm{b}_{j}\}_{j=1}^{m}$.

\begin{lemma}\label{lemma:normbound}For $\mathcal{A}$ defined in
\eqref{eq:operatordef}, with probability at least $1-m^{-\gamma}$,
\[
\left\Vert \mathcal{A}\right\Vert \leq\sqrt{2K\log K+\gamma\log m}.
\]
\end{lemma}\begin{proof}See \citet[Lemma 5.12]{li2019rapid}.\end{proof}

\begin{lemma}\label{lemma:dftbound}For any $m\geq3$ and any $1\leq l\leq m$,
we have 
\[
\sum_{j=1}^{m}\left|\bm{b}_{l}^{\mathsf{H}}\bm{b}_{j}\right|\leq4\log m.
\]
\end{lemma}\begin{proof}See \citet[Lemma 48]{ma2017implicit}.\end{proof}

\subsubsection{Leave-one-out auxiliary sequences}

The key to establishing the incoherence hypotheses (\ref{eq:hypothesisincoherence1})
and (\ref{eq:hypothesisincoherence2}) is to introduce a collection
of auxiliary leave-one-out sequences --- an approach first introduced
by \cite{ma2017implicit}. Specifically, for each $1\leq l\leq m$,
define the leave-one-out loss function as follows
\[
f^{(l)}\left(\boldsymbol{h},\boldsymbol{x}\right)\coloneqq\sum_{j:j\neq l}\left|\boldsymbol{b}_{j}^{\mathsf{H}}\boldsymbol{h}\boldsymbol{x}^{\mathsf{H}}\boldsymbol{a}_{j}-y_{j}\right|^{2}+\lambda\left\Vert \bm{h}\right\Vert _{2}^{2}+\lambda\left\Vert \bm{x}\right\Vert _{2}^{2},
\]
which is obtained by discarding the $l$th sample. We then generate
the auxiliary sequence $\{\bm{h}^{(t),l},\bm{x}^{(t),l}\}_{t\geq0}$
by running the same nonconvex algorithm w.r.t.~$f^{(l)}(\cdot,\cdot)$,
as summarized in Algorithm \ref{alg:gd-mc-ncvx-LOO}. In a nutshell,
the resulting leave-one-out sequence $\{\bm{h}^{(t),l},\bm{x}^{(t),l}\}_{t\geq0}$
is statistically independent from the design vector $\bm{a}_{l}$
and is expected to stay exceedingly close to the original sequence
(given that only a single sample is dropped), which in turn facilitate
the analysis of the correlation of $\bm{a}_{l}$ and $\bm{x}^{t}$
as claimed in (\ref{eq:hypothesisincoherence1}). In the mean time,
this strategy also proves useful in controlling the correlation of
$\bm{b}_{l}$ and $\bm{h}^{t}$ as in (\ref{eq:hypothesisincoherence2}),
albeit with more delicate arguments.

\begin{algorithm}[h]
\caption{The $l$th leave-one-out sequence for nonconvex blind deconvolution}

\label{alg:gd-mc-ncvx-LOO}\begin{algorithmic}

\STATE \textbf{{Input}}: $\left\{ \bm{a}_{j}\right\} _{1\leq j\leq m,j\neq l}$,
$\left\{ \bm{b}_{j}\right\} _{1\leq j\leq m,j\neq l}$ and $\left\{ y_{j}\right\} _{1\leq j\leq m,j\neq l}$.

\STATE \textbf{{Spectral initialization}}: let $\sigma_{1}\left(\bm{M}^{\left(l\right)}\right)$,
$\check{\bm{h}}^{0,\left(l\right)}$ and $\check{\bm{x}}^{0,\left(l\right)}$
be the leading singular value, the leading left and right singular
vectors of 
\begin{equation}
\bm{M}^{\left(l\right)}\coloneqq\sum_{j:j\neq l}y_{j}\bm{b}_{j}\bm{a}_{j}^{\mathsf{H}},\label{eq:defn-Ml}
\end{equation}
respectively. Set $\bm{h}^{0,\left(l\right)}=\sqrt{\sigma_{1}\left(\bm{M}^{\left(l\right)}\right)}\,\check{\bm{h}}^{0,\left(l\right)}$
and $\bm{x}^{0,\left(l\right)}=\sqrt{\sigma_{1}\left(\bm{M}^{\left(l\right)}\right)}\,\check{\bm{x}}^{0,\left(l\right)}$.

\STATE \textbf{{Gradient updates}}: \textbf{for }$t=0,1,\ldots,t_{0}-1$
\textbf{do}

\STATE \vspace{-1em}
 \begin{subequations}\label{subeq:gradient_update_ncvx-1} 
\begin{align}
\begin{aligned}\left[\begin{array}{c}
\boldsymbol{h}^{t+1/2,\left(l\right)}\\
\boldsymbol{x}^{t+1/2,\left(l\right)}
\end{array}\right]= & \left[\begin{array}{c}
\boldsymbol{h}^{t,\left(l\right)}\\
\boldsymbol{x}^{t,\left(l\right)}
\end{array}\right]-\eta\left[\begin{array}{c}
\nabla_{\boldsymbol{h}}f^{\left(l\right)}\left(\boldsymbol{h}^{t},\boldsymbol{x}^{t}\right)\\
\nabla_{\boldsymbol{x}}f^{\left(l\right)}\left(\boldsymbol{h}^{t},\boldsymbol{x}^{t}\right)
\end{array}\right],\\
\left[\begin{array}{c}
\boldsymbol{h}^{t+1,\left(l\right)}\\
\boldsymbol{x}^{t+1,\left(l\right)}
\end{array}\right]= & \left[\begin{array}{c}
\sqrt{\frac{\left\Vert \bm{x}^{t+1/2,\left(l\right)}\right\Vert _{2}}{\left\Vert \bm{h}^{t+1/2,\left(l\right)}\right\Vert _{2}}}\boldsymbol{h}^{t+1/2,\left(l\right)}\\
\sqrt{\frac{\left\Vert \bm{h}^{t+1/2,\left(l\right)}\right\Vert _{2}}{\left\Vert \bm{x}^{t+1/2,\left(l\right)}\right\Vert _{2}}}\boldsymbol{x}^{t+1/2,\left(l\right)}
\end{array}\right].
\end{aligned}
\label{eq:defn-WF-LOO-BD}
\end{align}
\end{subequations}

\end{algorithmic} 
\end{algorithm}

Similar to the notation adopted for the original sequence, we shall
define the alignment parameter for the leave-one-out sequence as follows\begin{subequations}
\begin{align}
\alpha^{t,(l)} & \coloneqq\arg\min_{\alpha\in\mathbb{C}}\left\{ \left\Vert \tfrac{1}{\overline{\alpha}}\boldsymbol{h}^{t,(l)}-\boldsymbol{h}^{\star}\right\Vert _{2}^{2}+\big\|\alpha\boldsymbol{x}^{t,(l)}-\boldsymbol{x}^{\star}\big\|_{2}^{2}\right\} ,\label{eq:alignment-LOO}\\
\alpha^{t+1/2,(l)} & \coloneqq\arg\min_{\alpha\in\mathbb{C}}\left\{ \left\Vert \tfrac{1}{\overline{\alpha}}\boldsymbol{h}^{t+1/2,(l)}-\boldsymbol{h}^{\star}\right\Vert _{2}^{2}+\big\|\alpha\boldsymbol{x}^{t+1/2,(l)}-\boldsymbol{x}^{\star}\big\|_{2}^{2}\right\} ,
\end{align}
\end{subequations}along with the properly rescaled iterates\begin{subequations}
\begin{align}
\widetilde{\boldsymbol{z}}^{t,\left(l\right)}=\left[\begin{array}{c}
\widetilde{\boldsymbol{h}}^{t,(l)}\\
\widetilde{\boldsymbol{x}}^{t,(l)}
\end{array}\right] & :=\left[\begin{array}{c}
\frac{1}{\overline{\alpha^{t,\left(l\right)}}}\boldsymbol{h}^{t,(l)}\\
\alpha^{t,(l)}\boldsymbol{x}^{t,(l)}
\end{array}\right],\label{eq:LOO-zt}\\
\widetilde{\boldsymbol{z}}^{t+1/2,\left(l\right)}=\left[\begin{array}{c}
\widetilde{\boldsymbol{h}}^{t+1/2,(l)}\\
\widetilde{\boldsymbol{x}}^{t+1/2,(l)}
\end{array}\right] & :=\left[\begin{array}{c}
\frac{1}{\overline{\alpha^{t+1/2,\left(l\right)}}}\boldsymbol{h}^{t+1/2,(l)}\\
\alpha^{t+1/2,(l)}\boldsymbol{x}^{t+1/2,(l)}
\end{array}\right].
\end{align}
\end{subequations}Further we define the alignment parameter between
$\boldsymbol{z}^{t,\left(l\right)}$ and $\widetilde{\bm{z}}^{t}$
as\begin{subequations}
\begin{align}
\alpha_{\text{\text{mutual}}}^{t,\left(l\right)} & \coloneqq\arg\min_{\alpha\in\mathbb{C}}\left\{ \left\Vert \tfrac{1}{\overline{\alpha}}\boldsymbol{h}^{t,\left(l\right)}-\tfrac{1}{\overline{\alpha^{t}}}\boldsymbol{h}^{t}\right\Vert _{2}^{2}+\left\Vert \alpha\boldsymbol{x}^{t,\left(l\right)}-\alpha^{t}\boldsymbol{x}^{t}\right\Vert _{2}^{2}\right\} ,\label{eq:defn-alpha-mutual-loo}\\
\alpha_{\text{\text{mutual}}}^{t+1/2,\left(l\right)} & \coloneqq\arg\min_{\alpha\in\mathbb{C}}\left\{ \left\Vert \tfrac{1}{\overline{\alpha}}\boldsymbol{h}^{t+1/2,\left(l\right)}-\tfrac{1}{\overline{\alpha^{t+1/2}}}\boldsymbol{h}^{t+1/2}\right\Vert _{2}^{2}+\left\Vert \alpha\boldsymbol{x}^{t+1/2,\left(l\right)}-\alpha^{t+1/2}\boldsymbol{x}^{t+1/2}\right\Vert _{2}^{2}\right\} .
\end{align}
\end{subequations}Hereafter, we shall also denote\begin{subequations}
\begin{align}
\widehat{\boldsymbol{z}}^{t,\left(l\right)}:=\left[\begin{array}{c}
\widehat{\boldsymbol{h}}^{t,\left(l\right)}\\
\widehat{\boldsymbol{x}}^{t,\left(l\right)}
\end{array}\right] & =\left[\begin{array}{c}
\tfrac{1}{\overline{\alpha_{\text{mutual}}^{t,\left(l\right)}}}\boldsymbol{h}^{t,\left(l\right)}\\
\alpha_{\text{mutual}}^{t,\left(l\right)}\boldsymbol{x}^{t,\left(l\right)}
\end{array}\right],\label{eq:defn-zt-l-mutual}\\
\widehat{\boldsymbol{z}}^{t+1/2,\left(l\right)}:=\left[\begin{array}{c}
\widehat{\boldsymbol{h}}^{t+1/2,\left(l\right)}\\
\widehat{\boldsymbol{x}}^{t+1/2,\left(l\right)}
\end{array}\right] & =\left[\begin{array}{c}
\tfrac{1}{\overline{\alpha_{\text{mutual}}^{t+1/2,\left(l\right)}}}\boldsymbol{h}^{t+1/2,\left(l\right)}\\
\alpha_{\text{mutual}}^{t+1/2,\left(l\right)}\boldsymbol{x}^{t+1/2,\left(l\right)}
\end{array}\right].
\end{align}
\end{subequations}

\subsubsection{Additional induction hypotheses}

In addition to the set of induction hypotheses already listed in \eqref{sec:hypotheses-ncvx},
we find it convenient to include the following hypotheses concerning
the leave-one-out sequences. Specifically, for any $0<t\leq t_{0}$
and any $1\leq l\leq m$, the hypotheses claim that\begin{subequations}\label{subeq:Additional-induction-hypotheses}
\begin{align}
\mathsf{dist}\big(\boldsymbol{z}^{t,\left(l\right)},\widetilde{\bm{z}}^{t}\big) & \leq C_{2}\left(\frac{\mu}{\sqrt{m}}\sqrt{\frac{\mu^{2}K\log^{9}m}{m}}+\frac{\sigma}{\log^{2}m}\right)\label{eq:hypothesisloo}\\
\big\|\widetilde{\bm{z}}^{t,\left(l\right)}-\widetilde{\bm{z}}^{t}\big\|_{2} & \lesssim C_{2}\left(\frac{\mu}{\sqrt{m}}\sqrt{\frac{\mu^{2}K\log^{9}m}{m}}+\frac{\sigma}{\log^{2}m}\right)\label{eq:ztilde-l-ztilde-t-L2}\\
\mathsf{dist}\big(\boldsymbol{z}^{0,(l)},\boldsymbol{z}^{\star}\big) & \lesssim\sqrt{\frac{\mu^{2}K\log m}{m}}+\sigma\sqrt{K\log m}\label{eq:hypothesis-loo-dist-truth}\\
\mathsf{dist}\big(\bm{z}^{0,\left(l\right)},\widetilde{\bm{z}}^{0}\big) & \lesssim\frac{\mu}{\sqrt{m}}\sqrt{\frac{\mu^{2}K\log^{5}m}{m}}+\frac{\sigma}{\log^{2}m}\label{eq:hypothesis-loo-dist-mutual}
\end{align}
\end{subequations}for some constant $C_{2}\gg C_{4}^{2}$. Furthermore,
there are several immediate consequences of the hypotheses \eqref{sec:hypotheses-ncvx}
and \eqref{subeq:Additional-induction-hypotheses} that are also useful
in the analysis, which we gather as follows. Note that the notation
$(\widetilde{\bm{h}}^{t},\widetilde{\bm{x}}^{t})$, $(\widehat{\bm{h}}^{t},\widehat{\bm{x}}^{t})$,
$(\widehat{\bm{h}}^{t,(l)},\widehat{\bm{x}}^{t,(l)})$ and $\alpha^{t}$
has been defined in \eqref{eq:defn-htilde-xtilde}, \eqref{eq:defn-hhat-xhat},
\eqref{eq:defn-zt-l-mutual} and \eqref{eq:defn-alignment-alphat},
respectively. 

\begin{lemma}\label{lem:consequence}Instate the notation and assumptions
in Theorem \ref{thm:nonconvex}. For $t\geq0$, suppose that the hypotheses
\eqref{sec:hypotheses-ncvx} and \eqref{subeq:Additional-induction-hypotheses}
hold in the first $t$ iterations. Then there exist some constants
$C_{1},C>0$ such that for any $1\le l\leq m$, \begin{subequations}
\begin{align}
\mathsf{dist}\left(\boldsymbol{z}^{t},\boldsymbol{z}^{\star}\right) & \leq C_{1}\left(\sqrt{\frac{\mu^{2}K\log m}{m}}+\lambda+\sigma\sqrt{K\log m}\right),\label{eq:dist-bound}\\
\left\Vert \bm{h}^{t}\big(\bm{x}^{t}\big)^{\mathsf{H}}-\boldsymbol{h}^{\star}\boldsymbol{x}^{\mathsf{\star H}}\right\Vert  & \leq C\left(\sqrt{\frac{\mu^{2}K\log m}{m}}+\lambda+\sigma\sqrt{K\log m}\right),\label{eq:hx-normbound}\\
\big\|\widetilde{\boldsymbol{z}}^{t,\left(l\right)}-\boldsymbol{z}^{\star}\big\|_{2} & \leq2C_{1}\left(\sqrt{\frac{\mu^{2}K\log m}{m}}+\lambda+\sigma\sqrt{K\log m}\right),\label{eq:conseq-2}\\
\frac{1}{2}\leq\left\Vert \widetilde{\bm{x}}^{t}\right\Vert _{2}\leq\frac{3}{2}, & \qquad\frac{1}{2}\le\big\|\widetilde{\bm{h}}^{t}\big\|_{2}\leq\frac{3}{2},\label{eq:tilde-hx}\\
\frac{1}{2}\le\big\|\widetilde{\bm{x}}^{t,\left(l\right)}\big\|_{2}\leq\frac{3}{2}, & \qquad\frac{1}{2}\le\big\|\widetilde{\bm{h}}^{t,\left(l\right)}\big\|_{2}\leq\frac{3}{2},\label{eq:tilde-hx-loo}\\
\frac{1}{2}\le\big\|\widehat{\bm{x}}^{t,\left(l\right)}\big\|_{2}\leq\frac{3}{2}, & \qquad\frac{1}{2}\le\big\|\widehat{\bm{h}}^{t,\left(l\right)}\big\|_{2}\leq\frac{3}{2}.\label{eq:hat-hx-loo}
\end{align}
In addition, if $t>0$, then one also has
\begin{align}
\big\|\widehat{\bm{z}}^{t-1/2}-\boldsymbol{z}^{\star}\big\|_{2} & \leq C\left(\sqrt{\frac{\mu^{2}K\log m}{m}}+\lambda+\sigma\sqrt{K\log m}\right).\label{eq:conseq-3}
\end{align}
\end{subequations}\end{lemma}\begin{proof}See Appendix \ref{subsec:Proof-of-Lemmaconsequence}.\end{proof}

\subsection{Inductive analysis}

In this subsection, we carry out the analysis by induction. 

\subsubsection{Step 1: Characterizing local geometry}

Similar to \citet[Lemma 14]{ma2017implicit}, local linear convergence
is made possible when some sort of restricted strong convexity and
smoothness are present simultaneously. To be specific, define the
following squared loss that excludes the regularization term
\begin{equation}
f_{\mathsf{reg}\text{-}\mathsf{free}}\left(\bm{z}\right)=f_{\mathsf{reg}\text{-}\mathsf{free}}\left(\bm{h},\boldsymbol{x}\right):=\sum_{j=1}^{m}\big|\boldsymbol{b}_{j}^{\mathsf{H}}\boldsymbol{hx}^{\mathsf{H}}\boldsymbol{a}_{j}-y_{j}\big|^{2}.\label{eq:reg-free-loss}
\end{equation}
Our result is this:

\begin{lemma}\label{lemma:geometry}Let $\delta:=c/\log^{2}m$ for
some sufficiently small constant $c>0$. Suppose that $m\geq C\mu^{2}K\log^{9}m$
for some sufficiently large constant $C>0$ and that $\sigma\sqrt{K\log^{5}m}\leq c_{1}$
for some sufficiently small constant $c_{1}>0$. Then with probability
$1-O\left(m^{-10}+e^{-K}\log m\right)$, one has
\begin{align*}
\boldsymbol{u}^{\mathsf{H}}\left[\bm{D}\nabla^{2}f\left(\boldsymbol{z}\right)+\nabla^{2}f\left(\boldsymbol{z}\right)\bm{D}\right]\boldsymbol{u} & \geq\left\Vert \boldsymbol{u}\right\Vert _{2}^{2}/8\quad\text{and}\\
\left\Vert \nabla^{2}f\left(\boldsymbol{z}\right)\right\Vert  & \leq4
\end{align*}
simultaneously for all points
\[
\boldsymbol{z}=\left[\begin{array}{c}
\boldsymbol{h}\\
\boldsymbol{x}
\end{array}\right],\quad\boldsymbol{u}=\left[\begin{array}{c}
\boldsymbol{h}_{1}-\boldsymbol{h}_{2}\\
\boldsymbol{x}_{1}-\boldsymbol{x}_{2}\\
\overline{\boldsymbol{h}_{1}-\boldsymbol{h}_{2}}\\
\overline{\boldsymbol{x}_{1}-\boldsymbol{x}_{2}}
\end{array}\right]\quad\text{and}\quad\bm{D}=\left[\begin{array}{cccc}
\gamma_{1}\bm{I}_{K}\\
 & \gamma_{2}\bm{I}_{K}\\
 &  & \gamma_{1}\bm{I}_{K}\\
 &  &  & \gamma_{2}\bm{I}_{K}
\end{array}\right]
\]
obeying the following properties:
\begin{itemize}
\item $\boldsymbol{z}$ satisfies
\begin{align*}
\max\left\{ \left\Vert \boldsymbol{h}-\boldsymbol{h}^{\star}\right\Vert _{2},\left\Vert \boldsymbol{x}-\boldsymbol{x}^{\star}\right\Vert _{2}\right\}  & \leq\delta,\\
\max_{1\leq j\leq m}\left|\boldsymbol{a}_{j}^{\mathsf{H}}\left(\boldsymbol{x}-\boldsymbol{x}^{\star}\right)\right| & \leq2C_{3}\tfrac{1}{\log^{3/2}m},\\
\max_{1\leq j\leq m}\left|\boldsymbol{b}_{j}^{\mathsf{H}}\boldsymbol{h}\right| & \leq2C_{4}\Big(\tfrac{\mu\log^{2}m}{\sqrt{m}}+\sigma\Big);
\end{align*}
\item $\bm{z}_{1}:=\left(\boldsymbol{h}_{1},\boldsymbol{x}_{1}\right)$
is aligned with $\bm{z}_{2}:=\left(\boldsymbol{h}_{2},\boldsymbol{x}_{2}\right)$
in the sense that $\|\bm{z}_{1}-\bm{z}_{2}\|_{2}=\mathsf{dist}(\bm{z}_{1},\bm{z}_{2})$;
in addition, they satisfy 
\[
\max\left\{ \left\Vert \boldsymbol{h}_{1}-\boldsymbol{h}^{\star}\right\Vert _{2},\left\Vert \boldsymbol{h}_{2}-\boldsymbol{h}^{\star}\right\Vert _{2},\left\Vert \boldsymbol{x}_{1}-\boldsymbol{x}^{\star}\right\Vert _{2},\left\Vert \boldsymbol{x}_{2}-\boldsymbol{x}^{\star}\right\Vert _{2}\right\} \leq\delta;
\]
\item $\gamma_{1},\gamma_{2}\in\mathbb{R}$ and obey
\[
\max\left\{ \left|\gamma_{1}-1\right|,\left|\gamma_{2}-1\right|\right\} \leq\delta.
\]
\end{itemize}
\end{lemma}\begin{proof}See Appendix \ref{subsec:proofgeometry}.\end{proof}In
words, the function $f(\cdot)$ resembles a strongly convex and smooth
function when we restrict attention to (i) a highly restricted set
of points $\bm{z}$ and (ii) a highly special set of directions $\bm{u}$.

\subsubsection{Step 2: $\ell_{2}$ error contraction}

Next, we demonstrate that under the hypotheses \eqref{sec:hypotheses-ncvx}
for the $t$th iteration, the next iterate will undergo $\ell_{2}$
error contraction, as long as the stepsize is properly chosen. The
proof is largely based on the restricted strong convexity and smoothness
established in Lemma \ref{lemma:geometry}.

\begin{lemma}\label{lemma:distance}Set $\lambda=C_{\lambda}\sigma\sqrt{K\log m}$
for some large constant $C_{\lambda}>0$. The stepsize parameter $\eta>0$
in Algorithm \ref{alg:gd-mc-ncvx-LOO} is taken to be some sufficiently
small constant. There exists some constant $C>0$ such that with probability
at least $1-O\left(m^{-100}+e^{-CK}\log m\right)$, if the hypotheses
\eqref{sec:hypotheses-ncvx} hold true at the $t$th iteration, then
\begin{align}
\mathsf{dist}\left(\boldsymbol{z}^{t+1},\boldsymbol{z}^{\star}\right) & \leq\big\|\widehat{\boldsymbol{z}}^{t+1/2}-\boldsymbol{z}^{\star}\big\|_{2}\leq\rho\mathsf{dist}\left(\bm{z}^{t},\bm{z}^{\star}\right)+C_{1}\eta\left(\lambda+\sigma\sqrt{K\log m}\right)\label{eq:lem-dist-1}
\end{align}
for some constants $\rho=1-\eta/16$ and $C_{1}>0$.\end{lemma}\begin{proof}See
Appendix \ref{subsec:Proof-of-Lemmadistance}.\end{proof}

To establish this lemma and many other results, we need to ensure
that the alignment parameters and the sizes of the iterates do not
change much, as stated below. 

\begin{corollary}\label{corollary:alpha}Instate the notation and
assumptions in Theorem \ref{thm:nonconvex}. For an integer $t>0$,
suppose that the hypotheses \eqref{sec:hypotheses-ncvx} and \eqref{subeq:Additional-induction-hypotheses}
hold in the first $t-1$ iterations. Then there exists some constant
$C>0$ such that for any $1\le l\leq m$, one has\begin{subequations}
\begin{align}
\left|\left|\alpha^{t}\right|-1\right| & \lesssim\mathsf{dist}\left(\widetilde{\bm{z}}^{t},\bm{z}^{\star}\right)\lesssim\sqrt{\frac{\mu^{2}K\log m}{m}}+\lambda+\sigma\sqrt{K\log m},\label{eq:alpha-asymp1}\\
\left|\frac{\alpha^{t-1/2}}{\alpha^{t-1}}-1\right| & \lesssim\eta\left(\sqrt{\frac{\mu^{2}K\log m}{m}}+\lambda+\sigma\sqrt{K\log m}\right),\label{eq:alpharatio-round1}\\
\left|\left|\alpha_{\mathrm{mutual}}^{t,\left(l\right)}\right|-1\right| & \lesssim\big\|\widehat{\bm{z}}^{t,\left(l\right)}-\bm{z}^{\star}\big\|_{2}\lesssim\sqrt{\frac{\mu^{2}K\log m}{m}}+\lambda+\sigma\sqrt{K\log m},\label{eq:alphaloo-asymp1}\\
\frac{1}{2}\leq\left\Vert \bm{x}^{t}\right\Vert _{2}\leq\frac{3}{2}, & \qquad\frac{1}{2}\leq\left\Vert \bm{h}^{t}\right\Vert _{2}\leq\frac{3}{2},\label{eq:hx-asymp1}\\
\frac{1}{2}\leq\big\|\bm{x}^{t,\left(l\right)}\big\|_{2}\leq\frac{3}{2}, & \qquad\frac{1}{2}\leq\big\|\bm{h}^{t,\left(l\right)}\big\|_{2}\leq\frac{3}{2}\label{eq:hxloo-asymp1}
\end{align}
\end{subequations}with probability at least $1-O\left(m^{-100}+e^{-CK}\log m\right)$.
\end{corollary}\begin{proof}See Appendix \ref{subsec:Proof-of-Lemmaalpha}.\end{proof}

\subsubsection{Step 3: Leave-one-out proximity}

We then move on to justifying the close proximity of the leave-one-out
sequences and the original sequences, as stated in the hypothesis
\eqref{eq:hypothesisloo}. 

\begin{lemma}\label{lemma:proximity}Suppose the sample complexity
obeys $m\geq C\mu^{2}K\log^{9}m$ for some sufficiently large constant
$C>0$. If the hypotheses \eqref{eq:hypothesisdist}-\eqref{eq:hypothesisincoherence2}
hold for the $t$th iteration, then with probability at least $1-O\left(m^{-100}+me^{-cK}\right)$
for some constant $c>0$, one has\begin{subequations}
\begin{align}
\max_{1\leq l\leq m}\mathsf{dist}\big(\boldsymbol{z}^{t+1,\left(l\right)},\widetilde{\boldsymbol{z}}^{t+1}\big) & \leq C_{2}\left(\frac{\mu}{\sqrt{m}}\sqrt{\frac{\mu^{2}K\log^{9}m}{m}}+\frac{\sigma}{\log^{2}m}\right)\label{eq:lemproximity-claim1}\\
\text{and}\qquad\max_{1\leq l\leq m}\big\|\widetilde{\bm{z}}^{t+1,\left(l\right)}-\widetilde{\bm{z}}^{t+1}\big\|_{2} & \lesssim C_{2}\left(\frac{\mu}{\sqrt{m}}\sqrt{\frac{\mu^{2}K\log^{9}m}{m}}+\frac{\sigma}{\log^{2}m}\right),\label{eq:lemproximity-claim2}
\end{align}
\end{subequations}provided that the stepsize $\eta>0$ is some sufficiently
small constant.\end{lemma}\begin{proof}See Appendix \ref{subsec:Proof-of-Lemmaproximity}.\end{proof}

\subsubsection{Step 4: Establishing incoherence}

The next step is to establish the hypotheses concerning incoherence,
namely, \eqref{eq:hypothesisincoherence1} and \eqref{eq:hypothesisincoherence2}
for the $(t+1)$-th iteration. 

We start with the incoherence of $\bm{a}_{l}$ and $\bm{x}^{t+1}$,
which is much easier to handle. The standard Gaussian concentration
inequality gives
\begin{equation}
\max_{1\leq l\leq m}\left|\boldsymbol{a}_{l}^{\mathsf{H}}\big(\widetilde{\boldsymbol{x}}^{t+1,\left(l\right)}-\boldsymbol{x}^{\star}\big)\right|\leq20\sqrt{\log m}\max_{1\leq l\leq m}\big\|\widetilde{\boldsymbol{x}}^{t+1,\left(l\right)}-\boldsymbol{x}^{\star}\big\|_{2}\label{eq:step4incoh-1}
\end{equation}
with probability exceeding $1-O\left(m^{-100}\right)$. Then the triangle
inequality and Cauchy-Schwarz inequality yield

\begin{align}
\left|\boldsymbol{a}_{l}^{\mathsf{H}}\big(\widetilde{\boldsymbol{x}}^{t+1}-\boldsymbol{x}^{\star}\big)\right| & \leq\left|\boldsymbol{a}_{l}^{\mathsf{H}}\big(\widetilde{\boldsymbol{x}}^{t+1}-\widetilde{\boldsymbol{x}}^{t+1,\left(l\right)}\big)\right|+\left|\boldsymbol{a}_{l}^{\mathsf{H}}\big(\widetilde{\boldsymbol{x}}^{t+1,\left(l\right)}-\boldsymbol{x}^{\star}\big)\right|\nonumber \\
 & \leq\left\Vert \boldsymbol{a}_{l}\right\Vert _{2}\big\|\widetilde{\boldsymbol{x}}^{t+1}-\widetilde{\boldsymbol{x}}^{t+1,\left(l\right)}\big\|_{2}+\left|\boldsymbol{a}_{l}^{\mathsf{H}}\big(\widetilde{\boldsymbol{x}}^{t+1,\left(l\right)}-\boldsymbol{x}^{\star}\big)\right|\nonumber \\
 & \leq10\sqrt{K}C_{2}\left(\frac{\mu}{\sqrt{m}}\sqrt{\frac{\mu^{2}K\log^{9}m}{m}}+\frac{\sigma}{\log^{2}m}\right)\nonumber \\
 & \quad+20\sqrt{\log m}\cdot2C_{1}\left(\sqrt{\frac{\mu^{2}K\log m}{m}}+\lambda+\sigma\sqrt{K\log m}\right)\nonumber \\
 & \leq C_{3}\left(\sqrt{\frac{\mu^{2}K\log^{2}m}{m}}+\lambda+\sigma\sqrt{K}\log m\right),\label{eq:incoherencea-proof}
\end{align}
where $C_{3}\gg C_{1}$, the penultimate inequality follows from \eqref{eq:useful2},
\eqref{eq:lemproximity-claim2}, \eqref{eq:step4incoh-1} and \eqref{eq:conseq-2}.
This establishes the hypothesis \eqref{eq:hypothesisincoherence1}
for the $(t+1)$-th iteration. 

Regarding the incoherence of $\bm{b}_{l}$ and $\bm{h}^{t+1}$ (as
stated in the hypothesis \eqref{eq:hypothesisincoherence2}), we have
the following lemma. 

\begin{lemma}\label{lemma:incoherenceb}Suppose the sample complexity
obeys $m\geq C\mu^{2}K\log^{9}m$ for some sufficiently large constant
$C>0$ and $\lambda=C_{\lambda}\sigma\sqrt{K\log m}$ for some absolute
constant $C_{\lambda}>0$. If the hypotheses \eqref{eq:hypothesisdist}-\eqref{eq:hypothesisincoherence2}
hold for the $t$th iteration, then with probability exceeding $1-O\left(m^{-100}+me^{-CK}\right)$
for some constant $C>0$, one has 
\[
\max_{1\leq l\leq m}\big|\boldsymbol{b}_{l}^{\mathsf{H}}\widetilde{\boldsymbol{h}}^{t+1}\big|\le C_{4}\left(\frac{\mu}{\sqrt{m}}\log^{2}m+\sigma\right),
\]
as long as $C_{4}>0$ is some sufficiently large constant and $\eta>0$
is taken to be some sufficiently small constant.\end{lemma}\begin{proof}See
Appendix \ref{subsec:Proof-of-Lemmaincoherenceb}.\end{proof}

\subsubsection{The base case: Spectral initialization}

To finish the induction analysis, it remains to justify the induction
hypotheses for the base case. Recall that $\sigma\left(\bm{M}\right),\check{\bm{h}}^{0}$
and $\check{\bm{x}}^{0}$ denote respectively the leading singular
value, the left and the right singular vectors of 
\[
\bm{M}\coloneqq\sum_{j=1}^{m}y_{j}\bm{b}_{j}\bm{a}_{j}^{\mathsf{H}}.
\]
The spectral initialization procedure sets $\bm{h}^{0}=\sqrt{\sigma_{1}\left(\bm{M}\right)}\check{\bm{h}}^{0}$
and $\bm{x}^{0}=\sqrt{\sigma_{1}\left(\bm{M}\right)}\check{\bm{x}}^{0}$. 

To begin with, the following lemma guarantees that $\left(\bm{h}^{0},\bm{x}^{0}\right)$
satisfies the desired conditions \eqref{eq:initialization-z0} and
\eqref{eq:initialization-alpha0}.

\begin{lemma}\label{lemma:specinit-1}Suppose the sample size obeys
$m\geq C\mu^{2}K\log^{4}m$ for some sufficiently large constant $C>0$.
Then with probability at least $1-O\left(m^{-100}\right)$, we have
\begin{align*}
\min_{\alpha\in\mathbb{C},\left|\alpha\right|=1}\left\{ \left\Vert \alpha\bm{h}^{0}-\bm{h}^{\star}\right\Vert _{2}+\left\Vert \alpha\bm{x}^{0}-\bm{x}^{\star}\right\Vert _{2}\right\}  & \lesssim\sqrt{\frac{\mu^{2}K\log m}{m}}+\sigma\sqrt{K\log m}
\end{align*}
and $\left|\left|\alpha^{0}\right|-1\right|\leq1/4$.\end{lemma}

In view of the definition of $\mathsf{dist}\left(\cdot,\cdot\right)$,
we can invoke Lemma \ref{lemma:specinit-1} to reach
\begin{align}
\mathsf{dist}\big(\bm{z}^{0},\bm{z}^{\star}\big) & =\min_{\alpha\in\mathbb{C}}\sqrt{\left\Vert \tfrac{1}{\overline{\alpha}}\bm{h}^{0}-\bm{h}^{\star}\right\Vert _{2}^{2}+\left\Vert \alpha\bm{x}^{0}-\bm{x}^{\star}\right\Vert _{2}^{2}}\le\min_{\alpha\in\mathbb{C}}\left\{ \left\Vert \tfrac{1}{\overline{\alpha}}\bm{h}^{0}-\bm{h}^{\star}\right\Vert _{2}+\left\Vert \alpha\bm{x}^{0}-\bm{x}^{\star}\right\Vert _{2}\right\} \nonumber \\
 & \leq\min_{\alpha\in\mathbb{C},\left|\alpha\right|=1}\left\{ \left\Vert \alpha\bm{h}^{0}-\bm{h}^{\star}\right\Vert _{2}+\left\Vert \alpha\bm{x}^{0}-\bm{x}^{\star}\right\Vert _{2}\right\} \leq C_{1}\left(\sqrt{\frac{\mu^{2}K\log m}{m}}+\sigma\sqrt{K\log m}\right).\label{eq:lemma19cor1}
\end{align}
Repeating the same arguments yields that, with probability exceeding
$1-O(m^{-20})$,
\begin{equation}
\mathsf{dist}\big(\bm{z}^{0,\left(l\right)},\bm{z}^{\star}\big)\leq C_{1}\left(\sqrt{\frac{\mu^{2}K\log m}{m}}+\sigma\sqrt{K\log m}\right),\qquad1\leq l\leq m,\label{eq:lemma19cor2}
\end{equation}
and $\left|\left|\alpha^{0,\left(l\right)}\right|-1\right|\leq1/4$,
as asserted in the hypothesis \eqref{eq:hypothesis-loo-dist-truth}. 

The following lemma justifies \eqref{eq:hypothesis-loo-dist-mutual}
as well as \eqref{eq:hypothesisincoherence2} for the base case.

\begin{lemma}\label{lemma:specinit-2}Suppose the sample size obeys
$m\geq C\mu^{2}K\log^{9}m$ for some sufficiently large constant $C>0$
and the noise satisfies $\sigma\sqrt{K\log m}\leq c/\log^{2}m$ for
some sufficiently small constant $c>0$. Let $\tau=C_{\tau}\log^{4}m$
for some sufficiently large constant $C_{\tau}>0$ such that $\tau$
is an integer. Then with probability at least $1-O\left(m^{-100}+me^{-cK}\right)$
for some constant $c>0$, we have\begin{subequations}
\begin{align}
\max_{1\leq l\leq m}\mathsf{dist}\big(\bm{z}^{0,\left(l\right)},\widetilde{\bm{z}}^{0}\big) & \lesssim\frac{\mu}{\sqrt{m}}\sqrt{\frac{\mu^{2}K\log^{5}m}{m}}+\frac{\sigma}{\log^{2}m},\label{eq:specinit-2-1}\\
\max_{1\leq l\le m}\left|\bm{b}_{l}^{\mathsf{H}}\widetilde{\bm{h}}^{0}\right| & \lesssim\frac{\mu\log^{2}m}{\sqrt{m}}+\sigma,\label{eq:specinit-2-2}\\
\max_{1\le j\leq\tau}\big|\big(\bm{b}_{j}-\bm{b}_{1}\big)^{\mathsf{H}}\widetilde{\bm{h}}^{0}\big| & \lesssim\frac{\mu}{\sqrt{m}}\frac{1}{\log m}+\frac{\sigma}{\log m}.\label{eq:specinit-2-3}
\end{align}
\end{subequations}\end{lemma}

Finally, we establish the hypothesis \eqref{eq:hypothesisincoherence1}
for the base case, which concerns the incoherence of $\bm{x}^{0}$
with respect to the design vectors $\left\{ \bm{a}_{l}\right\} $.

\begin{lemma}\label{lemma:specinit-3}Suppose the sample size obeys
$m\geq C\mu^{2}K\log^{6}m$ for some sufficiently large constant $C>0$
and $\sigma\sqrt{K\log^{5}m}\leq c$ for some small constant $c>0$.
Then with probability at least $1-O\left(m^{-100+}me^{-c_{2}K}\right)$
for some constant $c_{2}>0$, we have
\begin{align*}
\max_{1\leq j\leq m}\left|\boldsymbol{a}_{j}^{\mathsf{H}}\big(\widetilde{\boldsymbol{x}}^{0}-\boldsymbol{x}^{\star}\big)\right| & \lesssim\sqrt{\frac{\mu^{2}K\log^{2}m}{m}}+\sigma\sqrt{K}\log m.
\end{align*}
\end{lemma}

The proof of these three lemmas can be easily obtained via straightforward
modifications to \citet[Lemmas 19,20,21]{ma2017implicit}; we omit
the details here for the sake of brevity.

\subsubsection{Proof of Theorem \ref{thm:nonconvex}}

With the above results in place, it is straightforward to prove Theorem
\ref{thm:nonconvex}. The first two claims follows respectively from
\eqref{eq:lemma19cor1} and \eqref{eq:dist-geometric}. Regarding
\eqref{eq:Fro-error-ncvx}, it follows that
\begin{align*}
\left\Vert \bm{h}^{t}\big(\bm{x}^{t}\big)^{\mathsf{H}}-\boldsymbol{h}^{\star}\boldsymbol{x}^{\mathsf{\star H}}\right\Vert _{\mathrm{F}} & \leq\left\Vert \bm{h}^{t}\big(\bm{x}^{t}\big)^{\mathsf{H}}-\bm{h}^{\star}\big(\bm{x}^{t}\big)^{\mathsf{H}}\right\Vert _{\mathrm{F}}+\left\Vert \bm{h}^{\star}\big(\bm{x}^{t}\big)^{\mathsf{H}}-\boldsymbol{h}^{\star}\boldsymbol{x}^{\mathsf{\star H}}\right\Vert _{\mathrm{F}}\\
 & \leq\left\Vert \bm{h}^{t}-\bm{h}^{\star}\right\Vert _{2}\left\Vert \bm{x}^{t}\right\Vert _{2}+\left\Vert \bm{h}^{\star}\right\Vert _{2}\left\Vert \bm{x}^{t}-\boldsymbol{x}^{\mathsf{\star}}\right\Vert _{2}\\
 & \leq2\left\Vert \bm{z}^{\star}\right\Vert _{2}\left(\rho^{t}\mathsf{dist}\left(\bm{z}^{0},\bm{z}^{\star}\right)+\frac{C_{1}\left(\lambda+\sigma\sqrt{K\log m}\right)}{c_{\rho}\left\Vert \bm{z}^{\star}\right\Vert _{2}}\right)
\end{align*}
where the last inequality follows from \eqref{eq:dist-geometric}
and the fact that
\[
\left\Vert \bm{x}^{t}\right\Vert _{2}\leq\left\Vert \bm{x}^{\star}\right\Vert _{2}+\left\Vert \bm{x}^{t}-\bm{x}^{\star}\right\Vert _{2}\leq\left\Vert \bm{z}^{\star}\right\Vert _{2}+\rho^{t}\mathsf{dist}\left(\bm{z}^{0},\bm{z}^{\star}\right)+\frac{C_{1}\left(\lambda+\sigma\sqrt{K\log m}\right)}{c_{\rho}\left\Vert \bm{z}^{\star}\right\Vert _{2}}\leq2\left\Vert \bm{z}^{\star}\right\Vert _{2}.
\]
This concludes the proof. 

\subsection{Proof of Lemma \ref{lem:consequence}\label{subsec:Proof-of-Lemmaconsequence}}
\begin{enumerate}
\item Condition \eqref{eq:dist-bound} follows directly from the $\ell_{2}$
contraction \eqref{eq:hypothesisdist} and the bound \eqref{eq:initialization-z0}
for the base case. 
\item \eqref{eq:hx-normbound} is direct consequence of \eqref{eq:dist-bound}
and triangle inequality. We have
\begin{align*}
\left\Vert \bm{h}^{t}\bm{x}^{t\mathsf{H}}-\bm{h}^{\star}\bm{x}^{\star\mathsf{H}}\right\Vert _{\text{F}} & =\left\Vert \widetilde{\bm{h}}^{t}\widetilde{\bm{x}}^{t\mathsf{H}}-\bm{h}^{\star}\bm{x}^{\star\mathsf{H}}\right\Vert _{\text{F}}\\
 & \leq\left\Vert \widetilde{\bm{h}}^{t}\widetilde{\bm{x}}^{t\mathsf{H}}-\widetilde{\bm{h}}^{t}\bm{x}^{\star\mathsf{H}}\right\Vert _{\text{F}}+\left\Vert \widetilde{\bm{h}}^{t}\bm{x}^{\star\mathsf{H}}-\bm{h}^{\star}\bm{x}^{\star\mathsf{H}}\right\Vert _{\text{F}}\\
 & \leq\left\Vert \widetilde{\bm{h}}^{t}\right\Vert _{2}\left\Vert \widetilde{\bm{x}}^{t}-\bm{x}^{\star}\right\Vert _{2}+\left\Vert \widetilde{\bm{h}}^{t}-\bm{h}^{\star}\right\Vert _{2}\left\Vert \bm{x}^{\star}\right\Vert _{2}\\
 & \leq\left(1+\mathsf{dist}\left(\bm{z}^{t},\bm{z}^{\star}\right)\right)\mathsf{dist}\left(\bm{z}^{t},\bm{z}^{\star}\right)+\mathsf{dist}\left(\bm{z}^{t},\bm{z}^{\star}\right)\\
 & \leq C\left(\sqrt{\frac{\mu^{2}K\log m}{m}}+\lambda+\sigma\sqrt{K\log m}\right),
\end{align*}
where the first equality follows from the definitions of $\widetilde{\bm{h}}^{t}$
and $\widetilde{\bm{x}}^{t}$ (cf. \eqref{eq:defn-htilde-xtilde})
and $C>0$ is some sufficiently large constant.
\item Regarding \eqref{eq:conseq-2}, it follows from the triangle inequality
that 
\begin{align*}
\max_{1\leq l\leq m}\big\|\widetilde{\boldsymbol{z}}^{t,\left(l\right)}-\boldsymbol{z}^{\star}\big\|_{2} & \leq\max_{1\leq l\leq m}\Big\{\big\|\widetilde{\boldsymbol{z}}^{t,\left(l\right)}-\widetilde{\boldsymbol{z}}^{t}\big\|_{2}+\big\|\widetilde{\boldsymbol{z}}^{t}-\boldsymbol{z}^{\star}\big\|_{2}\Big\}\\
 & \leq\widetilde{C}C_{2}\left(\frac{\mu}{\sqrt{m}}\sqrt{\frac{\mu^{2}K\log^{9}m}{m}}+\frac{\sigma}{\log^{2}m}\right)+C_{1}\left(\sqrt{\frac{\mu^{2}K\log m}{m}}+\lambda+\sigma\sqrt{K\log m}\right)\\
 & \leq2C_{1}\left(\sqrt{\frac{\mu^{2}K\log m}{m}}+\lambda+\sigma\sqrt{K\log m}\right)
\end{align*}
for $t>0$. Here, the penultimate inequality follows from the distance
bounds \eqref{eq:ztilde-l-ztilde-t-L2} and \eqref{eq:dist-bound},
while the last inequality holds as long as $m\geq C\mu^{2}\log^{8}m$
for some sufficiently large constant $C>0$. The base case follows
from \eqref{eq:hypothesis-loo-dist-truth}.
\item Condition \eqref{eq:tilde-hx} immediately results from \eqref{eq:dist-bound},
the assumption $\|\bm{x}^{\star}\|_{2}=\|\bm{h}^{\star}\|_{2}=1$,
the definition of $\mathsf{dist}\left(\cdot,\cdot\right)$, and the
triangle inequality.
\item With regards to \eqref{eq:tilde-hx-loo} and \eqref{eq:hat-hx-loo},
we shall only provide the proof for the result concerning $\bm{h}$;
the result concerning $\bm{x}$ can be derived analogously. In terms
of \eqref{eq:hat-hx-loo}, one has 
\begin{align*}
\big\|\widehat{\bm{h}}^{t,\left(l\right)}\big\|_{2} & \leq\big\|\widetilde{\bm{h}}^{t}\big\|_{2}+\big\|\widehat{\bm{h}}^{t,\left(l\right)}-\widetilde{\bm{h}}^{t}\big\|_{2}=\big\|\widetilde{\bm{h}}^{t}\big\|_{2}+\mathsf{dist}\big(\boldsymbol{h}^{t,\left(l\right)},\widetilde{\bm{h}}^{t}\big)\\
 & \lesssim1+C_{2}\left(\sqrt{\frac{\mu^{4}K\log^{9}m}{m^{2}}}+\frac{\sigma}{\log^{2}m}\right)\asymp1.
\end{align*}
Here, the first line comes from triangle inequality as well as the
definitions of $\widehat{\bm{h}}^{t,\left(l\right)}$ and $\widetilde{\bm{h}}^{t}$,
whereas the last inequality comes from \eqref{eq:hypothesisloo}.
A lower bound can be derived in a similar manner:
\begin{align*}
\big\|\widehat{\bm{h}}^{t,\left(l\right)}\big\|_{2} & \geq\big\|\widetilde{\bm{h}}^{t}\big\|_{2}-\big\|\widehat{\bm{h}}^{t,\left(l\right)}-\widetilde{\bm{h}}^{t}\big\|_{2}\gtrsim1-C_{2}\left(\sqrt{\frac{\mu^{4}K\log^{9}m}{m^{2}}}+\frac{\sigma}{\log^{2}m}\right)\asymp1.
\end{align*}
Regarding \eqref{eq:tilde-hx-loo}, apply \eqref{eq:ztilde-l-ztilde-t-L2}
and \eqref{eq:tilde-hx} to obtain
\begin{align*}
\big\|\widetilde{\bm{h}}^{t,\left(l\right)}\big\|_{2} & \leq\big\|\widetilde{\bm{h}}^{t,\left(l\right)}-\widetilde{\bm{h}}^{t}\big\|_{2}+\big\|\widetilde{\bm{h}}^{t}\big\|_{2}\lesssim C_{2}\left(\frac{\mu}{\sqrt{m}}\sqrt{\frac{\mu^{2}K\log^{9}m}{m}}+\frac{\sigma}{\log^{2}m}\right)+1\asymp1
\end{align*}
and, similarly,
\begin{align*}
\big\|\widetilde{\bm{h}}^{t,\left(l\right)}\big\|_{2} & \geq\big\|\widetilde{\bm{h}}^{t}\big\|_{2}-\big\|\widetilde{\bm{h}}^{t,\left(l\right)}-\widetilde{\bm{h}}^{t}\big\|_{2}\gtrsim1-C_{2}\left(\frac{\mu}{\sqrt{m}}\sqrt{\frac{\mu^{2}K\log^{9}m}{m}}+\frac{\sigma}{\log^{2}m}\right)\asymp1.
\end{align*}
The base case follows from similar deduction using \eqref{eq:hypothesis-loo-dist-mutual},
\eqref{eq:tilde-hx} and triangle inequality.
\item When it comes to Condition \eqref{eq:conseq-3}, it is seen from \eqref{eq:hypothesisdist}
and the choice $\rho=1-c_{\rho}\eta$ that
\begin{align*}
\left\Vert \widehat{\boldsymbol{z}}^{t-1/2}-\boldsymbol{z}^{\star}\right\Vert _{2} & \leq\rho^{t}\mathsf{dist}\left(\boldsymbol{z}^{0},\boldsymbol{z}^{\star}\right)+\frac{C_{1}}{1-\rho}\eta\left(\lambda+\sigma\sqrt{K\log m}\right)\\
 & =\rho^{t}\mathsf{dist}\left(\boldsymbol{z}^{0},\boldsymbol{z}^{\star}\right)+\frac{C_{1}}{c_{\rho}}\left(\lambda+\sigma\sqrt{K\log m}\right).
\end{align*}
Combining this with \eqref{eq:initialization-z0} guarantees the existence
of some sufficiently large constant $\widetilde{C}>0$ such that 
\begin{align*}
\left\Vert \widehat{\boldsymbol{z}}^{t-1/2}-\boldsymbol{z}^{\star}\right\Vert _{2} & \leq\rho^{t}\cdot\widetilde{C}\left(\sqrt{\frac{\mu^{2}K\log m}{m}}+\sigma\sqrt{K\log m}\right)+\frac{C_{1}}{c_{\rho}}\left(\lambda+\sigma\sqrt{K\log m}\right)\\
 & \leq C\left(\sqrt{\frac{\mu^{2}K\log m}{m}}+\lambda+\sigma\sqrt{K\log m}\right),
\end{align*}
provided that the constant $C>0$ is large enough. 
\end{enumerate}

\subsection{Proof of Corollary \ref{corollary:alpha}\label{subsec:Proof-of-Lemmaalpha}}
\begin{enumerate}
\item To establish \eqref{eq:alpha-asymp1}, we recall that the balancing
operation \eqref{eq:gradient-descent-BD-proj} guarantees $\left\Vert \bm{h}^{t}\right\Vert _{2}=\left\Vert \bm{x}^{t}\right\Vert _{2}$.
Hence, in view of the definitions of $\widetilde{\bm{h}}^{t}$ and
$\widetilde{\bm{x}}^{t}$ in \eqref{eq:defn-htilde-xtilde}, we have
\[
0=\left\Vert \bm{h}^{t}\right\Vert _{2}^{2}-\left\Vert \bm{x}^{t}\right\Vert _{2}^{2}=\left|\alpha^{t}\right|^{2}\big\|\widetilde{\bm{h}}^{t}\big\|_{2}^{2}-\frac{1}{\left|\alpha^{t}\right|^{2}}\left\Vert \widetilde{\bm{x}}^{t}\right\Vert _{2}^{2}.
\]
It then follows from the triangle inequality and the assumption $\|\bm{x}^{\star}\|_{2}=\|\bm{h}^{\star}\|_{2}$
that
\begin{align*}
0 & =\left|\alpha^{t}\right|^{2}\big\|\widetilde{\bm{h}}^{t}\big\|_{2}^{2}-\frac{1}{\left|\alpha^{t}\right|^{2}}\left\Vert \widetilde{\bm{x}}^{t}\right\Vert _{2}^{2}\leq\left|\alpha^{t}\right|^{2}\left(1+\big\|\widetilde{\bm{h}}^{t}-\bm{h}^{\star}\big\|_{2}\right)^{2}-\frac{\left(1-\left\Vert \widetilde{\bm{x}}^{t}-\bm{x}^{\star}\right\Vert _{2}\right)^{2}}{\left|\alpha^{t}\right|^{2}};\\
0 & =\left|\alpha^{t}\right|^{2}\big\|\widetilde{\bm{h}}^{t}\big\|_{2}^{2}-\frac{1}{\left|\alpha^{t}\right|^{2}}\left\Vert \widetilde{\bm{x}}^{t}\right\Vert _{2}^{2}\geq\left|\alpha^{t}\right|^{2}\left(1-\big\|\widetilde{\bm{h}}^{t}-\bm{h}^{\star}\big\|_{2}\right)^{2}-\frac{\left(1+\left\Vert \widetilde{\bm{x}}^{t}-\bm{x}^{\star}\right\Vert _{2}\right)^{2}}{\left|\alpha^{t}\right|^{2}}.
\end{align*}
Rearranging terms, we are left with
\[
\sqrt{\frac{1-\left\Vert \widetilde{\bm{x}}^{t}-\bm{x}^{\star}\right\Vert _{2}}{1+\big\|\widetilde{\bm{h}}^{t}-\bm{h}^{\star}\big\|_{2}}}\leq\left|\alpha^{t}\right|\leq\sqrt{\frac{1+\left\Vert \widetilde{\bm{x}}^{t}-\bm{x}^{\star}\right\Vert _{2}}{1-\big\|\widetilde{\bm{h}}^{t}-\bm{h}^{\star}\big\|_{2}}}.
\]
Combining this with \eqref{eq:dist-bound}, we arrive at
\[
\left|\left|\alpha^{t}\right|-1\right|\lesssim\big\|\widetilde{\bm{x}}^{t}-\bm{x}^{\star}\big\|_{2}+\big\|\widetilde{\bm{h}}^{t}-\bm{h}^{\star}\big\|_{2}\lesssim\mathsf{dist}\big(\widetilde{\bm{z}}^{t},\bm{z}^{\star}\big)\leq C_{1}\left(\sqrt{\frac{\mu^{2}K\log m}{m}}+\lambda+\sigma\sqrt{K\log m}\right).
\]
\item Regarding \eqref{eq:alpha-asymp1}, take $\bm{x}_{1}=\alpha^{t-1}\bm{x}^{t-1/2}$,
$\bm{h}_{1}=\bm{h}^{t-1/2}/\overline{\alpha^{t-1}}$, $\bm{x}_{2}=\alpha^{t-1}\bm{x}^{t-1}$
and $\bm{h}_{2}=\bm{h}^{t-1}/\overline{\alpha^{t-1}}$. Then we check
that these vectors satisfy the conditions of \citet[Lemma 54]{ma2017implicit}.
Towards this, observe that
\begin{align*}
 & \max\left\{ \left\Vert \bm{x}_{1}-\bm{x}^{\star}\right\Vert _{2},\left\Vert \bm{h}_{1}-\bm{h}^{\star}\right\Vert _{2},\left\Vert \bm{x}_{2}-\bm{x}^{\star}\right\Vert _{2},\left\Vert \bm{h}_{2}-\bm{h}^{\star}\right\Vert _{2}\right\} \\
 & \quad\leq\max\left\{ \left\Vert \widehat{\bm{z}}^{t-1/2}-\bm{z}^{\star}\right\Vert _{2},\mathsf{dist}\left(\bm{z}^{t-1},\bm{z}^{\star}\right)\right\} \\
 & \quad\lesssim\sqrt{\frac{\mu^{2}K\log m}{m}}+\lambda+\sigma\sqrt{K\log m}
\end{align*}
holds with probability over $1-O(m^{-100}+e^{-CK}\log m)$ for some
constant $C>0$. Here, the first inequality comes from the definitions
of $\widehat{\bm{z}}^{t-1/2}$ (cf.~\eqref{eq:defn-hhat-xhat-1/2}),
and the last inequality follows from \eqref{eq:dist-bound} and \eqref{eq:lem-dist-1}.
Hence, the condition of \citet[Lemma 54]{ma2017implicit} is satisfied.
Note that the statement of \citet[Lemma 54]{ma2017implicit} involves
two quantities $\alpha_{1}$ and $\alpha_{2}$, which in our case
are given by $\alpha_{1}=\alpha^{t-1/2}/\alpha^{t-1}$ and $\alpha_{2}=1$.
\citet[Lemma 54]{ma2017implicit} tells us that
\[
\left|\alpha_{1}-\alpha_{2}\right|=\left|\frac{\alpha^{t-1/2}}{\alpha^{t-1}}-1\right|\lesssim\left\Vert \alpha^{t-1}\bm{x}^{t-1/2}-\alpha^{t-1}\bm{x}^{t-1}\right\Vert _{2}+\left\Vert \frac{\bm{h}^{t-1/2}}{\overline{\alpha^{t-1}}}-\frac{\bm{h}^{t-1}}{\overline{\alpha^{t-1}}}\right\Vert _{2}.
\]
Additionally, the gradient update rule \eqref{eq:gradient-descent-BD}
reveals that 
\begin{align*}
 & \left\Vert \left[\begin{array}{c}
\frac{\boldsymbol{h}^{t-1/2}}{\overline{\alpha^{t-1}}}-\frac{\bm{h}^{t-1}}{\overline{\alpha^{t-1}}}\\
\alpha^{t-1}\boldsymbol{x}^{t-1/2}-\alpha^{t-1}\bm{x}^{t-1}
\end{array}\right]\right\Vert _{2}\\
 & =\left\Vert \left[\begin{array}{c}
-\frac{\eta}{\left|\alpha^{t-1}\right|^{2}}\nabla_{\boldsymbol{h}}f_{\mathsf{reg}\text{-}\mathsf{free}}\big(\widetilde{\bm{z}}^{t-1}\big)-\eta\lambda\widetilde{\bm{h}}^{t-1}\\
-\eta\left|\alpha^{t-1}\right|^{2}\nabla_{\boldsymbol{x}}f_{\mathsf{reg}\text{-}\mathsf{free}}\big(\widetilde{\bm{z}}^{t-1}\big)-\eta\lambda\widetilde{\bm{x}}^{t-1}
\end{array}\right]\right\Vert _{2}\\
 & =\left\Vert \left[\begin{array}{c}
-\frac{\eta}{\left|\alpha^{t-1}\right|^{2}}\left(\nabla_{\boldsymbol{h}}f_{\mathsf{reg}\text{-}\mathsf{free}}\big(\widetilde{\bm{z}}^{t-1}\big)-\nabla_{\boldsymbol{h}}f_{\mathsf{reg}\text{-}\mathsf{free}}\left(\bm{z}^{\star}\right)\right)-\eta\lambda\widetilde{\bm{h}}^{t-1}-\frac{\eta}{\left|\alpha^{t-1}\right|^{2}}\nabla_{\boldsymbol{h}}f_{\mathsf{reg}\text{-}\mathsf{free}}\left(\bm{z}^{\star}\right)\\
-\eta\left|\alpha^{t-1}\right|^{2}\left(\nabla_{\boldsymbol{x}}f_{\mathsf{reg}\text{-}\mathsf{free}}\big(\widetilde{\bm{z}}^{t-1}\big)-\nabla_{\boldsymbol{x}}f_{\mathsf{reg}\text{-}\mathsf{free}}\left(\bm{z}^{\star}\right)\right)-\eta\lambda\widetilde{\bm{x}}^{t-1}-\eta\left|\alpha^{t-1}\right|^{2}\nabla_{\boldsymbol{x}}f_{\mathsf{reg}\text{-}\mathsf{free}}\left(\bm{z}^{\star}\right)
\end{array}\right]\right\Vert _{2}\\
 & \le\left\Vert \left[\begin{array}{c}
\frac{\eta}{\left|\alpha^{t-1}\right|^{2}}\left(\nabla_{\boldsymbol{h}}f_{\mathsf{reg}\text{-}\mathsf{free}}\big(\widetilde{\bm{z}}^{t-1}\big)-\nabla_{\boldsymbol{h}}f_{\mathsf{reg}\text{-}\mathsf{free}}\left(\bm{z}^{\star}\right)\right)\\
\eta\left|\alpha^{t-1}\right|^{2}\left(\nabla_{\boldsymbol{x}}f_{\mathsf{reg}\text{-}\mathsf{free}}\big(\widetilde{\bm{z}}^{t-1}\big)-\nabla_{\boldsymbol{x}}f_{\mathsf{reg}\text{-}\mathsf{free}}\left(\bm{z}^{\star}\right)\right)
\end{array}\right]\right\Vert _{2}+\left\Vert \left[\begin{array}{c}
\eta\lambda\widetilde{\bm{h}}^{t-1}\\
\eta\lambda\widetilde{\bm{x}}^{t-1}
\end{array}\right]\right\Vert _{2}\\
\text{\ensuremath{\quad}} & \quad+\left\Vert \left[\begin{array}{c}
\frac{\eta}{\left|\alpha^{t-1}\right|^{2}}\nabla_{\boldsymbol{h}}f_{\mathsf{reg}\text{-}\mathsf{free}}\left(\bm{z}^{\star}\right)\\
\eta\left|\alpha^{t-1}\right|^{2}\nabla_{\boldsymbol{x}}f_{\mathsf{reg}\text{-}\mathsf{free}}\left(\bm{z}^{\star}\right)
\end{array}\right]\right\Vert _{2}\\
 & \le4\eta\left\Vert \nabla f_{\mathsf{reg-free}}\left(\widetilde{\bm{z}}^{t-1}\right)-\nabla f_{\mathsf{reg-free}}\left(\bm{z}^{\star}\right)\right\Vert _{2}+\eta\lambda\left\Vert \widetilde{\bm{z}}^{t-1}\right\Vert _{2}+4\eta\left\Vert \nabla f_{\mathsf{reg-free}}\left(\bm{z}^{\star}\right)\right\Vert _{2},
\end{align*}
where the last inequality utilizes the consequence of \eqref{eq:alpha-asymp1}
that 
\begin{align*}
\frac{1}{2}\leq1-\left|\left|\alpha^{t-1}\right|-1\right|\leq\left|\alpha^{t-1}\right| & \leq1+\left|\left|\alpha^{t-1}\right|-1\right|\leq2.
\end{align*}
Then, one has
\[
\left[\begin{array}{c}
\nabla f_{\mathsf{reg}\text{-}\mathsf{free}}\left(\widetilde{\bm{z}}^{t-1}\right)-\nabla f_{\mathsf{reg}\text{-}\mathsf{free}}\left(\bm{z}^{\star}\right)\\
\overline{\nabla f_{\mathsf{reg}\text{-}\mathsf{free}}\left(\widetilde{\bm{z}}^{t-1}\right)-\nabla f_{\mathsf{reg}\text{-}\mathsf{free}}\left(\bm{z}^{\star}\right)}
\end{array}\right]=\int_{0}^{1}\nabla^{2}f_{\mathsf{reg}\text{-}\mathsf{free}}\left(\bm{z}\left(s\right)\right)\mathrm{d}s\left[\begin{array}{c}
\widetilde{\bm{z}}^{t}-\bm{z}^{\star}\\
\overline{\widetilde{\bm{z}}^{t}-\bm{z}^{\star}}
\end{array}\right],
\]
where $\bm{z}\left(s\right)=\bm{z}^{\star}+s\left(\widetilde{\bm{z}}^{t}-\bm{z}^{\star}\right)$.
Therefore, for all $0\leq s\leq1$ we have 
\begin{align*}
\max\left\{ \left\Vert \boldsymbol{h}\left(s\right)-\boldsymbol{h}^{\star}\right\Vert _{2},\left\Vert \boldsymbol{x}\left(s\right)-\boldsymbol{x}^{\star}\right\Vert _{2}\right\}  & \leq\frac{c}{\log^{2}m},\\
\max_{1\leq j\leq m}\left|\boldsymbol{a}_{j}^{\mathsf{H}}\left(\boldsymbol{x}\left(s\right)-\boldsymbol{x}^{\star}\right)\right| & \leq2C_{3}\tfrac{1}{\log^{3/2}m},\\
\max_{1\leq j\leq m}\left|\boldsymbol{b}_{j}^{\mathsf{H}}\boldsymbol{h}\left(s\right)\right| & \leq2C_{4}\Big(\tfrac{\mu\log^{2}m}{\sqrt{m}}+\sigma\Big),
\end{align*}
which are guaranteed by the induction hypotheses \eqref{sec:hypotheses-ncvx}.
The conditions of Lemma \eqref{lemma:geometry} are satisfied, allowing
us to obtain 
\[
\left\Vert \int_{0}^{1}\nabla^{2}f_{\mathsf{reg}\text{-}\mathsf{free}}\left(\bm{z}\left(s\right)\right)\mathrm{d}s\right\Vert \leq\left\Vert \int_{0}^{1}\nabla^{2}f\left(\bm{z}\left(s\right)\right)\mathrm{d}s\right\Vert +\lambda\leq4+\lambda\leq5.
\]
Consequently, it follows that
\begin{align*}
\left\Vert \left[\begin{array}{c}
\frac{\boldsymbol{h}^{t-1/2}}{\overline{\alpha^{t-1}}}-\frac{\bm{h}^{t-1}}{\overline{\alpha^{t-1}}}\\
\alpha^{t-1}\boldsymbol{x}^{t-1/2}-\alpha^{t-1}\bm{x}^{t-1}
\end{array}\right]\right\Vert _{2} & \leq20\eta\left\Vert \widetilde{\bm{z}}^{t-1}-\bm{z}^{\star}\right\Vert _{2}+\eta\lambda\left\Vert \widetilde{\bm{z}}^{t-1}\right\Vert _{2}+4\eta\left\Vert \nabla f_{\mathsf{reg}\text{-}\mathsf{free}}\left(\bm{z}^{\star}\right)\right\Vert _{2}\\
 & \leq C\eta\left(\sqrt{\frac{\mu^{2}K\log m}{m}}+\lambda+\sigma\sqrt{K\log m}\right),
\end{align*}
where the last inequality results from \eqref{eq:dist-bound}, \eqref{eq:tilde-hx},
and \eqref{eq:gradientstar}. Hence, we arrive at
\[
\left|\frac{\alpha^{t-1/2}}{\alpha^{t-1}}-1\right|\lesssim\eta\left(\sqrt{\frac{\mu^{2}K\log m}{m}}+\lambda+\sigma\sqrt{K\log m}\right).
\]
\item Similarly, the balancing step \eqref{eq:defn-WF-LOO-BD} implies $\big\|\bm{h}^{t,\left(l\right)}\big\|_{2}^{2}=\big\|\bm{x}^{t,\left(l\right)}\big\|_{2}^{2}$.
From the definitions of $\alpha_{\text{mutual}}^{t,\left(l\right)}$
(cf.~\eqref{eq:defn-alpha-mutual-loo}), $\widehat{\bm{h}}^{t,\left(l\right)}$
and $\widehat{\bm{x}}^{t,\left(l\right)}$ (cf.~\eqref{eq:defn-zt-l-mutual}),
we have
\[
0=\big\|\bm{h}^{t,\left(l\right)}\big\|_{2}^{2}-\big\|\bm{x}^{t,\left(l\right)}\big\|_{2}^{2}=\big|\alpha_{\text{mutual}}^{t,\left(l\right)}\big|^{2}\big\|\widehat{\bm{h}}^{t,\left(l\right)}\big\|_{2}^{2}-\big|\alpha_{\text{mutual}}^{t,\left(l\right)}\big|^{-2}\big\|\widehat{\bm{x}}^{t,\left(l\right)}\big\|_{2}^{2}.
\]
Then the triangle inequality together with the assumption $\|\bm{x}^{\star}\|_{2}=\|\bm{h}^{\star}\|_{2}$
gives 
\begin{align*}
0 & =\big|\alpha_{\text{mutual}}^{t,\left(l\right)}\big|^{2}\big\|\widehat{\bm{h}}^{t,\left(l\right)}\big\|_{2}^{2}-\frac{1}{\big|\alpha_{\text{mutual}}^{t,\left(l\right)}\big|^{2}}\big\|\widehat{\bm{x}}^{t,\left(l\right)}\big\|_{2}^{2}\leq\big|\alpha_{\text{mutual}}^{t,\left(l\right)}\big|^{2}\left(1+\big\|\widehat{\bm{h}}^{t,\left(l\right)}-\bm{h}^{\star}\big\|_{2}\right)^{2}-\frac{\left(1-\left\Vert \widehat{\bm{x}}^{t,\left(l\right)}-\bm{x}^{\star}\right\Vert _{2}\right)^{2}}{\big|\alpha_{\text{mutual}}^{t,\left(l\right)}\big|^{2}},\\
0 & =\big|\alpha_{\text{mutual}}^{t,\left(l\right)}\big|^{2}\big\|\widehat{\bm{h}}^{t,\left(l\right)}\big\|_{2}^{2}-\frac{1}{\big|\alpha_{\text{mutual}}^{t,\left(l\right)}\big|^{2}}\big\|\widehat{\bm{x}}^{t,\left(l\right)}\big\|_{2}^{2}\geq\big|\alpha_{\text{mutual}}^{t,\left(l\right)}\big|^{2}\left(1-\big\|\widehat{\bm{h}}^{t,\left(l\right)}-\bm{h}^{\star}\big\|_{2}\right)^{2}-\frac{\left(1+\left\Vert \widehat{\bm{x}}^{t,\left(l\right)}-\bm{x}^{\star}\right\Vert _{2}\right)^{2}}{\big|\alpha_{\text{mutual}}^{t,\left(l\right)}\big|^{2}},
\end{align*}
which in turn lead to
\[
\sqrt{\frac{1-\left\Vert \widehat{\bm{x}}^{t,\left(l\right)}-\bm{x}^{\star}\right\Vert _{2}}{1+\big\|\widehat{\bm{h}}^{t,\left(l\right)}-\bm{h}^{\star}\big\|_{2}}}\leq\big|\alpha_{\text{mutual}}^{t,\left(l\right)}\big|\leq\sqrt{\frac{1+\left\Vert \widehat{\bm{x}}^{t,\left(l\right)}-\bm{x}^{\star}\right\Vert _{2}}{1-\big\|\widehat{\bm{h}}^{t,\left(l\right)}-\bm{h}^{\star}\big\|_{2}}}.
\]
Taking this together with \eqref{eq:hypothesisloo} and \eqref{eq:dist-bound},
we reach
\begin{align*}
\left|\big|\alpha_{\text{mutual}}^{t,\left(l\right)}\big|-1\right| & \lesssim\big\|\widehat{\bm{z}}^{t,\left(l\right)}-\bm{z}^{\star}\big\|_{2}\leq\big\|\widehat{\bm{z}}^{t,\left(l\right)}-\widetilde{\bm{z}}^{t}\big\|_{2}+\big\|\widetilde{\bm{z}}^{t}-\bm{z}^{\star}\big\|_{2}\\
 & \leq C_{2}\left(\sqrt{\frac{\mu^{4}K\log^{9}m}{m^{2}}}+\frac{\sigma}{\log^{2}m}\right)+C_{1}\left(\sqrt{\frac{\mu^{2}K\log m}{m}}+\lambda+\sigma\sqrt{K\log m}\right)\\
 & \leq\left(C_{1}+C_{2}\right)\left(\sqrt{\frac{\mu^{2}K\log m}{m}}+\lambda+\sigma\sqrt{K\log m}\right),
\end{align*}
where the second line follows from the distance bounds \eqref{eq:hypothesisloo}
and \eqref{eq:dist-bound}, and the last line holds with the proviso
that $m\geq\mu^{2}K\log^{8}m$. This establishes the claim \eqref{eq:alphaloo-asymp1}. 
\item Finally, \eqref{eq:hx-asymp1} and \eqref{eq:hxloo-asymp1} are direct
consequences of \eqref{eq:alpha-asymp1}, \eqref{eq:alphaloo-asymp1}
as well as the fact that $\left\Vert \bm{h}^{\star}\right\Vert _{2}=\left\Vert \bm{x}^{\star}\right\Vert _{2}=1$.
We omit the details for the sake of brevity. 
\end{enumerate}

\subsection{Proof of Lemma \ref{lemma:geometry}\label{subsec:proofgeometry}}

Define another loss function as follows
\begin{align*}
f_{\mathsf{clean}}\left(\boldsymbol{z}\right) & :=\sum_{j=1}^{m}\big|\boldsymbol{b}_{j}^{\mathsf{H}}\boldsymbol{hx}^{\mathsf{H}}\boldsymbol{a}_{j}-\boldsymbol{b}_{j}^{\mathsf{H}}\boldsymbol{h}^{\star}\boldsymbol{x}^{\star\mathsf{H}}\boldsymbol{a}_{j}\big|^{2},
\end{align*}
which excludes both the noise $\bm{\xi}$ and the regularization term
from consideration when compared with the original loss $f(\cdot)$.
By virtue of \eqref{eq:whessian}, it is easily seen that
\begin{align}
\nabla^{2}f_{\mathsf{reg}\text{-}\mathsf{free}}\left(\bm{z}\right)=\nabla^{2}f_{\mathsf{clean}}\left(\boldsymbol{z}\right)+ & \left[\begin{array}{cc}
\boldsymbol{M} & \boldsymbol{0}\\
0 & \overline{\boldsymbol{M}}
\end{array}\right],
\end{align}
where
\begin{align*}
\boldsymbol{M}:= & \left[\begin{array}{cc}
\boldsymbol{0} & -\sum_{j=1}^{m}\xi_{j}\bm{b}_{j}\bm{a}_{j}^{\mathsf{H}}\\
-\left(\sum_{j=1}^{m}\xi_{j}\bm{b}_{j}\bm{a}_{j}^{\mathsf{H}}\right)^{\mathsf{H}} & \boldsymbol{0}
\end{array}\right]\in\mathbb{C}^{2K\times2K}.
\end{align*}
By setting
\[
\boldsymbol{u}=\left[\begin{array}{c}
\boldsymbol{h}_{1}-\boldsymbol{h}_{2}\\
\boldsymbol{x}_{1}-\boldsymbol{x}_{2}\\
\overline{\boldsymbol{h}_{1}-\boldsymbol{h}_{2}}\\
\overline{\boldsymbol{x}_{1}-\boldsymbol{x}_{2}}
\end{array}\right]\eqqcolon\left[\begin{array}{c}
\bm{u}_{\bm{h}}\\
\bm{u}_{\bm{x}}\\
\overline{\bm{u}_{\bm{h}}}\\
\overline{\bm{u}_{\bm{x}}}
\end{array}\right]
\]
and recalling the definitions of $\bm{D}$, $\gamma_{1}$, $\gamma_{2}$
in the statement of Lemma \ref{lemma:geometry}, we arrive at
\begin{align*}
 & \boldsymbol{u}^{\mathsf{H}}\left[\bm{D}\nabla^{2}f_{\mathsf{reg}\text{-}\mathsf{free}}\left(\boldsymbol{z}\right)+\nabla^{2}f_{\mathsf{reg}\text{-}\mathsf{free}}\left(\boldsymbol{z}\right)\bm{D}\right]\boldsymbol{u}\\
 & =\boldsymbol{u}^{\mathsf{H}}\left[\bm{D}\nabla^{2}f_{\mathsf{clean}}\left(\boldsymbol{z}\right)+\nabla^{2}f_{\mathsf{clean}}\left(\boldsymbol{z}\right)\bm{D}\right]\boldsymbol{u}-\ensuremath{2\left(\gamma_{1}+\gamma_{2}\right)}\text{Re}\left(\boldsymbol{u_{h}}^{\mathsf{H}}\sum_{j=1}^{m}\xi_{j}\boldsymbol{b}_{j}\boldsymbol{a}_{j}^{\mathsf{H}}\boldsymbol{u_{x}}\right)\\
 & \quad-\ensuremath{2\left(\gamma_{1}+\gamma_{2}\right)}\text{Re}\left(\overline{\boldsymbol{u_{h}}}^{\mathsf{H}}\overline{\sum_{j=1}^{m}\xi_{j}\boldsymbol{b}_{j}\boldsymbol{a}_{j}^{\mathsf{H}}}\overline{\boldsymbol{u_{x}}}\right)\\
 & =\boldsymbol{u}^{\mathsf{H}}\left[\bm{D}\nabla^{2}f_{\mathsf{clean}}\left(\boldsymbol{z}\right)+\nabla^{2}f_{\mathsf{clean}}\left(\boldsymbol{z}\right)\bm{D}\right]\boldsymbol{u}-\ensuremath{4\left(\gamma_{1}+\gamma_{2}\right)}\text{Re}\left(\boldsymbol{u_{h}}^{\mathsf{H}}\sum_{j=1}^{m}\xi_{j}\boldsymbol{b}_{j}\boldsymbol{a}_{j}^{\mathsf{H}}\boldsymbol{u_{x}}\right).
\end{align*}
Consequently, with high probability one has
\begin{align}
 & \left|\boldsymbol{u}^{\mathsf{H}}\left[\bm{D}\nabla^{2}f_{\mathsf{reg}\text{-}\mathsf{free}}\left(\boldsymbol{z}\right)+\nabla^{2}f_{\mathsf{reg}\text{-}\mathsf{free}}\left(\boldsymbol{z}\right)\bm{D}\right]\boldsymbol{u}-\boldsymbol{u}^{\mathsf{H}}\left[\bm{D}\nabla^{2}f_{\mathsf{clean}}\left(\boldsymbol{z}\right)+\nabla^{2}f_{\mathsf{clean}}\left(\boldsymbol{z}\right)\bm{D}\right]\boldsymbol{u}\right|\nonumber \\
 & \quad\leq\ensuremath{4\left(\gamma_{1}+\gamma_{2}\right)}\left|\text{Re}\left(\boldsymbol{u_{h}}^{\mathsf{H}}\sum_{j=1}^{m}\xi_{j}\boldsymbol{b}_{j}\boldsymbol{a}_{j}^{\mathsf{H}}\boldsymbol{u_{x}}\right)\right|\leq\ensuremath{4\left(\gamma_{1}+\gamma_{2}\right)}\left\Vert \sum_{j=1}^{m}\xi_{j}\boldsymbol{b}_{j}\boldsymbol{a}_{j}^{\mathsf{H}}\right\Vert \left\Vert \boldsymbol{u}\right\Vert _{2}^{2}\nonumber \\
 & \quad\lesssim\sigma\sqrt{K\log m}\left\Vert \boldsymbol{u}\right\Vert _{2}^{2}=:\mathcal{E}_{\mathsf{res}}\label{eq:defn-E-res-bound}
\end{align}
for any vector $\bm{u}$, where the last inequality follows from Lemma
\ref{lemma:useful} as well as the assumptions $\gamma_{1},\gamma_{2}\asymp1$. 

The above bound allows us to turn attention to $\nabla^{2}f_{\mathsf{clean}}$,
which has been studied in \cite{ma2017implicit}. In particular, it
has been shown in \cite{ma2017implicit} that
\[
\boldsymbol{u}^{\mathsf{H}}\left[\bm{D}\nabla^{2}f_{\mathsf{clean}}\left(\boldsymbol{z}\right)+\nabla^{2}f_{\mathsf{clean}}\left(\boldsymbol{z}\right)\bm{D}\right]\boldsymbol{u}\geq\left(1/4\right)\cdot\left\Vert \boldsymbol{u}\right\Vert _{2}^{2}\quad\text{and}\ensuremath{\quad}\ensuremath{\left\Vert \nabla^{2}f_{\mathsf{clean}}\left(\boldsymbol{z}\right)\right\Vert \leq}3
\]
under the assumptions stated in the lemma. These bounds together with
\eqref{eq:defn-E-res-bound} yield\begin{subequations}\label{eq:Hessian-reg-free-UB1}
\begin{align}
\boldsymbol{u}^{\mathsf{H}}\left[\bm{D}\nabla^{2}f_{\mathsf{reg}\text{-}\mathsf{free}}\left(\boldsymbol{z}\right)+\nabla^{2}f_{\mathsf{reg}\text{-}\mathsf{free}}\left(\boldsymbol{z}\right)\bm{D}\right]\boldsymbol{u} & \geq\left(1/4\right)\cdot\left\Vert \boldsymbol{u}\right\Vert _{2}^{2}-\mathcal{E}_{\mathsf{res}}\geq\left(1/8\right)\cdot\left\Vert \boldsymbol{u}\right\Vert _{2}^{2},\\
\quad\text{and}\quad\left\Vert \nabla^{2}f_{\mathsf{reg}\text{-}\mathsf{free}}\left(\boldsymbol{z}\right)\right\Vert  & \leq\ensuremath{\left\Vert \nabla^{2}f_{\mathsf{clean}}\left(\boldsymbol{z}\right)\right\Vert }+\sup_{\bm{u}\neq\bm{0}}\frac{\mathcal{E}_{\mathsf{res}}}{\|\bm{u}\|_{2}^{2}}\leq7/2,
\end{align}
\end{subequations}provided that $\sigma\sqrt{K\log m}\leq0.5$. To
finish up, we recall that
\[
\nabla^{2}f\left(\bm{z}\right)=\nabla^{2}f_{\mathsf{reg}\text{-}\mathsf{free}}\left(\boldsymbol{z}\right)+\lambda\bm{I},
\]
which combined with \eqref{eq:Hessian-reg-free-UB1} and the assumption
$\lambda\leq C_{\lambda}\sigma\sqrt{K\log m}\leq C_{\lambda}c_{1}/\log^{2}m\ll1$
yields
\begin{align*}
\boldsymbol{u}^{\mathsf{H}}\left[\bm{D}\nabla^{2}f\left(\boldsymbol{z}\right)+\nabla^{2}f\left(\boldsymbol{z}\right)\bm{D}\right]\boldsymbol{u} & =\boldsymbol{u}^{\mathsf{H}}\left[\bm{D}\nabla^{2}f_{\mathsf{reg}\text{-}\mathsf{free}}\left(\boldsymbol{z}\right)+\nabla^{2}f_{\mathsf{reg}\text{-}\mathsf{free}}\left(\boldsymbol{z}\right)\bm{D}\right]\boldsymbol{u}+2\lambda\boldsymbol{u}^{\mathsf{H}}\bm{D}\bm{u}\\
 & \geq\boldsymbol{u}^{\mathsf{H}}\left[\bm{D}\nabla^{2}f_{\mathsf{reg}\text{-}\mathsf{free}}\left(\boldsymbol{z}\right)+\nabla^{2}f_{\mathsf{reg}\text{-}\mathsf{free}}\left(\boldsymbol{z}\right)\bm{D}\right]\boldsymbol{u}\\
 & \geq\left\Vert \boldsymbol{u}\right\Vert _{2}^{2}/8
\end{align*}
and
\begin{align*}
\left\Vert \nabla^{2}f\left(\bm{z}\right)\right\Vert  & \leq\left\Vert \nabla^{2}f_{\mathsf{reg}\text{-}\mathsf{free}}\left(\boldsymbol{z}\right)\right\Vert +\lambda\leq4.
\end{align*}

\subsection{Proof of Lemma \ref{lemma:distance}\label{subsec:Proof-of-Lemmadistance}}

Recognizing that
\[
f_{\mathsf{reg}\text{-}\mathsf{free}}\left(\bm{h},\bm{x}\right)=f_{\mathsf{reg}\text{-}\mathsf{free}}\left(\frac{1}{\overline{\alpha}}\bm{h},\alpha\bm{x}\right)\quad\text{and}\quad\nabla f_{\mathsf{reg}\text{-}\mathsf{free}}\left(\bm{h},\bm{x}\right)=\left[\begin{array}{c}
\frac{1}{\alpha}\nabla_{\bm{h}}f_{\mathsf{reg}\text{-}\mathsf{free}}\left(\frac{1}{\overline{\alpha}}\bm{h},\alpha\bm{x}\right)\\
\overline{\alpha}\nabla_{\bm{x}}f_{\mathsf{reg}\text{-}\mathsf{free}}\left(\frac{1}{\overline{\alpha}}\bm{h},\alpha\bm{x}\right)
\end{array}\right]
\]
and recalling the definitions of $\big(\widetilde{\bm{h}}^{t},\widetilde{\bm{x}}^{t}\big):=\big(\tfrac{1}{\overline{\alpha^{t}}}\boldsymbol{h}^{t},\alpha^{t}\bm{x}^{t}\big)$,
we can deduce that
\begin{align}
 & \mathsf{dist}\left(\boldsymbol{z}^{t+1},\boldsymbol{z}^{\star}\right)=\mathsf{dist}\big(\boldsymbol{z}^{t+1/2},\boldsymbol{z}^{\star}\big)\leq\left\Vert \left[\begin{array}{c}
\frac{1}{\overline{\alpha^{t}}}\boldsymbol{h}^{t+1/2}-\boldsymbol{h}^{\star}\\
\alpha^{t}\boldsymbol{x}^{t+1/2}-\boldsymbol{x}^{\star}
\end{array}\right]\right\Vert _{2}\\
 & \quad=\left\Vert \left[\begin{array}{c}
\widetilde{\bm{h}}^{t}-\frac{\eta}{\left|\alpha^{t}\right|^{2}}\nabla_{\boldsymbol{h}}f_{\mathsf{reg}\text{-}\mathsf{free}}\big(\widetilde{\bm{z}}^{t}\big)-\eta\lambda\widetilde{\bm{h}}^{t}-\left(\boldsymbol{h}^{\star}-\frac{\eta}{\left|\alpha^{t}\right|^{2}}\nabla_{\boldsymbol{h}}f_{\mathsf{reg}\text{-}\mathsf{free}}\left(\bm{z}^{\star}\right)\right)-\frac{\eta}{\left|\alpha^{t}\right|^{2}}\nabla_{\boldsymbol{h}}f_{\mathsf{reg}\text{-}\mathsf{free}}\left(\bm{z}^{\star}\right)\\
\widetilde{\bm{x}}^{t}-\eta\left|\alpha^{t}\right|^{2}\nabla_{\boldsymbol{x}}f_{\mathsf{reg}\text{-}\mathsf{free}}\big(\widetilde{\bm{z}}^{t}\big)-\eta\lambda\widetilde{\bm{x}}^{t}-\left(\boldsymbol{x}^{\star}-\eta\left|\alpha^{t}\right|^{2}\nabla_{\boldsymbol{x}}f_{\mathsf{reg}\text{-}\mathsf{free}}\left(\bm{z}^{\star}\right)\right)-\eta\left|\alpha^{t}\right|^{2}\nabla_{\boldsymbol{x}}f_{\mathsf{reg}\text{-}\mathsf{free}}\left(\bm{z}^{\star}\right)
\end{array}\right]\right\Vert _{2}\nonumber \\
 & \quad\leq\underbrace{\left\Vert \left[\begin{array}{c}
\widetilde{\bm{h}}^{t}-\frac{\eta}{\left|\alpha^{t}\right|^{2}}\nabla_{\boldsymbol{h}}f_{\mathsf{reg}\text{-}\mathsf{free}}\big(\widetilde{\bm{z}}^{t}\big)-\left(\boldsymbol{h}^{\star}-\frac{\eta}{\left|\alpha^{t}\right|^{2}}\nabla_{\boldsymbol{h}}f_{\mathsf{reg}\text{-}\mathsf{free}}\left(\bm{z}^{\star}\right)\right)\\
\widetilde{\bm{x}}^{t}-\eta\left|\alpha^{t}\right|^{2}\nabla_{\boldsymbol{x}}f_{\mathsf{reg}\text{-}\mathsf{free}}\big(\widetilde{\bm{z}}^{t}\big)-\left(\boldsymbol{x}^{\star}-\eta\left|\alpha^{t}\right|^{2}\nabla_{\boldsymbol{x}}f_{\mathsf{reg}\text{-}\mathsf{free}}\left(\bm{z}^{\star}\right)\right)
\end{array}\right]\right\Vert _{2}}_{\eqqcolon\beta_{1}}\nonumber \\
 & \qquad\qquad+\underbrace{\left\Vert \left[\begin{array}{c}
\frac{\eta}{\left|\alpha^{t}\right|^{2}}\nabla_{\boldsymbol{h}}f_{\mathsf{reg}\text{-}\mathsf{free}}\left(\bm{z}^{\star}\right)\\
\eta\left|\alpha^{t}\right|^{2}\nabla_{\boldsymbol{x}}f_{\mathsf{reg}\text{-}\mathsf{free}}\left(\bm{z}^{\star}\right)
\end{array}\right]\right\Vert _{2}}_{\eqqcolon\beta_{2}}+\underbrace{\eta\lambda\left\Vert \left[\begin{array}{c}
\widetilde{\bm{h}}^{t}\\
\widetilde{\bm{x}}^{t}
\end{array}\right]\right\Vert _{2}}_{\eqqcolon\beta_{3}}.\label{eq:lem-dist-decom}
\end{align}

Using an argument similar to the proof idea of \citet[Equation (210)]{ma2017implicit},
we can obtain 
\begin{align}
\beta_{1}^{2} & =\left\Vert \widetilde{\bm{h}}^{t}-\frac{\eta}{\left|\alpha^{t}\right|^{2}}\nabla_{\boldsymbol{h}}f_{\mathsf{reg}\text{-}\mathsf{free}}\big(\widetilde{\bm{z}}^{t}\big)-\left(\boldsymbol{h}^{\star}-\frac{\eta}{\left|\alpha^{t}\right|^{2}}\nabla_{\boldsymbol{h}}f_{\mathsf{reg}\text{-}\mathsf{free}}\left(\bm{z}^{\star}\right)\right)\right\Vert _{2}^{2}\nonumber \\
 & \quad\quad+\left\Vert \widetilde{\bm{x}}^{t}-\eta\left|\alpha^{t}\right|^{2}\nabla_{\boldsymbol{x}}f_{\mathsf{reg}\text{-}\mathsf{free}}\big(\widetilde{\bm{z}}^{t}\big)-\left(\boldsymbol{x}^{\star}-\eta\left|\alpha^{t}\right|^{2}\nabla_{\boldsymbol{x}}f_{\mathsf{reg}\text{-}\mathsf{free}}\left(\bm{z}^{\star}\right)\right)\right\Vert _{2}^{2}\nonumber \\
 & \leq\left(1-\frac{\eta}{8}\right)\left\Vert \widetilde{\bm{z}}^{t}-\bm{z}^{\star}\right\Vert _{2}^{2}.\label{eq:lem-dist-beta1}
\end{align}
Regarding $\beta_{2}$, we first invoke Lemma \ref{lemma:noise} and
the fact $\nabla f_{\mathsf{clean}}\left(\boldsymbol{z}^{\star}\right)=\bm{0}$
to derive 
\begin{align}
\left\Vert \nabla f_{\mathsf{reg}\text{-}\mathsf{free}}\left(\bm{z}^{\star}\right)\right\Vert _{\text{2}} & \leq\|\nabla f_{\mathsf{clean}}\left(\boldsymbol{z}^{\star}\right)\|_{2}+\left\Vert \mathcal{A}^{*}\left(\bm{\xi}\right)\right\Vert \left\Vert \boldsymbol{h}^{\star}\right\Vert _{2}+\left\Vert \mathcal{A}^{*}\left(\bm{\xi}\right)\right\Vert \left\Vert \boldsymbol{x}^{\star}\right\Vert _{2}\nonumber \\
 & \lesssim\sigma\sqrt{K\log m}.\label{eq:gradientstar}
\end{align}
A little algebra then yields
\begin{align*}
\beta_{2}^{2} & =\left\Vert \frac{\eta}{|\alpha^{t}|^{2}}\nabla_{\boldsymbol{h}}f_{\mathsf{reg}\text{-}\mathsf{free}}\left(\bm{z}^{\star}\right)\right\Vert _{2}^{2}+\left\Vert \eta\left|\alpha^{t}\right|^{2}\nabla_{\boldsymbol{x}}f_{\mathsf{reg}\text{-}\mathsf{free}}\left(\bm{z}^{\star}\right)\right\Vert _{2}^{2}\\
 & \leq\Big(\frac{\eta^{2}}{\left|\alpha^{t}\right|^{4}}+\eta^{2}\left|\alpha^{t}\right|^{4}\Big)\left\Vert \nabla f_{\mathsf{reg}\text{-}\mathsf{free}}\left(\bm{z}^{\star}\right)\right\Vert _{2}^{2}\\
 & \lesssim\eta^{2}\left(\sigma\sqrt{K\log m}\right)^{2},
\end{align*}
which relies on the observation that $|\alpha^{t}|\asymp1$ (see Corollary
\ref{corollary:alpha}). Finally, when it comes to $\beta_{3}$, we
have 
\[
\beta_{3}^{2}=\eta^{2}\lambda^{2}\big\|\widetilde{\bm{h}}^{t}\big\|_{2}^{2}+\eta^{2}\lambda^{2}\big\|\widetilde{\bm{x}}^{t}\big\|_{2}^{2}\leq8\eta^{2}\lambda^{2},
\]
using the fact that $\big\|\widetilde{\bm{x}}^{t}\big\|_{2}\asymp\big\|\widetilde{\bm{h}}^{t}\big\|_{2}\asymp1$
(see Lemma \ref{lem:consequence}). 

As a result, as long as $\eta>0$ is taken to be some constant small
enough, combining \eqref{eq:lem-dist-decom} and the above bounds
on $\beta_{1},\beta_{2}$ gives 
\begin{align*}
\mathsf{dist}\left(\boldsymbol{z}^{t+1},\boldsymbol{z}^{\star}\right) & \leq\left\Vert \widehat{\boldsymbol{z}}^{t+1/2}-\boldsymbol{z}^{\star}\right\Vert _{2}^{2}\leq\sqrt{\left(1-\eta/8\right)}\left\Vert \widetilde{\boldsymbol{z}}^{t}-\boldsymbol{z}^{\star}\right\Vert _{2}+C_{1}\eta\left(\lambda+\sigma\sqrt{K\log m}\right),
\end{align*}
which together with the elementary fact $\sqrt{1-x}\leq1-x/2$ leads
to
\begin{align*}
\mathsf{dist}\left(\boldsymbol{z}^{t+1},\boldsymbol{z}^{\star}\right) & \leq\left\Vert \widehat{\boldsymbol{z}}^{t+1/2}-\boldsymbol{z}^{\star}\right\Vert _{2}\leq\left(1-\eta/16\right)\big\|\widetilde{\boldsymbol{z}}^{t}-\boldsymbol{z}^{\star}\big\|_{2}+C_{1}\eta\left(\lambda+\sigma\sqrt{K\log m}\right)\\
 & =\left(1-\eta/16\right)\mathsf{dist}\left(\boldsymbol{z}^{t},\boldsymbol{z}^{\star}\right)+C_{1}\eta\left(\lambda+\sigma\sqrt{K\log m}\right).
\end{align*}
The advertised claim then follows, provided that $C_{1}$ is large
enough. 

\subsection{Proof of Lemma \ref{lemma:proximity}\label{subsec:Proof-of-Lemmaproximity}}

The lemma can be established in a similar manner as \citet[Lemma 17]{ma2017implicit}.
We have
\begin{align}
\mathsf{dist}\big(\boldsymbol{z}^{t+1,\left(l\right)},\widetilde{\bm{z}}^{t+1}\big) & =\mathsf{dist}\big(\boldsymbol{z}^{t+1/2,\left(l\right)},\widetilde{\bm{z}}^{t+1/2}\big)\nonumber \\
 & \leq\max\left\{ \left|\frac{\alpha^{t+1/2}}{\alpha^{t}}\right|,\left|\frac{\alpha^{t}}{\alpha^{t+1/2}}\right|\right\} \left\Vert \left[\begin{array}{c}
\frac{1}{\overline{\alpha_{\text{mutual}}^{t,\left(l\right)}}}\boldsymbol{h}^{t+1/2,\left(l\right)}-\frac{1}{\overline{\alpha^{t}}}\boldsymbol{h}^{t+1/2}\\
\alpha_{\text{mutual}}^{t,\left(l\right)}\boldsymbol{x}^{t+1,\left(l\right)}-\alpha^{t}\boldsymbol{x}^{t+1/2}
\end{array}\right]\right\Vert _{2},\label{eq:lemproximity-1}
\end{align}
where the second line comes from the same calculation as \citet[Eqn.~(212)]{ma2017implicit}.
Repeating the analysis in \citet[Appendix C.3]{ma2017implicit} and
using the gradient update rule, we obtain
\begin{align}
 & \left[\begin{array}{c}
\frac{1}{\overline{\alpha_{\text{mutual}}^{t,\left(l\right)}}}\boldsymbol{h}^{t+1/2,\left(l\right)}-\frac{1}{\overline{\alpha^{t}}}\boldsymbol{h}^{t+1/2}\\
\alpha_{\text{mutual}}^{t,\left(l\right)}\boldsymbol{x}^{t+1,\left(l\right)}-\alpha^{t}\boldsymbol{x}^{t+1/2}
\end{array}\right]\nonumber \\
 & =\underbrace{\left[\begin{array}{c}
\widehat{\bm{h}}^{t,\left(l\right)}-\frac{\eta}{\big|\alpha_{\text{mutual}}^{t,\left(l\right)}\big|^{2}}\nabla_{\boldsymbol{h}}f_{\mathsf{reg}\text{-}\mathsf{free}}\big(\widehat{\bm{z}}^{t,\left(l\right)}\big)-\left(\widetilde{\bm{h}}^{t}-\frac{\eta}{\big|\alpha_{\text{mutual}}^{t,\left(l\right)}\big|^{2}}\nabla_{\boldsymbol{h}}f_{\mathsf{reg}\text{-}\mathsf{free}}\big(\widetilde{\bm{z}}^{t}\big)\right)\\
\widehat{\bm{x}}^{t,\left(l\right)}-\eta\big|\alpha_{\text{mutual}}^{t,\left(l\right)}\big|^{2}\nabla_{\boldsymbol{x}}f_{\mathsf{reg}\text{-}\mathsf{free}}\big(\widehat{\bm{z}}^{t,\left(l\right)}\big)-\left(\widetilde{\bm{x}}^{t}-\eta\big|\alpha_{\text{mutual}}^{t,\left(l\right)}\big|^{2}\nabla_{\boldsymbol{x}}f_{\mathsf{reg}\text{-}\mathsf{free}}\big(\widetilde{\bm{z}}^{t}\big)\right)
\end{array}\right]}_{\eqqcolon\bm{\nu}_{1}}\nonumber \\
 & \quad+\eta\underbrace{\left[\begin{array}{c}
\Big(\frac{1}{\left|\alpha^{t}\right|^{2}}-\frac{1}{\big|\alpha_{\text{mutual}}^{t,\left(l\right)}\big|^{2}}\Big)\nabla_{\boldsymbol{h}}f_{\mathsf{reg}\text{-}\mathsf{free}}\big(\widehat{\bm{z}}^{t}\big)\\
\left(\left|\alpha^{t}\right|^{2}-\big|\alpha_{\text{mutual}}^{t,\left(l\right)}\big|^{2}\right)\nabla_{\boldsymbol{x}}f_{\mathsf{reg}\text{-}\mathsf{free}}\big(\widehat{\bm{z}}^{t}\big)
\end{array}\right]}_{\eqqcolon\bm{\nu}_{2}}-\eta\underbrace{\left[\begin{array}{c}
\frac{1}{\big|\alpha_{\text{mutual}}^{t,\left(l\right)}\big|^{2}}\left(\bm{b}_{l}^{\mathsf{H}}\widehat{\bm{h}}^{t,\left(l\right)}\widehat{\bm{x}}^{t,\left(l\right)\mathsf{H}}\bm{a}_{l}-y_{l}\right)\bm{b}_{l}\bm{a}_{l}^{\mathsf{H}}\widehat{\bm{x}}^{t,\left(l\right)}\\
\big|\alpha_{\text{mutual}}^{t,\left(l\right)}\big|^{2}\overline{\left(\bm{b}_{l}^{\mathsf{H}}\widehat{\bm{h}}^{t,\left(l\right)}\widehat{\bm{x}}^{t,\left(l\right)\mathsf{H}}\bm{a}_{l}-y_{l}\right)}\bm{a}_{l}\bm{b}_{l}^{\mathsf{H}}\widehat{\bm{h}}^{t,\left(l\right)}
\end{array}\right]}_{\eqqcolon\bm{\nu}_{3}}\nonumber \\
 & \quad+\eta\lambda\underbrace{\left[\begin{array}{c}
\widehat{\bm{h}}^{t,\left(l\right)}-\widetilde{\bm{h}}^{t}\\
\widehat{\bm{x}}^{t,\left(l\right)}-\widetilde{\bm{x}}^{t}
\end{array}\right]}_{\eqqcolon\bm{\nu}_{4}}.\label{eq:proximity-ncvx-decom}
\end{align}
In what follows, we shall look at $\bm{\nu}_{1}$, $\bm{\nu}_{2}$,
$\bm{\nu}_{3}$ and $\bm{\nu}_{4}$ separately. 
\begin{itemize}
\item It has been shown in \citet[Lemma 17]{ma2017implicit} that
\begin{align}
\left\Vert \boldsymbol{\nu}_{1}\right\Vert _{2} & \leq\left(1-\eta/16\right)\big\|\widehat{\boldsymbol{z}}^{t,\left(l\right)}-\widetilde{\boldsymbol{z}}^{t}\big\|_{2};\qquad\left\Vert \boldsymbol{\nu}_{2}\right\Vert _{2}\lesssim C_{1}\frac{1}{\log^{2}m}\big\|\widehat{\boldsymbol{z}}^{t,\left(l\right)}-\widetilde{\boldsymbol{z}}^{t}\big\|_{2}.\label{eq:prox-nu12}
\end{align}
\item Regarding $\boldsymbol{\nu}_{3}$, we have \begin{subequations}\label{eq:lem-loo-ncvx}
\begin{align}
\left\Vert \boldsymbol{\nu}_{3}\right\Vert _{2} & =\sqrt{\frac{1}{\big|\alpha_{\text{mutual}}^{t,\left(l\right)}\big|^{4}}\left\Vert \left(\bm{b}_{l}^{\mathsf{H}}\widehat{\bm{h}}^{t,\left(l\right)}\widehat{\bm{x}}^{t,\left(l\right)\mathsf{H}}\bm{a}_{l}-y_{l}\right)\bm{b}_{l}\bm{a}_{l}^{\mathsf{H}}\widehat{\bm{x}}^{t,\left(l\right)}\right\Vert _{2}^{2}+\big|\alpha_{\text{mutual}}^{t,\left(l\right)}\big|^{4}\left\Vert \overline{\left(\bm{b}_{l}^{\mathsf{H}}\widehat{\bm{h}}^{t,\left(l\right)}\widehat{\bm{x}}^{t,\left(l\right)\mathsf{H}}\bm{a}_{l}-y_{l}\right)}\bm{a}_{l}\bm{b}_{l}^{\mathsf{H}}\widehat{\bm{h}}^{t,\left(l\right)}\right\Vert _{2}^{2}}\nonumber \\
 & \leq\frac{1}{\big|\alpha_{\text{mutual}}^{t,\left(l\right)}\big|^{2}}\left\Vert \left(\bm{b}_{l}^{\mathsf{H}}\widehat{\bm{h}}^{t,\left(l\right)}\widehat{\bm{x}}^{t,\left(l\right)\mathsf{H}}\bm{a}_{l}-y_{l}\right)\bm{b}_{l}\bm{a}_{l}^{\mathsf{H}}\widehat{\bm{x}}^{t,\left(l\right)}\right\Vert _{2}+\big|\alpha_{\text{mutual}}^{t,\left(l\right)}\big|^{2}\left\Vert \overline{\left(\bm{b}_{l}^{\mathsf{H}}\widehat{\bm{h}}^{t,\left(l\right)}\widehat{\bm{x}}^{t,\left(l\right)\mathsf{H}}\bm{a}_{l}-y_{l}\right)}\bm{a}_{l}\bm{b}_{l}^{\mathsf{H}}\widehat{\bm{h}}^{t,\left(l\right)}\right\Vert _{2}\nonumber \\
 & \leq\frac{1}{\big|\alpha_{\text{mutual}}^{t,\left(l\right)}\big|^{2}}\underbrace{\left\Vert \boldsymbol{b}_{l}^{\mathsf{H}}\left(\widehat{\boldsymbol{h}}^{t,(l)}\widehat{\boldsymbol{x}}^{t,(l)\mathsf{H}}-\boldsymbol{h}^{\star}\boldsymbol{x}^{\star\mathsf{H}}\right)\boldsymbol{a}_{l}\bm{b}_{l}\bm{a}_{l}^{\mathsf{H}}\widehat{\bm{x}}^{t,\left(l\right)}\right\Vert _{2}}_{\eqqcolon\nu_{31}}\nonumber \\
 & \quad+\big|\alpha_{\text{mutual}}^{t,\left(l\right)}\big|^{2}\underbrace{\left\Vert \overline{\boldsymbol{b}_{l}^{\mathsf{H}}\left(\widehat{\boldsymbol{h}}^{t,(l)}\widehat{\boldsymbol{x}}^{t,(l)\mathsf{H}}-\boldsymbol{h}^{\star}\boldsymbol{x}^{\star\mathsf{H}}\right)\boldsymbol{a}_{l}}\bm{a}_{l}\bm{b}_{l}^{\mathsf{H}}\widehat{\bm{h}}^{t,\left(l\right)}\right\Vert _{2}}_{\eqqcolon\nu_{32}}\nonumber \\
 & \quad+\frac{1}{\big|\alpha_{\text{mutual}}^{t,\left(l\right)}\big|^{2}}\underbrace{\left\Vert \xi_{l}\boldsymbol{b}_{l}\boldsymbol{a}_{l}^{\mathsf{H}}\widehat{\boldsymbol{x}}^{t,(l)}\right\Vert _{2}}_{\eqqcolon\nu_{33}}+\big|\alpha_{\text{mutual}}^{t,\left(l\right)}\big|^{2}\underbrace{\left\Vert \overline{\xi_{l}}\bm{a}_{l}\bm{b}_{l}^{\mathsf{H}}\widehat{\bm{h}}^{t,\left(l\right)}\right\Vert _{2}}_{\eqqcolon\nu_{34}},\label{eq:lemma-proximity-nu3-decompose}
\end{align}
where the first inequality comes from the elementary inequality $\sqrt{a+b}\leq\sqrt{a}+\sqrt{b}$
for $a,b\geq0$, and the second inequality follows from the triangle
inequality. The bounds of $\nu_{31}$ and $\nu_{32}$ follow from
the same derivation as \citet[Equation (217)]{ma2017implicit} and
are thus omitted here for simplicity. The quantity $\nu_{31}$ can
be upper bounded by
\begin{align}
\nu_{31} & \leq\left|\boldsymbol{b}_{l}^{\mathsf{H}}\widehat{\boldsymbol{h}}^{t,(l)}\widehat{\boldsymbol{x}}^{t,(l)\mathsf{H}}\boldsymbol{a}_{l}-\boldsymbol{b}_{l}^{\mathsf{H}}\boldsymbol{h}^{\star}\boldsymbol{x}^{\star\mathsf{H}}\boldsymbol{a}_{l}\right|\left\Vert \boldsymbol{b}_{l}\right\Vert _{2}\left|\boldsymbol{a}_{l}^{\mathsf{H}}\widehat{\boldsymbol{x}}^{t,(l)}\right|\nonumber \\
 & \leq\left|\boldsymbol{b}_{l}^{\mathsf{H}}\widehat{\boldsymbol{h}}^{t,(l)}\widehat{\boldsymbol{x}}^{t,(l)\mathsf{H}}\boldsymbol{a}_{l}-\boldsymbol{b}_{l}^{\mathsf{H}}\boldsymbol{h}^{\star}\boldsymbol{x}^{\star\mathsf{H}}\boldsymbol{a}_{l}\right|\cdot\sqrt{\frac{K}{m}}\cdot20\sqrt{\log m}\cdot\big\|\widehat{\boldsymbol{x}}^{t,(l)}\big\|_{2}\nonumber \\
 & \leq40\sqrt{\frac{K\log m}{m}}\left|\boldsymbol{b}_{l}^{\mathsf{H}}\widehat{\boldsymbol{h}}^{t,(l)}\widehat{\boldsymbol{x}}^{t,(l)\mathsf{H}}\boldsymbol{a}_{l}-\boldsymbol{b}_{l}^{\mathsf{H}}\boldsymbol{h}^{\star}\boldsymbol{x}^{\star\mathsf{H}}\boldsymbol{a}_{l}\right|,\label{eq:nu31}
\end{align}
where the penultimate inequality follows from the fact that $\left\Vert \boldsymbol{b}_{l}\right\Vert _{2}=\sqrt{K/m}$
and \eqref{eq:useful1}, and the last line makes use of \eqref{eq:hat-hx-loo}.
Regarding $\nu_{32}$, one has
\begin{align}
\nu_{32} & \leq\left|\boldsymbol{b}_{l}^{\mathsf{H}}\widehat{\boldsymbol{h}}^{t,(l)}\widehat{\boldsymbol{x}}^{t,(l)\mathsf{H}}\boldsymbol{a}_{l}-\boldsymbol{b}_{l}^{\mathsf{H}}\boldsymbol{h}^{\star}\boldsymbol{x}^{\star\mathsf{H}}\boldsymbol{a}_{l}\right|\left\Vert \bm{a}_{l}\right\Vert _{2}\left|\bm{b}_{l}^{\mathsf{H}}\widehat{\bm{h}}^{t,\left(l\right)}\right|\nonumber \\
 & \leq\left|\boldsymbol{b}_{l}^{\mathsf{H}}\widehat{\boldsymbol{h}}^{t,(l)}\widehat{\boldsymbol{x}}^{t,(l)\mathsf{H}}\boldsymbol{a}_{l}-\boldsymbol{b}_{l}^{\mathsf{H}}\boldsymbol{h}^{\star}\boldsymbol{x}^{\star\mathsf{H}}\boldsymbol{a}_{l}\right|\cdot10\sqrt{K}\cdot\left(\sqrt{\frac{K}{m}}\left\Vert \widehat{\bm{h}}^{t,(l)}-\widetilde{\bm{h}}^{t}\right\Vert _{2}+\left|\bm{b}_{l}^{\mathsf{H}}\widetilde{\bm{h}}^{t}\right|\right)\nonumber \\
 & \leq\left|\boldsymbol{b}_{l}^{\mathsf{H}}\widehat{\boldsymbol{h}}^{t,(l)}\widehat{\boldsymbol{x}}^{t,(l)\mathsf{H}}\boldsymbol{a}_{l}-\boldsymbol{b}_{l}^{\mathsf{H}}\boldsymbol{h}^{\star}\boldsymbol{x}^{\star\mathsf{H}}\boldsymbol{a}_{l}\right|\cdot10\sqrt{K}\cdot\sqrt{\frac{K}{m}}C_{2}\left(\frac{\mu}{\sqrt{m}}\sqrt{\frac{\mu^{2}K\log^{9}m}{m}}+\frac{\sigma}{\log^{2}m}\right)\nonumber \\
 & \quad+\left|\boldsymbol{b}_{l}^{\mathsf{H}}\widehat{\boldsymbol{h}}^{t,(l)}\widehat{\boldsymbol{x}}^{t,(l)\mathsf{H}}\boldsymbol{a}_{l}-\boldsymbol{b}_{l}^{\mathsf{H}}\boldsymbol{h}^{\star}\boldsymbol{x}^{\star\mathsf{H}}\boldsymbol{a}_{l}\right|\cdot10\sqrt{K}\cdot C_{4}\left(\frac{\mu}{\sqrt{m}}\log^{2}m+\sigma\right)\nonumber \\
 & \leq20C_{4}\left(\frac{\mu\sqrt{K}}{\sqrt{m}}\log^{2}m+\sigma\sqrt{K}\right)\left|\boldsymbol{b}_{l}^{\mathsf{H}}\widehat{\boldsymbol{h}}^{t,(l)}\widehat{\boldsymbol{x}}^{t,(l)\mathsf{H}}\boldsymbol{a}_{l}-\boldsymbol{b}_{l}^{\mathsf{H}}\boldsymbol{h}^{\star}\boldsymbol{x}^{\star\mathsf{H}}\boldsymbol{a}_{l}\right|,\label{eq:nu32}
\end{align}
where the second line follows from \eqref{eq:useful2}, triangle inequality
and the fact that $\left\Vert \boldsymbol{b}_{l}\right\Vert _{2}=\sqrt{K/m}$;
the penultimate inequality follows from \eqref{eq:hypothesisloo}
and \eqref{eq:hypothesisincoherence2}; the last line holds as long
as $m\gg\mu^{2}K\log^{3}m$. Further we have
\begin{align}
\left|\bm{b}_{l}^{\mathsf{H}}\big(\widehat{\boldsymbol{h}}^{t,(l)}-\bm{h}^{\star}\big)\right| & \leq\left|\bm{b}_{l}^{\mathsf{H}}\left(\widehat{\bm{h}}^{t,(l)}-\widetilde{\bm{h}}^{t}\right)\right|+\left|\bm{b}_{l}^{\mathsf{H}}\widetilde{\bm{h}}^{t}\right|+\left|\bm{b}_{l}^{\mathsf{H}}\bm{h}^{\star}\right|\nonumber \\
 & \leq\sqrt{\frac{K}{m}}\left\Vert \widehat{\bm{h}}^{t,(l)}-\widetilde{\bm{h}}^{t}\right\Vert _{2}+\left|\bm{b}_{l}^{\mathsf{H}}\widetilde{\bm{h}}^{t}\right|+\left|\bm{b}_{l}^{\mathsf{H}}\bm{h}^{\star}\right|\nonumber \\
 & \leq\sqrt{\frac{K}{m}}C_{2}\left(\frac{\mu}{\sqrt{m}}\sqrt{\frac{\mu^{2}K\log^{9}m}{m}}+\frac{\sigma}{\log^{2}m}\right)+C_{4}\left(\frac{\mu}{\sqrt{m}}\log^{2}m+\sigma\right)+\frac{\mu}{\sqrt{m}}\nonumber \\
 & \leq2C_{4}\left(\frac{\mu}{\sqrt{m}}\log^{2}m+\sigma\right),\label{eq:incohb-loo}
\end{align}
where the second line follows from the fact that $\left\Vert \boldsymbol{b}_{l}\right\Vert _{2}=\sqrt{K/m}$;
the penultimate inequality follows from \eqref{eq:hypothesisloo},
\eqref{eq:hypothesisincoherence2} and \eqref{eq:incoherence-condition};
the last line holds as long as $m\gg\mu^{2}K\log^{3}m$. Therefore,
\begin{align}
 & \left|\boldsymbol{b}_{l}^{\mathsf{H}}\widehat{\boldsymbol{h}}^{t,(l)}\widehat{\boldsymbol{x}}^{t,(l)\mathsf{H}}\boldsymbol{a}_{l}-\boldsymbol{b}_{l}^{\mathsf{H}}\boldsymbol{h}^{\star}\boldsymbol{x}^{\star\mathsf{H}}\boldsymbol{a}_{l}\right|\nonumber \\
 & \leq\left|\boldsymbol{b}_{l}^{\mathsf{H}}\widehat{\boldsymbol{h}}^{t,(l)}\left(\widehat{\boldsymbol{x}}^{t,(l)}-\bm{x}^{\star}\right)^{\mathsf{H}}\boldsymbol{a}_{l}\right|+\left|\boldsymbol{b}_{l}^{\mathsf{H}}\big(\widehat{\boldsymbol{h}}^{t,(l)}-\bm{h}^{\star}\big)\bm{x}^{\star\mathsf{H}}\boldsymbol{a}_{l}\right|\nonumber \\
 & \leq\left(\left|\boldsymbol{b}_{l}^{\mathsf{H}}\big(\widehat{\boldsymbol{h}}^{t,(l)}-\bm{h}^{\star}\big)\right|+\left|\boldsymbol{b}_{l}^{\mathsf{H}}\bm{h}^{\star}\right|\right)\cdot20\sqrt{\log m}\left(\left\Vert \widehat{\boldsymbol{x}}^{t,(l)}-\widetilde{\bm{x}}^{t}\right\Vert _{2}+\left\Vert \widetilde{\bm{x}}^{t}-\bm{x}^{\star}\right\Vert _{2}\right)+\left|\boldsymbol{b}_{l}^{\mathsf{H}}\big(\widehat{\boldsymbol{h}}^{t,(l)}-\bm{h}^{\star}\big)\right|\cdot\left|\bm{x}^{\star\mathsf{H}}\boldsymbol{a}_{l}\right|\nonumber \\
 & \le2C_{4}\left(\frac{\mu}{\sqrt{m}}\log^{2}m+\sigma\right)\cdot20\sqrt{\log m}\cdot C_{2}\left(\frac{\mu}{\sqrt{m}}\sqrt{\frac{\mu^{2}K\log^{9}m}{m}}+\frac{\sigma}{\log^{2}m}\right)\nonumber \\
 & \quad+2C_{4}\left(\frac{\mu}{\sqrt{m}}\log^{2}m+\sigma\right)\cdot C_{1}\left(\sqrt{\frac{\mu^{2}K\log m}{m}}+\lambda+\sigma\sqrt{K\log m}\right)\cdot20\sqrt{\log m}\nonumber \\
 & \quad+2C_{4}\left(\frac{\mu}{\sqrt{m}}\log^{2}m+\sigma\right)\cdot20\sqrt{\log m}\nonumber \\
 & \lesssim C_{4}\left(\frac{\mu}{\sqrt{m}}\log^{2.5}m+\sigma\sqrt{\log m}\right),\label{eq:diff-loo}
\end{align}
where the second inequality follows from triangle inequality and \eqref{eq:useful1};
the penultimate inequality follows from \eqref{eq:incohb-loo}, \eqref{eq:hypothesisloo},
\eqref{eq:dist-bound} and \eqref{eq:useful1}; the last line holds
as long as $m\gg\mu^{2}K\log m$. Substituting \eqref{eq:diff-loo}
into \eqref{eq:nu31} and \eqref{eq:nu32}, we reach
\begin{align}
\nu_{31}+\nu_{32} & \lesssim\left(40\sqrt{\frac{K\log m}{m}}+20C_{4}\left(\frac{\mu\sqrt{K}}{\sqrt{m}}\log^{2}m+\sigma\sqrt{K}\right)\right)C_{4}\left(\frac{\mu}{\sqrt{m}}\log^{2.5}m+\sigma\sqrt{\log m}\right)\nonumber \\
 & \leq\left(C_{4}\right)^{2}\frac{\mu}{\sqrt{m}}\sqrt{\frac{\mu^{2}K\log^{9}m}{m}}+C_{4}\frac{\sigma}{\log^{2}m},\label{eq:lem-loo-2}
\end{align}
as long as $m\gg\mu^{2}K\log^{9}m$. Regarding $\nu_{33}$ and $\nu_{34}$,
it is seen that
\begin{align}
\big\|\xi_{l}\boldsymbol{b}_{l}\boldsymbol{a}_{l}^{\mathsf{H}}\widehat{\boldsymbol{x}}^{t,(l)}\big\|_{2} & \leq\left|\xi_{l}\right|\left\Vert \boldsymbol{b}_{l}\right\Vert _{2}\left|\boldsymbol{a}_{l}^{\mathsf{H}}\widehat{\boldsymbol{x}}^{t,(l)}\right|\overset{\text{(i)}}{\lesssim}\sigma\sqrt{\frac{K}{m}}\big\|\widehat{\boldsymbol{x}}^{t,(l)}\big\|_{2}\log m\overset{\text{(ii)}}{\leq}2\sigma\sqrt{\frac{K}{m}}\log m,\label{eq:lem-loo-3}\\
\big\|\overline{\xi_{l}}\bm{a}_{l}\bm{b}_{l}^{\mathsf{H}}\widehat{\bm{h}}^{t,\left(l\right)}\big\|_{2} & \leq\left|\xi_{l}\right|\left\Vert \bm{a}_{l}\right\Vert _{2}\left|\bm{b}_{l}^{\mathsf{H}}\widehat{\bm{h}}^{t,\left(l\right)}\right|\overset{\text{(iii)}}{\lesssim}\sigma\sqrt{K}\left(\left|\bm{b}_{l}^{\mathsf{H}}\left(\widehat{\bm{h}}^{t,(l)}-\bm{h}^{\star}\right)\right|+\left|\bm{b}_{l}^{\mathsf{H}}\bm{h}^{\star}\right|\right)\nonumber \\
 & \overset{(\text{iv})}{\lesssim}\sigma\sqrt{K}\left(2C_{4}\left(\frac{\mu}{\sqrt{m}}\log^{2}m+\sigma\right)+\frac{\mu}{\sqrt{m}}\right)\nonumber \\
 & \lesssim C_{4}\frac{\sigma}{\log^{2.5}m}+C_{4}\sigma\sqrt{\frac{\mu^{2}K\log^{4}m}{m}},\label{eq:lem-loo-4}
\end{align}
\end{subequations}where (i) holds by the property of sub-Gaussian
variables (cf.~\citet[Proposition 2.5.2]{vershynin2018high}) and
the independence between $\xi_{l},\bm{a}_{l}$ and $\widehat{\boldsymbol{x}}^{t,(l)}$,
(ii) holds by \eqref{eq:hat-hx-loo}, (iii) is due to Lemma \eqref{lemma:useful},
the triangle inequality and \eqref{eq:incoherence-condition}, and
(iv) follows from \eqref{eq:incohb-loo} and \eqref{eq:incoherence-condition}.
Consequently, by \eqref{eq:lem-loo-2}-\eqref{eq:lem-loo-4} we have
\begin{equation}
\left\Vert \boldsymbol{\nu}_{3}\right\Vert _{2}\lesssim\left(C_{4}\right)^{2}\frac{\mu}{\sqrt{m}}\sqrt{\frac{\mu^{2}K\log^{9}m}{m}}+C_{4}\frac{\sigma}{\log^{2}m}.\label{eq:prox-nu3-1-1}
\end{equation}
\item Finally, in terms of $\bm{\nu}_{4}$ one has 
\begin{align}
\left\Vert \boldsymbol{\nu}_{4}\right\Vert _{2} & =\left\Vert \left[\begin{array}{c}
\widehat{\bm{h}}^{t,\left(l\right)}-\widetilde{\bm{h}}^{t}\\
\widehat{\bm{x}}^{t,\left(l\right)}-\widetilde{\bm{x}}^{t}
\end{array}\right]\right\Vert _{2}=\left\Vert \widehat{\bm{z}}^{t,\left(l\right)}-\widetilde{\bm{z}}^{t}\right\Vert _{2}.\label{eq:prox-nu4}
\end{align}
\end{itemize}
With the above bounds in place, we can demonstrate that
\begin{align}
\mathsf{dist}\big(\boldsymbol{z}^{t+1,\left(l\right)},\widetilde{\boldsymbol{z}}^{t+1}\big) & \leq\max\left\{ \left|\frac{\alpha^{t+1/2}}{\alpha^{t}}\right|,\left|\frac{\alpha^{t}}{\alpha^{t+1/2}}\right|\right\} \left\Vert \left[\begin{array}{c}
\frac{1}{\overline{\alpha_{\text{mutual}}^{t,\left(l\right)}}}\boldsymbol{h}^{t+1/2,\left(l\right)}-\frac{1}{\overline{\alpha^{t}}}\boldsymbol{h}^{t+1/2}\\
\alpha_{\text{mutual}}^{t,\left(l\right)}\boldsymbol{x}^{t+1/2,\left(l\right)}-\alpha^{t}\boldsymbol{x}^{t+1/2}
\end{array}\right]\right\Vert _{2}\nonumber \\
 & \overset{(\text{i})}{\leq}\frac{1-\eta/32}{1-\eta/16}\left(\left\Vert \boldsymbol{\nu}_{1}\right\Vert _{2}+\left\Vert \boldsymbol{\nu}_{2}\right\Vert _{2}+\left\Vert \boldsymbol{\nu}_{3}\right\Vert _{2}+\left\Vert \boldsymbol{\nu}_{4}\right\Vert _{2}\right)\nonumber \\
 & \overset{(\text{ii})}{\leq}\left(1-\eta/32\right)\big\|\widehat{\boldsymbol{z}}^{t,\left(l\right)}-\widetilde{\boldsymbol{z}}^{t}\big\|_{2}+\frac{1-\eta/32}{1-\eta/16}C\eta\times C_{1}\frac{1}{\log^{2}m}\big\|\widehat{\boldsymbol{z}}^{t,\left(l\right)}-\widetilde{\boldsymbol{z}}^{t}\big\|_{2}\nonumber \\
 & \quad+\frac{1-\eta/32}{1-\eta/16}C\eta\left(\left(C_{4}\right)^{2}\frac{\mu}{\sqrt{m}}\sqrt{\frac{\mu^{2}K\log^{9}m}{m}}+C_{4}\frac{\sigma}{\log^{2}m}\right)+\frac{1-\eta/32}{1-\eta/16}\eta\lambda\big\|\widehat{\boldsymbol{z}}^{t,\left(l\right)}-\widetilde{\boldsymbol{z}}^{t}\big\|_{2}\nonumber \\
 & \leq\left(1-\eta/32+\frac{1-\eta/32}{1-\eta/16}\eta\lambda+\frac{1-\eta/32}{1-\eta/16}CC_{1}\frac{\eta}{\log^{2}m}\right)\big\|\widehat{\boldsymbol{z}}^{t,\left(l\right)}-\widetilde{\boldsymbol{z}}^{t}\big\|_{2}\nonumber \\
 & \quad\quad+\frac{1-\eta/32}{1-\eta/16}C\eta\left(\left(C_{4}\right)^{2}\frac{\mu}{\sqrt{m}}\sqrt{\frac{\mu^{2}K\log^{9}m}{m}}+C_{4}\frac{\sigma}{\log^{2}m}\right)\nonumber \\
 & \leq\left(1-\frac{\eta}{64}\right)\mathsf{dist}\big(\boldsymbol{z}^{t,\left(l\right)},\widetilde{\boldsymbol{z}}^{t}\big)+\eta C\left(C_{4}\right)^{2}\frac{\mu}{\sqrt{m}}\sqrt{\frac{\mu^{2}K\log^{9}m}{m}}+\eta CC_{4}\frac{\sigma}{\log^{2}m}\nonumber \\
 & \leq C_{2}\left(\frac{\mu}{\sqrt{m}}\sqrt{\frac{\mu^{2}K\log^{9}m}{m}}+\frac{\sigma}{\log^{2}m}\right),\label{eq:lemmaproximity-3}
\end{align}
provided that $\eta>0$ is some sufficiently small constant and $C_{2}\gg C_{4}^{2}$.
To see why (i) holds, we observe that
\[
\left|\left|\frac{\alpha^{t+1/2}}{\alpha^{t}}\right|-1\right|\le\left|\frac{\alpha^{t+1/2}}{\alpha^{t}}-1\right|\leq C\left(\sqrt{\frac{\mu^{2}K\log m}{m}}+\lambda+\sigma\sqrt{K\log m}\right)
\]
as shown in Corollary \ref{corollary:alpha}, which implies that
\[
\left|\frac{\alpha^{t+1/2}}{\alpha^{t}}\right|\le1+\frac{\eta/32}{1-\eta/16}=\frac{1-\eta/32}{1-\eta/16}
\]
as long as $m\gg\mu^{2}K\log m$ and $\sigma\sqrt{K\log m}\ll1$;
a similar argument also reveals that
\[
\left|\frac{\alpha^{t}}{\alpha^{t+1/2}}\right|\leq\frac{1-\eta/32}{1-\eta/16}.
\]
In addition, (ii) follows from \eqref{eq:prox-nu12}, \eqref{eq:prox-nu3-1-1}
and \eqref{eq:prox-nu4}, whereas the last inequality of \eqref{eq:lemmaproximity-3}
relies on the hypothesis \eqref{eq:hypothesisloo}. 

Next, we turn to the second inequality claimed in the lemma. In view
of \eqref{eq:dist-bound} in Lemma \ref{lem:consequence}, we have
\[
\left\Vert \widetilde{\bm{z}}^{t+1}-\bm{z}^{\star}\right\Vert _{2}\leq C_{1}\left(\sqrt{\frac{\mu^{2}K\log m}{m}}+\lambda+\sigma\sqrt{K\log m}\right),
\]
which together with the triangle inequality and \eqref{eq:lemmaproximity-3}
yields
\begin{align}
\big\|\widehat{\bm{z}}^{t+1,\left(l\right)}-\bm{z}^{\star}\big\|_{2} & \leq\big\|\widehat{\bm{z}}^{t+1,\left(l\right)}-\widetilde{\bm{z}}^{t+1}\big\|_{2}+\|\widetilde{\bm{z}}^{t+1}-\bm{z}^{\star}\|_{2}\nonumber \\
 & \leq C_{2}\left(\frac{\mu}{\sqrt{m}}\sqrt{\frac{\mu^{2}K\log^{9}m}{m}}+\frac{\sigma}{\log^{2}m}\right)+C_{1}\left(\sqrt{\frac{\mu^{2}K\log m}{m}}+\lambda+\sigma\sqrt{K\log m}\right)\nonumber \\
 & \lesssim\sqrt{\frac{\mu^{2}K\log m}{m}}+\sigma\sqrt{K\log m}+\lambda.\label{eq:lemmaproximity-4}
\end{align}
In other words, both $\widetilde{\bm{z}}^{t+1}$ and $\widehat{\bm{z}}^{t+1,(l)}$
are sufficiently close to the truth $\bm{z}^{\star}$. Consequently,
we are ready to invoke \citet[Lemma 55]{ma2017implicit}. Taking $\bm{h}_{1}=\widetilde{\bm{h}}^{t+1}$,
$\bm{x}_{1}=\widetilde{\bm{x}}^{t+1}$, $\bm{h}_{2}=\widehat{\bm{h}}^{t+1,\left(l\right)}$
and $\bm{x}_{2}=\widehat{\bm{x}}^{t+1,\left(l\right)}$ in \citet[Lemma 55]{ma2017implicit}
yields
\begin{equation}
\big\|\widetilde{\bm{z}}^{t+1,\left(l\right)}-\widetilde{\bm{z}}^{t+1}\big\|_{2}\lesssim\big\|\widehat{\bm{z}}^{t+1,\left(l\right)}-\widetilde{\bm{z}}^{t+1}\big\|_{2}\leq C_{2}\left(\frac{\mu}{\sqrt{m}}\sqrt{\frac{\mu^{2}K\log^{9}m}{m}}+\frac{\sigma}{\log^{2}m}\right),\label{eq:lemproxlast}
\end{equation}
where the last inequality follows from \eqref{eq:lemmaproximity-4}. 

\subsection{Proof of Lemma \ref{lemma:incoherenceb}\label{subsec:Proof-of-Lemmaincoherenceb}}

Recall from Corollary \ref{corollary:alpha} that there exist some
constant $C>0$ such that 
\begin{equation}
\left|\frac{\alpha^{t+1/2}}{\alpha^{t}}-1\right|\leq C\eta\left(\sqrt{\frac{\mu^{2}K\log m}{m}}+\lambda+\sigma\sqrt{K\log m}\right)\eqqcolon\delta,\label{eq:alpharatio-delta}
\end{equation}
with $\delta\ll1$, thus indicating that
\begin{align*}
\max_{1\leq l\leq m}\left|\boldsymbol{b}_{l}^{\mathsf{H}}\frac{1}{\overline{\alpha^{t+1}}}\boldsymbol{h}^{t+1}\right| & =\max_{1\leq l\leq m}\left|\boldsymbol{b}_{l}^{\mathsf{H}}\frac{1}{\overline{\alpha^{t+1/2}}}\boldsymbol{h}^{t+1/2}\right|\leq\left|\frac{\alpha^{t}}{\alpha^{t+1/2}}\right|\max_{1\leq l\leq m}\left|\boldsymbol{b}_{l}^{\mathsf{H}}\frac{1}{\overline{\alpha^{t}}}\boldsymbol{h}^{t+1/2}\right|\\
 & \leq\left(1+\delta\right)\max_{1\leq l\leq m}\left|\boldsymbol{b}_{l}^{\mathsf{H}}\frac{1}{\overline{\alpha^{t}}}\boldsymbol{h}^{t+1/2}\right|.
\end{align*}
The gradient update rule regarding $\boldsymbol{h}^{t+1}$ then leads
to
\[
\frac{1}{\overline{\alpha^{t}}}\boldsymbol{h}^{t+1/2}=\widetilde{\boldsymbol{h}}^{t}-\frac{\eta}{\left|\alpha^{t}\right|^{2}}\sum_{j=1}^{m}\left(\boldsymbol{b}_{j}^{\mathsf{H}}\widetilde{\boldsymbol{h}}^{t}\widetilde{\boldsymbol{x}}^{t\mathsf{H}}\boldsymbol{a}_{j}-y_{j}\right)\boldsymbol{b}_{j}\boldsymbol{a}_{j}^{\mathsf{H}}\widetilde{\boldsymbol{x}}^{t}-\eta\lambda\widetilde{\bm{h}}^{t},
\]
where we recall that $\widetilde{\boldsymbol{h}}^{t}=\boldsymbol{h}^{t}/\overline{\alpha^{t}}$
and $\widetilde{\boldsymbol{x}}^{t}=\alpha^{t}\boldsymbol{x}^{t}$.
Expanding terms further and using the assumption $\sum_{j=1}^{m}\boldsymbol{b}_{j}\boldsymbol{b}_{j}^{\mathsf{H}}=\bm{I}$
give
\begin{align}
\frac{1}{\overline{\alpha^{t}}}\boldsymbol{h}^{t+1/2} & =\widetilde{\boldsymbol{h}}^{t}-\frac{\eta}{\left|\alpha^{t}\right|^{2}}\sum_{j=1}^{m}\boldsymbol{b}_{j}\boldsymbol{b}_{j}^{\mathsf{H}}\left(\widetilde{\boldsymbol{h}}^{t}\widetilde{\boldsymbol{x}}^{t\mathsf{H}}-\boldsymbol{h}^{\star}\boldsymbol{x}^{\star\mathsf{H}}\right)\boldsymbol{a}_{j}\boldsymbol{a}_{j}^{\mathsf{H}}\widetilde{\boldsymbol{x}}^{t}+\frac{\eta}{\left|\alpha^{t}\right|^{2}}\sum_{j=1}^{m}\xi_{j}\boldsymbol{b}_{j}\boldsymbol{a}_{j}^{\mathsf{H}}\widetilde{\boldsymbol{x}}^{t}-\eta\lambda\widetilde{\boldsymbol{h}}^{t}\nonumber \\
 & =\underbrace{\left(1-\eta\lambda-\frac{\eta}{\left|\alpha^{t}\right|^{2}}\left\Vert \boldsymbol{x}^{\star}\right\Vert _{2}^{2}\right)\widetilde{\boldsymbol{h}}^{t}}_{\eqqcolon\bm{\nu}_{0}}-\frac{\eta}{\left|\alpha^{t}\right|^{2}}\underbrace{\sum_{j=1}^{m}\boldsymbol{b}_{j}\boldsymbol{b}_{j}^{\mathsf{H}}\widetilde{\boldsymbol{h}}^{t}\left(\left|\boldsymbol{a}_{j}^{\mathsf{H}}\widetilde{\boldsymbol{x}}^{t}\right|^{2}-\left|\boldsymbol{a}_{j}^{\mathsf{H}}\boldsymbol{x}^{\star}\right|^{2}\right)}_{\eqqcolon\bm{\nu}_{1}}\nonumber \\
 & \ -\frac{\eta}{\left|\alpha^{t}\right|^{2}}\underbrace{\sum_{j=1}^{m}\boldsymbol{b}_{j}\boldsymbol{b}_{j}^{\mathsf{H}}\widetilde{\boldsymbol{h}}^{t}\left(\left|\boldsymbol{a}_{j}^{\mathsf{H}}\boldsymbol{x}^{\star}\right|^{2}-\left\Vert \boldsymbol{x}^{\star}\right\Vert _{2}^{2}\right)}_{\eqqcolon\bm{\nu}_{2}}+\frac{\eta}{\left|\alpha^{t}\right|^{2}}\underbrace{\sum_{j=1}^{m}\boldsymbol{b}_{j}\boldsymbol{b}_{j}^{\mathsf{H}}\boldsymbol{h}^{\star}\boldsymbol{x}^{\star\mathsf{H}}\boldsymbol{a}_{j}\boldsymbol{a}_{j}^{\mathsf{H}}\widetilde{\boldsymbol{x}}^{t}}_{\eqqcolon\bm{\nu}_{3}}+\frac{\eta}{\left|\alpha^{t}\right|^{2}}\underbrace{\sum_{j=1}^{m}\xi_{j}\boldsymbol{b}_{j}\boldsymbol{a}_{j}^{\mathsf{H}}\widetilde{\boldsymbol{x}}^{t}}_{\eqqcolon\bm{\nu}_{4}}.\label{eq:lem-incohb-decomp}
\end{align}
The first three terms can be controlled via the same arguments as
\citet[Appendix C.4]{ma2017implicit}, which are built upon the induction
hypotheses \eqref{eq:hypothesisdist}-\eqref{eq:hypothesisincoherence2}
at the $t$th iteration as well as the following claim (which is the
counterpart of \citet[Claim 224]{ma2017implicit}). 

\begin{claim}\label{claim:224} Suppose that $m\gg\tau K\log^{4}m$.
For some sufficiently small constant $c>0$, it holds that 
\[
\max_{1\leq j\leq\tau}\left|\left(\boldsymbol{b}_{j}-\boldsymbol{b}_{1}\right)^{\mathsf{H}}\widetilde{\boldsymbol{h}}^{t}\right|\leq cC_{4}\left(\frac{\mu}{\sqrt{m}}\log m+\frac{\sigma}{\log m}\right).
\]
\end{claim}

\noindent The corresponding bounds obtained from \citet[Appendix C.4]{ma2017implicit}
are listed below:\begin{subequations}\label{eq:lem-incohb-nu}
\begin{align}
\left|\boldsymbol{b}_{l}^{\mathsf{H}}\bm{\nu}_{1}\right| & \leq0.1\max_{1\leq j\leq m}\big|\boldsymbol{b}_{j}^{\mathsf{H}}\widetilde{\bm{h}}^{t}\big|,\label{eq:lem-incohb-nu1}\\
\left|\boldsymbol{b}_{l}^{\mathsf{H}}\bm{\nu}_{2}\right| & \leq0.2\max_{1\leq j\leq m}\big|\boldsymbol{b}_{j}^{\mathsf{H}}\widetilde{\bm{h}}^{t}\big|+\max_{1\leq j\leq\tau}\left|\left(\boldsymbol{b}_{j}-\boldsymbol{b}_{1}\right)^{\mathsf{H}}\widetilde{\boldsymbol{h}}^{t}\right|\log m,\label{eq:lem-incohb-nu2}\\
\left|\boldsymbol{b}_{l}^{\mathsf{H}}\bm{\nu}_{3}\right| & \lesssim\frac{\mu}{\sqrt{m}}+\frac{\mu}{\sqrt{m}}\log^{3/2}m\max_{1\leq j\leq m}\left|\bm{a}_{j}^{\mathsf{H}}(\widetilde{\bm{x}}^{t}-\bm{x}^{\star})\right|.\label{eq:lem-incohb-nu3}
\end{align}
\end{subequations}

When it comes to the last term of \eqref{eq:lem-incohb-decomp} concerning
$\bm{\nu}_{4}$, it is seen that
\begin{align*}
\left|\bm{b}_{l}^{\mathsf{H}}\bm{\nu}_{4}\right| & \leq\underbrace{\Big|\sum_{j=1}^{m}\xi_{j}\bm{b}_{l}^{\mathsf{H}}\bm{b}_{j}\boldsymbol{a}_{j}^{\mathsf{H}}(\widetilde{\bm{x}}^{t}-\bm{x}^{\star})\Big|}_{\eqqcolon\varsigma_{1}}+\underbrace{\Big|\sum_{j=1}^{m}\xi_{j}\bm{b}_{l}^{\mathsf{H}}\boldsymbol{b}_{j}\boldsymbol{a}_{j}^{\mathsf{H}}\bm{x}^{\star}\Big|}_{\eqqcolon\varsigma_{2}},
\end{align*}
leaving us with two terms to control. 
\begin{itemize}
\item With regards to $\varsigma_{1}$, we have
\begin{align*}
\ensuremath{\varsigma_{1}} & \leq\sum_{j=1}^{m}\left|\bm{b}_{l}^{\mathsf{H}}\bm{b}_{j}\right|\cdot\max_{1\leq j\leq m}\left|\xi_{j}\right|\cdot\max_{1\leq j\leq m}\left|\boldsymbol{a}_{j}^{\mathsf{H}}(\widetilde{\bm{x}}^{t}-\bm{x}^{\star})\right|\\
 & \lesssim(4\log m)\cdot\sigma\sqrt{\log m}\cdot\max_{1\leq j\leq m}\left|\boldsymbol{a}_{j}^{\mathsf{H}}(\widetilde{\bm{x}}^{t}-\bm{x}^{\star})\right|\\
 & \asymp\sigma\log^{1.5}m\max_{1\leq j\leq m}\left|\boldsymbol{a}_{j}^{\mathsf{H}}(\widetilde{\bm{x}}^{t}-\bm{x}^{\star})\right|,
\end{align*}
where the second inequality follows from \citet[Lemma 48]{ma2017implicit}
and standard sub-Gaussian concentration inequalities.
\item Regarding $\varsigma_{2}$, since $\left\{ \boldsymbol{a}_{j}^{\mathsf{H}}\bm{x}^{\star}\right\} _{j=1}^{m}$
are i.i.d.~Gaussian variables with variance $\left\Vert \bm{x}^{\star}\right\Vert _{2}=1$,
we see that 
\[
\Big\|\xi_{j}\boldsymbol{a}_{j}^{\mathsf{H}}\bm{x}^{\star}\Big\|_{\psi_{1}}\leq\left\Vert \xi_{j}\right\Vert _{\psi_{2}}\Big\|\boldsymbol{a}_{j}^{\mathsf{H}}\bm{x}^{\star}\Big\|_{\psi_{2}}\leq\sigma,
\]
where $\|\cdot\|_{\psi_{1}}$ and $\|\cdot\|_{\psi_{2}}$ denote the
sub-exponential norm and the sub-Gaussian norm, respectively. In view
of the Bernstein inequality \citet[Theorem 2.8.2]{vershynin2018high},
we have
\begin{align}
\mathbb{P}\left\{ \left|\sum_{j=1}^{m}\xi_{j}\bm{b}_{l}^{\mathsf{H}}\boldsymbol{b}_{j}\boldsymbol{a}_{j}^{\mathsf{H}}\bm{x}^{\star}\right|\geq t\right\}  & \leq2\exp\left(-c\min\left(\frac{\tau^{2}}{\sigma^{2}\sum_{j=1}^{m}\left|\bm{b}_{l}^{\mathsf{H}}\boldsymbol{b}_{j}\right|^{2}},\frac{\tau}{\sigma\max_{1\leq j\leq m}\left|\bm{b}_{l}^{\mathsf{H}}\boldsymbol{b}_{j}\right|}\right)\right)\label{eq:lem-incohb-sigma2}
\end{align}
for any $\tau>0$. Recognizing that 
\[
\sum_{j=1}^{m}\left|\bm{b}_{l}^{\mathsf{H}}\boldsymbol{b}_{j}\right|^{2}=\bm{b}_{l}^{\mathsf{H}}\Big(\sum_{j=1}^{m}\boldsymbol{b}_{j}\bm{b}_{j}^{\mathsf{H}}\Big)\bm{b}_{l}=\frac{K}{m}\qquad\text{and}\qquad\max_{1\leq j\leq m}\left|\bm{b}_{l}^{\mathsf{H}}\boldsymbol{b}_{j}\right|\leq\max_{1\leq j\leq m}\left\Vert \bm{b}_{l}\right\Vert _{2}\left\Vert \bm{b}_{j}\right\Vert _{2}=\frac{K}{m}
\]
and setting $\tau=C\sigma\sqrt{\frac{K}{m}\log m}$ for some large
enough constant $C>0$, one obtains
\[
\mathbb{P}\left\{ \varsigma_{2}\geq C\sigma\sqrt{\frac{K}{m}\log m}\right\} \leq2\exp\left(-c\min\left(C^{2}\log m,C\sqrt{\frac{m\log m}{K}}\right)\right)\lesssim m^{-100},
\]
provided that $m\gg K\log m$. 
\item Combining the above two pieces implies that, with probability exceeding
$1-O\left(m^{-100}\right)$, 
\begin{align}
\left|\bm{b}_{l}^{\mathsf{H}}\bm{\nu}_{4}\right| & \lesssim\sigma\log^{1.5}m\max_{1\leq j\leq m}\left|\boldsymbol{a}_{j}^{\mathsf{H}}(\widetilde{\bm{x}}^{t}-\bm{x}^{\star})\right|+\sigma\sqrt{\frac{K}{m}\log m},\label{eq:lem-incohb-nu4}\\
 & \leq\sigma\log^{1.5}m\cdot C_{3}\left(\sqrt{\frac{\mu^{2}K\log^{2}m}{m}}+\sqrt{\log m}\left(\lambda+\sigma\sqrt{K\log m}\right)\right)+\sigma\sqrt{\frac{K}{m}\log m}\nonumber \\
 & \lesssim C_{3}\sigma.
\end{align}
where the penultimate inequality follows from the hypothesis \eqref{eq:hypothesisincoherence1},
and the last line holds as long as $m\gg\mu^{2}K\log^{5}m$, $\sigma\sqrt{K\log^{5}m}\ll1$.
\end{itemize}
Combining the bounds \eqref{eq:lem-incohb-nu} with \eqref{eq:lem-incohb-decomp}
and \eqref{eq:lem-incohb-nu4}, we arrive at
\begin{align*}
\left|\boldsymbol{b}_{l}^{\mathsf{H}}\widetilde{\boldsymbol{h}}^{t+1}\right| & \leq\left(1+\delta\right)\left(1-\eta\lambda-\frac{\eta}{\left|\alpha^{t}\right|^{2}}\left\Vert \boldsymbol{x}^{\star}\right\Vert _{2}^{2}\right)\left|\boldsymbol{b}_{l}^{\mathsf{H}}\widetilde{\boldsymbol{h}}^{t}\right|+\left(1+\delta\right)0.3\frac{\eta}{\left|\alpha^{t}\right|^{2}}\max_{1\leq j\leq m}\left|\bm{b}_{j}^{\mathsf{H}}\widetilde{\bm{h}}^{t}\right|\\
 & \quad+\left(1+\delta\right)\frac{\eta}{\left|\alpha^{t}\right|^{2}}\times C\left(\frac{\mu}{\sqrt{m}}+\frac{\mu}{\sqrt{m}}\log^{3/2}m\max_{1\leq j\leq m}\left|\bm{a}_{j}^{\mathsf{H}}\left(\widetilde{\bm{x}}^{t}-\bm{x}^{\star}\right)\right|\right)\\
 & \quad+\left(1+\delta\right)\frac{\eta}{\left|\alpha^{t}\right|^{2}}\max_{1\leq j\leq\tau}\left|\left(\boldsymbol{b}_{j}-\boldsymbol{b}_{1}\right)^{\mathsf{H}}\widetilde{\boldsymbol{h}}^{t}\right|\log m+\left(1+\delta\right)\frac{\eta}{\left|\alpha^{t}\right|^{2}}\left|\bm{b}_{l}^{\mathsf{H}}\bm{\nu}_{4}\right|\\
 & \leq C_{4}\left(\frac{\mu}{\sqrt{m}}\log^{2}m+\sigma\right),
\end{align*}
as long as $m\gg\mu^{2}K\log^{9}m$ for some large enough constant
$C_{4}\gg C_{3}$. Here, the last inequality invokes the induction
hypotheses \eqref{sec:hypotheses-ncvx} at the $t$th iteration, Claim
\ref{claim:224}, as well as the fact $\left|\alpha^{t}\right|\asymp1$
(cf.~Corollary \ref{corollary:alpha}).

\subsubsection{Proof of Claim \ref{claim:224}\label{subsec:Proof-of-Claim224}}

To begin with, we make the observation that
\begin{align*}
\left|\left(\boldsymbol{b}_{j}-\boldsymbol{b}_{1}\right)^{\mathsf{H}}\widetilde{\boldsymbol{h}}^{t}\right| & =\left|\left(\boldsymbol{b}_{j}-\boldsymbol{b}_{1}\right)^{\mathsf{H}}\widetilde{\boldsymbol{h}}^{t-1/2}\right|=\left|\frac{\alpha^{t-1}}{\alpha^{t-1/2}}\right|\left|\left(\boldsymbol{b}_{j}-\boldsymbol{b}_{1}\right)^{\mathsf{H}}\frac{\boldsymbol{h}^{t-1/2}}{\overline{\alpha^{t-1}}}\right|\\
 & \leq\left(1+\delta\right)\left|\left(\boldsymbol{b}_{j}-\boldsymbol{b}_{1}\right)^{\mathsf{H}}\frac{\boldsymbol{h}^{t-1/2}}{\overline{\alpha^{t-1}}}\right|,
\end{align*}
with $\delta\ll1$ defined in \eqref{eq:alpharatio-delta}. This inequality
allows us to turn attention to $\frac{1}{\overline{\alpha^{t-1}}}\left(\boldsymbol{b}_{j}-\boldsymbol{b}_{1}\right)^{\mathsf{H}}\boldsymbol{h}^{t-1/2}$
instead.

Use the gradient update rule with respect to $\boldsymbol{h}^{t}$,
we obtain 
\[
\frac{1}{\overline{\alpha^{t-1}}}\boldsymbol{h}^{t-1/2}=\frac{1}{\overline{\alpha^{t-1}}}\left(\boldsymbol{h}^{t-1}-\eta\left(\sum_{l=1}^{m}\boldsymbol{b}_{l}\boldsymbol{b}_{l}^{\mathsf{H}}\left(\boldsymbol{h}^{t-1}\boldsymbol{x}^{t-1\mathsf{H}}-\boldsymbol{h}^{\star}\boldsymbol{x}^{\star\mathsf{H}}\right)\boldsymbol{a}_{l}\boldsymbol{a}_{l}^{\mathsf{H}}\boldsymbol{x}^{t-1}-\sum_{l=1}^{m}\xi_{l}\boldsymbol{b}_{l}\boldsymbol{a}_{l}^{\mathsf{H}}\boldsymbol{x}^{t-1}+\lambda\bm{h}^{t-1}\right)\right).
\]
Therefore, one can decompose
\begin{align}
\left(\boldsymbol{b}_{j}-\boldsymbol{b}_{1}\right)^{\mathsf{H}}\frac{1}{\overline{\alpha^{t-1}}}\boldsymbol{h}^{t} & =\left(1-\eta\lambda-\frac{\eta}{\left|\alpha^{t}\right|^{2}}\left\Vert \widetilde{\boldsymbol{x}}^{t-1}\right\Vert _{2}^{2}\right)\left(\boldsymbol{b}_{j}-\boldsymbol{b}_{1}\right)^{\mathsf{H}}\widetilde{\boldsymbol{h}}^{t-1}+\frac{\eta}{\left|\alpha^{t}\right|^{2}}\underbrace{\left(\boldsymbol{b}_{j}-\boldsymbol{b}_{1}\right)\boldsymbol{h}^{\star}\boldsymbol{x}^{\star\mathsf{H}}\widetilde{\boldsymbol{x}}^{t-1}}_{\eqqcolon\beta_{1}}\nonumber \\
 & \quad-\frac{\eta}{\left|\alpha^{t}\right|^{2}}\underbrace{\left(\boldsymbol{b}_{j}-\boldsymbol{b}_{1}\right)^{\mathsf{H}}\sum_{l=1}^{m}\boldsymbol{b}_{l}\boldsymbol{b}_{l}^{\mathsf{H}}\left(\widetilde{\boldsymbol{h}}^{t-1}\widetilde{\boldsymbol{x}}^{t-1\mathsf{H}}-\boldsymbol{h}^{\star}\boldsymbol{x}^{\star\mathsf{H}}\right)\left(\boldsymbol{a}_{l}\boldsymbol{a}_{l}^{\mathsf{H}}-\bm{I}_{k}\right)\widetilde{\boldsymbol{x}}^{t-1}}_{\eqqcolon\beta_{2}}\nonumber \\
 & \quad+\frac{\eta}{\left|\alpha^{t}\right|^{2}}\underbrace{\left(\boldsymbol{b}_{j}-\boldsymbol{b}_{1}\right)^{\mathsf{H}}\sum_{l=1}^{m}\xi_{l}\boldsymbol{b}_{l}\boldsymbol{a}_{l}^{\mathsf{H}}\widetilde{\boldsymbol{x}}^{t-1}}_{\eqqcolon\beta_{3}}.\label{eq:proof-claim-decomp}
\end{align}
Except $\beta_{3}$, the bounds of the other terms can be obtained
by the same arguments as in \citet[Appendix C.4.3]{ma2017implicit};
we thus omit the detailed proof but only list the results below:
\begin{align*}
\left|\beta_{1}\right| & \leq4\frac{\mu}{\sqrt{m}}\\
\left|\beta_{2}\right| & \leq\frac{c}{\log m}\left(\max_{1\leq l\leq m}\left|\bm{b}_{l}^{\mathsf{H}}\widetilde{\bm{h}}^{t-1}\right|+\frac{\mu}{\sqrt{m}}\right)
\end{align*}
with $c$ some small constant $c>0$, as long as $m\gg K\log^{8}m$.
When it comes to the remaining term $\beta_{3}$, the triangle inequality
yields
\begin{align*}
\left|\beta_{3}\right| & \leq\underbrace{\left|\sum_{l=1}^{m}\xi_{l}\left(\boldsymbol{b}_{j}-\boldsymbol{b}_{1}\right)^{\mathsf{H}}\boldsymbol{b}_{l}\boldsymbol{a}_{l}^{\mathsf{H}}\left(\widetilde{\boldsymbol{x}}^{t-1}-\bm{x}^{\star}\right)\right|}_{\eqqcolon\omega_{1}}+\underbrace{\left|\sum_{l=1}^{m}\xi_{l}\left(\boldsymbol{b}_{j}-\boldsymbol{b}_{1}\right)^{\mathsf{H}}\boldsymbol{b}_{l}\boldsymbol{a}_{l}^{\mathsf{H}}\bm{x}^{\star}\right|}_{\eqqcolon\omega_{2}}.
\end{align*}

\begin{itemize}
\item Regarding $\omega_{1}$, we have
\begin{align*}
\ensuremath{\omega_{1}} & \leq\sum_{j=1}^{m}\left|(\boldsymbol{b}_{j}-\boldsymbol{b}_{1})^{\mathsf{H}}\boldsymbol{b}_{l}\right|\cdot\max_{1\leq j\leq m}\left|\xi_{j}\right|\cdot\max_{1\leq j\leq m}\left|\boldsymbol{a}_{j}^{\mathsf{H}}(\widetilde{\bm{x}}^{t}-\bm{x}^{\star})\right|\\
 & \lesssim\frac{1}{\log^{2}m}\cdot\sigma\sqrt{\log m}\cdot\max_{1\leq j\leq m}\left|\boldsymbol{a}_{j}^{\mathsf{H}}(\widetilde{\bm{x}}^{t}-\bm{x}^{\star})\right|\\
 & \lesssim\frac{\sigma}{\log^{1.5}m}\max_{1\leq j\leq m}\left|\boldsymbol{a}_{j}^{\mathsf{H}}(\widetilde{\bm{x}}^{t}-\bm{x}^{\star})\right|,
\end{align*}
where the second inequality follows from \citet[Lemma 50]{ma2017implicit}
and standard sub-Gaussian concentration inequalities.
\item For $\omega_{2}$, similar to \eqref{eq:lem-incohb-sigma2}, we can
invoke the Bernstein inequality \citet[Theorem 2.8.2]{vershynin2018high}
to reach
\begin{align}
 & \mathbb{P}\left\{ \left|\sum_{l=1}^{m}\xi_{l}\left(\boldsymbol{b}_{j}-\boldsymbol{b}_{1}\right)^{\mathsf{H}}\boldsymbol{b}_{l}\boldsymbol{a}_{l}^{\mathsf{H}}\bm{x}^{\star}\right|\geq\tau\right\} \nonumber \\
 & \qquad\leq2\exp\left(-c\min\left(\frac{\tau^{2}}{\sigma^{2}\sum_{l=1}^{m}\left|\left(\boldsymbol{b}_{j}-\boldsymbol{b}_{1}\right)^{\mathsf{H}}\boldsymbol{b}_{l}\right|^{2}},\frac{\tau}{\sigma\max_{1\leq j\leq m}\left|\left(\boldsymbol{b}_{j}-\boldsymbol{b}_{1}\right)^{\mathsf{H}}\boldsymbol{b}_{l}\right|}\right)\right)\label{eq:Bernstein-1234}
\end{align}
for any $\tau\geq0$. In addition, observe that
\begin{align*}
\sum_{j=1}^{m}\left|\left(\boldsymbol{b}_{j}-\boldsymbol{b}_{1}\right)^{\mathsf{H}}\boldsymbol{b}_{l}\right|^{2} & \leq\left\{ \max_{1\leq j\leq m}\left|\left(\boldsymbol{b}_{j}-\boldsymbol{b}_{1}\right)^{\mathsf{H}}\boldsymbol{b}_{l}\right|\right\} \cdot\sum_{j=1}^{m}\left|\left(\boldsymbol{b}_{j}-\boldsymbol{b}_{1}\right)^{\mathsf{H}}\boldsymbol{b}_{l}\right|\\
 & \leq2\frac{K}{m}\cdot\frac{c}{\log^{2}m},
\end{align*}
 where the last inequality follows from \citet[Lemma 48, 49]{ma2017implicit}.
Taking $\tau=C\sigma\sqrt{K\log^{2}m/m}$ in \eqref{eq:Bernstein-1234}
for some large enough constant $C>0$, one arrives at
\[
\mathbb{P}\left\{ \omega_{2}\geq C\sigma\sqrt{\frac{K\log m}{m}}\right\} \leq2\exp\left(-c\min\left(C^{2}\log^{3}m,C\frac{m}{K}\sqrt{\log m}\right)\right)\lesssim m^{-100}.
\]
\item The above bounds taken collectively imply that: with probability exceeding
$1-O\left(m^{-100}\right)$,
\begin{align}
\left|\beta_{3}\right| & \lesssim\frac{\sigma}{\log^{1.5}m}\max_{1\leq j\leq m}\left|\boldsymbol{a}_{j}^{\mathsf{H}}(\widetilde{\bm{x}}^{t}-\bm{x}^{\star})\right|+\sigma\sqrt{\frac{K\log m}{m}}\nonumber \\
 & \lesssim C_{3}\frac{\sigma}{\log^{1.5}m}\left(\sqrt{\frac{\mu^{2}K\log^{2}m}{m}}+\sqrt{\log m}\left(\lambda+\sigma\sqrt{K\log m}\right)\right)+\sigma\sqrt{\frac{K\log m}{m}}\nonumber \\
 & \lesssim\frac{\sigma}{\log^{3}m}.\label{eq:proof-claim-beta3}
\end{align}
\end{itemize}
Putting together the above results, we demonstrate that
\begin{align*}
\left|\left(\boldsymbol{b}_{j}-\boldsymbol{b}_{1}\right)^{\mathsf{H}}\widetilde{\boldsymbol{h}}^{t}\right| & \leq\left(1+\delta\right)\left(1-\eta\lambda-\frac{\eta}{\left|\alpha^{t}\right|^{2}}\left\Vert \widetilde{\boldsymbol{x}}^{t-1}\right\Vert _{2}^{2}\right)\left|\left(\boldsymbol{b}_{j}-\boldsymbol{b}_{1}\right)^{\mathsf{H}}\widetilde{\boldsymbol{h}}^{t-1}\right|+4\left(1+\delta\right)\frac{\eta}{\left|\alpha^{t}\right|^{2}}\frac{\mu}{\sqrt{m}}\\
 & \quad+c\left(1+\delta\right)\frac{\eta}{\left|\alpha^{t}\right|^{2}}\frac{1}{\log m}\left[\max_{1\leq l\leq m}\left|\bm{b}_{l}^{\mathsf{H}}\widetilde{\bm{h}}^{t-1}\right|+\frac{\mu}{\sqrt{m}}\right]+\left(1+\delta\right)\frac{\eta}{\left|\alpha^{t}\right|^{2}}\frac{\sigma}{\log^{3}m}\\
 & \leq cC_{4}\left(\frac{\mu}{\sqrt{m}}\log m+\frac{\sigma}{\log m}\right)
\end{align*}
if $\eta>0$ is sufficiently small, where the last inequality utilizes
$\|\widetilde{\bm{x}}^{t-1}\|_{2}\asymp1$ and $|\alpha^{t}|\asymp1$
in Lemma \ref{lem:consequence}.

\section{Analysis under Fourier design: connections between convex and nonconvex
solutions\label{appendix:cvx}}

\subsection{Proof outline for Theorem \ref{theorem:cvx}\label{sec:Proof-outline-for-theorem-cvx}}

As the empirical evidence (cf.~Figure \ref{fig:dist_cvx_ncvx})
suggests, an approximate nonconvex optimizer produced by a simple
gradient-type algorithm is exceedingly close to the convex minimizer
of (\ref{eq:objcvx}). In what follows, we shall start by introducing
an auxiliary nonconvex gradient method, and formalize its connection
to the convex program. Without loss of generality, we assume that $\Vert\bm{h}^{\star}\Vert_{2}=\Vert\bm{x}^{\star}\Vert_{2}=1$ throughout the proof. 

\paragraph{An auxiliary nonconvex algorithm.} Let us consider the
iterates obtained by running a variant of (Wirtinger) gradient descent,
as summarized in Algorithm \ref{alg:gd-bd-ncvx-2}. A crucial difference
from Algorithm \ref{alg:gd-BD-ncvx} lies in the initialization stage
--- namely, Algorithm \ref{alg:gd-bd-ncvx-2} initializes the algorithm
from the ground truth $(\bm{h}^{\star},\bm{x}^{\star})$ rather than
a spectral estimate as adopted in Algorithm \ref{alg:gd-BD-ncvx}.
While initialization at the truth is not practically implementable,
it is introduced here solely for analytical purpose, namely, it creates a sequence of ancillary random variables that approximate our estimators and are close to the ground truth. This is how we establish the convergence rate of our estimators.

\begin{algorithm}[h]
	\caption{Auxiliary gradient descent for blind deconvolution (for analysis purpose only)}
	
	\label{alg:gd-bd-ncvx-2}\begin{algorithmic}
		
		\STATE \textbf{{Input}}: $\left\{ \bm{a}_{j}\right\} _{1\leq j\leq m}$,
		$\left\{ \bm{b}_{j}\right\} _{1\leq j\leq m}$, $\left\{ y_{j}\right\} _{1\leq j\leq m}$,
		$\bm{h}^{\star}$ and $\bm{x}^{\star}$.
		
		\STATE \textbf{{Initialization}}: $\bm{h}^{0}=\bm{h}^{\star}$
		and $\bm{x}^{0}=\bm{x}^{\star}$.
		
		\STATE \textbf{{Gradient updates}}: \textbf{for }$t=0,1,\ldots,t_{0}-1$
		\textbf{do}
		
		\STATE \vspace{-1em}
		\begin{subequations}\label{subeq:gradient_update_ncvx-2} 
			\begin{align}
			\left[\begin{array}{c}
			\boldsymbol{h}^{t+1/2}\\
			\boldsymbol{x}^{t+1/2}
			\end{array}\right]= & \left[\begin{array}{c}
			\boldsymbol{h}^{t}\\
			\boldsymbol{x}^{t}
			\end{array}\right]-\eta\left[\begin{array}{c}
			\nabla_{\boldsymbol{h}}f\left(\boldsymbol{h}^{t},\boldsymbol{x}^{t}\right)\\
			\nabla_{\boldsymbol{x}}f\left(\boldsymbol{h}^{t},\boldsymbol{x}^{t}\right)
			\end{array}\right],\\
			\left[\begin{array}{c}
			\boldsymbol{h}^{t+1}\\
			\boldsymbol{x}^{t+1}
			\end{array}\right]= & \left[\begin{array}{c}
			\sqrt{\frac{\left\Vert \bm{x}^{t+1/2}\right\Vert _{2}}{\left\Vert \bm{h}^{t+1/2}\right\Vert _{2}}}\boldsymbol{h}^{t+1/2}\\
			\sqrt{\frac{\left\Vert \bm{h}^{t+1/2}\right\Vert _{2}}{\left\Vert \bm{x}^{t+1/2}\right\Vert _{2}}}\boldsymbol{x}^{t+1/2}
			\end{array}\right],
			\end{align}
		\end{subequations}where $\nabla_{\boldsymbol{h}}f(\cdot)$ and $\nabla_{\boldsymbol{x}}f(\cdot)$
		represent the Wirtinger gradient (see \cite[Section 3.3]{li2019rapid}
		and Appendix~\ref{subsec:Wirtinger-calculus}) of $f(\cdot)$ w.r.t.~$\bm{h}$
		and $\bm{x}$, respectively.
		
	\end{algorithmic}
\end{algorithm}

\paragraph{Properties of the auxiliary nonconvex algorithm.} The
trajectory of this auxiliary nonconvex algorithm enjoys several important
properties. In the following lemma, the results are stated for the
properly rescaled iterate
\[
\widetilde{\bm{z}}^{t}=\left(\widetilde{\bm{h}}^{t},\widetilde{\bm{x}}^{t}\right)\coloneqq\left(\frac{1}{\overline{\alpha^{t}}}\bm{h}^{t},\alpha^{t}\bm{x}^{t}\right),
\]
with alignment parameter defined by 
\[
\alpha^{t}\coloneqq\arg\min_{\alpha\in\mathbb{C}}\left\{ \left\Vert \frac{1}{\overline{\alpha}}\bm{h}^{t}-\bm{h}^{\star}\right\Vert _{2}^{2}+\left\Vert \alpha\bm{x}^{t}-\bm{x}^{\star}\right\Vert _{2}^{2}\right\} .
\]

\begin{lemma}\label{lemma:cvx-noncvx}Take $\lambda=C_{\lambda}\sigma\sqrt{K\log m}$
	for some large enough constant $C_{\lambda}>0$. Assume the number
	of measurements obeys $m\geq C\mu^{2}K\log^{9}m$ for some sufficiently
	large constant $C>0$, and the noise satisfies $\sigma\sqrt{K\log m}\leq c/\log^{2}m$
	for some sufficiently small constant $c>0$. Then, with probability
	at least $1-O\left(m^{-100}+me^{-cK}\right)$ for some constant $c>0$,
	the iterates $\left\{ \boldsymbol{h}^{t},\boldsymbol{x}^{t}\right\} _{0<t\leq t_{0}}$
	of Algorithm (\ref{alg:gd-bd-ncvx-2}) satisfy \begin{subequations}\label{eq:intro-induction}
		\begin{align}
		\mathsf{dist}\left(\boldsymbol{z}^{t},\boldsymbol{z}^{\star}\right) & \leq\rho\mathsf{dist}\left(\boldsymbol{z}^{t-1},\boldsymbol{z}^{\star}\right)+C_{5}\eta\left(\lambda+\sigma\sqrt{K\log m}\right)\label{eq:induction1}\\
		\max_{1\leq j\leq m}\left|\boldsymbol{a}_{j}^{\mathsf{H}}\big(\widetilde{\boldsymbol{x}}^{t}-\boldsymbol{x}^{\star}\big)\right| & \leq C_{7}\sqrt{\log m}\left(\lambda+\sigma\sqrt{K\log m}\right)\label{eq:induction3}\\
		\max_{1\leq j\leq m}\left|\boldsymbol{b}_{j}^{\mathsf{H}}\widetilde{\boldsymbol{h}}^{t}\right| & \leq C_{8}\left(\frac{\mu}{\sqrt{m}}\log m+\sigma\right)\label{eq:induction4}\\
		\max_{1\leq j\leq m}\left|\boldsymbol{b}_{j}^{\mathsf{H}}\big(\widetilde{\bm{h}}^{t}-\bm{h}^{\star}\big)\right| & \leq C_{9}\sigma\label{eq:incohc}
		\end{align}
		for any $0<t\leq t_{0}$, where $\rho=1-c_{\rho}\eta\in\left(0,1\right)$
		for some small constant $c_{\rho}>0$, and we take $t_{0}=m^{20}$.
		Here, $C_{5}$, $\ldots$, $C_{9}$ are constants obeying $C_{7}\gg C_{5}$.
		In addition, we have 
		\begin{equation}
		\min_{0\leq t\leq t_{0}}\left\Vert \nabla f\left(\boldsymbol{h}^{t},\boldsymbol{x}^{t}\right)\right\Vert _{2}\leq\frac{\lambda}{m^{10}}.\label{eq:smallgradient}
		\end{equation}
\end{subequations}\end{lemma}

Most of the inequalities of this lemma (as well as their proofs) resemble
the ones derived for Algorithm \ref{alg:gd-BD-ncvx} in Appendix \ref{appendix:noncvx}.
It is worth emphasizing, however, that the establishment of the inequality
(\ref{eq:incohc}) relies heavily on the idealized initialization
$(\bm{h}^{0},\bm{x}^{0})=(\bm{h}^{\star},\bm{x}^{\star})$, and the
current proof does not work if the algorithm is spectrally initialized.
The proof of this lemma is deferred to Appendix \ref{subsec:Proof-of-Lemmacvx-ncvx}. 

\paragraph{Connection between the approximate nonconvex minimizer and the convex solution.}

As it turns out, the above type of features of the nonconvex iterates
together with the first-order optimality of the convex program allows
us to control the proximity of the convex minimizer and the approximate
nonconvex optimizer. Before proceeding to develop this idea formally, we first introduce the following operators for notational convenience. For any $\bm{z}=[z_{j}]_{1\leq j\leq m}$ and any $\bm{Z}\in\mathbb{C}^{K\times K}$, we define
\begin{align}
\mathcal{A}\left(\bm{Z}\right) & \coloneqq\left\{ \bm{b}_{j}^{\mathsf{H}}\bm{Z}\bm{a}_{j}\right\} _{j=1}^{m},\qquad\mathcal{A}^{*}\left(\bm{z}\right)=\sum_{j=1}^{m}z_{j}\bm{b}_{j}\bm{a}_{j}^{\mathsf{H}},\nonumber \\
\mathcal{T}\left(\boldsymbol{Z}\right) & \coloneqq\mathcal{A}^{*}\mathcal{A}\left(\bm{Z}\right)=\sum_{j=1}^{m}\boldsymbol{b}_{j}\boldsymbol{b}_{j}^{\mathsf{H}}\bm{Z}\boldsymbol{a}_{j}\boldsymbol{a}_{j}^{\mathsf{H}}.\label{eq:operatordef}
\end{align}
Below are several key conditions on these operators concerned with the interplay between the noise size, the estimation
accuracy of the nonconvex estimate $(\bm{h}, \bm{x})$ and the regularization parameters $\lambda$. 

\begin{condition}\label{condition:fourier-lambda}
	The regularization parameter $\lambda$ satisfies
	\begin{enumerate}
		\item $
		\| \mathcal{T}\left(\boldsymbol{h}\boldsymbol{x}^{\mathsf{H}}-\boldsymbol{h}^{\star}\boldsymbol{x}^{\star\mathsf{H}}\right)-\left(\boldsymbol{h}\boldsymbol{x}^{\mathsf{H}}-\boldsymbol{h}^{\star}\boldsymbol{x}^{\star\mathsf{H}}\right)\| <\lambda/8.
		$
		\item $
		\left\Vert \mathcal{A}^{*}\left(\bm{\xi}\right)\right\Vert =\|\sum_{j=1}^{m}\xi_{j}\boldsymbol{b}_{j}\boldsymbol{a}_{j}^{\mathsf{H}}\|\leq c\lambda,
		$ for some small constant $c>0$. 
	\end{enumerate}
\end{condition}
Condition \ref{condition:fourier-lambda} requires that the regularization parameter $\lambda$ dominate the norm of the deviation of $\mathcal{T}(\boldsymbol{h}\boldsymbol{x}^{\mathsf{H}}-\boldsymbol{h}^{\star}\boldsymbol{x}^{\star\mathsf{H}})$ from its mean $\boldsymbol{h}\boldsymbol{x}^{\mathsf{H}}-\boldsymbol{h}^{\star}\boldsymbol{x}^{\star\mathsf{H}}$, and also the norm of the noise operated on by $\mathcal{A}^*$. As can be seen shortly, these two conditions can be met with high probability when $(\boldsymbol{h}, \boldsymbol{x})$ is sufficiently close to $(\boldsymbol{h}^{\star}, \boldsymbol{x}^{\star})$. 

Another critical condition is the following injectivity condition on $\mathcal{A}$. 
\begin{condition}\label{condition:fourier-inj}
	Let $T$ be the tangent space of $\boldsymbol{h}\boldsymbol{x}^{\mathsf{H}}$. Then for all $\boldsymbol{Z}\in T$, one has
	\[
	\left\Vert \mathcal{A}\left(\boldsymbol{Z}\right)\right\Vert _{2}^{2}\geq\frac{1}{16}\left\Vert \boldsymbol{Z}\right\Vert _{\mathrm{F}}^{2}.
	\]
\end{condition}

When these two conditions hold, the aforementioned intimate connection between approximate nonconvex minimizer and the convex solution can be formalized in the following crucial lemma. 

\begin{lemma}\label{lemma:proximity-convex-ncvx}
	Suppose that $\left(\boldsymbol{h},\boldsymbol{x}\right)$ obeys\begin{subequations}\label{eq:hx-properties-all}
		\begin{align}
		\left\Vert \nabla f\left(\boldsymbol{h},\boldsymbol{x}\right)\right\Vert _{2}\leq C\frac{\lambda}{m^{10}},\label{eq:small-gradient}
		\end{align}
	\end{subequations}for some constants $C>0$. Then under Conditions \ref{condition:fourier-lambda} and \ref{condition:fourier-inj},
	any minimizer $\boldsymbol{Z}_{\mathsf{cvx}}$ of the convex problem
	(\ref{eq:objcvx}) satisfies
	\[
	\left\Vert \boldsymbol{h}\boldsymbol{x}^{\mathsf{H}}-\boldsymbol{Z}_{\mathsf{cvx}}\right\Vert _{\mathrm{F}}\lesssim\left\Vert \nabla f\left(\boldsymbol{h},\boldsymbol{x}\right)\right\Vert _{2}.
	\]
\end{lemma}\begin{proof}See Appendix \ref{subsec:Proof-of-Lemma2}.\end{proof}
In words, if we can find a point $(\bm{h},\bm{x})$ that has vanishingly
small gradient (cf.~(\ref{eq:small-gradient})) and that satisfies the additional Conditions \ref{condition:fourier-lambda} and \ref{condition:fourier-inj}, 
then the matrix $\bm{h}\bm{x}^{\mathsf{H}}$ is guaranteed to be exceedingly
close to the solution of the convex program. Encouragingly, Lemma
\ref{lemma:cvx-noncvx} hints at the existence of a point along the
trajectory of Algorithm (\ref{alg:gd-bd-ncvx-2}) satisfying these
conditions (\ref{eq:hx-properties-all}); if this were true, then
one could transfer the properties of the approximate nonconvex optimizer
to the convex solution, as a means to certify the statistical efficiency
of convex programming. As we will see soon, this is indeed the case that with Assumption \ref{assumptions:models}, we can prove that under some mild sample size and noise level conditions, Conditions \ref{condition:fourier-lambda} and \ref{condition:fourier-inj} would hold with high probability. To begin with, the following lemma corresponds to the first point in Condition \ref{condition:fourier-lambda}.

\begin{lemma}\label{lemma:T-uniform-mean}Suppose that the sample
	complexity satisfies $m\geq C\mu^{2}K\log^{4}m$ for some sufficiently
	large constant $C>0$. Take $\lambda=C_{\lambda}\sigma\sqrt{K\log m}$
	for some large enough constant $C_{\lambda}>0$. Then with probability
	at least $1-O\left(m^{-10}+me^{-CK}\right)$, we have
	\[
	\left\Vert \mathcal{T}\left(\boldsymbol{h}\boldsymbol{x}^{\mathsf{H}}-\boldsymbol{h}^{\star}\boldsymbol{x}^{\star\mathsf{H}}\right)-\left(\boldsymbol{h}\boldsymbol{x}^{\mathsf{H}}-\boldsymbol{h}^{\star}\boldsymbol{x}^{\star\mathsf{H}}\right)\right\Vert <\lambda/8,
	\]
	simultaneously for any $\left(\boldsymbol{h},\boldsymbol{x}\right)$
	obeying \begin{subequations}\label{eq:hx-properties-all}
		\begin{align}
		\left\Vert \bm{h}\right\Vert _{2}=\left\Vert \bm{x}\right\Vert _{2},\quad\left\Vert \boldsymbol{h}-\boldsymbol{h}^{\star}\right\Vert _{2} & \leq\frac{C_{5}}{1-\rho}\eta\left(\lambda+\sigma\sqrt{K\log m}\right),\quad\left\Vert \boldsymbol{x}-\boldsymbol{x}^{\star}\right\Vert _{2}\leq\frac{C_{5}}{1-\rho}\eta\left(\lambda+\sigma\sqrt{K\log m}\right),\label{eq:hx-properties}\\
		\max_{1\leq j\leq m}\left|\boldsymbol{b}_{j}^{\mathsf{H}}\left(\bm{h}-\bm{h}^{\star}\right)\right| & \leq C_{9}\sigma\quad\text{and}\quad\max_{1\leq j\leq m}\left|\boldsymbol{a}_{j}^{\mathsf{H}}\left(\bm{x}-\boldsymbol{x}^{\star}\right)\right|\leq C_{7}\sqrt{\log m}\left(\lambda+\sigma\sqrt{K\log m}\right),\label{eq:hx-properties2}
		\end{align}
	\end{subequations}for some constants $C_{5},C_{7},C_{9}>0$. 
\end{lemma}\begin{proof}See Appendix \ref{subsec:Proof-of-Lemma8innoisy}.\end{proof}
Recall the definition of operator $\mathcal{T}$ in \eqref{eq:operatordef}.
The lemma above states that for all $(\bm{h},\bm{x})$ sufficiently
close to $(\bm{h}^{\star},\bm{x}^{\star})$, the matrix $\mathcal{T}\left(\boldsymbol{h}\boldsymbol{x}^{\mathsf{H}}-\boldsymbol{h}^{\star}\boldsymbol{x}^{\star\mathsf{H}}\right)$
is close to the expectation $\boldsymbol{h}\boldsymbol{x}^{\mathsf{H}}-\boldsymbol{h}^{\star}\boldsymbol{x}^{\star\mathsf{H}}$. 

Next we turn to the second point in Condition \ref{condition:fourier-lambda}. 
\begin{lemma}\label{lemma:noise}
	Suppose that Asumption \ref{assumptions:models} holds and $m\gtrsim K\log^{3}m$. With probability at least $1-O\left(m^{-100}\right)$,
	one has 
	\[
	\left\Vert \mathcal{A}^{*}\left(\bm{\xi}\right)\right\Vert =\Bigg\|\sum_{j=1}^{m}\xi_{j}\boldsymbol{b}_{j}\boldsymbol{a}_{j}^{\mathsf{H}}\Bigg\|\lesssim\sigma\sqrt{K\log m}.
	\]
\end{lemma}\begin{proof}See Appendix \ref{subsec:Proof-oflemmanoise}.\end{proof}

Regarding Condition \ref{condition:fourier-inj}, we have the following lemma.
\begin{lemma}\label{lemma:inj}Suppose that the sample complexity
	satisfies $m\geq C\mu^{2}K\log m$ for some sufficiently large constant
	$C>0$. Then with probability at least $1-O\left(m^{-10}\right)$,
	\[
	\left\Vert \mathcal{A}\left(\boldsymbol{Z}\right)\right\Vert _{2}^{2}\geq\frac{1}{16}\left\Vert \boldsymbol{Z}\right\Vert _{\mathrm{F}}^{2},\quad\forall\boldsymbol{Z}\in T
	\]
	holds simultaneously for all $T$ for which the associated point $\left(\boldsymbol{h},\boldsymbol{x}\right)$
	obeys \eqref{eq:hx-properties} and \eqref{eq:hx-properties2}. Here,
	$T$ denotes the tangent space of $\boldsymbol{h}\boldsymbol{x}^{\mathsf{H}}$.\end{lemma}\begin{proof}See
	Appendix \ref{subsec:Proof-of-Lemma-inj}.\end{proof}
Basically, this lemma
reveals that when $(\bm{h},\bm{x})$ is sufficiently close to $(\bm{h}^{\star},\bm{x}^{\star})$,
the operator $\mathcal{A}(\cdot)$ --- restricted to the tangent
space $T$ of $\bm{h}\bm{x}^{\mathsf{H}}$ --- is injective. 


Now we are ready to present the proof of Theorem \ref{theorem:cvx}.

\paragraph{Proof of Theorem \ref{theorem:cvx}.} Armed with this
result and the properties about the nonconvex trajectory, we are ready
to establish Theorem \ref{theorem:cvx} as follows. Let $\overline{t}\coloneqq\arg\min_{0\leq t\leq t_{0}}\left\Vert \nabla f\left(\boldsymbol{h}^{t},\boldsymbol{x}^{t}\right)\right\Vert _{\text{F}}$,
and take $\left(\boldsymbol{h}_{\mathsf{ncvx}},\boldsymbol{x}_{\mathsf{ncvx}}\right)=\left(\frac{1}{\overline{\alpha^{\overline{t}}}}\boldsymbol{h}^{\overline{t}},\alpha^{\overline{t}}\boldsymbol{x}^{\overline{t}}\right)$.
By virtue of Lemma \ref{lemma:cvx-noncvx}, we see that $\left(\boldsymbol{h}_{\mathsf{ncvx}},\boldsymbol{x}_{\mathsf{ncvx}}\right)$
satisfies --- with high probability --- the small gradient property
\eqref{eq:smallgradient} as well as all conditions required to invoke
Lemma \ref{lemma:proximity-convex-ncvx}. As a consequence, invoke
Lemma \ref{lemma:proximity-convex-ncvx} to obtain
\begin{equation}
\left\Vert \boldsymbol{Z}_{\mathsf{cvx}}-\boldsymbol{h}_{\mathsf{ncvx}}\boldsymbol{x}_{\mathsf{ncvx}}^{\mathsf{H}}\right\Vert _{\text{F}}\lesssim\frac{1}{c_{\text{inj}}}\left\Vert \nabla f\left(\boldsymbol{h}_{\mathsf{ncvx}},\boldsymbol{x}_{\mathsf{ncvx}}\right)\right\Vert _{\text{F}}\lesssim\text{\ensuremath{\frac{\lambda}{m^{10}}}}.\label{eq:proof-thm-cvx-1}
\end{equation}
Further, it is seen that
\begin{align}
\left\Vert \bm{h}_{\mathsf{ncvx}}\big(\bm{x}_{\mathsf{ncvx}}\big)^{\mathsf{H}}-\boldsymbol{h}^{\star}\boldsymbol{x}^{\mathsf{\star H}}\right\Vert _{\mathrm{F}} & \leq\left\Vert \bm{h}_{\mathsf{ncvx}}\big(\bm{x}_{\mathsf{ncvx}}\big)^{\mathsf{H}}-\bm{h}^{\star}\big(\bm{x}_{\mathsf{ncvx}}\big)^{\mathsf{H}}\right\Vert _{\mathrm{F}}+\left\Vert \bm{h}^{\star}\big(\bm{x}_{\mathsf{ncvx}}\big)^{\mathsf{H}}-\boldsymbol{h}^{\star}\boldsymbol{x}^{\mathsf{\star H}}\right\Vert _{\mathrm{F}}\nonumber \\
& \leq\left\Vert \bm{h}_{\mathsf{ncvx}}-\bm{h}^{\star}\right\Vert _{2}\left\Vert \bm{x}_{\mathsf{ncvx}}\right\Vert _{2}+\left\Vert \bm{h}^{\star}\right\Vert _{2}\left\Vert \bm{x}_{\mathsf{ncvx}}-\boldsymbol{x}^{\mathsf{\star}}\right\Vert _{2}\nonumber \\
& \leq2\left\Vert \bm{z}^{\star}\right\Vert _{2}\cdot\frac{C_{5}\eta}{\left(1-\rho\right)\left\Vert \bm{z}^{\star}\right\Vert _{2}}\left(\lambda+\sigma\sqrt{K\log m}\right)\nonumber \\
& =\frac{2C_{5}}{c_{\rho}}\left(\lambda+\sigma\sqrt{K\log m}\right),\label{eq:proof-thm-cvx-2}
\end{align}
where the penultimate line follows from (\ref{eq:dist-bound-1}) and
the inequality
\[
\left\Vert \bm{x}_{\mathsf{ncvx}}\right\Vert _{2}\leq\left\Vert \bm{x}^{\star}\right\Vert _{2}+\left\Vert \bm{x}_{\mathsf{ncvx}}-\bm{x}^{\star}\right\Vert _{2}\leq\left\Vert \bm{z}^{\star}\right\Vert _{2}+\frac{C_{5}\eta}{\left(1-\rho\right)\left\Vert \bm{z}^{\star}\right\Vert _{2}}\left(\lambda+\sigma\sqrt{K\log m}\right)\leq2\left\Vert \bm{z}^{\star}\right\Vert _{2}.
\]
Taking (\ref{eq:proof-thm-cvx-1}) and (\ref{eq:proof-thm-cvx-2})
collectively yields
\begin{align*}
\left\Vert \boldsymbol{Z}_{\mathsf{cvx}}-\boldsymbol{h}^{\star}\boldsymbol{x}^{\mathsf{\star H}}\right\Vert _{\text{F}} & \leq\left\Vert \boldsymbol{Z}_{\mathsf{cvx}}-\boldsymbol{h}_{\mathsf{ncvx}}\boldsymbol{x}_{\mathsf{ncvx}}^{\mathsf{H}}\right\Vert _{\text{F}}+\left\Vert \boldsymbol{h}_{\mathsf{ncvx}}\boldsymbol{x}_{\mathsf{ncvx}}^{\mathsf{H}}-\boldsymbol{h}^{\star}\boldsymbol{x}^{\mathsf{\star H}}\right\Vert _{\text{F}}\\
& \lesssim\frac{\lambda}{m^{10}}+\lambda+\sigma\sqrt{K\log m}\\
& \lesssim\lambda+\sigma\sqrt{K\log m}.
\end{align*}
This together with the elementary bound $\left\Vert \boldsymbol{Z}_{\mathsf{cvx}}-\boldsymbol{h}^{\star}\boldsymbol{x}^{\mathsf{\star H}}\right\Vert \leq\left\Vert \boldsymbol{Z}_{\mathsf{cvx}}-\boldsymbol{h}^{\star}\boldsymbol{x}^{\mathsf{\star H}}\right\Vert _{\text{F}}$
concludes the proof, as long as the above key lemmas can be justified. 

To prove the results also holds for $\boldsymbol{Z}_{\mathsf{cvx,}1}$, we recall that $\boldsymbol{Z}_{\mathsf{cvx,}1}$ is the best rank-1 approximation of $\boldsymbol{Z}_{\mathsf{cvx}}$ and this implies that,
\[
\left\Vert \boldsymbol{Z}_{\mathsf{cvx}}-\boldsymbol{Z}_{\mathsf{cvx,}1}\right\Vert _{\text{F}} \leq \left\Vert \boldsymbol{Z}_{\mathsf{cvx}}-\boldsymbol{h}_{\mathsf{ncvx}}\boldsymbol{x}_{\mathsf{ncvx}}^{\mathsf{H}}\right\Vert _{\text{F}}\lesssim\text{\ensuremath{\frac{\lambda}{m^{10}}}}.
\]
Hence, repeating the above calculations for $\boldsymbol{Z}_{\mathsf{cvx,}1}$ reveals that \eqref{eq:thm:gaussian-cvx} continues to holds if $\boldsymbol{Z}_{\mathsf{cvx}}$ is replaced by $\boldsymbol{Z}_{\mathsf{cvx,}1}$.

In what follows, we establish the key lemmas stated above. 

\subsection{Preliminary facts}

Before proceeding, there are a couple of immediate consequences of
Lemma \ref{lemma:cvx-noncvx} that will prove useful, which we summarize
as follows. 

\begin{lemma}\label{lem:consequence-cvx}Instate the notation and
assumptions in Theorem \ref{thm:nonconvex}. For $t\geq0$, suppose
that the hypotheses \eqref{eq:induction} hold in the first $t$ iterations.
Then there exist some constants $C_{5}>0$ such that for any $1\le l\leq m$,
\begin{subequations}
\begin{align}
\mathsf{dist}\left(\boldsymbol{z}^{t},\boldsymbol{z}^{\star}\right) & \leq\frac{C_{5}}{c_{\rho}}\left(\lambda+\sigma\sqrt{K\log m}\right),\label{eq:dist-bound-1}\\
\big\|\widetilde{\boldsymbol{z}}^{t,\left(l\right)}-\boldsymbol{z}^{\star}\big\|_{2} & \leq2\frac{C_{5}}{c_{\rho}}\left(\lambda+\sigma\sqrt{K\log m}\right),\label{eq:conseq-2-1}\\
\frac{1}{2}\leq\left\Vert \widetilde{\bm{x}}^{t}\right\Vert _{2}\leq\frac{3}{2}, & \qquad\frac{1}{2}\le\big\|\widetilde{\bm{h}}^{t}\big\|_{2}\leq\frac{3}{2},\label{eq:tilde-hx-1}\\
\frac{1}{2}\le\big\|\widetilde{\bm{x}}^{t,\left(l\right)}\big\|_{2}\leq\frac{3}{2}, & \qquad\frac{1}{2}\le\big\|\widetilde{\bm{h}}^{t,\left(l\right)}\big\|_{2}\leq\frac{3}{2},\label{eq:tilde-hx-loo-1}\\
\frac{1}{2}\le\big\|\widehat{\bm{x}}^{t,\left(l\right)}\big\|_{2}\leq\frac{3}{2}, & \qquad\frac{1}{2}\le\big\|\widehat{\bm{h}}^{t,\left(l\right)}\big\|_{2}\leq\frac{3}{2},\label{eq:hat-hx-loo-1}\\
\left\Vert \bm{h}^{t}\right\Vert _{2}^{2}=\left\Vert \bm{x}^{t}\right\Vert _{2}^{2}=\left\Vert \bm{h}^{t}\right\Vert _{2}\left\Vert \bm{x}^{t}\right\Vert _{2} & =\big\|\widetilde{\bm{h}}^{t-1/2}\big\|_{2}\big\|\widetilde{\bm{x}}^{t-1/2}\big\|_{2}=\big\|\widetilde{\bm{h}}^{t}\big\|_{2}\big\|\widetilde{\bm{x}}^{t}\big\|_{2}.\label{eq:balance}
\end{align}
In addition, for an integer $t>0$, suppose that the hypotheses \eqref{eq:induction}
hold in the first $t-1$ iterations. Then there exists some constant
$C>0$ such that with probability at least $1-O\left(m^{-100}+e^{-CK}\log m\right)$,
there holds
\begin{align}
\big\|\widehat{\bm{z}}^{t}-\boldsymbol{z}^{\star}\big\|_{2} & \leq\frac{C_{5}}{c_{\rho}}\left(\lambda+\sigma\sqrt{K\log m}\right),\label{eq:conseq-3-1}\\
\left|\left|\alpha^{t}\right|-1\right| & \lesssim\frac{C_{5}}{c_{\rho}}\left(\lambda+\sigma\sqrt{K\log m}\right),\label{eq:alpha-asymp1-1}\\
\left|\frac{\alpha^{t-1/2}}{\alpha^{t-1}}-1\right| & \lesssim\eta\frac{C_{5}}{c_{\rho}}\left(\lambda+\sigma\sqrt{K\log m}\right),\label{eq:alpharatio-round1-1}\\
\left|\alpha^{t-1/2}-\alpha^{t-1}\right| & \lesssim\eta\frac{C_{5}}{c_{\rho}}\left(\lambda+\sigma\sqrt{K\log m}\right),\label{eq:alphadiff-1}\\
\frac{1}{2} & \leq\left|\frac{\alpha^{t-1}}{\alpha^{t-1/2}}\right|\le\frac{3}{2},\label{eq:alpha-round1-1}\\
\frac{1}{2} & \leq\left|\alpha^{t}\right|\le\frac{3}{2}.\label{eq:alpha-bound}
\end{align}
\end{subequations} \end{lemma}\begin{proof}The proof follows from
the same argument as in the proof of Lemma \ref{lem:consequence}
and Corollary \ref{corollary:alpha}, and is thus omitted here for
brevity.\end{proof}

\subsection{Proof of Lemma \ref{lemma:cvx-noncvx}\label{subsec:Proof-of-Lemmacvx-ncvx}}

After the introduction of the proof idea in Appendix \ref{appendix:noncvx},
we state a more complete version of Lemma \ref{lemma:cvx-noncvx}
here.

\begin{lemma}\label{lemma:cvx-noncvx-1}Take $\lambda=C_{\lambda}\sigma\sqrt{K\log m}$
for some large enough constant $C_{\lambda}>0$. Assume the number
of measurements obeys $m\geq C\mu^{2}K\log^{9}m$ for some sufficiently
large constant $C>0$, and the noise satisfies $\sigma\sqrt{K\log m}\leq c/\log^{2}m$
for some sufficiently small constant $c>0$. Then, with probability
at least $1-O\left(m^{-100}+me^{-cK}\right)$ for some constant $c>0$,
the iterates $\left\{ \boldsymbol{h}^{t},\boldsymbol{x}^{t}\right\} _{0<t\leq t_{0}}$
of Algorithm \eqref{alg:gd-bd-ncvx-2} satisfy \begin{subequations}\label{eq:induction}
\begin{align}
\mathsf{dist}\left(\boldsymbol{z}^{t},\boldsymbol{z}^{\star}\right) & \leq\rho\mathsf{dist}\left(\boldsymbol{z}^{t-1},\boldsymbol{z}^{\star}\right)+C_{5}\eta\left(\lambda+\sigma\sqrt{K\log m}\right)\label{eq:induction1-1}\\
\max_{1\leq l\leq m}\mathsf{dist}\big(\boldsymbol{z}^{t,\left(l\right)},\widetilde{\boldsymbol{z}}^{t}\big) & \leq C_{6}\frac{\sigma}{\log^{2}m}\label{eq:induction2-1}\\
\max_{1\leq l\leq m}\big\|\widetilde{\bm{z}}^{t,\left(l\right)}-\widetilde{\bm{z}}^{t}\big\|_{2} & \lesssim C_{6}\frac{\sigma}{\log^{2}m}\label{eq:induction-2-1}\\
\max_{1\leq j\leq m}\left|\boldsymbol{a}_{j}^{\mathsf{H}}\big(\widetilde{\boldsymbol{x}}^{t}-\boldsymbol{x}^{\star}\big)\right| & \leq C_{7}\sqrt{\log m}\left(\lambda+\sigma\sqrt{K\log m}\right)\label{eq:induction3-1}\\
\max_{1\leq j\leq m}\left|\boldsymbol{b}_{j}^{\mathsf{H}}\widetilde{\boldsymbol{h}}^{t}\right| & \leq C_{8}\left(\frac{\mu}{\sqrt{m}}\log m+\sigma\right)\label{eq:induction4-1}\\
\max_{1\leq j\leq m}\left|\boldsymbol{b}_{j}^{\mathsf{H}}\big(\widetilde{\bm{h}}^{t}-\bm{h}^{\star}\big)\right| & \leq C_{9}\sigma\label{eq:incohc-1}
\end{align}
for any $0<t\leq t_{0}$, where $\rho=1-c_{\rho}\eta\in\left(0,1\right)$
for some small constant $c_{\rho}>0$, and we take $t_{0}=m^{20}$.
Here, $C_{5}$, $\ldots$, $C_{9}$ are constants obeying $C_{7}\gg C_{5}$.
In addition, we have 
\begin{equation}
\min_{0\leq t\leq t_{0}}\left\Vert \nabla f\left(\boldsymbol{h}^{t},\boldsymbol{x}^{t}\right)\right\Vert _{2}\leq\frac{\lambda}{m^{10}}.\label{eq:smallgradient-1}
\end{equation}
\end{subequations}\end{lemma}

The claims \eqref{eq:induction1-1}-\eqref{eq:induction4-1} are direct
consequences of Lemma \ref{lemma:distance}, Lemma \ref{lemma:proximity},
the relation \eqref{eq:incoherencea-proof}, and Lemma \ref{lemma:incoherenceb}.
As a result, the remaining steps lie in proving \eqref{eq:incohc}
and \eqref{eq:smallgradient}.

\subsubsection{Proof of the claim \eqref{eq:incohc}\label{subsec:Proof-of-Lemmaincoherencec}}

Recall the definition $\widetilde{\bm{h}}^{t}\coloneqq\boldsymbol{h}^{t}/\overline{\alpha^{t}}.$
We aim to prove inductively that
\begin{equation}
\max_{1\leq j\leq m}\left|\bm{b}_{j}^{\mathsf{H}}\big(\widetilde{\bm{h}}^{t}-\bm{h}^{\star}\big)\right|\leq C_{9}\sigma\label{eq:induction-5}
\end{equation}
holds for some constant $C_{9}>0$, provided that the algorithm is
initialized at the truth. 

It is self-evident that \eqref{eq:induction-5} holds for the base
case (i.e.~$t=0$) when $\bm{h}^{0}=\bm{h}^{\star}$. Assume for
the moment that \eqref{eq:induction-5} holds true at the $t$th iteration.
In view of the simple relation between $\alpha^{t+1}$ and $\alpha^{t+1/2}$
in \eqref{eq:balance-relation-1} and the balancing step \eqref{subeq:gradient_update_ncvx-2},
one has
\[
\alpha^{t+1}=\sqrt{\frac{\left\Vert \bm{x}^{t+1/2}\right\Vert _{2}}{\left\Vert \bm{h}^{t+1/2}\right\Vert _{2}}}\,\alpha^{t+1/2},\qquad\text{and}\qquad\bm{h}^{t+1}=\sqrt{\frac{\left\Vert \bm{x}^{t+1/2}\right\Vert _{2}}{\left\Vert \bm{h}^{t+1/2}\right\Vert _{2}}}\,\boldsymbol{h}^{t+1/2}.
\]
It then follows that $\bm{h}^{t+1}/\overline{\alpha^{t+1}}=\bm{h}^{t+1/2}/\overline{\alpha^{t+1/2}}$
and, therefore,
\begin{align}
 & \frac{\overline{\alpha^{t+1/2}}}{\overline{\alpha^{t}}}\left(\frac{\bm{h}^{t+1}}{\overline{\alpha^{t+1}}}-\bm{h}^{\star}\right)=\frac{\overline{\alpha^{t+1/2}}}{\overline{\alpha^{t}}}\left(\frac{\bm{h}^{t+1/2}}{\overline{\alpha^{t+1/2}}}-\bm{h}^{\star}\right)\nonumber \\
 & \overset{(\text{i})}{=}\frac{\overline{\alpha^{t+1/2}}}{\overline{\alpha^{t}}}\left(\frac{1}{\overline{\alpha^{t+1/2}}}\left(\bm{h}^{t}-\eta\nabla_{\boldsymbol{h}}f\left(\boldsymbol{h}^{t},\boldsymbol{x}^{t}\right)\right)-\bm{h}^{\star}\right)\nonumber \\
 & =\widetilde{\bm{h}}^{t}-\frac{\eta}{\left|\alpha^{t}\right|^{2}}\nabla_{\boldsymbol{h}}f\big(\widetilde{\boldsymbol{h}}^{t},\widetilde{\boldsymbol{x}}^{t}\big)-\frac{\overline{\alpha^{t+1/2}}}{\overline{\alpha^{t}}}\bm{h}^{\star}\nonumber \\
 & \overset{(\text{ii})}{=}\left(1-\eta\lambda-\frac{\overline{\alpha^{t+1/2}}}{\overline{\alpha^{t}}}\right)\bm{h}^{\star}+\left(1-\eta\lambda\right)\left(\widetilde{\bm{h}}^{t}-\bm{h}^{\star}\right)-\frac{\eta}{\left|\alpha^{t}\right|^{2}}\sum_{j=1}^{m}\boldsymbol{b}_{j}\boldsymbol{b}_{j}^{\mathsf{H}}\left(\widetilde{\bm{h}}^{t}\widetilde{\boldsymbol{x}}^{t\mathsf{H}}-\boldsymbol{h}^{\star}\boldsymbol{x}^{\star\mathsf{H}}\right)\boldsymbol{a}_{j}\boldsymbol{a}_{j}^{\mathsf{H}}\widetilde{\boldsymbol{x}}^{t}+\frac{\eta}{\left|\alpha^{t}\right|^{2}}\sum_{j=1}^{m}\xi_{j}\boldsymbol{b}_{j}\boldsymbol{a}_{j}^{\mathsf{H}}\widetilde{\boldsymbol{x}}^{t}\nonumber \\
 & =\left(1-\eta\lambda-\frac{\overline{\alpha^{t+1/2}}}{\overline{\alpha^{t}}}\right)\bm{h}^{\star}+\left(1-\eta\lambda\right)\left(\widetilde{\bm{h}}^{t}-\bm{h}^{\star}\right)-\frac{\eta}{\left|\alpha^{t}\right|^{2}}\sum_{j=1}^{m}\boldsymbol{b}_{j}\boldsymbol{b}_{j}^{\mathsf{H}}\left(\widetilde{\bm{h}}^{t}-\bm{h}^{\star}\right)\widetilde{\boldsymbol{x}}^{t\mathsf{H}}\boldsymbol{a}_{j}\boldsymbol{a}_{j}^{\mathsf{H}}\widetilde{\boldsymbol{x}}^{t}\nonumber \\
 & \quad\quad-\frac{\eta}{\left|\alpha^{t}\right|^{2}}\sum_{j=1}^{m}\boldsymbol{b}_{j}\boldsymbol{b}_{j}^{\mathsf{H}}\bm{h}^{\star}\left(\widetilde{\boldsymbol{x}}^{t}-\bm{x}^{\star}\right)^{\mathsf{H}}\boldsymbol{a}_{j}\boldsymbol{a}_{j}^{\mathsf{H}}\widetilde{\boldsymbol{x}}^{t}+\frac{\eta}{\left|\alpha^{t}\right|^{2}}\sum_{j=1}^{m}\xi_{j}\boldsymbol{b}_{j}\boldsymbol{a}_{j}^{\mathsf{H}}\widetilde{\boldsymbol{x}}^{t}\nonumber \\
 & =\left(1-\eta\lambda-\frac{\overline{\alpha^{t+1/2}}}{\overline{\alpha^{t}}}\right)\bm{h}^{\star}+\left(1-\eta\lambda-\frac{\eta}{\left|\alpha^{t}\right|^{2}}\right)\left(\widetilde{\bm{h}}^{t}-\bm{h}^{\star}\right)-\frac{\eta}{\left|\alpha^{t}\right|^{2}}\underbrace{\sum_{j=1}^{m}\boldsymbol{b}_{j}\boldsymbol{b}_{j}^{\mathsf{H}}\left(\widetilde{\bm{h}}^{t}-\bm{h}^{\star}\right)\left(\left|\boldsymbol{a}_{j}^{\mathsf{H}}\widetilde{\boldsymbol{x}}^{t}\right|^{2}-\left|\boldsymbol{a}_{j}^{\mathsf{H}}\bm{x}^{\star}\right|^{2}\right)}_{=:\bm{\nu}_{1}}\nonumber \\
 & \quad-\frac{\eta}{\left|\alpha^{t}\right|^{2}}\underbrace{\sum_{j=1}^{m}\boldsymbol{b}_{j}\boldsymbol{b}_{j}^{\mathsf{H}}\left(\widetilde{\bm{h}}^{t}-\bm{h}^{\star}\right)\left(\left|\boldsymbol{a}_{j}^{\mathsf{H}}\bm{x}^{\star}\right|^{2}-\left\Vert \bm{x}^{\star}\right\Vert _{2}^{2}\right)}_{=:\bm{\nu}_{2}}-\frac{\eta}{\left|\alpha^{t}\right|^{2}}\underbrace{\sum_{j=1}^{m}\boldsymbol{b}_{j}\boldsymbol{b}_{j}^{\mathsf{H}}\bm{h}^{\star}\left(\widetilde{\boldsymbol{x}}^{t}-\bm{x}^{\star}\right)^{\mathsf{H}}\boldsymbol{a}_{j}\boldsymbol{a}_{j}^{\mathsf{H}}\widetilde{\boldsymbol{x}}^{t}}_{=:\bm{\nu}_{3}}+\frac{\eta}{\left|\alpha^{t}\right|^{2}}\underbrace{\sum_{j=1}^{m}\xi_{j}\boldsymbol{b}_{j}\boldsymbol{a}_{j}^{\mathsf{H}}\widetilde{\boldsymbol{x}}^{t}}_{=:\bm{\nu}_{4}},\label{eq:incohc-decomp}
\end{align}
where (i) comes from the gradient update rule \eqref{subeq:gradient_update_ncvx-2}
and (ii) is due to the expression \eqref{eq:whessian}.
\begin{itemize}
\item Applying a similar argument as for \citet[Equation (219)]{ma2017implicit}
yields
\[
\left|\bm{b}_{l}^{\mathsf{H}}\bm{\nu}_{1}\right|\leq0.1\max_{1\leq j\leq m}\left|\bm{b}_{j}^{\mathsf{H}}\big(\widetilde{\bm{h}}^{t}-\bm{h}^{\star}\big)\right|.
\]
\item The $\bm{\nu}_{2}$ can be controlled as follows
\begin{align*}
\left|\bm{b}_{l}^{\mathsf{H}}\bm{\nu}_{2}\right| & \leq0.2\max_{1\leq j\leq m}\left|\bm{b}_{j}^{\mathsf{H}}\big(\widetilde{\bm{h}}^{t}-\bm{h}^{\star}\big)\right|+C\log m\max_{0\leq l\leq m-\tau,1\leq j\leq\tau}\left|\left(\bm{b}_{l+j}-\bm{b}_{l+1}\right)^{\mathsf{H}}\big(\widetilde{\bm{h}}^{t}-\bm{h}^{\star}\big)\right|\\
 & \leq0.2\max_{1\leq j\leq m}\left|\bm{b}_{j}^{\mathsf{H}}\big(\widetilde{\bm{h}}^{t}-\bm{h}^{\star}\big)\right|+\left(C\log m\right)C_{11}\frac{\sigma}{\log^{3}m}.
\end{align*}
The first inequality can be derived via a similar argument as in \citet[Equation (221)]{ma2017implicit}
(the detailed proof is omitted here for the sake of simplicity), whereas
the second inequality results from the following claim.
\end{itemize}
\begin{claim}\label{claim:new}For some constant $C_{11}\gg C_{7}$,
we have
\[
\max_{0\leq l\leq m-\tau,1\leq j\leq\tau}\left|\left(\bm{b}_{l+j}-\bm{b}_{l+1}\right)^{\mathsf{H}}\big(\widetilde{\bm{h}}^{t}-\bm{h}^{\star}\big)\right|\leq C_{11}\frac{\sigma}{\log^{3}m}.
\]
\end{claim}\begin{proof}See Appendix \ref{subsec:Proof-of-Claimnew}.\end{proof}
\begin{itemize}
\item When it comes to the term $\bm{\nu}_{3}$, we observe that
\begin{align*}
\left|\bm{b}_{l}^{\mathsf{H}}\bm{\nu}_{3}\right| & \leq\left|\sum_{j=1}^{m}\bm{b}_{l}^{\mathsf{H}}\boldsymbol{b}_{j}\boldsymbol{b}_{j}^{\mathsf{H}}\bm{h}^{\star}\left(\widetilde{\bm{x}}^{t}-\bm{x}^{\star}\right)^{\mathsf{H}}\boldsymbol{a}_{j}\boldsymbol{a}_{j}^{\mathsf{H}}\left(\widetilde{\bm{x}}^{t}-\bm{x}^{\star}\right)\right|+\left|\sum_{j=1}^{m}\bm{b}_{l}^{\mathsf{H}}\boldsymbol{b}_{j}\boldsymbol{b}_{j}^{\mathsf{H}}\bm{h}^{\star}\left(\widetilde{\bm{x}}^{t}-\bm{x}^{\star}\right)^{\mathsf{H}}\boldsymbol{a}_{j}\boldsymbol{a}_{j}^{\mathsf{H}}\bm{x}^{\star}\right|\\
 & \leq\sum_{j=1}^{m}\left|\bm{b}_{l}^{\mathsf{H}}\boldsymbol{b}_{j}\right|\max_{1\leq j\leq m}\left|\boldsymbol{b}_{j}^{\mathsf{H}}\bm{h}^{\star}\right|\max_{1\leq j\leq m}\left|\boldsymbol{a}_{j}^{\mathsf{H}}\left(\widetilde{\bm{x}}^{t}-\bm{x}^{\star}\right)\right|^{2}+\sum_{j=1}^{m}\left|\bm{b}_{l}^{\mathsf{H}}\boldsymbol{b}_{j}\right|\left|\boldsymbol{b}_{j}^{\mathsf{H}}\bm{h}^{\star}\right|\max_{1\leq j\leq m}\left|\boldsymbol{a}_{j}^{\mathsf{H}}\left(\widetilde{\bm{x}}^{t}-\bm{x}^{\star}\right)\right|\max_{1\leq j\leq m}\left|\boldsymbol{a}_{j}^{\mathsf{H}}\bm{x}^{\star}\right|\\
 & \leq(4\log m)\frac{\mu}{\sqrt{m}}\left(\max_{1\leq j\leq m}\left|\boldsymbol{a}_{j}^{\mathsf{H}}\left(\widetilde{\bm{x}}^{t}-\bm{x}^{\star}\right)\right|\right)^{2}+(4\log m)\frac{\mu}{\sqrt{m}}\max_{1\leq j\leq m}\left|\boldsymbol{a}_{j}^{\mathsf{H}}\left(\widetilde{\bm{x}}^{t}-\bm{x}^{\star}\right)\right|\max_{1\leq j\leq m}\left|\boldsymbol{a}_{j}^{\mathsf{H}}\bm{x}^{\star}\right|\\
 & \lesssim C_{7}\frac{\mu}{\sqrt{m}}\log^{2}m\left(\lambda+\sigma\sqrt{K\log m}\right).
\end{align*}
Here, the penultimate inequality follows from the incoherence condition
\eqref{eq:induction3-1} and Lemma \ref{lemma:dftbound}, whereas
the last inequality follows from the induction hypothesis \eqref{eq:induction3-1}.
\item Finally, we turn to the term $\bm{\nu}_{4}$. Clearly, it is of the
same form as $\bm{\nu}_{4}$ in \eqref{eq:lem-incohb-decomp}; therefore,
via the same line of analysis, one can deduce the following bound
(similar to \eqref{eq:lem-incohb-nu4})
\begin{align*}
\left|\bm{b}_{l}^{\mathsf{H}}\bm{\nu}_{4}\right| & \lesssim(\sigma\log^{1.5}m)\max_{1\leq j\leq m}\left|\boldsymbol{a}_{j}^{\mathsf{H}}(\widetilde{\bm{x}}^{t}-\bm{x}^{\star})\right|+\sigma\sqrt{\frac{K}{m}\log m}\\
 & \lesssim\sigma\log^{1.5}m\left(C_{7}\sqrt{\log m}\left(\lambda+\sigma\sqrt{K\log m}\right)\right)+\sigma\sqrt{\frac{K}{m}\log m},
\end{align*}
where the last inequality invokes \eqref{eq:induction3-1}. 
\end{itemize}
With all the preceding results in place, we can combine them to demonstrate
that
\begin{align*}
 & \left|\frac{\overline{\alpha^{t+1/2}}}{\overline{\alpha^{t}}}\right|\max_{1\leq j\leq m}\left|\boldsymbol{b}_{j}^{\mathsf{H}}\left(\widetilde{\bm{h}}^{t+1}-\bm{h}^{\star}\right)\right|\\
\leq & \left(1-\eta\lambda-\frac{\overline{\alpha^{t+1/2}}}{\overline{\alpha^{t}}}\right)\max_{1\leq j\leq m}\left|\boldsymbol{b}_{j}^{\mathsf{H}}\bm{h}^{\star}\right|+\left(1-\eta\lambda-\frac{\eta}{\left|\alpha^{t}\right|^{2}}\right)\max_{1\leq j\leq m}\left|\boldsymbol{b}_{j}^{\mathsf{H}}\left(\widetilde{\bm{h}}^{t}-\bm{h}^{\star}\right)\right|\\
\quad & \quad+\frac{\eta}{\left|\alpha^{t}\right|^{2}}\left(0.3\max_{1\leq j\leq m}\left|\boldsymbol{b}_{j}^{\mathsf{H}}\left(\widetilde{\bm{h}}^{t}-\bm{h}^{\star}\right)\right|+\log m\times C_{11}\frac{\sigma}{\log^{3}m}\right)\\
\quad & \quad+\frac{\eta}{\left|\alpha^{t}\right|^{2}}CC_{7}\frac{\mu}{\sqrt{m}}\log^{2}m\left(\lambda+\sigma\sqrt{K\log m}\right)+\frac{\eta C}{\left|\alpha^{t}\right|^{2}}\left(\sigma\log^{1.5}m\left(C_{7}\sqrt{\log m}\left(\lambda+\sigma\sqrt{K\log m}\right)\right)+\sigma\sqrt{\frac{K}{m}\log m}\right)\\
\overset{(\text{i})}{\leq} & \left(1-\frac{7\eta}{40}\right)\max_{1\leq j\leq m}\left|\boldsymbol{b}_{j}^{\mathsf{H}}\left(\widetilde{\bm{h}}^{t}-\bm{h}^{\star}\right)\right|+\left(\eta\lambda+\left|1-\frac{\alpha^{t+1/2}}{\alpha^{t}}\right|\right)\frac{\mu}{\sqrt{m}}+\frac{4C_{11}\eta\sigma}{\log^{2}m}+CC_{7}\frac{\mu}{\sqrt{m}}\log^{2}m\left(\lambda+\sigma\sqrt{K\log m}\right)\\
\quad & \quad+4\eta C\left[\sigma\log^{1.5}m\left(C_{7}\sqrt{\log m}\left(\lambda+\sigma\sqrt{K\log m}\right)\right)+\sigma\sqrt{\frac{K}{m}\log m}\right]\\
\leq & \left(1-\frac{7\eta}{40}\right)C_{9}\sigma+c\eta\sigma,
\end{align*}
for some constant $C>0$ and sufficiently small constant $c>0$. Here
(i) uses triangle inequality and \eqref{eq:alpha-bound} and the proviso
that $m\gg\mu^{2}K\log^{5}m$ and $\sigma\sqrt{K\log^{4}m}\ll1$. 

Finally, making use of \eqref{eq:alpharatio-round1-1} we obtain
\begin{align*}
\max_{1\leq j\leq m}\left|\boldsymbol{b}_{j}^{\mathsf{H}}\left(\widetilde{\bm{h}}^{t+1}-\bm{h}^{\star}\right)\right| & \leq\frac{\left(1-\frac{7\eta}{40}\right)C_{9}\sigma+c\eta\sigma}{\left|\frac{\overline{\alpha^{t+1/2}}}{\overline{\alpha^{t}}}\right|}\leq\frac{\left(1-\frac{7\eta}{40}\right)C_{9}\sigma+c\eta\sigma}{1-\left|\overline{\frac{\alpha^{t+1/2}}{\alpha^{t}}}-1\right|}\\
 & \leq\frac{\left(1-\frac{7\eta}{40}\right)C_{9}\sigma+c\eta\sigma}{1-\eta\frac{CC_{5}}{c_{\rho}}\left(\lambda+\sigma\sqrt{K\log m}\right)}\\
 & \leq C_{9}\sigma,
\end{align*}
where $C>0$ is some constant and the last inequality holds since
$c$ is sufficiently small. 

\subsubsection{Proof of the claim \eqref{eq:smallgradient}\label{subsec:Proof-of-the-small-gradient}}

To prove \eqref{eq:smallgradient}, we need to show that the objective
value decreases as the algorithm progresses.

\begin{claim}\label{claim:decreasingobj}If the iterates satisfy
the induction hypotheses \eqref{eq:induction1-1}-\eqref{eq:induction4-1}
in the $t$th iteration, then with probability exceeding $1-O\left(m^{-100}+e^{-CK}\log m\right),$
\begin{equation}
f\left(\boldsymbol{h}^{t+1},\boldsymbol{x}^{t+1}\right)\leq f\left(\boldsymbol{h}^{t},\boldsymbol{x}^{t}\right)-\frac{\eta}{2}\left\Vert \nabla f\left(\boldsymbol{h}^{t},\boldsymbol{x}^{t}\right)\right\Vert _{2}^{2}.\label{eq:objvaldecrease}
\end{equation}
\end{claim}\begin{proof}See Appendix \ref{subsec:Proof-of-decreasingobj}.\end{proof}

When summed over $t$, the inequality in Lemma \ref{claim:decreasingobj}
leads to the following telescopic sum
\[
f\left(\bm{z}^{t_{0}}\right)\leq f\left(\boldsymbol{z}^{0}\right)-\frac{\eta}{2}\sum_{t=0}^{t_{0}-1}\left\Vert \nabla f\left(\boldsymbol{z}^{t}\right)\right\Vert _{2}^{2}.
\]
This further gives
\begin{equation}
\min_{0\leq t<t_{0}}\left\Vert \nabla f\left(\boldsymbol{z}^{t}\right)\right\Vert _{2}\leq\left\{ \frac{1}{t_{0}}\sum_{t=0}^{t_{0}-1}\left\Vert \nabla f\left(\boldsymbol{z}^{t}\right)\right\Vert _{2}^{2}\right\} ^{1/2}\leq\left\{ \frac{2}{\eta t_{0}}\left[f\left(\boldsymbol{z}^{\star}\right)-f\left(\boldsymbol{z}^{t_{0}}\right)\right]\right\} ^{1/2},\label{eq:sg1-1}
\end{equation}
where we have assumed that $\boldsymbol{z}^{0}=\boldsymbol{z}^{\star}$. 

We then proceed to control $f\left(\boldsymbol{z}^{\star}\right)-f\left(\boldsymbol{z}^{t_{0}}\right)$.
From the mean value theorem (cf.~\citet[Appendix D.3.1]{ma2017implicit}),
we can write
\begin{align*}
f\left(\boldsymbol{z}^{t_{0}}\right) & =f\left(\frac{\bm{h}^{t_{0}}}{\overline{\alpha^{t_{0}}}/\left|\alpha^{t_{0}}\right|},\frac{\alpha^{t_{0}}}{\left|\alpha^{t_{0}}\right|}\bm{x}^{t_{0}}\right)\\
 & =f\left(\boldsymbol{z}^{\star}\right)+\left[\begin{array}{c}
\nabla f\left(\boldsymbol{z}^{\star}\right)\\
\overline{\nabla f\left(\boldsymbol{z}^{\star}\right)}
\end{array}\right]^{\mathsf{H}}\left[\begin{array}{c}
\overline{\bm{z}}^{t_{0}}-\boldsymbol{z}^{\star}\\
\overline{\overline{\bm{z}}^{t_{0}}-\boldsymbol{z}^{\star}}
\end{array}\right]+\frac{1}{2}\left[\begin{array}{c}
\overline{\bm{z}}^{t_{0}}-\boldsymbol{z}^{\star}\\
\overline{\overline{\bm{z}}^{t_{0}}-\boldsymbol{z}^{\star}}
\end{array}\right]^{\mathsf{H}}\nabla^{2}f\left(\widehat{\bm{z}}\right)\left[\begin{array}{c}
\overline{\bm{z}}^{t_{0}}-\boldsymbol{z}^{\star}\\
\overline{\overline{\bm{z}}^{t_{0}}-\boldsymbol{z}^{\star}}
\end{array}\right]
\end{align*}
for some $\widehat{\bm{z}}$ lying between $\left(\frac{\bm{h}^{t_{0}}}{\overline{\alpha^{t_{0}}}/\left|\alpha^{t_{0}}\right|},\frac{\alpha^{t_{0}}}{\left|\alpha^{t_{0}}\right|}\bm{x}^{t_{0}}\right)$
and $\boldsymbol{z}^{\star}$. Then one has
\[
f\left(\boldsymbol{z}^{\star}\right)-f\left(\bm{z}^{t_{0}}\right)\leq2\left\Vert \nabla f\left(\boldsymbol{z}^{\star}\right)\right\Vert _{2}\left\Vert \overline{\bm{z}}^{t_{0}}-\boldsymbol{z}^{\star}\right\Vert _{2}+4\left\Vert \overline{\bm{z}}^{t_{0}}-\boldsymbol{z}^{\star}\right\Vert _{2}^{2}.
\]
The last inequality in the above formula invokes Lemma \ref{lemma:geometry},
whose assumptions are verified in the proof of Claim \ref{claim:decreasingobj}
(see Appendix \eqref{subsec:Proof-of-decreasingobj}). Further, the
relations \eqref{eq:lem-obj-gradient-true} and \eqref{eq:barz-dist}
in the proof of Claim \ref{claim:decreasingobj} lead to
\begin{equation}
f\left(\boldsymbol{z}^{\star}\right)-f\left(\boldsymbol{z}^{t_{0}}\right)\lesssim\left(\lambda+\sigma\sqrt{K\log m}\right)^{2}.\label{eq:sg2-1}
\end{equation}
It then follows from \eqref{eq:sg1-1} and \eqref{eq:sg2-1} that
\[
\min_{0\leq t<t_{0}}\left\Vert \nabla f\left(\boldsymbol{z}^{t}\right)\right\Vert _{2}\lesssim\sqrt{\frac{2}{\eta t_{0}}}\left(\lambda+\sigma\sqrt{K\log m}\right)\leq\frac{\lambda}{m^{10}}.
\]

\subsubsection{Proof of Claim \ref{claim:new}\label{subsec:Proof-of-Claimnew}}

We aim to prove by induction that there exists some constant $C_{11}>0$
such that 
\begin{equation}
\max_{0\leq l\leq m-\tau,\,1\leq j\leq\tau}\left|\left(\bm{b}_{l+j}-\bm{b}_{l+1}\right)^{\mathsf{H}}\left(\widetilde{\bm{h}}^{t}-\bm{h}^{\star}\right)\right|\leq C_{11}\frac{\sigma}{\log^{3}m}.\label{eq:induction-5-1}
\end{equation}
Apparently, \eqref{eq:induction-5-1} holds when $t=0$ given that
$\bm{h}^{0}=\bm{h}^{\star}$. In what follows, we shall assume that
\eqref{eq:induction-5-1} holds true at the $t$th iteration, and
examine this condition for the $(t+1)$th iteration.

Similar to the derivation of \eqref{eq:incohc-decomp}, we have the
following decomposition
\begin{align*}
\frac{\overline{\alpha^{t+1/2}}}{\overline{\alpha^{t}}}\left(\frac{1}{\overline{\alpha^{t+1}}}\boldsymbol{h}^{t+1}-\bm{h}^{\star}\right) & =\frac{\overline{\alpha^{t+1/2}}}{\overline{\alpha^{t}}}\left(\frac{1}{\overline{\alpha^{t+1/2}}}\boldsymbol{h}^{t+1/2}-\bm{h}^{\star}\right)\\
 & =\left(1-\eta\lambda-\frac{\overline{\alpha^{t+1/2}}}{\overline{\alpha^{t}}}\right)\bm{h}^{\star}+\left(1-\eta\lambda-\eta\left\Vert \bm{x}^{t}\right\Vert _{2}^{2}\right)\left(\widetilde{\bm{h}}^{t}-\bm{h}^{\star}\right)\\
 & \quad-\frac{\eta}{\left|\alpha^{t}\right|^{2}}\underbrace{\sum_{j=1}^{m}\boldsymbol{b}_{j}\boldsymbol{b}_{j}^{\mathsf{H}}\left(\widetilde{\boldsymbol{h}^{t}}-\bm{h}^{\star}\right)\widetilde{\bm{x}}^{t\mathsf{H}}\left(\bm{a}_{j}\bm{a}_{j}^{\mathsf{H}}-\bm{I}_{k}\right)\widetilde{\bm{x}}^{t}}_{=:\bm{\nu}_{1}}\\
 & \quad-\frac{\eta}{\left|\alpha^{t}\right|^{2}}\underbrace{\sum_{j=1}^{m}\boldsymbol{b}_{j}\boldsymbol{b}_{j}^{\mathsf{H}}\bm{h}^{\star}\left(\widetilde{\bm{x}}^{t}-\bm{x}^{\star}\right)^{\mathsf{H}}\bm{a}_{j}\bm{a}_{j}^{\mathsf{H}}\widetilde{\bm{x}}^{t}}_{=:\bm{\nu}_{2}}+\frac{\eta}{\left|\alpha^{t}\right|^{2}}\underbrace{\sum_{j=1}^{m}\xi_{j}\boldsymbol{b}_{j}\boldsymbol{a}_{j}^{\mathsf{H}}\widetilde{\bm{x}}^{t}}_{=:\bm{\nu}_{3}},
\end{align*}
leaving us with several terms to control. 
\begin{itemize}
\item For $\bm{\nu}_{1}$, we have that 
\begin{align*}
\left|\left(\bm{b}_{l}-\bm{b}_{1}\right)^{\mathsf{H}}\bm{\nu}_{1}\right| & \leq\sum_{j=1}^{m}\left|\left(\bm{b}_{l}-\bm{b}_{1}\right)^{\mathsf{H}}\boldsymbol{b}_{j}\right|\max_{1\leq j\leq m}\left|\boldsymbol{b}_{j}^{\mathsf{H}}\left(\widetilde{\boldsymbol{h}^{t}}-\bm{h}^{\star}\right)\widetilde{\bm{x}}^{t\mathsf{H}}\left(\bm{a}_{j}\bm{a}_{j}^{\mathsf{H}}-\bm{I}_{k}\right)\widetilde{\bm{x}}^{t}\right|\\
 & \leq\frac{c}{\log^{2}m}\max_{1\leq j\leq m}\left|\bm{b}_{j}^{\mathsf{H}}\left(\widetilde{\bm{h}}^{t}-\bm{h}^{\star}\right)\right|\max_{1\leq j\leq m}\left|\widetilde{\bm{x}}^{t\mathsf{H}}\left(\bm{a}_{j}\bm{a}_{j}^{\mathsf{H}}-\bm{I}_{k}\right)\widetilde{\bm{x}}^{t}\right|\\
 & \leq\frac{c}{\log^{2}m}\max_{1\leq j\leq m}\left|\bm{b}_{j}^{\mathsf{H}}\left(\widetilde{\bm{h}}^{t}-\bm{h}^{\star}\right)\right|\max_{1\leq j\leq m}\left(\left\Vert \bm{a}_{j}^{\mathsf{H}}\widetilde{\bm{x}}^{t}\right\Vert _{2}^{2}+\left\Vert \widetilde{\bm{x}}^{t}\right\Vert _{2}^{2}\right)\\
 & \lesssim\frac{c}{\log m}\max_{1\leq j\leq m}\left|\bm{b}_{j}^{\mathsf{H}}\left(\widetilde{\bm{h}}^{t}-\bm{h}^{\star}\right)\right|,
\end{align*}
where the second inequality follows from \citet[Lemma 50]{ma2017implicit}
and the last inequality utilizes the following consequence of \eqref{eq:induction3-1}
and Lemma \ref{lemma:useful}:
\[
\max_{1\leq j\leq m}\left(\left\Vert \bm{a}_{j}^{\mathsf{H}}\widetilde{\bm{x}}^{t}\right\Vert _{2}^{2}+\left\Vert \widetilde{\bm{x}}^{t}\right\Vert _{2}^{2}\right)\lesssim\max_{1\leq j\leq m}\left(2\left\Vert \bm{a}_{j}^{\mathsf{H}}\left(\widetilde{\bm{x}}^{t}-\bm{x}^{\star}\right)\right\Vert _{2}^{2}+2\left\Vert \bm{a}_{j}^{\mathsf{H}}\bm{x}^{\star}\right\Vert _{2}^{2}+\left\Vert \widetilde{\bm{x}}^{t}\right\Vert _{2}^{2}\right)\lesssim\log m.
\]
\item With regards to $\bm{\nu}_{2}$, we invoke the induction hypothesis
\eqref{eq:induction3-1} at the $t$th iteration to obtain
\begin{align*}
\left|\left(\bm{b}_{l}-\bm{b}_{1}\right)^{\mathsf{H}}\bm{\nu}_{2}\right| & \leq\sum_{j=1}^{m}\left|\left(\bm{b}_{l}-\bm{b}_{1}\right)^{\mathsf{H}}\boldsymbol{b}_{j}\right|\max_{1\leq j\leq m}\left|\boldsymbol{b}_{j}^{\mathsf{H}}\bm{h}^{\star}\right|\max_{1\leq j\leq m}\left|\left(\widetilde{\bm{x}}^{t}-\bm{x}^{\star}\right)^{\mathsf{H}}\bm{a}_{j}\bm{a}_{j}^{\mathsf{H}}\widetilde{\bm{x}}^{t}\right|\\
 & \leq\frac{c}{\log^{2}m}\frac{\mu}{\sqrt{m}}\left(\max_{1\leq j\leq m}\left|\left(\widetilde{\bm{x}}^{t}-\bm{x}^{\star}\right)^{\mathsf{H}}\bm{a}_{j}\right|^{2}+\max_{1\leq j\leq m}\left|\left(\widetilde{\bm{x}}^{t}-\bm{x}^{\star}\right)^{\mathsf{H}}\bm{a}_{j}\right|\max_{1\leq j\leq m}\left|\bm{a}_{j}^{\mathsf{H}}\bm{x}^{\star}\right|\right)\\
 & \lesssim C_{8}\frac{\mu}{\log m\sqrt{m}}\left(\lambda+\sigma\sqrt{\log m}\right),
\end{align*}
where the second inequality applies \citet[Lemma 50]{ma2017implicit}
and \eqref{eq:incoherence-condition}, and the last inequality results
from \eqref{eq:induction3-1} and \eqref{eq:useful1}. 
\item Finally, since $\left(\bm{b}_{l}-\bm{b}_{1}\right)^{\mathsf{H}}\bm{\nu}_{3}$
is of the same form as the quantity $\beta_{3}$ in \eqref{eq:proof-claim-decomp},
we can apply the analysis leading to \eqref{eq:proof-claim-beta3}
to derive
\begin{align*}
\left|\left(\bm{b}_{l}-\bm{b}_{1}\right)^{\mathsf{H}}\bm{\nu}_{3}\right| & \lesssim\frac{\sigma}{\log^{1.5}m}\max_{1\leq j\leq m}\left|\boldsymbol{a}_{j}^{\mathsf{H}}(\widetilde{\bm{x}}^{t}-\bm{x}^{\star})\right|+\sigma\sqrt{\frac{K\log^{2}m}{m}}\\
 & \lesssim\frac{\sigma}{\log^{1.5}m}\left(C_{7}\sqrt{\log m}\left(\lambda+\sigma\sqrt{K\log m}\right)\right)+\sigma\sqrt{\frac{K\log^{2}m}{m}}
\end{align*}
\end{itemize}
With the preceding results in hand, we have
\begin{align*}
 & \left|\frac{\overline{\alpha^{t+1/2}}}{\overline{\alpha^{t}}}\right|\max_{0\leq l\leq m-\tau,1\leq j\leq\tau}\left|\left(\bm{b}_{l+j}-\bm{b}_{l+1}\right)^{\mathsf{H}}\left(\widetilde{\bm{h}}^{t+1}-\bm{h}^{\star}\right)\right|\\
 & \quad\leq\left|1-\eta\lambda-\frac{\overline{\alpha^{t+1/2}}}{\overline{\alpha^{t}}}\right|\max_{0\leq l\leq m-\tau,1\leq j\leq\tau}\left|\left(\bm{b}_{l+j}-\bm{b}_{l+1}\right)^{\mathsf{H}}\bm{h}^{\star}\right|\\
 & \quad\quad\quad+\left(1-\eta\lambda-\eta\left\Vert \bm{x}^{t}\right\Vert _{2}^{2}\right)\max_{0\leq l\leq m-\tau,1\leq j\leq\tau}\left|\left(\bm{b}_{l+j}-\bm{b}_{l+1}\right)^{\mathsf{H}}\left(\widetilde{\bm{h}}^{t}-\bm{h}^{\star}\right)\right|\\
 & \quad\quad\quad+\frac{\eta CC_{9}}{\left|\alpha^{t}\right|^{2}}\frac{\mu\log m}{\sqrt{m}}\left(\lambda+\sigma\sqrt{K\log m}\right)+\frac{\eta CC_{8}}{\left|\alpha^{t}\right|^{2}}\left(\frac{\mu}{\log m\sqrt{m}}\left(\lambda+\sigma\sqrt{K\log m}\right)\right)\\
 & \quad\quad\quad+\frac{\eta C}{\left|\alpha^{t}\right|^{2}}\left[\frac{\sigma}{\log^{1.5}m}\left(C_{7}\sqrt{\log m}\left(\lambda+\sigma\sqrt{K\log m}\right)\right)+\sigma\sqrt{\frac{K\log^{2}m}{m}}\right]\\
 & \quad\overset{(\text{i})}{\leq}\left(\eta\lambda+\left|1-\frac{\alpha^{t+1/2}}{\alpha^{t}}\right|\right)\frac{2\mu}{\sqrt{m}}+\left(1-\frac{\eta}{16}\right)C_{11}\frac{\sigma}{\log^{3}m}\\
 & \quad\quad\quad+4\eta CC_{9}\frac{\mu\log m}{\sqrt{m}}\left(\lambda+\sigma\sqrt{K\log m}\right)+4\eta CC_{8}\left(\frac{\mu}{\log m\sqrt{m}}\left(\lambda+\sigma\sqrt{K\log m}\right)\right)\\
 & \quad\quad\quad+4\eta C\left[\frac{\sigma}{\log^{1.5}m}\left(C_{7}\sqrt{\log m}\left(\lambda+\sigma\sqrt{K\log m}\right)\right)+\sigma\sqrt{\frac{K\log^{2}m}{m}}\right]\\
 & \quad\overset{(\text{ii})}{\leq}\left(1-\frac{\eta}{16}\right)\frac{C_{11}\sigma}{\log^{3}m}+c\frac{\eta\sigma}{\log^{3}m}
\end{align*}
for some constant $C>0$ and some sufficiently small constant $c>0$.
Here, the relation (i) comes from the triangle inequality, \eqref{eq:alpha-bound},
as well as the consequence of \eqref{eq:tilde-hx-1} and \eqref{eq:alpha-bound}
\[
\left\Vert \bm{x}^{t}\right\Vert _{2}=\frac{\left\Vert \widetilde{\bm{x}}^{t}\right\Vert _{2}}{\left|\alpha^{t}\right|}\geq\frac{1/2}{2}=\frac{1}{4};
\]
the inequality (ii) invokes \eqref{eq:alpharatio-round1-1} and holds
with the proviso that $m\gg\mu^{2}K\log^{8}m$ and $\sigma\sqrt{K\log^{5}m}\ll1$. 

Finally, by \eqref{eq:alpharatio-round1-1} we obtain
\begin{align*}
\max_{0\leq l\leq m-\tau,1\leq j\leq\tau}\left|\left(\bm{b}_{l+j}-\bm{b}_{l+1}\right)^{\mathsf{H}}\left(\widetilde{\bm{h}}^{t+1}-\bm{h}^{\star}\right)\right| & \leq\frac{\left(1-\frac{\eta}{16}\right)C_{11}\frac{\sigma}{\log^{3}m}+c\frac{\eta\sigma}{\log^{3}m}}{\left|\frac{\overline{\alpha^{t+1/2}}}{\overline{\alpha^{t}}}\right|}\\
 & \leq\frac{\left(1-\frac{\eta}{16}\right)C_{11}\frac{\sigma}{\log^{3}m}+c\frac{\eta\sigma}{\log^{3}m}}{1-\left|\frac{\alpha^{t+1/2}}{\alpha^{t}}-1\right|}\\
 & \leq\frac{\left(1-\frac{\eta}{16}\right)C_{11}\frac{\sigma}{\log^{3}m}+c\frac{\eta\sigma}{\log^{3}m}}{1-\eta\frac{CC_{5}}{c_{\rho}}\left(\lambda+\sigma\sqrt{K\log m}\right)}\\
 & \leq C_{11}\frac{\sigma}{\log^{3}m},
\end{align*}
where $C>0$ is some constant. Here, the last inequality holds as
long as $c$ is sufficiently small. 

\subsubsection{Proof of Claim \ref{claim:decreasingobj}\label{subsec:Proof-of-decreasingobj}}

Before proceeding, we note that
\[
\nabla f\left(\bm{z}\right)=\nabla f_{\mathsf{reg}\text{-}\mathsf{free}}\left(\bm{z}\right)+\lambda\bm{z},
\]
and 
\begin{equation}
\left[\begin{array}{c}
\nabla_{\bm{h}}f\left(\frac{\bm{h}}{\overline{\alpha}},\alpha\bm{x}\right)\\
\nabla_{\bm{x}}f\left(\frac{\bm{h}}{\overline{\alpha}},\alpha\bm{x}\right)
\end{array}\right]=\left[\begin{array}{c}
\alpha\nabla_{\bm{h}}f_{\mathsf{reg}\text{-}\mathsf{free}}\left(\bm{h},\bm{x}\right)\\
\frac{1}{\overline{\alpha}}\nabla_{\bm{x}}f_{\mathsf{reg}\text{-}\mathsf{free}}\left(\bm{h},\bm{x}\right)
\end{array}\right]+\lambda\left[\begin{array}{c}
\frac{\bm{h}}{\overline{\alpha}}\\
\alpha\bm{x}
\end{array}\right].\label{eq:gradient-scale-relation}
\end{equation}
Another fact of use is that 
\[
\nabla^{2}f\left(\bm{h},\bm{x}\right)=\nabla^{2}f_{\mathsf{reg}\text{-}\mathsf{free}}\left(\bm{h},\bm{x}\right)+\lambda\bm{I}_{4K}.
\]
Letting 
\[
\beta^{t}=\frac{\alpha^{t}}{\left|\alpha^{t}\right|},\quad\quad\overline{\bm{h}}^{t}=\tfrac{1}{\overline{\beta^{t}}}\boldsymbol{h}^{t},\qquad\text{and}\qquad\overline{\bm{x}}^{t}=\beta^{t}\bm{x}^{t},
\]
we can write
\begin{align}
\left\Vert \nabla f\big(\overline{\bm{h}}^{t},\overline{\bm{x}}^{t}\big)\right\Vert _{2} & =\left\Vert \left[\begin{array}{c}
\beta^{t}\nabla_{\bm{h}}f_{\mathsf{reg}\text{-}\mathsf{free}}\left(\bm{h}^{t},\bm{x}^{t}\right)\\
\frac{1}{\overline{\beta^{t}}}\nabla_{\bm{x}}f_{\mathsf{reg}\text{-}\mathsf{free}}\left(\bm{h}^{t},\bm{x}^{t}\right)
\end{array}\right]+\lambda\left[\begin{array}{c}
\frac{\bm{h}^{t}}{\overline{\beta^{t}}}\\
\beta^{t}\bm{x}^{t}
\end{array}\right]\right\Vert _{2}\nonumber \\
 & =\left\Vert \left[\begin{array}{c}
\nabla_{\bm{h}}f_{\mathsf{reg}\text{-}\mathsf{free}}\left(\bm{h}^{t},\bm{x}^{t}\right)\\
\nabla_{\bm{x}}f_{\mathsf{reg}\text{-}\mathsf{free}}\left(\bm{h}^{t},\bm{x}^{t}\right)
\end{array}\right]+\lambda\left[\begin{array}{c}
\bm{h}^{t}\\
\bm{x}^{t}
\end{array}\right]\right\Vert _{2}\nonumber \\
 & =\left\Vert \nabla f\left(\bm{h}^{t},\bm{x}^{t}\right)\right\Vert _{2},\label{eq:gradient-scale}
\end{align}
where the first inequality is due to \eqref{eq:gradient-scale-relation},
and the second inequality comes from the simple fact that $\beta^{t}\overline{\beta^{t}}=1$
(by definition of $\beta^{t}$).

To begin with, we show that $f\left(\boldsymbol{h}^{t+1},\boldsymbol{x}^{t+1}\right)$
is upper bounded by $f\left(\bm{h}^{t+1/2},\bm{x}^{t+1/2}\right)$,
that is, 
\begin{align*}
f\left(\boldsymbol{h}^{t+1},\boldsymbol{x}^{t+1}\right) & =\sum_{j=1}^{m}\left|\boldsymbol{b}_{j}^{\mathsf{H}}\bm{h}^{t+1}\big(\bm{x}^{t+1}\big)^{\mathsf{H}}\boldsymbol{a}_{j}-y_{j}\right|^{2}+\lambda\left\Vert \boldsymbol{h}^{t+1}\right\Vert _{2}^{2}+\lambda\left\Vert \boldsymbol{x}^{t+1}\right\Vert _{2}^{2}\\
 & \overset{(\text{i})}{=}\sum_{j=1}^{m}\left|\boldsymbol{b}_{j}^{\mathsf{H}}\bm{h}^{t+1/2}\big(\bm{x}^{t+1/2}\big)^{\mathsf{H}}\boldsymbol{a}_{j}-y_{j}\right|^{2}+2\lambda\left\Vert \boldsymbol{h}^{t+1}\right\Vert _{2}\left\Vert \boldsymbol{x}^{t+1}\right\Vert _{2}\\
 & \overset{(\text{ii})}{=}\sum_{j=1}^{m}\left|\boldsymbol{b}_{j}^{\mathsf{H}}\bm{h}^{t+1/2}\big(\bm{x}^{t+1/2}\big)^{\mathsf{H}}\boldsymbol{a}_{j}-y_{j}\right|^{2}+2\lambda\left\Vert \bm{h}^{t+1/2}\right\Vert _{2}\left\Vert \bm{x}^{t+1/2}\right\Vert _{2}\\
 & \overset{(\text{iii})}{\leq}\sum_{j=1}^{m}\left|\boldsymbol{b}_{j}^{\mathsf{H}}\bm{h}^{t+1/2}\big(\bm{x}^{t+1/2}\big)^{\mathsf{H}}\boldsymbol{a}_{j}-y_{j}\right|^{2}+\lambda\big\|\bm{h}^{t+1/2}\big\|_{2}^{2}+\lambda\big\|\bm{x}^{t+1/2}\big\|_{2}^{2}\\
 & =f\left(\bm{h}^{t+1/2},\bm{x}^{t+1/2}\right),
\end{align*}
where (i) and (ii) come from \eqref{eq:balance}, and (iii) is due
to the elementary inequality $2ab\leq a^{2}+b^{2}$. In order to control
$f\left(\bm{h}^{t+1/2},\bm{x}^{t+1/2}\right)$, one observes that
\begin{align*}
f\left(\bm{h}^{t+1/2},\bm{x}^{t+1/2}\right) & =f\left(\frac{\bm{h}^{t+1/2}}{\overline{\beta^{t}}},\beta^{t}\bm{x}^{t+1/2}\right)\\
 & \overset{(\text{i})}{=}f\left(\overline{\bm{h}}^{t}-\frac{\eta}{\overline{\beta^{t}}}\left(\nabla_{\bm{h}}f_{\mathsf{reg}\text{-}\mathsf{free}}\left(\bm{z}^{t}\right)+\lambda\bm{h}^{t}\right),\overline{\bm{x}}^{t}-\eta\beta^{t}\left(\nabla_{\bm{x}}f_{\mathsf{reg}\text{-}\mathsf{free}}\left(\bm{z}^{t}\right)+\lambda\bm{x}^{t}\right)\right)\\
 & \overset{(\text{ii})}{=}f\left(\overline{\bm{h}}^{t}-\eta\nabla_{\bm{h}}f\left(\overline{\bm{z}}^{t}\right),\overline{\bm{x}}^{t}-\eta\nabla_{\bm{x}}f\left(\overline{\bm{z}}^{t}\right)\right)\\
 & \overset{(\text{iii})}{=}f\left(\overline{\bm{h}}^{t},\overline{\bm{x}}^{t}\right)-\eta\left[\begin{array}{c}
\nabla_{\bm{h}}f\left(\overline{\bm{h}}^{t},\overline{\bm{x}}^{t}\right)\\
\nabla_{\bm{x}}f\left(\overline{\bm{h}}^{t},\overline{\bm{x}}^{t}\right)\\
\overline{\nabla_{\bm{h}}f\left(\overline{\bm{h}}^{t},\overline{\bm{x}}^{t}\right)}\\
\overline{\nabla_{\bm{x}}f\left(\overline{\bm{h}}^{t},\overline{\bm{x}}^{t}\right)}
\end{array}\right]^{\mathsf{H}}\left[\begin{array}{c}
\nabla_{\bm{h}}f\left(\overline{\bm{h}}^{t},\overline{\bm{x}}^{t}\right)\\
\nabla_{\bm{x}}f\left(\overline{\bm{h}}^{t},\overline{\bm{x}}^{t}\right)\\
\overline{\nabla_{\bm{h}}f\left(\overline{\bm{h}}^{t},\overline{\bm{x}}^{t}\right)}\\
\overline{\nabla_{\bm{x}}f\left(\overline{\bm{h}}^{t},\overline{\bm{x}}^{t}\right)}
\end{array}\right]\\
 & \quad\quad\quad+\frac{\eta^{2}}{2}\left[\begin{array}{c}
\nabla_{\bm{h}}f\left(\overline{\bm{h}}^{t},\overline{\bm{x}}^{t}\right)\\
\nabla_{\bm{x}}f\left(\overline{\bm{h}}^{t},\overline{\bm{x}}^{t}\right)\\
\overline{\nabla_{\bm{h}}f\left(\overline{\bm{h}}^{t},\overline{\bm{x}}^{t}\right)}\\
\overline{\nabla_{\bm{x}}f\left(\overline{\bm{h}}^{t},\overline{\bm{x}}^{t}\right)}
\end{array}\right]^{\mathsf{H}}\nabla^{2}f\left(\widehat{\bm{z}}\right)\left[\begin{array}{c}
\nabla_{\bm{h}}f\left(\overline{\bm{h}}^{t},\overline{\bm{x}}^{t}\right)\\
\nabla_{\bm{x}}f\left(\overline{\bm{h}}^{t},\overline{\bm{x}}^{t}\right)\\
\overline{\nabla_{\bm{h}}f\left(\overline{\bm{h}}^{t},\overline{\bm{x}}^{t}\right)}\\
\overline{\nabla_{\bm{x}}f\left(\overline{\bm{h}}^{t},\overline{\bm{x}}^{t}\right)}
\end{array}\right]\\
 & \overset{\text{(iv)}}{\leq}f\left(\overline{\bm{h}}^{t},\overline{\bm{x}}^{t}\right)-2\eta\left\Vert \nabla_{\bm{h}}f\left(\overline{\bm{h}}^{t},\overline{\bm{x}}^{t}\right)\right\Vert _{2}^{2}-2\eta\left\Vert \nabla_{\bm{x}}f\left(\overline{\bm{h}}^{t},\overline{\bm{x}}^{t}\right)\right\Vert _{2}^{2}\\
 & \quad\quad+\frac{\eta^{2}}{2}\cdot4\left[2\left\Vert \nabla_{\bm{h}}f\left(\overline{\bm{h}}^{t},\overline{\bm{x}}^{t}\right)\right\Vert _{2}^{2}+2\left\Vert \nabla_{\bm{x}}f\left(\overline{\bm{h}}^{t},\overline{\bm{x}}^{t}\right)\right\Vert _{2}^{2}\right]\\
 & \overset{\text{(v)}}{\leq}f\left(\overline{\bm{h}}^{t},\overline{\bm{x}}^{t}\right)-\frac{\eta}{2}\left\Vert \nabla_{\bm{h}}f\left(\overline{\bm{h}}^{t},\overline{\bm{x}}^{t}\right)\right\Vert _{2}^{2}-\frac{\eta}{2}\left\Vert \nabla_{\bm{x}}f\left(\overline{\bm{h}}^{t},\overline{\bm{x}}^{t}\right)\right\Vert _{2}^{2}\\
 & =f\left(\bm{h}^{t},\bm{x}^{t}\right)-\frac{\eta}{2}\left\Vert \nabla f\left(\bm{h}^{t},\bm{x}^{t}\right)\right\Vert _{2}^{2},
\end{align*}
where $\widehat{\bm{z}}$ is a point lying between $\overline{\bm{z}}^{t}-\eta\nabla f\left(\overline{\bm{z}}^{t}\right)$
and $\overline{\bm{z}}^{t}$. Here, (i) resorts to the gradient update
rule \eqref{subeq:gradient_update_ncvx-2}; (ii) utilizes the relation
\eqref{eq:gradient-scale-relation}; (iii) comes from the mean value
theorem \citet[Appendix D.3.1]{ma2017implicit}; (iv) follows from
Lemma \ref{lemma:geometry} (which we shall verify shortly); (v) holds
true for sufficiently small $\eta>0$; and the last equality follows
from the identity \eqref{eq:gradient-scale}. Therefore, it only remains
to verify the conditions required to invoke Lemma \ref{lemma:geometry}
in Step (iv). In particular, we would need to justify that both $\overline{\bm{z}}^{t}$
and $\overline{\bm{z}}^{t}-\eta\nabla f\left(\overline{\bm{z}}^{t}\right)$
satisfy the conditions of Lemma \ref{lemma:geometry}.
\begin{itemize}
\item We first show that $\overline{\bm{z}}^{t}$ satisfies the conditions
of Lemma \ref{lemma:geometry}. Towards this, it is first seen that
\begin{align}
\left\Vert \overline{\bm{h}}^{t}-\bm{h}^{\star}\right\Vert _{2}^{2}+\left\Vert \overline{\bm{x}}^{t}-\bm{x}^{\star}\right\Vert _{2}^{2} & =\left\Vert \frac{\bm{h}^{t}}{\overline{\alpha^{t}/\left|\alpha^{t}\right|}}-\bm{h}^{\star}\right\Vert _{2}^{2}+\left\Vert \frac{\alpha^{t}}{\left|\alpha^{t}\right|}\bm{x}^{t}-\bm{x}^{\star}\right\Vert _{2}^{2}\nonumber \\
 & \leq\left(\left\Vert \frac{\bm{h}^{t}}{\overline{\alpha^{t}/\left|\alpha^{t}\right|}}-\frac{\bm{h}^{t}}{\overline{\alpha^{t}}}\right\Vert _{2}+\left\Vert \frac{\bm{h}^{t}}{\overline{\alpha^{t}}}-\bm{h}^{\star}\right\Vert _{2}\right)^{2}+\left(\left\Vert \frac{\alpha^{t}}{\left|\alpha^{t}\right|}\bm{x}^{t}-\alpha^{t}\bm{x}^{t}\right\Vert _{2}+\left\Vert \alpha^{t}\bm{x}^{t}-\bm{x}^{\star}\right\Vert _{2}\right)^{2}\nonumber \\
 & =\left(\left|\left|\alpha^{t}\right|-1\right|\left\Vert \widetilde{\bm{h}}^{t}\right\Vert _{2}+\left\Vert \widetilde{\bm{h}}^{t}-\bm{h}^{\star}\right\Vert _{2}\right)^{2}+\left(\left|\frac{\left|\alpha^{t}\right|-1}{\left|\alpha^{t}\right|}\right|\left\Vert \widetilde{\bm{x}}^{t}\right\Vert _{2}+\left\Vert \widetilde{\bm{x}}^{t}-\bm{x}^{\star}\right\Vert _{2}\right)^{2}\nonumber \\
 & \lesssim\left(\frac{C_{5}}{c_{\rho}}\left(\lambda+\sigma\sqrt{K\log m}\right)\right)^{2},\label{eq:barz-dist}
\end{align}
where the last inequality comes from \eqref{eq:dist-bound-1} and
\eqref{eq:alpha-asymp1-1}. Further,
\begin{align}
\max_{1\leq j\leq m}\left|\boldsymbol{a}_{j}^{\mathsf{H}}\left(\overline{\bm{x}}^{t}-\bm{x}^{\star}\right)\right| & \leq\max_{1\leq j\leq m}\left|\boldsymbol{a}_{j}^{\mathsf{H}}\left(\frac{\alpha^{t}}{\left|\alpha^{t}\right|}\bm{x}^{t}-\alpha^{t}\bm{x}^{t}\right)\right|+\max_{1\leq j\leq m}\left|\boldsymbol{a}_{j}^{\mathsf{H}}\left(\widetilde{\bm{x}}^{t}-\bm{x}^{\star}\right)\right|\nonumber \\
 & \leq\left|\frac{\left|\alpha^{t}\right|-1}{\left|\alpha^{t}\right|}\right|\max_{1\leq j\leq m}\left|\boldsymbol{a}_{j}^{\mathsf{H}}\widetilde{\bm{x}}^{t}\right|+\max_{1\leq j\leq m}\left|\boldsymbol{a}_{j}^{\mathsf{H}}\left(\widetilde{\bm{x}}^{t}-\bm{x}^{\star}\right)\right|\nonumber \\
 & \leq\left|\frac{\left|\alpha^{t}\right|-1}{\left|\alpha^{t}\right|}\right|\left(\max_{1\leq j\leq m}\left|\boldsymbol{a}_{j}^{\mathsf{H}}\left(\widetilde{\bm{x}}^{t}-\bm{x}^{\star}\right)\right|+\max_{1\leq j\leq m}\left|\boldsymbol{a}_{j}^{\mathsf{H}}\bm{x}^{\star}\right|\right)+\max_{1\leq j\leq m}\left|\boldsymbol{a}_{j}^{\mathsf{H}}\left(\widetilde{\bm{x}}^{t}-\bm{x}^{\star}\right)\right|\nonumber \\
 & \lesssim\left(\lambda+\sigma\sqrt{K\log m}\right)\sqrt{\log m},\label{eq:barx-incoha}
\end{align}
where the last inequality follows from \eqref{eq:alpha-asymp1-1},
\eqref{eq:induction3-1} and Lemma \ref{lemma:useful}. Similarly,
one has 
\begin{align}
\max_{1\leq j\leq m}\left|\boldsymbol{b}_{j}^{\mathsf{H}}\overline{\bm{h}}^{t}\right| & \leq\max_{1\leq j\leq m}\left|\boldsymbol{b}_{j}^{\mathsf{H}}\left(\frac{\bm{h}^{t}}{\overline{\alpha^{t}/\left|\alpha^{t}\right|}}-\frac{\bm{h}^{t}}{\overline{\alpha^{t}}}\right)\right|+\max_{1\leq j\leq m}\left|\boldsymbol{b}_{j}^{\mathsf{H}}\widetilde{\bm{h}}^{t}\right|\nonumber \\
 & \leq\left|\left|\alpha^{t}\right|-1\right|\max_{1\leq j\leq m}\left|\boldsymbol{b}_{j}^{\mathsf{H}}\frac{\bm{h}^{t}}{\overline{\alpha^{t}}}\right|+\max_{1\leq j\leq m}\left|\boldsymbol{b}_{j}^{\mathsf{H}}\widetilde{\bm{h}}^{t}\right|\nonumber \\
 & \leq2\max_{1\leq j\leq m}\left|\boldsymbol{b}_{j}^{\mathsf{H}}\widetilde{\bm{h}}^{t}\right|\label{eq:barh-incohb}\\
 & \lesssim\frac{\mu}{\sqrt{m}}\log m+\sigma,
\end{align}
where the last inequality comes from \eqref{eq:induction4-1}. Given
that $\overline{\bm{z}}^{t}$ satisfies the conditions in Lemma \ref{lemma:geometry},
we can invoke Lemma \ref{lemma:geometry} to demonstrate that
\begin{equation}
\left\Vert \nabla_{\bm{h}}f\left(\overline{\bm{z}}^{t}\right)-\nabla_{\bm{h}}f\left(\bm{z}^{\star}\right)\right\Vert _{2}\leq4\left\Vert \overline{\bm{z}}^{t}-\bm{z}^{\star}\right\Vert _{2}.\label{eq:lem-obj-gradient}
\end{equation}
\item Next, we move on to show that $\overline{\bm{z}}^{t}-\eta\nabla f\left(\overline{\bm{z}}^{t}\right)$
also satisfies the conditions of Lemma \ref{lemma:geometry}. To begin
with,
\begin{align}
\left\Vert \overline{\bm{z}}^{t}-\eta\nabla f\left(\overline{\bm{z}}^{t}\right)-\bm{z}^{\star}\right\Vert _{2}\leq\left\Vert \overline{\bm{z}}^{t}-\bm{z}^{\star}\right\Vert _{2}+\eta\left\Vert \nabla f\left(\overline{\bm{z}}^{t}\right)-\nabla f\left(\bm{z}^{\star}\right)\right\Vert _{2}+\eta\left\Vert \nabla f\left(\bm{z}^{\star}\right)\right\Vert _{2}.\label{eq:bargradient-decomp}
\end{align}
We observe that
\begin{align}
\left\Vert \nabla f\left(\boldsymbol{z}^{\star}\right)\right\Vert _{\text{2}} & \leq\|\nabla f_{\mathsf{clean}}\left(\boldsymbol{z}^{\star}\right)\|_{2}+\left\Vert \mathcal{A}^{*}\left(\bm{\xi}\right)\boldsymbol{h}^{\star}\right\Vert _{2}+\left\Vert \mathcal{A}^{*}\left(\bm{\xi}\right)\boldsymbol{x}^{\star}\right\Vert _{2}+\lambda\left\Vert \bm{h}^{\star}\right\Vert _{2}+\lambda\left\Vert \bm{z}^{\star}\right\Vert _{2}\nonumber \\
 & \lesssim\lambda+\sigma\sqrt{K\log m}.\label{eq:lem-obj-gradient-true}
\end{align}
Taking \eqref{eq:lem-obj-gradient-true}, \eqref{eq:lem-obj-gradient},
\eqref{eq:barz-dist} and \eqref{eq:bargradient-decomp} collectively,
one arrives at
\begin{align*}
\left\Vert \overline{\bm{z}}^{t}-\eta\nabla f\left(\overline{\bm{z}}^{t}\right)-\bm{z}^{\star}\right\Vert _{2} & \lesssim\lambda+\sigma\sqrt{K\log m}.
\end{align*}
With regards to the incoherence condition w.r.t.~$\bm{a}_{j}$, we
have
\begin{align}
 & \max_{1\leq j\leq m}\left|\bm{a}_{j}^{\mathsf{H}}\left(\overline{\bm{x}}^{t}-\eta\nabla_{\bm{x}}f\left(\overline{\bm{z}}^{t}\right)-\bm{x}^{\star}\right)\right|\nonumber \\
 & \leq\max_{1\leq j\leq m}\left|\bm{a}_{j}^{\mathsf{H}}\left(\overline{\bm{x}}^{t}-\bm{x}^{\star}\right)\right|+\eta\max_{1\leq j\leq m}\left|\bm{a}_{j}^{\mathsf{H}}\nabla_{\bm{x}}f\left(\overline{\bm{z}}^{t}\right)\right|\nonumber \\
 & \leq\max_{1\leq j\leq m}\left|\bm{a}_{j}^{\mathsf{H}}\left(\overline{\bm{x}}^{t}-\bm{x}^{\star}\right)\right|+\eta\left(\max_{1\leq j\leq m}\left|\bm{a}_{j}^{\mathsf{H}}\nabla_{\bm{x}}f\left(\overline{\bm{z}}^{t}-\widetilde{\boldsymbol{z}}^{t,\left(l\right)}\right)\right|+\max_{1\leq j\leq m}\left|\bm{a}_{j}^{\mathsf{H}}\nabla_{\bm{x}}f\left(\widetilde{\boldsymbol{z}}^{t,\left(l\right)}\right)\right|\right)\nonumber \\
 & \leq C\sqrt{\log m}\left(\lambda+\sigma\sqrt{K\log m}\right)+4\eta\left(10\sqrt{K}\times4\max_{1\leq j\leq m}\left\Vert \widetilde{\bm{z}}^{t}-\widetilde{\boldsymbol{z}}^{t,\left(l\right)}\right\Vert _{2}+20\sqrt{\log m}\max_{1\leq j\leq m}\left\Vert \nabla_{\bm{x}}f\left(\widetilde{\boldsymbol{z}}^{t,\left(l\right)}\right)\right\Vert _{2}\right),\label{eq:lem-obj-incoha}
\end{align}
where the last inequality follows from \eqref{eq:barx-incoha} for
some constant $C>0$, \eqref{eq:lem-obj-gradient} and Lemma \ref{lemma:useful}.
Further, it is self-evident that $\widetilde{\boldsymbol{z}}^{t,\left(l\right)}$
satisfies the conditions of Lemma \ref{lemma:geometry}, so that we
have
\begin{align*}
\left\Vert \nabla_{\bm{x}}f\left(\widetilde{\boldsymbol{z}}^{t,\left(l\right)}\right)\right\Vert _{2} & \leq\left\Vert \nabla_{\bm{x}}f\left(\widetilde{\boldsymbol{z}}^{t,\left(l\right)}\right)-\nabla_{\bm{x}}f\left(\bm{z}^{\star}\right)\right\Vert _{2}+\left\Vert \nabla_{\bm{x}}f\left(\bm{z}^{\star}\right)\right\Vert _{2}\\
 & \leq4\left\Vert \widetilde{\boldsymbol{z}}^{t,\left(l\right)}-\bm{z}^{\star}\right\Vert _{2}+C\left(\lambda+\sigma\sqrt{K\log m}\right)\\
 & \leq4\left(\left\Vert \widetilde{\boldsymbol{z}}^{t,\left(l\right)}-\widetilde{\bm{z}}^{t}\right\Vert _{2}+\left\Vert \widetilde{\bm{z}}^{t}-\bm{z}^{\star}\right\Vert _{2}\right)+C\left(\lambda+\sigma\sqrt{K\log m}\right),
\end{align*}
where the second inequality invokes Lemma \ref{lemma:geometry} and
\eqref{eq:lem-obj-gradient-true}. This together with \eqref{eq:lem-obj-incoha}
and \eqref{eq:induction} gives
\[
\max_{1\leq j\leq m}\left|\bm{a}_{j}^{\mathsf{H}}\left(\overline{\bm{x}}^{t}-\eta\nabla_{\bm{x}}f\left(\overline{\bm{z}}^{t}\right)-\bm{x}^{\star}\right)\right|\lesssim\sqrt{\log m}\left(\lambda+\sigma\sqrt{K\log m}\right).
\]
For the other incoherence condition w.r.t.~$\bm{b}_{j}$, we can
invoke similar argument to show that
\begin{align}
 & \max_{1\leq j\leq m}\left|\bm{b}_{j}^{\mathsf{H}}\left(\overline{\bm{h}}^{t}-\eta\nabla_{\bm{h}}f\left(\overline{\bm{z}}^{t}\right)-\bm{h}^{\star}\right)\right|\nonumber \\
 & \quad\leq\max_{1\leq j\leq m}\left|\bm{b}_{j}^{\mathsf{H}}\left(\overline{\bm{h}}^{t}-\bm{h}^{\star}\right)\right|+\eta\max_{1\leq j\leq m}\left|\bm{b}_{j}^{\mathsf{H}}\nabla_{\bm{h}}f\left(\overline{\bm{z}}^{t}\right)\right|\nonumber \\
 & \quad\leq\max_{1\leq j\leq m}\left|\bm{b}_{j}^{\mathsf{H}}\left(\overline{\bm{h}}^{t}-\bm{h}^{\star}\right)\right|+\eta\max_{1\leq j\leq m}\left|\bm{b}_{j}^{\mathsf{H}}\left(\sum_{l=1}^{m}\left(\boldsymbol{b}_{l}^{\mathsf{H}}\widetilde{\bm{h}}^{t}\widetilde{\bm{x}}^{t,\mathsf{H}}\boldsymbol{a}_{l}-y_{l}\right)\boldsymbol{b}_{l}\boldsymbol{a}_{l}^{\mathsf{H}}\overline{\bm{x}}^{t}+\lambda\overline{\bm{h}}^{t}\right)\right|\nonumber \\
 & \quad\leq\max_{1\leq j\leq m}\left|\boldsymbol{b}_{j}^{\mathsf{H}}\left(\frac{\bm{h}^{t}}{\overline{\alpha^{t}/\left|\alpha^{t}\right|}}-\frac{\bm{h}^{t}}{\overline{\alpha^{t}}}\right)\right|+\max_{1\leq j\leq m}\left|\boldsymbol{b}_{j}^{\mathsf{H}}\left(\widetilde{\bm{h}}^{t}-\bm{h}^{\star}\right)\right|\nonumber \\
 & \qquad+\eta\left(\lambda\left|\alpha^{t}\right|\max_{1\leq j\leq m}\left|\bm{b}_{j}^{\mathsf{H}}\widetilde{\bm{h}}^{t}\right|+\left|\alpha^{t}\right|^{-1}\underbrace{\max_{1\leq j\leq m}\left|\bm{b}_{j}^{\mathsf{H}}\left(\sum_{l=1}^{m}\left(\boldsymbol{b}_{l}^{\mathsf{H}}\widetilde{\bm{h}}^{t}\widetilde{\bm{x}}^{t,\mathsf{H}}\boldsymbol{a}_{l}-y_{l}\right)\boldsymbol{b}_{l}\boldsymbol{a}_{l}^{\mathsf{H}}\widetilde{\bm{x}}^{t}\right)\right|}_{\eqqcolon\tau}\right)\nonumber \\
 & \quad\leq\left|\left|\alpha^{t}\right|-1\right|\max_{1\leq j\leq m}\left|\boldsymbol{b}_{j}^{\mathsf{H}}\widetilde{\bm{h}}^{t}\right|+\max_{1\leq j\leq m}\left|\boldsymbol{b}_{j}^{\mathsf{H}}\left(\widetilde{\bm{h}}^{t}-\bm{h}^{\star}\right)\right|+\eta\left(2\lambda\max_{1\leq j\leq m}\left|\bm{b}_{j}^{\mathsf{H}}\widetilde{\bm{h}}^{t}\right|+2\tau\right).\label{eq:UB-15678}
\end{align}
Here, the last inequality utilizes the fact $\left\Vert \bm{b}_{j}\right\Vert _{2}=\sqrt{K/m}$
and \eqref{eq:alpha-asymp1-1}. The quantity $\tau$ can be controlled
by using the same analysis as Appendix \ref{subsec:Proof-of-Lemmaincoherenceb}.
Specifically,
\begin{align*}
\tau & =\max_{1\leq j\leq m}\left|\bm{b}_{j}^{\mathsf{H}}\nabla_{\bm{h}}f_{\mathsf{reg}\text{-}\mathsf{free}}\left(\widetilde{\bm{z}}^{t}\right)\right|\\
 & \leq\max_{1\leq j\leq m}\left(\left|\bm{b}_{j}^{\mathsf{H}}\bm{\nu}_{1}\right|+\left|\bm{b}_{j}^{\mathsf{H}}\bm{\nu}_{2}\right|+\left|\bm{b}_{j}^{\mathsf{H}}\bm{\nu}_{3}\right|+\left|\bm{b}_{j}^{\mathsf{H}}\bm{\nu}_{4}\right|+\left\Vert \bm{x}^{\star}\right\Vert _{2}^{2}\left|\bm{b}_{j}^{\mathsf{H}}\widetilde{\bm{h}}^{t}\right|\right)\\
 & \lesssim\frac{\mu}{\sqrt{m}}\log m+\sigma,
\end{align*}
where $\{\bm{\nu}_{i}\}_{i=1}^{4}$ are defined in \eqref{eq:lem-incohb-decomp},
and the last inequality is a direct consequence of Appendix \ref{subsec:Proof-of-Lemmaincoherenceb}.
Finally, continue the bound \eqref{eq:UB-15678} to demonstrate that
\begin{align*}
 & \max_{1\leq j\leq m}\left|\bm{b}_{j}^{\mathsf{H}}\left(\overline{\bm{h}}^{t}-\eta\nabla_{\bm{h}}f\left(\overline{\bm{z}}^{t}\right)-\bm{h}^{\star}\right)\right|\\
 & \lesssim\frac{C_{5}}{c_{\rho}}\left(\lambda+\sigma\sqrt{K\log m}\right)C_{8}\left(\frac{\mu}{\sqrt{m}}\log m+\sigma\right)+C_{9}\sigma+\eta\left(2C_{8}\lambda\left(\frac{\mu}{\sqrt{m}}\log m+\sigma\right)+2\left(\frac{\mu}{\sqrt{m}}\log m+\sigma\right)\right)\\
 & \lesssim\frac{\mu}{\sqrt{m}}\log m+\sigma,
\end{align*}
where the penultimate inequality is due to \eqref{eq:alpha-asymp1-1},
\eqref{eq:induction4-1} and \eqref{eq:incohc}.
\end{itemize}

\subsection{Proof of Lemma \ref{lemma:proximity-convex-ncvx}\label{subsec:Proof-of-Lemma2}}

Before proceeding, let us introduce some additional convenient notation.
Define 
\begin{equation}
\boldsymbol{Z}:=\boldsymbol{h}\boldsymbol{x}^{\mathsf{H}},\label{eq:defn-Z-hx}
\end{equation}
and denote by $T$ the tangent space of $\bm{Z}$, namely, 
\begin{equation}
T\coloneqq\left\{ \boldsymbol{X}:\boldsymbol{X}=\bm{h}\boldsymbol{v}^{\mathsf{H}}+\boldsymbol{u}\bm{x}^{\mathsf{H}},\boldsymbol{v}\in\mathbb{C}^{K},\boldsymbol{u}\in\mathbb{C}^{K}\right\} .\label{eq:defn-tangent-space}
\end{equation}
Further, define two associated projection operators as follows\begin{subequations}
\begin{align}
\mathcal{P}_{T}\left(\boldsymbol{X}\right) & :=\frac{1}{\|\bm{h}\|_{2}^{2}}\bm{h}\bm{h}^{\mathsf{H}}\boldsymbol{X}+\frac{1}{\|\bm{x}\|_{2}^{2}}\boldsymbol{X}\bm{x}\bm{x}^{\mathsf{H}}-\frac{1}{\|\bm{h}\|_{2}^{2}\|\bm{x}\|_{2}^{2}}\bm{h}\bm{h}^{\mathsf{H}}\bm{X}\bm{x}\bm{x}^{\mathsf{H}},\label{eq:defn-PT-projection}\\
\mathcal{P}_{T^{\perp}}\left(\boldsymbol{X}\right) & :=\left(\boldsymbol{I}-\frac{1}{\|\bm{h}\|_{2}^{2}}\bm{h}\bm{h}^{\mathsf{H}}\right)\boldsymbol{X}\left(\boldsymbol{I}-\frac{1}{\|\bm{x}\|_{2}^{2}}\bm{x}\bm{x}^{\mathsf{H}}\right).\label{eq:defn-PTperp-projection}
\end{align}
\end{subequations}

We further introduce a key lemma below. It proves useful in connecting the first order optimality
conditions of convex and nonconvex formulation. 

\begin{lemma}\label{lem:useful-3}Under the assumptions of Lemma
\ref{lemma:proximity-convex-ncvx}, one has 
\[
\text{\ensuremath{\mathcal{T}}\ensuremath{\left(\boldsymbol{h}\boldsymbol{x}^{\mathsf{H}}-\boldsymbol{h}^{\star}\boldsymbol{x}^{\mathsf{\star H}}\right)}}-\mathcal{A}^{*}\left(\boldsymbol{\xi}\right)=-\frac{\lambda}{\left\Vert \bm{h}\right\Vert _{2}\left\Vert \bm{x}\right\Vert _{2}}\boldsymbol{h}\boldsymbol{x}^{\mathsf{H}}+\boldsymbol{R},
\]
where $\boldsymbol{R}\in\mathbb{C}^{K\times K}$ is some residual
matrix satisfying
\[
\left\Vert \mathcal{P}_{T}\left(\boldsymbol{R}\right)\right\Vert _{\mathrm{F}}\leq2\left\Vert \nabla f\left(\boldsymbol{h},\boldsymbol{x}\right)\right\Vert _{2}\quad\quad\text{and\ensuremath{\quad}\quad\ensuremath{\left\Vert \mathcal{P}_{T^{\bot}}\left(\boldsymbol{R}\right)\right\Vert \leq}\ensuremath{\ensuremath{\lambda}/2}.}
\]
\end{lemma}\begin{proof}See Appendix \ref{subsec:Proof-of-useful-3}.\end{proof}

With these supporting lemmas in hand, we are ready to prove Lemma
\ref{lemma:proximity-convex-ncvx}. Suppose $\boldsymbol{Z}_{\mathsf{cvx}}$
is the minimizer of \eqref{eq:objcvx}.
\begin{enumerate}
\item Let $\boldsymbol{\Delta}\coloneqq\boldsymbol{Z}_{\mathsf{cvx}}-\boldsymbol{h}\boldsymbol{x}^{\mathsf{H}}$.
The optimality of $\bm{Z}_{\mathsf{cvx}}$ yields that 
\[
\left\Vert \mathcal{A}\left(\boldsymbol{h}\boldsymbol{x}^{\mathsf{H}}+\boldsymbol{\Delta}-\boldsymbol{h}^{\star}\boldsymbol{x}^{\star\mathsf{H}}\right)-\bm{\xi}\right\Vert _{2}^{2}+2\lambda\left\Vert \boldsymbol{h}\boldsymbol{x}^{\mathsf{H}}+\boldsymbol{\Delta}\right\Vert _{*}\leq\left\Vert \mathcal{A}\left(\boldsymbol{h}\boldsymbol{x}^{\mathsf{H}}-\boldsymbol{h}^{\star}\boldsymbol{x}^{\star\mathsf{H}}\right)-\bm{\xi}\right\Vert _{2}^{2}+2\lambda\left\Vert \boldsymbol{h}\boldsymbol{x}^{\mathsf{H}}\right\Vert _{*}.
\]
By simple calculation, it leads to
\[
\left\Vert \mathcal{A}\left(\boldsymbol{\Delta}\right)\right\Vert _{2}^{2}\leq-\left\langle \mathcal{T}\left(\boldsymbol{h}\boldsymbol{x}^{\mathsf{H}}-\boldsymbol{h}^{\star}\boldsymbol{x}^{\star\mathsf{H}}\right)-\mathcal{A}^{*}\left(\bm{\xi}\right),\boldsymbol{\Delta}\right\rangle +2\lambda\left\Vert \boldsymbol{h}\boldsymbol{x}^{\mathsf{H}}\right\Vert _{*}-2\lambda\left\Vert \boldsymbol{h}\boldsymbol{x}^{\mathsf{H}}+\boldsymbol{\Delta}\right\Vert _{*}.
\]
The convexity of the nuclear norm gives that for any $\boldsymbol{W}\in\boldsymbol{T}^{\bot}$
with $\left\Vert \boldsymbol{W}\right\Vert \leq1$, there holds
\[
\left\Vert \boldsymbol{h}\boldsymbol{x}^{\mathsf{H}}+\boldsymbol{\Delta}\right\Vert _{*}\geq\left\Vert \boldsymbol{h}\boldsymbol{x}^{\mathsf{H}}\right\Vert _{*}+\left\langle \boldsymbol{p}\boldsymbol{q}^{\mathsf{H}}+\boldsymbol{W},\boldsymbol{\Delta}\right\rangle ,
\]
where we denote by $\boldsymbol{p}\coloneqq\boldsymbol{h}/\left\Vert \boldsymbol{h}\right\Vert _{2}$
and $\boldsymbol{q}\coloneqq\boldsymbol{x}/\left\Vert \boldsymbol{x}\right\Vert _{2}$.
We choose $\boldsymbol{W}$ such that $\left\langle \boldsymbol{W},\boldsymbol{\Delta}\right\rangle =\left\Vert \mathcal{P}_{\boldsymbol{T}^{\bot}}\left(\boldsymbol{\Delta}\right)\right\Vert _{*}$.
Then, combining the above two equations gives rise to 
\begin{align}
0\leq\left\Vert \mathcal{A}\left(\boldsymbol{\Delta}\right)\right\Vert _{2}^{2} & \leq-\left\langle \mathcal{T}\left(\boldsymbol{h}\boldsymbol{x}^{\mathsf{H}}-\boldsymbol{h}^{\star}\boldsymbol{x}^{\star\mathsf{H}}\right)-\mathcal{A}^{*}\left(\bm{\xi}\right),\boldsymbol{\Delta}\right\rangle -2\lambda\left\langle \boldsymbol{p}\boldsymbol{q}^{\mathsf{H}}+\boldsymbol{W},\boldsymbol{\Delta}\right\rangle \nonumber \\
 & =-\left\langle \mathcal{T}\left(\boldsymbol{h}\boldsymbol{x}^{\mathsf{H}}-\boldsymbol{h}^{\star}\boldsymbol{x}^{\star\mathsf{H}}\right)-\mathcal{A}^{*}\left(\bm{\xi}\right),\boldsymbol{\Delta}\right\rangle -2\lambda\left\langle \boldsymbol{p}\boldsymbol{q}^{\mathsf{H}},\boldsymbol{\Delta}\right\rangle -2\lambda\left\Vert \mathcal{P}_{\boldsymbol{T}^{\bot}}\left(\boldsymbol{\Delta}\right)\right\Vert _{*}\nonumber \\
 & \overset{\left(\text{i}\right)}{=}-\left\langle \boldsymbol{R},\boldsymbol{\Delta}\right\rangle -2\lambda\left\Vert \mathcal{P}_{\boldsymbol{T}^{\bot}}\left(\boldsymbol{\Delta}\right)\right\Vert _{*}\nonumber \\
 & =-\left\langle \mathcal{P}_{\boldsymbol{T}}\left(\boldsymbol{R}\right),\boldsymbol{\Delta}\right\rangle -\left\langle \mathcal{P}_{\boldsymbol{T}^{\bot}}\left(\boldsymbol{R}\right),\boldsymbol{\Delta}\right\rangle -2\lambda\left\Vert \mathcal{P}_{\boldsymbol{T}^{\bot}}\left(\boldsymbol{\Delta}\right)\right\Vert _{*},\label{eq:lem-connect-1}
\end{align}
where $\bm{R}$ in (i) is defined in Lemma \ref{lem:useful-3}. Hence,
\begin{align*}
 & -\left\Vert \mathcal{P}_{\boldsymbol{T}}\left(\boldsymbol{R}\right)\right\Vert _{\text{F}}\left\Vert \mathcal{P}_{\boldsymbol{T}}\left(\boldsymbol{\Delta}\right)\right\Vert _{\text{F}}-\left\Vert \mathcal{P}_{\boldsymbol{T}^{\bot}}\left(\boldsymbol{R}\right)\right\Vert \left\Vert \mathcal{P}_{\boldsymbol{T}^{\bot}}\left(\boldsymbol{\Delta}\right)\right\Vert _{*}+2\lambda\left\Vert \mathcal{P}_{\boldsymbol{T}^{\bot}}\left(\boldsymbol{\Delta}\right)\right\Vert _{*}\\
\leq & \left\langle \mathcal{P}_{\boldsymbol{T}}\left(\boldsymbol{R}\right),\boldsymbol{\Delta}\right\rangle +\left\langle \mathcal{P}_{\boldsymbol{T}^{\bot}}\left(\boldsymbol{R}\right),\boldsymbol{\Delta}\right\rangle +2\lambda\left\Vert \mathcal{P}_{\boldsymbol{T}^{\bot}}\left(\boldsymbol{\Delta}\right)\right\Vert _{*}\leq0.
\end{align*}
Lemma \ref{lem:useful-3} gives $\left\Vert \mathcal{P}_{\boldsymbol{T}^{\bot}}\left(\boldsymbol{R}\right)\right\Vert \leq\lambda/2$,
then we have
\[
\left\Vert \mathcal{P}_{\boldsymbol{T}}\left(\boldsymbol{R}\right)\right\Vert _{\text{F}}\left\Vert \mathcal{P}_{\boldsymbol{T}}\left(\boldsymbol{\Delta}\right)\right\Vert _{\text{F}}\geq-\left\Vert \mathcal{P}_{\boldsymbol{T}^{\bot}}\left(\boldsymbol{R}\right)\right\Vert \left\Vert \mathcal{P}_{\boldsymbol{T}^{\bot}}\left(\boldsymbol{\Delta}\right)\right\Vert _{*}+2\lambda\left\Vert \mathcal{P}_{\boldsymbol{T}^{\bot}}\left(\boldsymbol{\Delta}\right)\right\Vert _{*}\geq\frac{3\lambda}{2}\left\Vert \mathcal{P}_{\boldsymbol{T}^{\bot}}\left(\boldsymbol{\Delta}\right)\right\Vert _{*},
\]
and it immediately reveals that
\begin{align*}
\left\Vert \mathcal{P}_{\boldsymbol{T}^{\bot}}\left(\boldsymbol{\Delta}\right)\right\Vert _{*} & \leq\frac{2}{3\lambda}\left\Vert \mathcal{P}_{\boldsymbol{T}}\left(\boldsymbol{R}\right)\right\Vert _{\text{F}}\left\Vert \mathcal{P}_{\boldsymbol{T}}\left(\boldsymbol{\Delta}\right)\right\Vert _{\text{F}}\\
 & \leq\frac{4}{3\lambda}\left\Vert \nabla f\left(\boldsymbol{h},\boldsymbol{x}\right)\right\Vert _{2}\left\Vert \mathcal{P}_{\boldsymbol{T}}\left(\boldsymbol{\Delta}\right)\right\Vert _{\text{F}}\\
 & \leq C\frac{4}{3m^{10}}\left\Vert \mathcal{P}_{\boldsymbol{T}}\left(\boldsymbol{\Delta}\right)\right\Vert _{\text{F}},
\end{align*}
where the second inequality invokes Lemma \ref{lem:useful-3}. We
then arrive at
\begin{equation}
\left\Vert \mathcal{P}_{\boldsymbol{T}^{\bot}}\left(\boldsymbol{\Delta}\right)\right\Vert _{\text{F}}\leq\left\Vert \mathcal{P}_{\boldsymbol{T}^{\bot}}\left(\boldsymbol{\Delta}\right)\right\Vert _{*}\leq C\frac{4}{3m^{10}}\left\Vert \mathcal{P}_{\boldsymbol{T}}\left(\boldsymbol{\Delta}\right)\right\Vert _{\text{F}}\leq\left\Vert \mathcal{P}_{\boldsymbol{T}}\left(\boldsymbol{\Delta}\right)\right\Vert _{\text{F}}.\label{eq:lemma-connect-step1}
\end{equation}
\item Next, we return to \eqref{eq:lem-connect-1} to deduce that 
\begin{align}
\left\Vert \mathcal{A}\left(\boldsymbol{\Delta}\right)\right\Vert _{2}^{2} & \leq-\left\langle \mathcal{P}_{\boldsymbol{T}}\left(\boldsymbol{R}\right),\boldsymbol{\Delta}\right\rangle -\left\langle \mathcal{P}_{\boldsymbol{T}^{\bot}}\left(\boldsymbol{R}\right),\boldsymbol{\Delta}\right\rangle -2\lambda\left\Vert \mathcal{P}_{\boldsymbol{T}^{\bot}}\left(\boldsymbol{\Delta}\right)\right\Vert _{*}\nonumber \\
 & \leq\left\Vert \mathcal{P}_{\boldsymbol{T}}\left(\boldsymbol{R}\right)\right\Vert _{\text{F}}\left\Vert \mathcal{P}_{\boldsymbol{T}}\left(\boldsymbol{\Delta}\right)\right\Vert _{\text{F}}+\left\Vert \mathcal{P}_{\boldsymbol{T}^{\bot}}\left(\boldsymbol{R}\right)\right\Vert \left\Vert \mathcal{P}_{\boldsymbol{T}^{\bot}}\left(\boldsymbol{\Delta}\right)\right\Vert _{*}-2\lambda\left\Vert \mathcal{P}_{\boldsymbol{T}^{\bot}}\left(\boldsymbol{\Delta}\right)\right\Vert _{*}\\
 & \overset{(\text{i})}{\leq}\left\Vert \mathcal{P}_{\boldsymbol{T}}\left(\boldsymbol{R}\right)\right\Vert _{\text{F}}\left\Vert \mathcal{P}_{\boldsymbol{T}}\left(\boldsymbol{\Delta}\right)\right\Vert _{\text{F}}-\frac{3\lambda}{2}\left\Vert \mathcal{P}_{\boldsymbol{T}^{\bot}}\left(\boldsymbol{\Delta}\right)\right\Vert _{*}\nonumber \\
 & \leq\left\Vert \mathcal{P}_{\boldsymbol{T}}\left(\boldsymbol{R}\right)\right\Vert _{\text{F}}\left\Vert \mathcal{P}_{\boldsymbol{T}}\left(\boldsymbol{\Delta}\right)\right\Vert _{\text{F}}\label{eq:lemma2step2}\\
 & \overset{(\text{ii})}{\leq}2\left\Vert \nabla f\left(\boldsymbol{h},\boldsymbol{x}\right)\right\Vert _{2}\left\Vert \boldsymbol{\Delta}\right\Vert _{\text{F}},
\end{align}
where (i) and (ii) come from Lemma \ref{lem:useful-3}.
\item For the final step, we turn to lower bound $\left\Vert \mathcal{A}\left(\boldsymbol{\Delta}\right)\right\Vert _{\text{F}}$.
One has
\begin{align}
\left\Vert \mathcal{A}\left(\boldsymbol{\Delta}\right)\right\Vert _{2} & =\left\Vert \mathcal{A}\left(\mathcal{P}_{\boldsymbol{T}}\left(\boldsymbol{\Delta}\right)\right)+\mathcal{A}\left(\mathcal{P}_{\boldsymbol{T}^{\bot}}\left(\boldsymbol{\Delta}\right)\right)\right\Vert _{2}\nonumber \\
 & \geq\left\Vert \mathcal{A}\left(\mathcal{P}_{\boldsymbol{T}}\left(\boldsymbol{\Delta}\right)\right)\right\Vert _{2}-\left\Vert \mathcal{A}\left(\mathcal{P}_{\boldsymbol{T}^{\bot}}\left(\boldsymbol{\Delta}\right)\right)\right\Vert _{2}\nonumber \\
 & \geq\left\Vert \mathcal{P}_{\boldsymbol{T}}\left(\boldsymbol{\Delta}\right)\right\Vert _{\text{F}}/4-\sqrt{2K\log K+\gamma\log m}\left\Vert \mathcal{P}_{\boldsymbol{T}^{\bot}}\left(\boldsymbol{\Delta}\right)\right\Vert _{\text{F}},\label{eq:lemma-connect-step3}
\end{align}
where the last inequality comes from Lemma \ref{lemma:inj} and Lemma
\ref{lemma:normbound}. Since \eqref{eq:lemma-connect-step1} gives
\[
\sqrt{2K\log K+\gamma\log m}\left\Vert \mathcal{P}_{\boldsymbol{T}^{\bot}}\left(\boldsymbol{\Delta}\right)\right\Vert _{\text{F}}\leq\sqrt{2K\log K+\gamma\log m}\times C\frac{4}{3m^{10}}\left\Vert \mathcal{P}_{\boldsymbol{T}}\left(\boldsymbol{\Delta}\right)\right\Vert _{\text{F}}\leq\frac{1}{8}\left\Vert \mathcal{P}_{\boldsymbol{T}}\left(\boldsymbol{\Delta}\right)\right\Vert _{\text{F}},
\]
as long as $m\gg K$, \eqref{eq:lemma-connect-step3} yields
\[
\left\Vert \mathcal{A}\left(\boldsymbol{\Delta}\right)\right\Vert _{2}\geq\frac{1}{8}\left\Vert \mathcal{P}_{\boldsymbol{T}}\left(\boldsymbol{\Delta}\right)\right\Vert _{\text{F}}.
\]
In addition, \eqref{eq:lemma-connect-step1} implies
\[
\left\Vert \boldsymbol{\Delta}\right\Vert _{\text{F}}\leq\left\Vert \mathcal{P}_{\boldsymbol{T}}\left(\boldsymbol{\Delta}\right)\right\Vert _{\text{F}}+\left\Vert \mathcal{P}_{\boldsymbol{T}^{\bot}}\left(\boldsymbol{\Delta}\right)\right\Vert _{\text{F}}\leq2\left\Vert \mathcal{P}_{\boldsymbol{T}}\left(\boldsymbol{\Delta}\right)\right\Vert _{\text{F}}.
\]
Consequently,
\begin{equation}
\left\Vert \mathcal{A}\left(\boldsymbol{\Delta}\right)\right\Vert _{2}\geq\frac{1}{8}\left\Vert \mathcal{P}_{\boldsymbol{T}}\left(\boldsymbol{\Delta}\right)\right\Vert _{\text{F}}\geq\frac{1}{16}\left\Vert \boldsymbol{\Delta}\right\Vert _{\text{F}}.\label{eq:lemma2step3}
\end{equation}
\end{enumerate}
Combining \eqref{eq:lemma2step2} and \eqref{eq:lemma2step3}, we
have 
\[
\frac{1}{256}\left\Vert \boldsymbol{\Delta}\right\Vert _{\text{F}}^{2}\leq\left\Vert \mathcal{A}\left(\boldsymbol{\Delta}\right)\right\Vert _{2}^{2}\leq2\left\Vert \nabla f\left(\boldsymbol{h},\boldsymbol{x}\right)\right\Vert _{2}\left\Vert \boldsymbol{\Delta}\right\Vert _{\text{F}},
\]
and therefore
\[
\left\Vert \boldsymbol{\Delta}\right\Vert _{\text{F}}\lesssim\left\Vert \nabla f\left(\boldsymbol{h},\boldsymbol{x}\right)\right\Vert _{2}.
\]

\subsubsection{Proof of Lemma \ref{lem:useful-3}\label{subsec:Proof-of-useful-3}}

Recall the definition of $\mathcal{T}^{\mathsf{debias}}$ in \eqref{eq:operatordef}.
Letting
\begin{equation}
\boldsymbol{p}=\frac{1}{\left\Vert \boldsymbol{h}\right\Vert _{2}}\boldsymbol{h}\qquad\text{and}\qquad\boldsymbol{q}=\frac{1}{\left\Vert \boldsymbol{x}\right\Vert _{2}}\boldsymbol{x}\label{eq:defn-p-q-unit}
\end{equation}
and rearranging terms, we can write
\begin{align}
\boldsymbol{h}^{\star}\boldsymbol{x}^{\star\mathsf{H}}+\mathcal{T}^{\mathsf{debias}}\left(\boldsymbol{h}^{\star}\boldsymbol{x}^{\star\mathsf{H}}-\boldsymbol{hx}^{\mathsf{H}}\right)+\mathcal{A}^{*}\left(\boldsymbol{\xi}\right) & =\boldsymbol{hx}^{\mathsf{H}}+\lambda\boldsymbol{p}\boldsymbol{q}^{\mathsf{H}}+\boldsymbol{R}\label{eq:lemma-useful-31}
\end{align}
for some matrix $\boldsymbol{R}$. In addition, in view of the small
gradient assumption \eqref{eq:small-gradient}, one has\begin{subequations}\label{eq:lemma-useful-3-defn-r}
	\begin{align}
	\left[\boldsymbol{h}^{\star}\boldsymbol{x}^{\star\mathsf{H}}+\mathcal{T}^{\mathsf{debias}}\left(\boldsymbol{h}^{\star}\boldsymbol{x}^{\star\mathsf{H}}-\boldsymbol{hx}^{\mathsf{H}}\right)+\mathcal{A}^{*}\left(\boldsymbol{\xi}\right)\right]\boldsymbol{x} & =\boldsymbol{hx}^{\mathsf{H}}\boldsymbol{x}+\lambda\boldsymbol{h}-\boldsymbol{r}_{1}\\
	\left[\boldsymbol{h}^{\star}\boldsymbol{x}^{\star\mathsf{H}}+\mathcal{T}^{\mathsf{debias}}\left(\boldsymbol{h}^{\star}\boldsymbol{x}^{\star\mathsf{H}}-\boldsymbol{hx}^{\mathsf{H}}\right)+\mathcal{A}^{*}\left(\boldsymbol{\xi}\right)\right]^{\mathsf{H}}\boldsymbol{h} & =\boldsymbol{x}\boldsymbol{h}^{\mathsf{H}}\boldsymbol{h}+\lambda\boldsymbol{x}-\boldsymbol{r}_{2}
	\end{align}
\end{subequations}for some vectors $\bm{r}_{1},\bm{r}_{2}\in\mathbb{C}^{K}$
obeying\begin{subequations}\label{subequation:small-r1-r2}
	\begin{align}
	\left\Vert \bm{r}_{1}\right\Vert _{2}=\left\Vert \lambda\boldsymbol{h}-\left(\mathcal{T}\left(\boldsymbol{h}^{\star}\boldsymbol{x}^{\star\mathsf{H}}-\boldsymbol{hx}^{\mathsf{H}}\right)+\mathcal{A}^{*}\left(\boldsymbol{\xi}\right)\right)\boldsymbol{x}\right\Vert _{2} & \leq\left\Vert \nabla f\left(\boldsymbol{h},\boldsymbol{x}\right)\right\Vert _{2}\leq C\frac{\lambda}{m^{10}},\label{eq:smallr1}\\
	\left\Vert \bm{r}_{2}\right\Vert _{2}=\left\Vert \lambda\boldsymbol{x}-\left(\mathcal{T}\left(\boldsymbol{h}^{\star}\boldsymbol{x}^{\star\mathsf{H}}-\boldsymbol{hx}^{\mathsf{H}}\right)+\mathcal{A}^{*}\left(\boldsymbol{\xi}\right)\right)^{\mathsf{H}}\boldsymbol{h}\right\Vert _{2} & \leq\left\Vert \nabla f\left(\boldsymbol{h},\boldsymbol{x}\right)\right\Vert _{2}\leq C\frac{\lambda}{m^{10}}.\label{eq:smallr2}
	\end{align}
\end{subequations}In what follows, we make of these properties to
control the size of $\bm{R}$. 
\begin{enumerate}
	\item We start by upper bounding $\left\Vert \mathcal{P}_{T}\left(\bm{R}\right)\right\Vert _{\text{F}}$
	as follows
	\begin{align*}
	\left\Vert \mathcal{P}_{T}\left(\bm{R}\right)\right\Vert _{\text{F}} & =\left\Vert \boldsymbol{p}\boldsymbol{p}^{\mathsf{H}}\bm{R}\left(\boldsymbol{I}_{K}-\boldsymbol{q}\boldsymbol{q}^{\mathsf{H}}\right)+\bm{R}\boldsymbol{q}\boldsymbol{q}^{\mathsf{H}}\right\Vert _{\text{F}}\\
	& \leq\left\Vert \boldsymbol{p}\right\Vert _{2}\left\Vert \boldsymbol{p}^{\mathsf{H}}\bm{R}\right\Vert _{2}\left\Vert \boldsymbol{I}_{K}-\boldsymbol{q}\boldsymbol{q}^{\mathsf{H}}\right\Vert +\left\Vert \bm{R}\boldsymbol{q}\right\Vert _{2}\left\Vert \boldsymbol{q}\right\Vert _{2}\\
	& \leq\left\Vert \boldsymbol{p}^{\mathsf{H}}\bm{R}\right\Vert _{2}+\left\Vert \bm{R}\boldsymbol{q}\right\Vert _{2},
	\end{align*}
	where $\bm{p}$ and $\bm{q}$ are unit vectors defined in \eqref{eq:defn-p-q-unit}.
	Recognizing that $\left\Vert \bm{h}\right\Vert _{2}=\left\Vert \bm{x}\right\Vert _{2}$
	(cf.~\eqref{eq:hx-properties}), we can use \eqref{eq:lemma-useful-31}
	and \eqref{eq:lemma-useful-3-defn-r} to obtain
	\begin{align*}
	\bm{R}^{\mathsf{H}}\boldsymbol{p} & =-\frac{\boldsymbol{r}_{2}}{\left\Vert \boldsymbol{h}\right\Vert _{2}}+\lambda\frac{\left\Vert \boldsymbol{x}\right\Vert _{2}}{\left\Vert \boldsymbol{h}\right\Vert _{2}}\bm{q}-\lambda\frac{\left\Vert \boldsymbol{h}\right\Vert _{2}}{\left\Vert \boldsymbol{x}\right\Vert _{2}}\bm{q}=-\frac{\boldsymbol{r}_{2}}{\left\Vert \boldsymbol{h}\right\Vert _{2}}\quad\text{and\ensuremath{\quad\bm{R}\boldsymbol{q}=-\frac{\boldsymbol{r}_{1}}{\left\Vert \boldsymbol{x}\right\Vert _{2}}}}.
	\end{align*}
	These together with \eqref{subequation:small-r1-r2} yield
	\begin{equation}
	\left\Vert \mathcal{P}_{T}\left(\bm{R}\right)\right\Vert _{\text{F}}\leq\left\Vert \boldsymbol{p}^{\mathsf{H}}\bm{R}\right\Vert _{2}+\left\Vert \bm{R}\boldsymbol{q}\right\Vert _{2}\leq2\left\Vert \nabla f\left(\boldsymbol{h},\boldsymbol{x}\right)\right\Vert _{2}\leq2C\frac{\lambda}{m^{10}}.\label{eq:UB_PTR_123}
	\end{equation}
	
	\item We them move on to control $\mathcal{P}_{T^{\bot}}\left(\boldsymbol{R}\right)$.
	Continue the relation \eqref{eq:lemma-useful-31} to derive
	\begin{equation}
	\boldsymbol{h}^{\star}\boldsymbol{x}^{\star\mathsf{H}}+\mathcal{T}^{\mathsf{debias}}\left(\boldsymbol{h}^{\star}\boldsymbol{x}^{\star\mathsf{H}}-\boldsymbol{hx}^{\mathsf{H}}\right)+\mathcal{A}^{*}\left(\boldsymbol{\xi}\right)-\mathcal{P}_{T}\left(\bm{R}\right)=\boldsymbol{p}\left(\left\Vert \boldsymbol{h}\right\Vert _{2}\left\Vert \boldsymbol{x}\right\Vert _{2}+\lambda\frac{\left\Vert \boldsymbol{h}\right\Vert _{2}}{\left\Vert \boldsymbol{x}\right\Vert _{2}}\right)\boldsymbol{q}^{\mathsf{H}}+\mathcal{P}_{T^{\bot}}\left(\bm{R}\right),\label{eq:lemma-useful-32}
	\end{equation}
	where we have used the assumption $\left\Vert \bm{h}\right\Vert _{2}/\left\Vert \bm{x}\right\Vert _{2}=1$
	(cf.~\eqref{eq:hx-properties}). Combine this with Lemma \ref{lemma:T-uniform-mean},
	Lemma \ref{lemma:noise} and \eqref{eq:UB_PTR_123} to derive
	\begin{align*}
	\left\Vert \mathcal{T}^{\mathsf{debias}}\left(\boldsymbol{h}^{\star}\boldsymbol{x}^{\star\mathsf{H}}-\boldsymbol{hx}^{\mathsf{H}}\right)+\mathcal{A}^{*}\left(\boldsymbol{\xi}\right)-\mathcal{P}_{T}\left(\bm{R}\right)\right\Vert  & \le\left\Vert \mathcal{T}^{\mathsf{debias}}\left(\boldsymbol{h}^{\star}\boldsymbol{x}^{\star\mathsf{H}}-\boldsymbol{hx}^{\mathsf{H}}\right)\right\Vert +\left\Vert \mathcal{A}^{*}\left(\boldsymbol{\xi}\right)\right\Vert +\left\Vert \mathcal{P}_{T}\left(\bm{R}\right)\right\Vert _{\text{F}}\\
	& \leq\frac{\lambda}{8}+\frac{\lambda}{8}+2C\frac{\lambda}{m^{10}}\\
	& <\frac{\lambda}{2},
	\end{align*}
	where the last inequality invokes the assumption \eqref{eq:smallgradient}.
	Invoking \eqref{eq:lemma-useful-32} and Weyl's inequality give
	\begin{align*}
	\sigma_{i}\left[\boldsymbol{p}\left(\left\Vert \boldsymbol{h}\right\Vert _{2}\left\Vert \boldsymbol{x}\right\Vert _{2}+\lambda\frac{\left\Vert \boldsymbol{h}\right\Vert _{2}}{\left\Vert \boldsymbol{x}\right\Vert _{2}}\right)\bm{q}^{\mathsf{H}}+\mathcal{P}_{T^{\bot}}\left(\bm{R}\right)\right] & \leq\sigma_{i}\left(\boldsymbol{h}^{\star}\boldsymbol{x}^{\star\mathsf{H}}\right)+\left\Vert \mathcal{T}^{\mathsf{debias}}\left(\boldsymbol{h}^{\star}\boldsymbol{x}^{\star\mathsf{H}}-\boldsymbol{hx}^{\mathsf{H}}\right)+\mathcal{A}^{*}\left(\boldsymbol{\xi}\right)-\mathcal{P}_{T}\left(\bm{R}\right)\right\Vert \\
	& <\lambda/2,
	\end{align*}
	for $K\geq i\geq2$. Additionally, when $i=1$, we have 
	\[
	\sigma_{1}\left[\boldsymbol{p}\left(\left\Vert \boldsymbol{h}\right\Vert _{2}\left\Vert \boldsymbol{x}\right\Vert _{2}+\lambda\frac{\left\Vert \boldsymbol{h}\right\Vert _{2}}{\left\Vert \boldsymbol{x}\right\Vert _{2}}\right)\bm{q}^{\mathsf{H}}\right]=\left\Vert \boldsymbol{h}\right\Vert _{2}\left\Vert \boldsymbol{x}\right\Vert _{2}+\lambda\frac{\left\Vert \boldsymbol{h}\right\Vert _{2}}{\left\Vert \boldsymbol{x}\right\Vert _{2}}\geq\lambda/2.
	\]
	This indicates that at least $K-1$ singular values of $\boldsymbol{p}\left(\left\Vert \boldsymbol{h}\right\Vert _{2}\left\Vert \boldsymbol{x}\right\Vert _{2}+\lambda\left\Vert \boldsymbol{h}\right\Vert _{2}/\left\Vert \boldsymbol{x}\right\Vert _{2}\right)\bm{q}^{\mathsf{H}}+\mathcal{P}_{T^{\bot}}\left(\bm{R}\right)$
	are no larger than $\lambda/2$, and these singular values cannot
	correspond to the direction of $\bm{p}\bm{q}^{\mathsf{H}}$. As a
	consequence, we conclude that
	\[
	\left\Vert \mathcal{P}_{T^{\bot}}\left(\boldsymbol{R}\right)\right\Vert \le\lambda/2.
	\]
\end{enumerate}

\subsection{Proof of Lemma \ref{lemma:T-uniform-mean}\label{subsec:Proof-of-Lemma8innoisy}}
For notational convenience, we define $\mathcal{T}^{\mathsf{debias}}$ by subtracting the expectation from $\mathcal{T}$ as follows:
\begin{align}
\mathcal{T}^{\mathsf{debias}}\left(\boldsymbol{Z}\right) & \coloneqq\mathcal{T}\left(\boldsymbol{Z}\right)-\bm{Z}=\left(\mathcal{A}^{*}\mathcal{A}-\mathcal{I}\right)\left(\bm{Z}\right)=\sum_{j=1}^{m}\boldsymbol{b}_{j}\boldsymbol{b}_{j}^{\mathsf{H}}\bm{Z}\boldsymbol{a}_{j}\boldsymbol{a}_{j}^{\mathsf{H}}-\boldsymbol{Z}.\nonumber 
\end{align}

For any
fixed vectors $\bm{h}$ and $\bm{x}$, we make note of the following
decomposition 
\begin{align*}
\bm{h}\bm{x}^{\mathsf{H}}-\bm{h}^{\star}\bm{x}^{\star\mathsf{H}} & =\left(\bm{\Delta}_{\bm{h}}+\bm{h}^{\star}\right)\left(\bm{\Delta}_{\bm{x}}+\bm{x}^{\star}\right)^{\mathsf{H}}-\bm{h}^{\star}\bm{x}^{\star\mathsf{H}}\\
& =\bm{h}^{\star}\bm{\Delta}_{\bm{x}}^{\mathsf{H}}+\bm{\Delta}_{\bm{h}}\bm{x}^{\star\mathsf{H}}+\bm{\Delta}_{\bm{h}}\bm{\Delta}_{\bm{x}}^{\mathsf{H}},
\end{align*}
which together with the triangle inequality gives
\begin{align*}
\left\Vert \mathcal{T}^{\mathsf{debias}}\left(\bm{h}\bm{x}^{\mathsf{H}}-\bm{h}^{\star}\bm{x}^{\star\mathsf{H}}\right)\right\Vert  & \leq\underbrace{\left\Vert \mathcal{T}^{\mathsf{debias}}\left(\bm{h}^{\star}\bm{\Delta}_{\bm{x}}^{\mathsf{H}}\right)\right\Vert }_{\eqqcolon\beta_{1}}+\underbrace{\left\Vert \mathcal{T}^{\mathsf{debias}}\left(\bm{\Delta}_{\bm{h}}\bm{x}^{\star}{}^{\mathsf{H}}\right)\right\Vert }_{\eqqcolon\beta_{2}}+\underbrace{\left\Vert \mathcal{T}^{\mathsf{debias}}\left(\bm{\Delta}_{\bm{h}}\bm{\Delta}_{\bm{x}}^{\mathsf{H}}\right)\right\Vert }_{\eqqcolon\beta_{3}}\text{.}
\end{align*}
In what follows, we shall upper bound $\beta_{1}$, $\beta_{2}$ and
$\beta_{3}$ separately. 
\begin{enumerate}
	\item For any fixed $\bm{x}$, the quantity $\beta_{1}$ is concerned with
	a matrix that can be written explicitly as follows
	\[
	\mathcal{T}^{\mathsf{debias}}\left(\bm{h}^{\star}\bm{\Delta}_{\bm{x}}^{\mathsf{H}}\right)=\sum_{j=1}^{m}\bm{b}_{j}\bm{b}_{j}^{\mathsf{H}}\bm{h}^{\star}\bm{\Delta}_{\bm{x}}^{\mathsf{H}}\left(\bm{a}_{j}\bm{a}_{j}^{\mathsf{H}}-\bm{I}_{K}\right).
	\]
	Consequently, for any fixed unit vectors $\bm{u}$, $\bm{v}\in\mathbb{C}^{K}$
	one has 
	\[
	\bm{u}^{\mathsf{H}}\mathcal{T}^{\mathsf{debias}}\left(\bm{h}^{\star}\bm{\Delta}_{\bm{x}}^{\mathsf{H}}\right)\bm{v}=\sum_{j=1}^{m}\left(\bm{u}^{\mathsf{H}}\bm{b}_{j}\bm{b}_{j}^{\mathsf{H}}\bm{h}^{\star}\bm{\Delta}_{\bm{x}}^{\mathsf{H}}\bm{a}_{j}\bm{a}_{j}^{\mathsf{H}}\bm{v}-\bm{u}^{\mathsf{H}}\bm{b}_{j}\bm{b}_{j}^{\mathsf{H}}\bm{h}^{\star}\bm{\Delta}_{\bm{x}}^{\mathsf{H}}\bm{v}\right),
	\]
	which is essentially a sum of independent variables. Letting $r\coloneqq\lambda+\sigma\sqrt{K\log m}$
	and $C_{4}\coloneqq10\max\left\{ C_{1},C_{3},1\right\} $, we can
	deduce that
	\begin{align*}
	& \sum_{j=1}^{m}\Bigg(\underbrace{\bm{u}^{\mathsf{H}}\bm{b}_{j}\bm{b}_{j}^{\mathsf{H}}\bm{h}^{\star}\bm{\Delta}_{\bm{x}}^{\mathsf{H}}\bm{a}_{j}\bm{a}_{j}^{\mathsf{H}}\bm{v}\ind_{\left\{ \left|\bm{\Delta}_{\bm{x}}^{\mathsf{H}}\bm{a}_{j}\right|\leq C_{4}r\sqrt{\log m}\right\} }}_{\eqqcolon z_{j}}-\bm{u}^{\mathsf{H}}\bm{b}_{j}\bm{b}_{j}^{\mathsf{H}}\bm{h}^{\star}\bm{\Delta}_{\bm{x}}^{\mathsf{H}}\bm{v}\Bigg)\\
	& \qquad=\sum_{j=1}^{m}\left(z_{j}-\mathbb{E}\left[z_{j}\right]\right)+\sum_{j=1}^{m}\left(\mathbb{E}\left[\bm{u}^{\mathsf{H}}\bm{b}_{j}\bm{b}_{j}^{\mathsf{H}}\bm{h}^{\star}\bm{\Delta}_{\bm{x}}^{\mathsf{H}}\bm{a}_{j}\bm{a}_{j}^{\mathsf{H}}\bm{v}\ind_{\left\{ \left|\bm{\Delta}_{\bm{x}}^{\mathsf{H}}\bm{a}_{j}\right|\leq C_{4}r\sqrt{\log m}\right\} }\right]-\bm{u}^{\mathsf{H}}\bm{b}_{j}\bm{b}_{j}^{\mathsf{H}}\bm{h}^{\star}\bm{\Delta}_{\bm{x}}^{\mathsf{H}}\bm{v}\right)\\
	& \qquad=\sum_{j=1}^{m}\left(z_{j}-\mathbb{E}\left[z_{j}\right]\right)+\sum_{j=1}^{m}\left(\mathbb{E}\left[\bm{u}^{\mathsf{H}}\bm{b}_{j}\bm{b}_{j}^{\mathsf{H}}\bm{h}^{\star}\bm{\Delta}_{\bm{x}}^{\mathsf{H}}\bm{a}_{j}\bm{a}_{j}^{\mathsf{H}}\bm{v}\ind_{\left\{ \left|\bm{\Delta}_{\bm{x}}^{\mathsf{H}}\bm{a}_{j}\right|\leq C_{4}r\sqrt{\log m}\right\} }\right]-\mathbb{E}\left[\bm{u}^{\mathsf{H}}\bm{b}_{j}\bm{b}_{j}^{\mathsf{H}}\bm{h}^{\star}\bm{\Delta}_{\bm{x}}^{\mathsf{H}}\bm{a}_{j}\bm{a}_{j}^{\mathsf{H}}\bm{v}\right]\right)\\
	& \qquad=\underbrace{\sum_{j=1}^{m}\left(z_{j}-\mathbb{E}\left[z_{j}\right]\right)}_{\eqqcolon\omega_{1}}-\underbrace{\sum_{j=1}^{m}\mathbb{E}\left[\bm{u}^{\mathsf{H}}\bm{b}_{j}\bm{b}_{j}^{\mathsf{H}}\bm{h}^{\star}\bm{\Delta}_{\bm{x}}^{\mathsf{H}}\bm{a}_{j}\bm{a}_{j}^{\mathsf{H}}\bm{v}\ind_{\left\{ \left|\bm{\Delta}_{\bm{x}}^{\mathsf{H}}\bm{a}_{j}\right|>C_{4}r\sqrt{\log m}\right\} }\right]}_{\eqqcolon\omega_{2}}.
	\end{align*}
	
	\begin{itemize}
		\item The term $\omega_{2}$ can be controlled by Cauchy-Schwarz as follows
		\begin{align}
		\left|\omega_{2}\right|= & \Bigg|\sum_{j=1}^{m}\mathbb{E}\left[\bm{u}^{\mathsf{H}}\bm{b}_{j}\bm{b}_{j}^{\mathsf{H}}\bm{h}^{\star}\bm{\Delta}_{\bm{x}}^{\mathsf{H}}\bm{a}_{j}\bm{a}_{j}^{\mathsf{H}}\bm{v}\ind_{\left\{ \left|\bm{\Delta}_{\bm{x}}^{\mathsf{H}}\bm{a}_{j}\right|>C_{4}r\sqrt{\log m}\right\} }\right]\Bigg|\nonumber \\
		\overset{\text{(i)}}{\leq} & \sum_{j=1}^{m}\sqrt{\mathbb{E}\left[\left|\bm{u}^{\mathsf{H}}\bm{b}_{j}\bm{b}_{j}^{\mathsf{H}}\bm{h}^{\star}\bm{\Delta}_{\bm{x}}^{\mathsf{H}}\bm{a}_{j}\bm{a}_{j}^{\mathsf{H}}\bm{v}\right|^{2}\right]\mathbb{P}\left[\left|\bm{\Delta}_{\bm{x}}^{\mathsf{H}}\bm{a}_{j}\right|>C_{4}r\sqrt{\log m}\right]}\nonumber \\
		\overset{\text{(ii)}}{\leq} & \sum_{j=1}^{m}\left|\bm{u}^{\mathsf{H}}\bm{b}_{j}\bm{b}_{j}^{\mathsf{H}}\bm{h}^{\star}\right|\sqrt{\left(2\left|\bm{\Delta}_{\bm{x}}^{\mathsf{H}}\bm{v}\right|^{2}+\left\Vert \bm{\Delta}_{\bm{x}}\right\Vert _{2}^{2}\left\Vert \bm{v}\right\Vert _{2}^{2}\right)2\exp\left(-\frac{C_{4}^{2}r^{2}\log m}{2\left\Vert \bm{\Delta}_{\bm{x}}\right\Vert _{2}^{2}}\right)}\nonumber \\
		\leq & \sum_{j=1}^{m}\left|\bm{u}^{\mathsf{H}}\bm{b}_{j}\bm{b}_{j}^{\mathsf{H}}\bm{h}^{\star}\right|\sqrt{6\left\Vert \bm{\Delta}_{\bm{x}}\right\Vert _{2}^{2}\exp\left(-50\log m\right)}\nonumber \\
		\overset{\text{(iii)}}{\leq} & \sum_{j=1}^{m}\left(\left|\bm{u}^{\mathsf{H}}\bm{b}_{j}\right|^{2}+\left|\bm{b}_{j}^{\mathsf{H}}\bm{h}^{\star}\right|^{2}\right)\frac{\sqrt{6}\left\Vert \bm{\Delta}_{\bm{x}}\right\Vert _{2}}{2m^{25}}\label{eq:lemma8omega2}\\
		\overset{(\text{iv})}{\leq} & \left(1+\mu^{2}\right)\frac{\sqrt{6}\left\Vert \bm{\Delta}_{\bm{x}}\right\Vert _{2}}{2m^{25}}\nonumber \\
		\overset{(\text{v})}{\leq} & \frac{\left\Vert \bm{\Delta}_{\bm{x}}\right\Vert _{2}}{m^{24}}.\nonumber 
		\end{align}
		Here, (i) follows from the Cauchy-Schwarz inequality, and (ii) comes
		from the property of sub-Gaussian variable $\bm{\Delta}_{\bm{x}}^{\mathsf{H}}\bm{a}_{j}$
		and
		\begin{align}
		\mathbb{E}\left[\left|\bm{u}^{\mathsf{H}}\bm{b}_{j}\bm{b}_{j}^{\mathsf{H}}\bm{h}^{\star}\bm{\Delta}_{\bm{x}}^{\mathsf{H}}\bm{a}_{j}\bm{a}_{j}^{\mathsf{H}}\bm{v}\right|^{2}\right]= & \left|\bm{u}^{\mathsf{H}}\bm{b}_{j}\bm{b}_{j}^{\mathsf{H}}\bm{h}^{\star}\right|^{2}\mathbb{E}\left[\left|\bm{\Delta}_{\bm{x}}^{\mathsf{H}}\bm{a}_{j}\bm{a}_{j}^{\mathsf{H}}\bm{v}\right|^{2}\right]\nonumber \\
		= & \left|\bm{u}^{\mathsf{H}}\bm{b}_{j}\bm{b}_{j}^{\mathsf{H}}\bm{h}^{\star}\right|\left(2\left|\bm{\Delta}_{\bm{x}}^{\mathsf{H}}\bm{v}\right|^{2}+\left\Vert \bm{\Delta}_{\bm{x}}\right\Vert _{2}^{2}\left\Vert \bm{v}\right\Vert _{2}^{2}\right),\label{eq:lem-T-unif-omega2}
		\end{align}
		where the last line is due to the property of Gaussian distributions.
		In addition, (iii) is a consequence of the elementary inequality $|ab|\leq(|a|^{2}+|b|^{2})/2$,
		(iv) comes from the incoherence condition \eqref{eq:incoherence-condition}
		and $\sum_{j=1}^{m}\left|\bm{u}^{\mathsf{H}}\bm{b}_{j}\right|^{2}=\left\Vert \bm{u}\right\Vert _{2}^{2}$,
		whereas (v) holds true as long as $m\gg\mu^{2}$. 
		\item Regarding $\omega_{1}$, note that $z_{j}$ is a sub-Gaussian random
		variable obeying
		\[
		\left\Vert z_{j}-\mathbb{E}\left[z_{j}\right]\right\Vert _{\psi_{2}}\lesssim\left|C_{4}r\sqrt{\log m}\left(\bm{u}^{\mathsf{H}}\bm{b}_{j}\right)\left(\bm{b}_{j}^{\mathsf{H}}\bm{h}^{\star}\right)\right|\leq C_{4}\frac{\mu\sqrt{\log m}}{\sqrt{m}}r\left|\bm{u}^{\mathsf{H}}\bm{b}_{j}\right|.
		\]
		Therefore, by invoking Hoeffding's inequality (cf.~\citet[Theorem 2.6.2]{vershynin2018high})
		we reach
		\[
		\mathbb{P}\left(\Bigg|\sum_{j=1}^{m}z_{j}-\mathbb{E}\left[z_{j}\right]\Bigg|\geq t\right)\leq2\exp\left(-\frac{ct^{2}}{\frac{C_{4}^{2}\mu^{2}r^{2}\log m}{m}\sum_{j=1}^{m}\left|\bm{u}^{\mathsf{H}}\bm{b}_{j}\right|^{2}}\right)=2\exp\left(-\frac{ct^{2}}{\frac{C_{4}^{2}\mu^{2}r^{2}\log m}{m}}\right)
		\]
		for any $t\geq0$. Setting $t=\frac{C\mu r\sqrt{K}\log m}{\sqrt{m}}$
		for some sufficiently large constant $C>0$ yields 
		\begin{equation}
		\mathbb{P}\left(\Bigg|\sum_{j=1}^{m}z_{j}-\mathbb{E}\left[z_{j}\right]\Bigg|\geq\frac{C\mu r\sqrt{K}\log m}{\sqrt{m}}\right)\leq2\exp\left(-10K\log m\right).\label{eq:cover-2-1}
		\end{equation}
		Next, we define $\mathcal{N}_{\bm{x}}$ to be an $\varepsilon_{1}$-net
		of $\mathcal{B}_{\bm{x}}\left(\frac{C_{5}}{1-\rho}\eta r\right)\coloneqq\left\{ \bm{x}:\left\Vert \bm{x}-\bm{x}^{\star}\right\Vert \leq\frac{C_{5}}{1-\rho}\eta r\right\} $,
		and $\mathcal{N}_{0}$ an $\varepsilon_{2}$-net of the unit sphere
		$\mathcal{S}^{K-1}=\left\{ \bm{u}\in\mathbb{C}^{K}:\left\Vert \bm{u}\right\Vert _{2}=1\right\} $,
		where we take $\varepsilon_{1}=r/\left(m\log m\right)$ and $\varepsilon_{2}=1/\left(m\log m\right)$.
		In view of \citet[Corollary 4.2.13]{vershynin2018high}, one can ensure
		that
		\[
		\left|\mathcal{N}_{\bm{x}}\right|\leq\left(1+\frac{2C_{5}\eta r}{\left(1-\rho\right)\varepsilon_{1}}\right)^{2K}\quad\mathrm{and}\quad\left|\mathcal{N}_{0}\right|\leq\left(1+\frac{2}{\varepsilon_{2}}\right)^{2K}.
		\]
		This together with the union bound leads to
		\[
		\Bigg|\sum_{j=1}^{m}z_{j}-\mathbb{E}\left[z_{j}\right]\Bigg|\geq\frac{C\mu r\sqrt{K}\log m}{\sqrt{m}},
		\]
		which holds uniformly for any $\bm{x}\in\mathcal{N}_{\bm{x}}$, $\bm{u},\bm{v}\in\mathcal{N}_{0}$
		and holds with probability at least 
		\[
		1-\left(1+\frac{2C_{5}\eta r}{\left(1-\rho\right)\varepsilon_{1}}\right)^{2K}\left(1+\frac{2}{\varepsilon_{2}}\right)^{4K}\cdot2e^{-10K\log m}\geq1-O\left(m^{-100}\right).
		\]
		As a result, with probability exceeding $1-O\left(m^{-10}+me^{-CK}\right)$
		there holds
		\begin{align}
		& \Bigg|\sum_{j=1}^{m}\left(\bm{u}^{\mathsf{H}}\bm{b}_{j}\bm{b}_{j}^{\mathsf{H}}\bm{h}^{\star}\bm{\Delta}_{\bm{x}}^{\mathsf{H}}\bm{a}_{j}\bm{a}_{j}^{\mathsf{H}}\bm{v}\ind_{\left\{ \left|\bm{\Delta}_{\bm{x}}^{\mathsf{H}}\bm{a}_{j}\right|\leq C_{4}r\sqrt{\log m}\right\} }-\bm{u}^{\mathsf{H}}\bm{b}_{j}\bm{b}_{j}^{\mathsf{H}}\bm{h}^{\star}\bm{\Delta}_{\bm{x}}^{\mathsf{H}}\bm{v}\right)\Bigg|\nonumber \\
		& \qquad\leq\Bigg|\sum_{j=1}^{m}\left(z_{j}-\mathbb{E}\left[z_{j}\right]\right)\Bigg|+\Bigg|\sum_{j=1}^{m}\mathbb{E}\left[\bm{u}^{\mathsf{H}}\bm{b}_{j}\bm{b}_{j}^{\mathsf{H}}\bm{h}^{\star}\bm{\Delta}_{\bm{x}}^{\mathsf{H}}\bm{a}_{j}\bm{a}_{j}^{\mathsf{H}}\bm{v}\ind_{\left\{ \left|\bm{\Delta}_{\bm{x}}^{\mathsf{H}}\bm{a}_{j}\right|\leq C_{4}r\sqrt{\log m}\right\} }\right]\Bigg|\nonumber \\
		& \qquad\leq\frac{C\mu r\sqrt{K}\log m}{\sqrt{m}}+\frac{\left\Vert \bm{\Delta}_{\bm{x}}\right\Vert _{2}}{m^{24}}\nonumber \\
		& \qquad\leq\frac{\lambda}{100}\label{eq:lemma85}
		\end{align}
		uniformly for any $\bm{x}\in\mathcal{N}_{\bm{x}}$, $\bm{u},\bm{v}\in\mathcal{N}_{0}$.
		Here, the penultimate inequality comes from \eqref{eq:lemma8omega2}
		and \eqref{eq:cover-2-1}. For any $\bm{x}$ obeying the assumption
		$\max_{j}\big|\left(\bm{x}-\bm{x}^{\star}\right)^{\mathsf{H}}\bm{a}_{j}\big|\le C_{3}r\sqrt{\log m}$
		and any $\bm{u}$, $\bm{v}\in\mathcal{S}^{K-1}$, we can find $\bm{x}_{0}\in\mathcal{N}_{\bm{x}}$,
		$\bm{u}_{0}\in\mathcal{N}_{0}$ and $\bm{v}_{0}\in\mathcal{N}_{0}$
		satisfying $\left\Vert \bm{x}-\bm{x}_{0}\right\Vert _{2}\leq\varepsilon_{1}$
		and $\max\left\{ \left\Vert \bm{u}-\bm{u}_{0}\right\Vert _{2},\left\Vert \bm{v}-\bm{v}_{0}\right\Vert _{2}\right\} \leq\varepsilon_{2}$.
		Given that $\max_{j}\left\Vert \bm{a}_{j}\right\Vert _{2}\leq10\sqrt{K}$
		with probability $1-me^{-CK}$ for some constant $C>0$, this yields
		that 
		\[
		\left|\bm{\Delta}_{\bm{x}_{0}}^{\mathsf{H}}\bm{a}_{j}\right|\leq\left|\bm{\Delta}_{\bm{x}}^{\mathsf{H}}\bm{a}_{j}\right|+10\varepsilon_{1}\sqrt{K}\leq2C_{3}\left(\lambda+\sigma\sqrt{K\log m}\right)\sqrt{\log m}.
		\]
		Recalling $C_{4}\geq10C_{3}$, we have 
		\[
		\left|\bm{\Delta}_{\bm{x}_{0}}^{\mathsf{H}}\bm{a}_{j}\right|\leq C_{4}\left(\lambda+\sigma\sqrt{K\log m}\right)\sqrt{\log m}=C_{4}r\sqrt{\log m},
		\]
		and hence $\ind_{\left\{ \left|\bm{\Delta}_{\bm{x}_{0}}^{\mathsf{H}}\bm{a}_{j}\right|\leq C_{4}r\sqrt{\log m}\right\} }=1$,
		$\forall j$. Therefore, if we let 
		\[
		f\left(\bm{x},\bm{u},\bm{v}\right)\coloneqq\sum_{j=1}^{m}\left(\bm{u}^{\mathsf{H}}\bm{b}_{j}\bm{b}_{j}^{\mathsf{H}}\bm{h}^{\star}\left(\bm{x}-\bm{x}^{\star}\right)^{\mathsf{H}}\bm{a}_{j}\bm{a}_{j}^{\mathsf{H}}\bm{v}\ind_{\left\{ \left|\left(\bm{x}-\bm{x}^{\star}\right)^{\mathsf{H}}\bm{a}_{j}\right|\leq C_{4}r\sqrt{\log m}\right\} }-\bm{u}^{\mathsf{H}}\bm{b}_{j}\bm{b}_{j}^{\mathsf{H}}\bm{h}^{\star}\left(\bm{x}-\bm{x}^{\star}\right)^{\mathsf{H}}\bm{v}\right),
		\]
		then we can demonstrate that
		\begin{align*}
		& \left|f\left(\bm{x},\bm{u},\bm{v}\right)-f\left(\bm{x}_{0},\bm{u}_{0},\bm{v}_{0}\right)\right|\\
		& \quad\leq\Bigg|\sum_{j=1}^{m}\bm{u}^{\mathsf{H}}\bm{b}_{j}\bm{b}_{j}^{\mathsf{H}}\bm{h}^{\star}\left(\bm{x}-\bm{x}_{0}\right)^{\mathsf{H}}\bm{a}_{j}\bm{a}_{j}^{\mathsf{H}}\bm{v}\Bigg|+\Bigg|\sum_{j=1}^{m}\bm{u}^{\mathsf{H}}\bm{b}_{j}\bm{b}_{j}^{\mathsf{H}}\bm{h}^{\star}\left(\bm{x}-\bm{x}_{0}\right)^{\mathsf{H}}\bm{v}\Bigg|\\
		& \qquad+\Bigg|\sum_{j=1}^{m}\left(\bm{u}-\bm{u}_{0}\right)^{\mathsf{H}}\bm{b}_{j}\bm{b}_{j}^{\mathsf{H}}\bm{h}^{\star}\left(\bm{x}_{0}-\bm{x}^{\star}\right)^{\mathsf{H}}\bm{a}_{j}\bm{a}_{j}^{\mathsf{H}}\bm{v}\Bigg|+\Bigg|\sum_{j=1}^{m}\left(\bm{u}-\bm{u}_{0}\right)^{\mathsf{H}}\bm{b}_{j}\bm{b}_{j}^{\mathsf{H}}\bm{h}^{\star}\left(\bm{x}_{0}-\bm{x}^{\star}\right)^{\mathsf{H}}\bm{v}\Bigg|\\
		& \qquad+\Bigg|\sum_{j=1}^{m}\bm{u}_{0}^{\mathsf{H}}\bm{b}_{j}\bm{b}_{j}^{\mathsf{H}}\bm{h}^{\star}\left(\bm{x}_{0}-\bm{x}^{\star}\right)^{\mathsf{H}}\bm{a}_{j}\bm{a}_{j}^{\mathsf{H}}\left(\bm{v}-\bm{v}_{0}\right)\Bigg|+\Bigg|\sum_{j=1}^{m}\bm{u}_{0}^{\mathsf{H}}\bm{b}_{j}\bm{b}_{j}^{\mathsf{H}}\bm{h}^{\star}\left(\bm{x}_{0}-\bm{x}^{\star}\right)^{\mathsf{H}}\left(\bm{v}-\bm{v}_{0}\right)\Bigg|\\
		& \quad\leq\left(\left\Vert \mathcal{A}\right\Vert ^{2}+1\right)\left(\left\Vert \bm{h}^{\star}\right\Vert _{2}\left\Vert \bm{x}-\bm{x}_{0}\right\Vert _{2}+\left\Vert \bm{x}_{0}-\bm{x}^{\star}\right\Vert _{2}\left\Vert \bm{u}-\bm{u}_{0}\right\Vert _{2}+\left\Vert \bm{x}_{0}-\bm{x}^{\star}\right\Vert _{2}\left\Vert \bm{v}-\bm{v}_{0}\right\Vert _{2}\right)\\
		& \quad\leq\left(2K\log K+10\log m+1\right)\left(\varepsilon_{1}+2C_{1}r\varepsilon_{2}\right),
		\end{align*}
		where the last inequality arises from \eqref{eq:a-normbound}. Consequently,
		\begin{align*}
		& \left|\bm{u}^{\mathsf{H}}\mathcal{T}^{\mathsf{debias}}\left(\bm{h}^{\star}\left(\bm{x}-\bm{x}^{\star}\right)^{\mathsf{H}}\right)\bm{v}\right|\\
		& \qquad=\Bigg|\sum_{j=1}^{m}\left(\bm{u}^{\mathsf{H}}\bm{b}_{j}\bm{b}_{j}^{\mathsf{H}}\bm{h}^{\star}\left(\bm{x}-\bm{x}^{\star}\right)^{\mathsf{H}}\bm{a}_{j}\bm{a}_{j}^{\mathsf{H}}\bm{v}\ind_{\left\{ \left|\left(\bm{x}-\bm{x}^{\star}\right)^{\mathsf{H}}\bm{a}_{j}\right|\leq C_{4}r\sqrt{\log m}\right\} }-\bm{u}^{\mathsf{H}}\bm{b}_{j}\bm{b}_{j}^{\mathsf{H}}\bm{h}^{\star}\left(\bm{x}-\bm{x}^{\star}\right)^{\mathsf{H}}\bm{v}\right)\Bigg|\\
		& \qquad\leq\left|f\left(\bm{x},\bm{u},\bm{v}\right)-f\left(\bm{x}_{0},\bm{u}_{0},\bm{v}_{0}\right)\right|+\left|f\left(\bm{x}_{0},\bm{u}_{0},\bm{v}_{0}\right)\right|\\
		& \qquad\leq\left(2K\log K+10\log m+1\right)\left(\varepsilon_{1}+2C_{1}r\varepsilon_{2}\right)+\frac{\lambda}{100}\\
		& \qquad\leq\frac{\lambda}{50},
		\end{align*}
		where the last inequality is due to the definitions $r=\lambda+\sigma\sqrt{K\log m}$,
		$\varepsilon_{1}=r/\left(m\log m\right)$, $\varepsilon_{2}=1/\left(m\log m\right)$
		and $m\gg K$. Therefore, for any $\left(\bm{h},\bm{x}\right)$ satisfying
		\eqref{eq:hx-properties-all}, there holds 
		\begin{align}
		\left\Vert \mathcal{T}^{\mathsf{debias}}\left(\bm{h}^{\star}\bm{\Delta}_{\bm{x}}^{\mathsf{H}}\right)\right\Vert  & =\sup_{\bm{u},\bm{v}\in\mathcal{S}^{K-1}}\bm{u}^{\mathsf{H}}\mathcal{T}^{\mathsf{debias}}\left(\bm{h}^{\star}\left(\bm{x}-\bm{x}^{\star}\right)^{\mathsf{H}}\right)\bm{v}\leq\frac{1}{50}\lambda\label{eq:lemma8bound1}
		\end{align}
		with probability exceeding $1-O\left(m^{-10}+me^{-CK}\right)$.
	\end{itemize}
	\item We now move on to $\beta_{2}$, for which we have a similar decomposition
	as follows
	\begin{align*}
	& \bm{u}^{\mathsf{H}}\mathcal{T}^{\mathsf{debias}}\left(\bm{\Delta}_{\bm{h}}\bm{x}^{\star}{}^{\mathsf{H}}\right)\bm{v}\\
	& \quad=\sum_{j=1}^{m}\left(\bm{u}^{\mathsf{H}}\bm{b}_{j}\bm{b}_{j}^{\mathsf{H}}\bm{\Delta}_{\bm{h}}\bm{x}^{\star\mathsf{H}}\bm{a}_{j}\bm{a}_{j}^{\mathsf{H}}\bm{v}-\bm{u}^{\mathsf{H}}\bm{b}_{j}\bm{b}_{j}^{\mathsf{H}}\bm{\Delta}_{\bm{h}}\bm{x}^{\star\mathsf{H}}\bm{v}\right)\\
	& \quad=\sum_{j=1}^{m}\Bigg(\underbrace{\bm{u}^{\mathsf{H}}\bm{b}_{j}\bm{b}_{j}^{\mathsf{H}}\bm{\Delta}_{\bm{h}}\bm{x}^{\star\mathsf{H}}\bm{a}_{j}\bm{a}_{j}^{\mathsf{H}}\bm{v}\ind_{\left\{ \left|\bm{x}^{\star\mathsf{H}}\bm{a}_{j}\right|\leq20\sqrt{\log m}\right\} }}_{\eqqcolon y_{j}}-\mathbb{E}\left[y_{j}\right]\Bigg)\\
	& \qquad-\underbrace{\sum_{j=1}^{m}\mathbb{E}\left[\bm{u}^{\mathsf{H}}\bm{b}_{j}\bm{b}_{j}^{\mathsf{H}}\bm{\Delta}_{\bm{h}}\bm{x}^{\star\mathsf{H}}\bm{a}_{j}\bm{a}_{j}^{\mathsf{H}}\bm{v}\ind_{\left\{ \left|\bm{x}^{\star\mathsf{H}}\bm{a}_{j}\right|>20\sqrt{\log m}\right\} }\right]}_{\eqqcolon\omega_{4}}+\sum_{j=1}^{m}\bm{u}^{\mathsf{H}}\bm{b}_{j}\bm{b}_{j}^{\mathsf{H}}\bm{\Delta}_{\bm{h}}\bm{x}^{\star\mathsf{H}}\bm{a}_{j}\bm{a}_{j}^{\mathsf{H}}\bm{v}\ind_{\left\{ \left|\bm{x}^{\star\mathsf{H}}\bm{a}_{j}\right|>20\sqrt{\log m}\right\} }.
	\end{align*}
	
	\begin{itemize}
		\item For $\omega_{4}$, similar to \eqref{eq:lemma8omega2} we have
		\begin{align}
		\left|\omega_{4}\right|= & \Bigg|\sum_{j=1}^{m}\mathbb{E}\left[\bm{u}^{\mathsf{H}}\bm{b}_{j}\bm{b}_{j}^{\mathsf{H}}\bm{\Delta}_{\bm{h}}\bm{x}^{\star\mathsf{H}}\bm{a}_{j}\bm{a}_{j}^{\mathsf{H}}\bm{v}\ind_{\left\{ \left|\bm{x}^{\star\mathsf{H}}\bm{a}_{j}\right|>20\sqrt{\log m}\right\} }\right]\Bigg|\nonumber \\
		\overset{(\text{i})}{\leq} & \sum_{j=1}^{m}\sqrt{\mathbb{E}\left[\left|\bm{u}^{\mathsf{H}}\bm{b}_{j}\bm{b}_{j}^{\mathsf{H}}\bm{\Delta}_{\bm{h}}\bm{x}^{\star\mathsf{H}}\bm{a}_{j}\bm{a}_{j}^{\mathsf{H}}\bm{v}\right|^{2}\right]\mathbb{P}\left(\left|\bm{x}^{\star\mathsf{H}}\bm{a}_{j}\right|>20\sqrt{\log m}\right)}\nonumber \\
		\overset{(\text{ii})}{\leq} & \sum_{j=1}^{m}\left|\bm{u}^{\mathsf{H}}\bm{b}_{j}\bm{b}_{j}^{\mathsf{H}}\bm{\Delta}_{\bm{h}}\right|\sqrt{\left(2\left|\bm{x}^{\star\mathsf{H}}\bm{v}\right|^{2}+\left\Vert \bm{x}^{\star}\right\Vert _{2}^{2}\left\Vert \bm{v}\right\Vert _{2}^{2}\right)2\exp\left(-200\log m\right)}\nonumber \\
		\leq & \sum_{j=1}^{m}\left|\bm{u}^{\mathsf{H}}\bm{b}_{j}\bm{b}_{j}^{\mathsf{H}}\bm{\Delta}_{\bm{h}}\right|\frac{4}{m^{100}}\nonumber \\
		\overset{\text{(iii)}}{\leq} & \sum_{j=1}^{m}\left\Vert \bm{b}_{j}\right\Vert _{2}\times C_{9}\sigma\times\frac{4}{m^{100}}\\
		\overset{(\text{iv})}{\leq} & \sqrt{\frac{K}{m}}\times m\times C_{9}\sigma\times\frac{4}{m^{100}}\nonumber \\
		\leq & \frac{\lambda}{m^{99}},\label{eq:lem-T-unif-4}
		\end{align}
		where (i) follows from Cauchy-Schwarz inequality, (ii) comes from
		the property of sub-Gaussian variable $\left|\bm{x}^{\star\mathsf{H}}\bm{a}_{j}\right|$
		and \eqref{eq:lem-T-unif-omega2}, (iii) is due to the assumption
		\eqref{eq:hx-properties2}, and (iv) comes from the fact $\left\Vert \bm{b}_{j}\right\Vert _{2}=\sqrt{K/m}$.
		\item Regarding the term $\omega_{3}\coloneqq\sum_{j=1}^{m}\left(y_{j}-\mathbb{E}\left[y_{j}\right]\right)$,
		we note that 
		\[
		\left\Vert \bm{u}^{\mathsf{H}}\bm{b}_{j}\bm{b}_{j}^{\mathsf{H}}\bm{\Delta}_{\bm{h}}\bm{x}^{\star\mathsf{H}}\bm{a}_{j}\bm{a}_{j}^{\mathsf{H}}\bm{v}\ind_{\left\{ \left|\bm{x}^{\star\mathsf{H}}\bm{a}_{j}\right|\leq20\sqrt{\log m}\right\} }\right\Vert _{\psi_{2}}\leq\frac{\mu\lambda}{\sqrt{m}}\log^{2}m\times20\sqrt{\log m}\left|\bm{u}^{\mathsf{H}}\bm{b}_{j}\right|.
		\]
		Hoeffding's inequality \citet[Theorem 2.6.3]{vershynin2018high} tells
		us that
		\[
		\mathbb{P}\left(\Bigg|\sum_{j=1}^{m}\left(y_{j}-\mathbb{E}\left[y_{j}\right]\right)\Bigg|\geq t\right)\leq2\exp\left(-\frac{ct^{2}}{400\frac{\mu^{2}\lambda^{2}}{m}\log^{5}m\sum_{j=1}^{m}\left|\bm{u}^{\mathsf{H}}\bm{b}_{j}\right|^{2}}\right)=2\exp\left(-\frac{ct^{2}}{400\frac{\mu^{2}\lambda^{2}}{m}\log^{5}m}\right)
		\]
		for any $t\geq0$. Setting $t=\frac{C\mu\lambda\sqrt{K}}{\sqrt{m}}\log^{3}m$
		for some sufficiently large constant $C>0$ yields 
		\begin{equation}
		\mathbb{P}\left(\Bigg|\sum_{j=1}^{m}\left(y_{j}-\mathbb{E}\left[y_{j}\right]\right)\Bigg|\geq C\frac{\mu\lambda\sqrt{K}}{\sqrt{m}}\log^{3}m\right)\leq2\exp\left(-10K\log m\right).\label{eq:lem-T-unif-5}
		\end{equation}
		Invoking a similar covering argument, we know that with probability
		exceeding $1-O\left(m^{-10}\right)$, 
		\[
		\Bigg|\sum_{j=1}^{m}\left(y_{j}-\mathbb{E}\left[y_{j}\right]\right)\Bigg|\geq C\frac{\mu\lambda\sqrt{K}}{\sqrt{m}}\log^{3}m
		\]
		holds uniformly for any $\bm{h}$ over the $\varepsilon_{1}$-net
		$\mathcal{N}_{\bm{h}}$ of $\mathcal{B}_{\bm{h}}\left(\frac{C_{5}}{1-\rho}\eta r\right)\coloneqq\left\{ \bm{h}:\left\Vert \bm{h}-\bm{h}^{\star}\right\Vert _{2}\leq\frac{C_{5}}{1-\rho}\eta r\right\} $
		and any $\bm{u}$, $\bm{v}$ over the $\varepsilon_{2}$-net $\mathcal{N}_{0}$
		of the unit sphere $\mathcal{S}^{K-1}$. As a result, one has
		\begin{align}
		& \Bigg|\sum_{j=1}^{m}\left(\bm{u}^{\mathsf{H}}\bm{b}_{j}\bm{b}_{j}^{\mathsf{H}}\bm{\Delta}_{\bm{h}}\bm{x}^{\star\mathsf{H}}\bm{a}_{j}\bm{a}_{j}^{\mathsf{H}}\bm{v}\ind_{\left\{ \left|\bm{x}^{\star\mathsf{H}}\bm{a}_{j}\right|\leq20\sqrt{\log m}\right\} }-\bm{u}^{\mathsf{H}}\bm{b}_{j}\bm{b}_{j}^{\mathsf{H}}\bm{\Delta}_{\bm{h}}\bm{x}^{\star\mathsf{H}}\bm{v}\right)\Bigg|\nonumber \\
		& \qquad\leq\Bigg|\sum_{j=1}^{m}\left(y_{j}-\mathbb{E}\left[y_{j}\right]\right)\Bigg|+\Bigg|\sum_{j=1}^{m}\mathbb{E}\left[\bm{u}^{\mathsf{H}}\bm{b}_{j}\bm{b}_{j}^{\mathsf{H}}\bm{\Delta}_{\bm{h}}\bm{x}^{\star\mathsf{H}}\bm{a}_{j}\bm{a}_{j}^{\mathsf{H}}\bm{v}\ind_{\left\{ \left|\bm{x}^{\star\mathsf{H}}\bm{a}_{j}\right|>20\sqrt{\log m}\right\} }\right]\Bigg|\nonumber \\
		& \qquad\leq C\frac{\mu\lambda\sqrt{K}}{\sqrt{m}}\log^{3}m+\frac{\lambda}{m^{99}}\nonumber \\
		& \qquad\leq\frac{\lambda}{100},\label{eq:lemma84}
		\end{align}
		where the penultimate inequality comes from \eqref{eq:lem-T-unif-4}
		and \eqref{eq:lem-T-unif-5}. Next, let us define
		\[
		g\left(\bm{h},\bm{u},\bm{v}\right)\coloneqq\sum_{j=1}^{m}\left(\bm{u}^{\mathsf{H}}\bm{b}_{j}\bm{b}_{j}^{\mathsf{H}}\left(\bm{h}-\bm{h}^{\star}\right)\bm{x}^{\star\mathsf{H}}\bm{a}_{j}\bm{a}_{j}^{\mathsf{H}}\bm{v}\ind_{\left\{ \left|\bm{x}^{\star\mathsf{H}}\bm{a}_{j}\right|\leq20\sqrt{\log m}\right\} }-\bm{u}^{\mathsf{H}}\bm{b}_{j}\bm{b}_{j}^{\mathsf{H}}\left(\bm{h}-\bm{h}^{\star}\right)\bm{x}^{\star\mathsf{H}}\bm{v}\right).
		\]
		Since we can always find some $\bm{x}_{0}\in\mathcal{N}_{\bm{x}}$,
		$\bm{u}_{0}$, $\bm{v}_{0}\in\mathcal{N}_{0}$ such that $\left\Vert \bm{h}-\bm{h}_{0}\right\Vert _{2}\leq\varepsilon_{1}$
		and $\max\left\{ \left\Vert \bm{u}-\bm{u}_{0}\right\Vert _{2},\left\Vert \bm{v}-\bm{v}_{0}\right\Vert _{2}\right\} \leq\varepsilon_{2}$,
		this guarantees that
		\begin{align*}
		& \left|g\left(\bm{h},\bm{u},\bm{v}\right)-g\left(\bm{h}_{0},\bm{u}_{0},\bm{v}_{0}\right)\right|\\
		& \quad\leq\Bigg|\sum_{j=1}^{m}\bm{u}^{\mathsf{H}}\bm{b}_{j}\bm{b}_{j}^{\mathsf{H}}\left(\bm{h}-\bm{h}_{0}\right)\bm{x}^{\star\mathsf{H}}\bm{a}_{j}\bm{a}_{j}^{\mathsf{H}}\bm{v}\ind_{\left\{ \left|\bm{x}^{\star\mathsf{H}}\bm{a}_{j}\right|\leq20\sqrt{\log m}\right\} }\Bigg|+\Bigg|\sum_{j=1}^{m}\bm{u}^{\mathsf{H}}\bm{b}_{j}\bm{b}_{j}^{\mathsf{H}}\left(\bm{h}-\bm{h}_{0}\right)\bm{x}^{\star\mathsf{H}}\bm{v}\Bigg|\\
		& \qquad+\Bigg|\sum_{j=1}^{m}\left(\left(\bm{u}-\bm{u}_{0}\right)^{\mathsf{H}}\bm{b}_{j}\bm{b}_{j}^{\mathsf{H}}\left(\bm{h}-\bm{h}^{\star}\right)\bm{x}^{\star\mathsf{H}}\bm{a}_{j}\bm{a}_{j}^{\mathsf{H}}\bm{v}\ind_{\left\{ \left|\bm{x}^{\star\mathsf{H}}\bm{a}_{j}\right|\leq20\sqrt{\log m}\right\} }-\left(\bm{u}-\bm{u}_{0}\right)^{\mathsf{H}}\bm{b}_{j}\bm{b}_{j}^{\mathsf{H}}\left(\bm{h}-\bm{h}^{\star}\right)\bm{x}^{\star\mathsf{H}}\bm{v}\right)\Bigg|\\
		& \qquad+\Bigg|\sum_{j=1}^{m}\left(\bm{u}_{0}^{\mathsf{H}}\bm{b}_{j}\bm{b}_{j}^{\mathsf{H}}\left(\bm{h}-\bm{h}^{\star}\right)\bm{x}^{\star\mathsf{H}}\bm{a}_{j}\bm{a}_{j}^{\mathsf{H}}\left(\bm{v}-\bm{v}_{0}\right)\ind_{\left\{ \left|\bm{x}^{\star\mathsf{H}}\bm{a}_{j}\right|\leq20\sqrt{\log m}\right\} }-\bm{u}_{0}^{\mathsf{H}}\bm{b}_{j}\bm{b}_{j}^{\mathsf{H}}\left(\bm{h}-\bm{h}^{\star}\right)\bm{x}^{\star\mathsf{H}}\left(\bm{v}-\bm{v}_{0}\right)\right)\Bigg|\\
		& \quad\leq\left(\left\Vert \mathcal{A}\right\Vert ^{2}+1\right)\left(\left\Vert \bm{x}^{\star}\right\Vert _{2}\left\Vert \bm{h}-\bm{h}_{0}\right\Vert _{2}+\left\Vert \left(\bm{h}-\bm{h}^{\star}\right)\right\Vert _{2}\left\Vert \bm{u}-\bm{u}_{0}\right\Vert _{2}+\left\Vert \bm{h}-\bm{h}^{\star}\right\Vert _{2}\left\Vert \bm{v}-\bm{v}_{0}\right\Vert _{2}\right)\\
		& \quad\leq\left(2K\log K+10\log m+1\right)\left(\varepsilon_{1}+2C_{1}r\varepsilon_{2}\right),
		\end{align*}
		where the last inequality comes from \eqref{eq:a-normbound}. Since
		$\mathbb{P}\left(\left|\bm{x}^{\star\mathsf{H}}\bm{a}_{j}\right|>20\sqrt{\log m}\right)\le O\left(m^{-100}\right)$
		(in view of \eqref{eq:useful1}), we have, with probability exceeding
		$1-O\left(m^{-10}\right)$, that
		\begin{align}
		\left\Vert \mathcal{T}^{\mathsf{debias}}\left(\bm{\Delta}_{\bm{h}}\bm{x}^{\star}{}^{\mathsf{H}}\right)\right\Vert  & =\sup_{\bm{u},\bm{v}\in\mathcal{S}^{K-1}}\Bigg|\sum_{j=1}^{m}\left(\bm{u}^{\mathsf{H}}\bm{b}_{j}\bm{b}_{j}^{\mathsf{H}}\bm{\Delta}_{\bm{h}}\bm{x}^{\star\mathsf{H}}\bm{a}_{j}\bm{a}_{j}^{\mathsf{H}}\bm{v}\ind_{\left\{ \left|\bm{x}^{\star\mathsf{H}}\bm{a}_{j}\right|\leq20\sqrt{\log m}\right\} }-\bm{u}^{\mathsf{H}}\bm{b}_{j}\bm{b}_{j}^{\mathsf{H}}\bm{\Delta}_{\bm{h}}\bm{x}^{\star\mathsf{H}}\bm{v}\right)\Bigg|\nonumber \\
		& \leq\sup_{\bm{u},\bm{v}\in\mathcal{S}^{K-1}}\left|g\left(\bm{h},\bm{u},\bm{v}\right)-g\left(\bm{h}_{0},\bm{u}_{0},\bm{v}_{0}\right)\right|+\left|g\left(\bm{h}_{0},\bm{u}_{0},\bm{v}_{0}\right)\right|\nonumber \\
		& \leq\left(2K\log K+10\log m+1\right)\left(\varepsilon_{1}+2C_{1}r\varepsilon_{2}\right)+\frac{\lambda}{100}\nonumber \\
		& \leq\frac{\lambda}{50}\label{eq:lemma8bound2}
		\end{align}
		holds uniformly over $\bm{h}\in\mathcal{B}_{\bm{h}}\left(C_{1}r\right)$,
		where the last inequality is due to the choices $\varepsilon_{1}=r/\left(m\log m\right)$,
		$\varepsilon_{2}=1/\left(m\log m\right)$ and $r=\lambda+\sigma\sqrt{K\log m}$.
	\end{itemize}
	\item Finally, we turn attention to $\beta_{3}$. Observe that for any fixed
	$\bm{h}$ and $\bm{x}$, one has
	\[
	\mathcal{T}^{\mathsf{debias}}\left(\bm{\Delta}_{\bm{h}}\bm{\Delta}_{\bm{x}}^{\mathsf{H}}\right)=\sum_{j=1}^{m}\bm{b}_{j}\bm{b}_{j}^{\mathsf{H}}\bm{\Delta}_{\bm{h}}\bm{\Delta}_{\bm{x}}^{\mathsf{H}}\left(\bm{a}_{j}\bm{a}_{j}^{\mathsf{H}}-\bm{I}_{K}\right).
	\]
	This indicates that for any fixed unit vectors $\bm{u}$, $\bm{v}\in\mathbb{C}^{K}$
	we have
	\[
	\bm{u}^{\mathsf{H}}\mathcal{T}^{\mathsf{debias}}\left(\bm{\Delta}_{\bm{h}}\bm{\Delta}_{\bm{x}}^{\mathsf{H}}\right)\bm{v}=\sum_{j=1}^{m}\left(\bm{u}^{\mathsf{H}}\bm{b}_{j}\bm{b}_{j}^{\mathsf{H}}\bm{\Delta}_{\bm{h}}\bm{\Delta}_{\bm{x}}^{\mathsf{H}}\bm{a}_{j}\bm{a}_{j}^{\mathsf{H}}\bm{v}-\bm{u}^{\mathsf{H}}\bm{b}_{j}\bm{b}_{j}^{\mathsf{H}}\bm{\Delta}_{\bm{h}}\bm{\Delta}_{\bm{x}}^{\mathsf{H}}\bm{v}\right),
	\]
	which is a sum of independent variables. Letting $r\coloneqq\lambda+\sigma\sqrt{K\log m}$
	and $C_{4}\coloneqq10\max\left\{ C_{1},C_{3},1\right\} $, we can
	demonstrate that
	\begin{align*}
	& \sum_{j=1}^{m}\Bigg(\underbrace{\bm{u}^{\mathsf{H}}\bm{b}_{j}\bm{b}_{j}^{\mathsf{H}}\bm{\Delta}_{\bm{h}}\bm{\Delta}_{\bm{x}}^{\mathsf{H}}\bm{a}_{j}\bm{a}_{j}^{\mathsf{H}}\bm{v}\ind_{\left\{ \left|\bm{\Delta}_{\bm{x}}^{\mathsf{H}}\bm{a}_{j}\right|\leq C_{4}r\sqrt{\log m}\right\} }}_{\eqqcolon s_{j}}-\bm{u}^{\mathsf{H}}\bm{b}_{j}\bm{b}_{j}^{\mathsf{H}}\bm{\Delta}_{\bm{h}}\bm{\Delta}_{\bm{x}}^{\mathsf{H}}\bm{v}\Bigg)\\
	& \qquad=\sum_{j=1}^{m}\left(s_{j}-\mathbb{E}\left[s_{j}\right]\right)+\sum_{j=1}^{m}\left(\mathbb{E}\left[\bm{u}^{\mathsf{H}}\bm{b}_{j}\bm{b}_{j}^{\mathsf{H}}\bm{\Delta}_{\bm{h}}\bm{\Delta}_{\bm{x}}^{\mathsf{H}}\bm{a}_{j}\bm{a}_{j}^{\mathsf{H}}\bm{v}\ind_{\left\{ \left|\bm{\Delta}_{\bm{x}}^{\mathsf{H}}\bm{a}_{j}\right|\leq C_{4}r\sqrt{\log m}\right\} }\right]-\bm{u}^{\mathsf{H}}\bm{b}_{j}\bm{b}_{j}^{\mathsf{H}}\bm{h}^{\star}\bm{\Delta}_{\bm{x}}^{\mathsf{H}}\bm{v}\right)\\
	& \qquad=\underbrace{\sum_{j=1}^{m}\left(s_{j}-\mathbb{E}\left[s_{j}\right]\right)}_{\eqqcolon\omega_{5}}-\underbrace{\sum_{j=1}^{m}\mathbb{E}\left[\bm{u}^{\mathsf{H}}\bm{b}_{j}\bm{b}_{j}^{\mathsf{H}}\bm{\Delta}_{\bm{h}}\bm{\Delta}_{\bm{x}}^{\mathsf{H}}\bm{a}_{j}\bm{a}_{j}^{\mathsf{H}}\bm{v}\ind_{\left\{ \left|\bm{\Delta}_{\bm{x}}^{\mathsf{H}}\bm{a}_{j}\right|>C_{4}r\sqrt{\log m}\right\} }\right]}_{\eqqcolon\omega_{6}}.
	\end{align*}
	
	\begin{itemize}
		\item With regards to $\omega_{6}$, similar to \eqref{eq:lemma8omega2}
		we have 
		\begin{align*}
		\left|\omega_{6}\right|= & \left|\sum_{j=1}^{m}\mathbb{E}\left[\bm{u}^{\mathsf{H}}\bm{b}_{j}\bm{b}_{j}^{\mathsf{H}}\bm{\Delta}_{\bm{h}}\bm{\Delta}_{\bm{x}}^{\mathsf{H}}\bm{a}_{j}\bm{a}_{j}^{\mathsf{H}}\bm{v}\ind_{\left\{ \left|\bm{\Delta}_{\bm{x}}^{\mathsf{H}}\bm{a}_{j}\right|>C_{4}r\sqrt{\log m}\right\} }\right]\right|\\
		\overset{(\text{i})}{\leq} & \sum_{j=1}^{m}\sqrt{\mathbb{E}\left[\left|\bm{u}^{\mathsf{H}}\bm{b}_{j}\bm{b}_{j}^{\mathsf{H}}\bm{\Delta}_{\bm{h}}\bm{\Delta}_{\bm{x}}^{\mathsf{H}}\bm{a}_{j}\bm{a}_{j}^{\mathsf{H}}\bm{v}\right|^{2}\right]\mathbb{P}\left[\ind_{\left\{ \left|\bm{\Delta}_{\bm{x}}^{\mathsf{H}}\bm{a}_{j}\right|>C_{4}r\sqrt{\log m}\right\} }\right]}\\
		\overset{(\text{ii})}{\leq} & \sum_{j=1}^{m}\left|\bm{u}^{\mathsf{H}}\bm{b}_{j}\bm{b}_{j}^{\mathsf{H}}\bm{\Delta}_{\bm{h}}\right|\sqrt{\left(2\left|\bm{\Delta}_{\bm{x}}^{\mathsf{H}}\bm{v}\right|^{2}+\left\Vert \bm{\Delta}_{\bm{x}}\right\Vert _{2}^{2}\left\Vert \bm{v}\right\Vert _{2}^{2}\right)2\exp\left(-\frac{C_{4}^{2}r^{2}\log m}{2\left\Vert \bm{\Delta}_{\bm{x}}\right\Vert _{2}^{2}}\right)}\\
		\leq & \sum_{j=1}^{m}\left|\bm{u}^{\mathsf{H}}\bm{b}_{j}\bm{b}_{j}^{\mathsf{H}}\bm{\Delta}_{\bm{h}}\right|\sqrt{6\left\Vert \bm{\Delta}_{\bm{x}}\right\Vert _{2}^{2}\exp\left(-50\log m\right)}\\
		\leq & \sum_{j=1}^{m}\left\Vert \bm{b}_{j}\right\Vert _{2}\left|\bm{b}_{j}^{\mathsf{H}}\bm{\Delta}_{\bm{h}}\right|\frac{\sqrt{6}\left\Vert \bm{\Delta}_{\bm{x}}\right\Vert _{2}}{m^{25}}\\
		\overset{(\text{iii})}{\leq} & \frac{\lambda\left\Vert \bm{\Delta}_{\bm{x}}\right\Vert _{2}}{m^{24}},
		\end{align*}
		where (i) follows from Cauchy-Schwarz inequality, (ii) comes from
		the property of sub-Gaussian variable $\left|\bm{\Delta}_{\bm{x}}^{\mathsf{H}}\bm{a}_{j}\right|$
		and \eqref{eq:lem-T-unif-omega2}, and (iii) is due to the fact $\left\Vert \bm{b}_{j}\right\Vert _{2}=\sqrt{K/m}$
		and the assumption \eqref{eq:hx-properties2}.
		\item Regarding $\omega_{5}$, we note that $s_{j}$ is a sub-Gaussian random
		variable satisfying
		\[
		\left\Vert s_{j}-\mathbb{E}\left[s_{j}\right]\right\Vert _{\psi_{2}}\lesssim C_{4}r\sqrt{\log m}\left|\left(\bm{u}^{\mathsf{H}}\bm{b}_{j}\right)\left(\bm{b}_{j}^{\mathsf{H}}\bm{\Delta}_{\bm{h}}\right)\right|\leq C_{4}\frac{\mu\sqrt{\log^{5}m}}{\sqrt{m}}r\left|\bm{u}^{\mathsf{H}}\bm{b}_{j}\right|.
		\]
		Therefore, invoking Hoeffding's inequality (cf.~\citet[Theorem 2.6.3]{vershynin2018high})
		reveals that
		\[
		\mathbb{P}\left(\left|\sum_{j=1}^{m}s_{j}-\mathbb{E}\left[s_{j}\right]\right|\geq t\right)\leq2\exp\left(-\frac{ct^{2}}{\frac{C_{4}^{2}\mu^{2}r^{2}\log^{5}m}{m}\sum_{j=1}^{m}\left|\bm{u}^{\mathsf{H}}\bm{b}_{j}\right|^{2}}\right)=2\exp\left(-\frac{ct^{2}}{\frac{C_{4}^{2}\mu^{2}r^{2}\log^{5}m}{m}}\right)
		\]
		for any $t\geq0$. Setting $t=\frac{C\mu r\sqrt{K}\log^{3}m}{\sqrt{m}}$
		for some sufficiently large constant $C>0$, we obtain
		\begin{equation}
		\mathbb{P}\left(\left|\sum_{j=1}^{m}s_{j}-\mathbb{E}\left[s_{j}\right]\right|\geq\frac{C\mu r\sqrt{K}\log^{3}m}{\sqrt{m}}\right)\leq2\exp\left(-10K\log m\right).\label{eq:cover-2-1-1}
		\end{equation}
		Let $\varepsilon_{1}=r/\left(m\log m\right)$ and $\varepsilon_{2}=1/\left(m\log m\right)$,
		and set $\mathcal{N}_{\bm{h}}$ to be an $\varepsilon_{1}$-net of
		$\mathcal{B}_{\bm{h}}\left(\frac{C_{5}}{1-\rho}\eta r\right)\coloneqq\left\{ \bm{h}:\left\Vert \bm{h}-\bm{h}^{\star}\right\Vert _{2}\leq\frac{C_{5}}{1-\rho}\eta r\right\} $,
		$\mathcal{N}_{\bm{x}}$ an $\varepsilon_{1}$-net of $\mathcal{B}_{\bm{x}}\left(\frac{C_{5}}{1-\rho}\eta r\right)\coloneqq\left\{ \bm{x}:\left\Vert \bm{x}-\bm{x}^{\star}\right\Vert _{2}\leq\frac{C_{5}}{1-\rho}\eta r\right\} $,
		and $\mathcal{N}_{0}$ an $\varepsilon_{2}$-net of the unit sphere
		$\mathcal{S}^{K-1}=\left\{ \bm{u}\in\mathbb{C}^{K}:\left\Vert \bm{u}\right\Vert _{2}=1\right\} $.
		In view of \citet[Corollary 4.2.13]{vershynin2018high}, these epsilon
		nets can be chosen to satisfy the following cardinality bounds
		\[
		\left|\mathcal{N}_{\bm{h}}\right|\leq\left(1+\frac{2C_{5}\eta r}{\left(1-\rho\right)\varepsilon_{1}}\right)^{2K},\quad\left|\mathcal{N}_{\bm{x}}\right|\leq\left(1+\frac{2C_{5}\eta r}{\left(1-\rho\right)\varepsilon_{1}}\right)^{2K}\quad\mathrm{and}\quad\left|\mathcal{N}_{0}\right|\leq\left(1+\frac{2}{\varepsilon_{2}}\right)^{2K}.
		\]
		By taking the union bound, we show that with probability at least
		\[
		1-\left(1+\frac{2C_{5}\eta r}{\left(1-\rho\right)\varepsilon_{1}}\right)^{4K}\left(1+\frac{2}{\varepsilon_{2}}\right)^{4K}e^{-10K\log m}\geq1-O\left(m^{-100}\right),
		\]
		the following bound
		\[
		\left|\sum_{j=1}^{m}s_{j}-\mathbb{E}\left[s_{j}\right]\right|\geq\frac{C\mu r\sqrt{K}\log^{3}m}{\sqrt{m}}
		\]
		holds uniformly for any $\bm{h}$ over $\mathcal{N}_{\bm{h}}$, any
		$\bm{x}$ over $\mathcal{N}_{\bm{x}}$, and any $\bm{u}$, $\bm{v}$
		over $\mathcal{N}_{0}$. Consequently, with probability exceeding
		$1-O(m^{-100})$, the inequality
		\begin{align}
		& \left|\sum_{j=1}^{m}\left(\underbrace{\bm{u}^{\mathsf{H}}\bm{b}_{j}\bm{b}_{j}^{\mathsf{H}}\bm{\Delta}_{\bm{h}}\bm{\Delta}_{\bm{x}}^{\mathsf{H}}\bm{a}_{j}\bm{a}_{j}^{\mathsf{H}}\bm{v}\ind_{\left\{ \left|\bm{\Delta}_{\bm{x}}^{\mathsf{H}}\bm{a}_{j}\right|\leq C_{4}r\sqrt{\log m}\right\} }}_{\eqqcolon s_{j}}-\bm{u}^{\mathsf{H}}\bm{b}_{j}\bm{b}_{j}^{\mathsf{H}}\bm{\Delta}_{\bm{h}}\bm{\Delta}_{\bm{x}}^{\mathsf{H}}\bm{v}\right)\right|\label{eq:lemma8rcovering}\\
		& \qquad\leq\frac{C\mu r\sqrt{K}\log^{3}m}{\sqrt{m}}+\frac{\lambda\left\Vert \bm{\Delta}_{\bm{x}}\right\Vert _{2}}{m^{24}}\leq\frac{\lambda}{100}\nonumber 
		\end{align}
		holds simultaneously for any $\bm{h}$ over $\mathcal{N}_{\bm{h}}$,
		any $\bm{x}$ over $\mathcal{N}_{\bm{x}}$, and any $\bm{u}$, $\bm{v}$
		over $\mathcal{N}_{0}$. Additionally, for any $\bm{x}$ obeying $\max_{1\leq j\leq m}\left|\left(\bm{x}-\bm{x}^{\star}\right)^{\mathsf{H}}\bm{a}_{j}\right|\le C_{3}r\sqrt{\log m}$
		and any $\bm{u}$, $\bm{v}\in\mathcal{S}^{K-1}$, we can find $\bm{h}_{0}\in\mathcal{N}_{\bm{h}}$,
		$\bm{x}_{0}\in\mathcal{N}_{\bm{x}}$, $\bm{u}_{0}\in\mathcal{N}_{0}$
		and $\bm{v}_{0}\in\mathcal{N}_{0}$ satisfying $\max\left\{ \left\Vert \bm{h}-\bm{h}_{0}\right\Vert _{2},\left\Vert \bm{x}-\bm{x}_{0}\right\Vert _{2}\right\} \leq\varepsilon_{1}$
		and $\max\left\{ \left\Vert \bm{u}-\bm{u}_{0}\right\Vert _{2},\left\Vert \bm{v}-\bm{v}_{0}\right\Vert _{2}\right\} \leq\varepsilon_{2}$.
		Recognizing that $\left\Vert \bm{a}_{j}\right\Vert _{2}\leq10\sqrt{K}$
		with probability $1-O\left(me^{-CK}\right)$ for some constant $C>0$
		(see \eqref{eq:useful2}), we can guarantee that 
		\[
		\left|\bm{\Delta}_{\bm{x}_{0}}^{\mathsf{H}}\bm{a}_{j}\right|\leq\left|\bm{\Delta}_{\bm{x}}^{\mathsf{H}}\bm{a}_{j}\right|+10\varepsilon_{1}\sqrt{K}\leq2C_{3}\left(\lambda+\sigma\sqrt{K\log m}\right)\sqrt{\log m}.
		\]
		Recalling that $C_{4}\geq10C_{3}$, we have 
		\[
		\left|\bm{\Delta}_{\bm{x}_{0}}^{\mathsf{H}}\bm{a}_{j}\right|\leq C_{4}\left(\lambda+\sigma\sqrt{K\log m}\right)\sqrt{\log m}=C_{4}r\sqrt{\log m},
		\]
		and hence $\ind_{\left\{ \left|\bm{\Delta}_{\bm{x}_{0}}^{\mathsf{H}}\bm{a}_{j}\right|\leq C_{4}r\sqrt{\log m}\right\} }=1$
		for all $1\leq j\leq m$. Therefore, if we take
		\[
		r\left(\bm{h},\bm{x},\bm{u},\bm{v}\right)\coloneqq\sum_{j=1}^{m}\left(\bm{u}^{\mathsf{H}}\bm{b}_{j}\bm{b}_{j}^{\mathsf{H}}\left(\bm{h}-\bm{h}^{\star}\right)\left(\bm{x}-\bm{x}^{\star}\right)^{\mathsf{H}}\bm{a}_{j}\bm{a}_{j}^{\mathsf{H}}\bm{v}\ind_{\left\{ \left|\bm{\Delta}_{\bm{x}}^{\mathsf{H}}\bm{a}_{j}\right|\leq C_{4}r\sqrt{\log m}\right\} }-\bm{u}^{\mathsf{H}}\bm{b}_{j}\bm{b}_{j}^{\mathsf{H}}\left(\bm{h}-\bm{h}^{\star}\right)\left(\bm{x}-\bm{x}^{\star}\right)^{\mathsf{H}}\bm{v}\right),
		\]
		then it follows that
		\begin{align}
		& \left|r\left(\bm{h},\bm{x},\bm{u},\bm{v}\right)-r\left(\bm{h}_{0},\bm{x}_{0},\bm{u}_{0},\bm{v}_{0}\right)\right|\nonumber \\
		& \quad\leq\left|\sum_{j=1}^{m}\bm{u}^{\mathsf{H}}\bm{b}_{j}\bm{b}_{j}^{\mathsf{H}}\left(\bm{h}-\bm{h}_{0}\right)\left(\bm{x}-\bm{x}^{\star}\right)^{\mathsf{H}}\bm{a}_{j}\bm{a}_{j}^{\mathsf{H}}\bm{v}\right|+\left|\sum_{j=1}^{m}\bm{u}^{\mathsf{H}}\bm{b}_{j}\bm{b}_{j}^{\mathsf{H}}\left(\bm{h}-\bm{h}_{0}\right)\left(\bm{x}-\bm{x}_{0}\right)^{\mathsf{H}}\bm{v}\right|\nonumber \\
		& \qquad+\left|\sum_{j=1}^{m}\bm{u}^{\mathsf{H}}\bm{b}_{j}\bm{b}_{j}^{\mathsf{H}}\left(\bm{h}_{0}-\bm{h}^{\star}\right)\left(\bm{x}-\bm{x}_{0}\right)^{\mathsf{H}}\bm{a}_{j}\bm{a}_{j}^{\mathsf{H}}\bm{v}\right|+\left|\sum_{j=1}^{m}\bm{u}^{\mathsf{H}}\bm{b}_{j}\bm{b}_{j}^{\mathsf{H}}\left(\bm{h}_{0}-\bm{h}^{\star}\right)\left(\bm{x}-\bm{x}_{0}\right)^{\mathsf{H}}\bm{v}\right|\nonumber \\
		& \qquad+\left|\sum_{j=1}^{m}\left(\left(\bm{u}-\bm{u}_{0}\right)^{\mathsf{H}}\bm{b}_{j}\bm{b}_{j}^{\mathsf{H}}\left(\bm{h}_{0}-\bm{h}^{\star}\right)\left(\bm{x}_{0}-\bm{x}^{\star}\right)^{\mathsf{H}}\bm{a}_{j}\bm{a}_{j}^{\mathsf{H}}\bm{v}-\left(\bm{u}-\bm{u}_{0}\right)^{\mathsf{H}}\bm{b}_{j}\bm{b}_{j}^{\mathsf{H}}\left(\bm{h}_{0}-\bm{h}^{\star}\right)\left(\bm{x}_{0}-\bm{x}^{\star}\right)^{\mathsf{H}}\bm{v}\right)\right|\nonumber \\
		& \qquad+\left|\sum_{j=1}^{m}\left(\bm{u}_{0}^{\mathsf{H}}\bm{b}_{j}\bm{b}_{j}^{\mathsf{H}}\left(\bm{h}_{0}-\bm{h}^{\star}\right)\left(\bm{x}_{0}-\bm{x}^{\star}\right)^{\mathsf{H}}\bm{a}_{j}\bm{a}_{j}^{\mathsf{H}}\left(\bm{v}-\bm{v}_{0}\right)-\bm{u}_{0}^{\mathsf{H}}\bm{b}_{j}\bm{b}_{j}^{\mathsf{H}}\left(\bm{h}_{0}-\bm{h}^{\star}\right)\left(\bm{x}_{0}-\bm{x}^{\star}\right)^{\mathsf{H}}\left(\bm{v}-\bm{v}_{0}\right)\right)\right|\nonumber \\
		& \quad\leq\left(\left\Vert \mathcal{A}\right\Vert ^{2}+1\right)\left\Vert \bm{h}-\bm{h}_{0}\right\Vert _{2}\left\Vert \bm{x}-\bm{x}_{0}\right\Vert _{2}+\left(\left\Vert \mathcal{A}\right\Vert ^{2}+1\right)\left\Vert \bm{h}_{0}-\bm{h}^{\star}\right\Vert _{2}\left\Vert \bm{x}-\bm{x}_{0}\right\Vert _{2}\nonumber \\
		& \qquad+\left(\left\Vert \mathcal{A}\right\Vert ^{2}+1\right)\left\Vert \bm{x}_{0}-\bm{x}^{\star}\right\Vert _{2}\left\Vert \bm{u}-\bm{u}_{0}\right\Vert _{2}+\left(\left\Vert \mathcal{A}\right\Vert ^{2}+1\right)\left\Vert \bm{x}_{0}-\bm{x}^{\star}\right\Vert _{2}\left\Vert \bm{v}-\bm{v}_{0}\right\Vert _{2}\nonumber \\
		& \quad\leq\left(2K\log K+10\log m+1\right)\left(2\left(\varepsilon_{1}\right)^{2}+2C_{1}r\varepsilon_{2}\right),\label{eq:lemma8r}
		\end{align}
		where the last inequality arises from \eqref{eq:a-normbound}. This
		further leads to
		\begin{align*}
		& \left|\bm{u}^{\mathsf{H}}\mathcal{T}^{\mathsf{debias}}\left(\left(\bm{h}-\bm{h}^{\star}\right)\left(\bm{x}-\bm{x}^{\star}\right)^{\mathsf{H}}\right)\bm{v}\right|\\
		& \qquad=\left|\sum_{j=1}^{m}\left(\bm{u}^{\mathsf{H}}\bm{b}_{j}\bm{b}_{j}^{\mathsf{H}}\left(\bm{h}-\bm{h}^{\star}\right)\left(\bm{x}-\bm{x}^{\star}\right)^{\mathsf{H}}\bm{a}_{j}\bm{a}_{j}^{\mathsf{H}}\bm{v}\ind_{\left\{ \left|\bm{\Delta}_{\bm{x}}^{\mathsf{H}}\bm{a}_{j}\right|\leq C_{4}r\sqrt{\log m}\right\} }-\bm{u}^{\mathsf{H}}\bm{b}_{j}\bm{b}_{j}^{\mathsf{H}}\left(\bm{h}-\bm{h}^{\star}\right)\left(\bm{x}-\bm{x}^{\star}\right)^{\mathsf{H}}\bm{v}\right)\right|\\
		& \qquad=\left|r\left(\bm{h},\bm{x},\bm{u},\bm{v}\right)-r\left(\bm{h}_{0},\bm{x}_{0},\bm{u}_{0},\bm{v}_{0}\right)\right|+\left|r\left(\bm{h}_{0},\bm{x}_{0},\bm{u}_{0},\bm{v}_{0}\right)\right|\\
		& \qquad\leq\left(2K\log K+10\log m+1\right)\left(2\left(\varepsilon_{1}\right)^{2}+2C_{1}r\varepsilon_{2}\right)+\frac{\lambda}{100}\\
		& \qquad\leq\frac{\lambda}{50},
		\end{align*}
		where the last inequality follows from \eqref{eq:lemma8rcovering}
		and \eqref{eq:lemma8r}. As a consequence, for any point $\left(\bm{h},\bm{x}\right)$
		satisfying \eqref{eq:hx-properties-all}, we have, with probability
		exceeding $1-O\left(m^{-10}+me^{-CK}\right)$, that
		\begin{align}
		& \quad\left\Vert \mathcal{T}^{\mathsf{debias}}\left(\left(\bm{h}-\bm{h}^{\star}\right)\left(\bm{x}-\bm{x}^{\star}\right)^{\mathsf{H}}\right)\right\Vert =\sup_{\bm{u},\bm{v}\in\mathcal{S}^{K-1}}\bm{u}^{\mathsf{H}}\mathcal{T}^{\mathsf{debias}}\left(\left(\bm{h}-\bm{h}^{\star}\right)\left(\bm{x}-\bm{x}^{\star}\right)^{\mathsf{H}}\right)\bm{v}\leq\frac{1}{50}\lambda.\label{eq:lemma8bound3}
		\end{align}
	\end{itemize}
\end{enumerate}
To finish up, combining the bounds obtained in \eqref{eq:lemma8bound1},
\eqref{eq:lemma8bound2} and \eqref{eq:lemma8bound3}, we arrive at
\[
\left\Vert \mathcal{T}^{\mathsf{debias}}\left(\bm{h}\bm{x}^{\mathsf{H}}-\bm{h}^{\star}\bm{x}^{\star\mathsf{H}}\right)\right\Vert \leq\frac{\lambda}{50}+\frac{\lambda}{50}+\frac{\lambda}{50}<\frac{\lambda}{8}.
\]

\subsection{Proof of Lemma \ref{lemma:noise}\label{subsec:Proof-oflemmanoise}}

We intend to invoke \citet[Proposition 2]{koltchinskii2011nuclear}
to bound the spectral norm of the random matrix of interest. Set $\boldsymbol{Z}_{i}=\xi_{i}\bm{b}_{i}\bm{a}_{i}^{\mathsf{H}}$.
Letting $\|\cdot\|_{\psi_{1}}$ (resp.~$\|\cdot\|_{\psi_{2}}$) denoting
the sub-exponential norm of a random variable \citet[Chapter 2]{vershynin2018high},
we have
\[
B_{\bm{Z}}:=\Big\|\big\|\xi_{j}\boldsymbol{b}_{j}\boldsymbol{a}_{j}^{\mathsf{H}}\big\|\Big\|_{\psi_{1}}=\Big\|\xi_{j}\|\boldsymbol{b}_{j}\|_{2}\|\boldsymbol{a}_{j}\|_{2}\Big\|_{\psi_{1}}\leq\left\Vert \xi_{j}\right\Vert _{\psi_{2}}\Big\|\|\boldsymbol{a}_{j}\|_{2}\Big\|_{\psi_{2}}\sqrt{\frac{K}{m}}\lesssim\sigma\frac{K}{\sqrt{m}}.
\]
Here, we have used the assumption that $\Vert\xi_{j}\Vert_{\psi_{2}}\lesssim\sigma$,
as well as the simple facts that $\Vert\bm{b}_{j}\Vert_{2}=\sqrt{K/m}$
and $\big\|\Vert\bm{a}_{j}\Vert_{2}\big\|_{\psi_{2}}\lesssim\sqrt{K}$
(cf.~\citet[Theorem 3.1.1]{vershynin2018high}). In addition, simple
calculation yields 
\begin{align*}
\left\Vert \sum_{j=1}^{m}\mathbb{\mathbb{E}}\big[\bm{Z}_{j}\bm{Z}_{j}^{\mathsf{H}}\big]\right\Vert  & =\left\Vert \sum_{j=1}^{m}\mathbb{\mathbb{E}}\Big[\left|\xi_{j}\right|^{2}\bm{b}_{i}\bm{a}_{i}^{\mathsf{H}}\bm{a}_{i}\bm{b}_{i}^{\mathsf{H}}\Big]\right\Vert =\left\Vert \sum_{j=1}^{m}\mathbb{E}\big[|\xi_{j}|^{2}\big]\mathbb{E}\big[\|\bm{a}_{j}\|_{2}^{2}\big]\bm{b}_{j}\bm{b}_{j}^{\mathsf{H}}\right\Vert \asymp\ensuremath{K\sigma^{2}},\\
\left\Vert \sum_{j=1}^{m}\mathbb{\mathbb{E}}\big[\bm{Z}_{j}^{\mathsf{H}}\bm{Z}_{j}\big]\right\Vert  & =\left\Vert \sum_{j=1}^{m}\mathbb{\mathbb{E}}\Big[\left|\xi_{j}\right|^{2}\bm{a}_{j}\bm{b}_{j}^{\mathsf{H}}\bm{b}_{j}\bm{a}_{j}^{\mathsf{H}}\Big]\right\Vert =\left\Vert \sum_{j=1}^{m}\mathbb{E}\big[|\xi_{j}|^{2}\big]\left\Vert \bm{b}_{j}\right\Vert _{2}^{2}\mathbb{E}\big[\bm{a}_{j}\bm{a}_{j}^{\mathsf{H}}\big]\right\Vert \asymp\ensuremath{K\sigma^{2}},
\end{align*}
which rely on the facts that $\mathbb{E}\big[|\xi_{j}|^{2}\big]\asymp\sigma^{2}$,
$\Vert\bm{b}_{j}\Vert_{2}=\sqrt{K/m}$, $\sum_{j=1}^{m}\bm{b}_{j}\bm{b}_{j}^{\mathsf{H}}=\bm{I}_{k}$
and $\mathbb{E}\big[\bm{a}_{j}\bm{a}_{j}^{\mathsf{H}}\big]=\bm{I}_{k}$.
As a result, by setting
\[
\sigma_{\bm{Z}}:=\max\left\{ \left\Vert \sum\nolimits _{j=1}^{m}\mathbb{\mathbb{E}}\big[\bm{Z}_{j}\bm{Z}_{j}^{\mathsf{H}}\big]\right\Vert ^{1/2},\left\Vert \sum\nolimits _{j=1}^{m}\mathbb{\mathbb{E}}\big[\bm{Z}_{j}^{\mathsf{H}}\bm{Z}_{j}\big]\right\Vert ^{1/2}\right\} \asymp\sigma\sqrt{K},
\]
we can apply the matrix Bernstein inequality \citet[Proposition 2]{koltchinskii2011nuclear}
to derive
\begin{equation}
\left\Vert \sum_{j=1}^{m}\xi_{j}\boldsymbol{b}_{j}\boldsymbol{a}_{j}^{\mathsf{H}}\right\Vert \lesssim\sigma_{\bm{Z}}\sqrt{\log m}+B_{\bm{Z}}\log\left(\frac{B_{\bm{Z}}}{\sigma_{\bm{Z}}}\right)\log m\lesssim\sigma\sqrt{K\log m}\label{eq:lemmanoise1}
\end{equation}
with probability exceeding $1-O(m^{-20})$, where the last inequality
holds as long as $m\gtrsim K\log^{3}m$.

\subsection{Proof of Lemma \ref{lemma:inj}\label{subsec:Proof-of-Lemma-inj}}

By the definition of $T$ (cf.~\eqref{eq:defn-tangent-space}), any
$\boldsymbol{Z}\in T$ takes the following form
\[
\boldsymbol{Z}=\boldsymbol{h}\boldsymbol{u}^{\mathsf{H}}+\boldsymbol{v}\boldsymbol{x}^{\mathsf{H}}
\]
for some $\boldsymbol{u},\boldsymbol{v}\in\mathbb{C}^{K}$. Since
this is an underdetermined system of equations, there might exist
more than one possibilities of $\left(\boldsymbol{h},\boldsymbol{x}\right)$
that enable and are compatible with this decomposition. Here, we shall
take a specific choice among them as follows 
\begin{align}
\left(\boldsymbol{h},\boldsymbol{x}\right) & \coloneqq\arg\min_{\left(\widetilde{\boldsymbol{h}},\widetilde{\boldsymbol{x}}\right)}\left\{ \frac{1}{2}\big\|\widetilde{\boldsymbol{h}}\big\|_{2}^{2}+\frac{1}{2}\left\Vert \widetilde{\boldsymbol{x}}\right\Vert _{2}^{2}\mid\boldsymbol{Z}=\widetilde{\boldsymbol{h}}\boldsymbol{u}^{\mathsf{H}}+\boldsymbol{v}\widetilde{\boldsymbol{x}}^{\mathsf{H}}\text{ for some }\bm{u}\text{ and }\bm{v}\right\} .\label{eq:special-choice-h-x}
\end{align}
As can be straightforwardly verified, this special choice enjoys the
following property
\[
\boldsymbol{h}^{\mathsf{H}}\boldsymbol{v}=\boldsymbol{u}^{\mathsf{H}}\boldsymbol{x},
\]
which plays a crucial role in the proof. 

The proof consists of two steps: (1) showing that
\begin{equation}
\left\Vert \boldsymbol{Z}\right\Vert _{\text{F}}^{2}\leq8\left(\left\Vert \boldsymbol{u}\right\Vert _{2}^{2}+\left\Vert \boldsymbol{v}\right\Vert _{2}^{2}\right),\label{eq:7.1}
\end{equation}
and (2) demonstrating that
\begin{align}
\left\Vert \mathcal{A}\left(\boldsymbol{Z}\right)\right\Vert _{2}^{2} & \geq\frac{1}{2}\left(\left\Vert \boldsymbol{u}\right\Vert _{2}^{2}+\left\Vert \boldsymbol{v}\right\Vert _{2}^{2}\right).\label{eq:7.2}
\end{align}
The first claim \eqref{eq:7.1} can be justified in the same way as
\citet[Equation (81)]{chen2019noisy}; we thus omit this part here
for brevity.

It then boils down to justifying the second claim \eqref{eq:7.2},
towards which we first decompose
\begin{align}
\left\Vert \mathcal{A}\left(\boldsymbol{Z}\right)\right\Vert _{2}^{2} & =\underbrace{\left\Vert \mathcal{A}\left(\boldsymbol{Z}\right)\right\Vert _{2}^{2}-\left\Vert \boldsymbol{Z}\right\Vert _{2}^{2}}_{\eqqcolon\alpha_{1}}+\underbrace{\left\Vert \boldsymbol{Z}\right\Vert _{2}^{2}}_{\eqqcolon\alpha_{2}}.\label{eq:AZ-alpha12}
\end{align}
By repeating the same argument as in \citet[Appendix C.3.1, 2(a)]{chen2019noisy},
we can lower bound $\alpha_{2}$ by
\[
\alpha_{2}\geq\left\Vert \boldsymbol{h}^{\star}\boldsymbol{u}^{\mathsf{H}}\right\Vert _{\text{F}}^{2}+\left\Vert \boldsymbol{v}\boldsymbol{x}^{\mathsf{\star H}}\right\Vert _{\text{F}}^{2}-\frac{1}{50}\left(\left\Vert \boldsymbol{u}\right\Vert _{2}^{2}+\left\Vert \boldsymbol{v}\right\Vert _{2}^{2}\right).
\]
We then turn attention to controlling $\alpha_{1}$. Letting $\boldsymbol{\Delta_{h}}=\bm{h}-\boldsymbol{h}^{\star}$
and $\boldsymbol{\Delta_{x}}=\boldsymbol{x}-\boldsymbol{x}^{\star}$,
we can write
\begin{align*}
\boldsymbol{h}\boldsymbol{u}^{\mathsf{H}}+\boldsymbol{v}\boldsymbol{x}^{\mathsf{H}} & =\left(\text{\ensuremath{\bm{h}^{\star}+\bm{\Delta_{h}}}}\right)\bm{u}^{\mathsf{H}}+\bm{v}\left(\bm{x}^{\star}+\bm{\Delta_{x}}\right)^{\mathsf{H}}\\
 & =\text{\ensuremath{\bm{h}^{\star}\bm{u}^{\mathsf{H}}+\bm{\Delta_{h}}}}\bm{u}^{\mathsf{H}}+\bm{v}\bm{x}^{\star\mathsf{H}}+\bm{v}\bm{\Delta_{x}}^{\mathsf{H}}.
\end{align*}
This implies that $\alpha_{1}$ can be expanded as follows
\begin{align*}
\alpha_{1} & =\underbrace{\left\Vert \mathcal{A}\left(\text{\ensuremath{\text{\ensuremath{\bm{h}^{\star}\bm{u}^{\mathsf{H}}}}+\bm{v}\bm{x}^{\star\mathsf{H}}}}\right)\right\Vert _{2}^{2}-\left\Vert \ensuremath{\text{\ensuremath{\bm{h}^{\star}\bm{u}^{\mathsf{H}}}}+\bm{v}\bm{x}^{\star\mathsf{H}}}\right\Vert _{\text{F}}^{2}}_{\eqqcolon\gamma_{1}}+\underbrace{\left\Vert \mathcal{A}\left(\ensuremath{\text{\ensuremath{\bm{\Delta_{h}}}}\bm{u}^{\mathsf{H}}}+\text{\ensuremath{\bm{v}\bm{\Delta_{x}}^{\mathsf{H}}}}\right)\right\Vert _{2}^{2}-\left\Vert \text{\ensuremath{\text{\ensuremath{\bm{\Delta_{h}}}}\bm{u}^{\mathsf{H}}}+\text{\ensuremath{\bm{v}\bm{\Delta_{x}}^{\mathsf{H}}}}}\right\Vert _{\text{F}}^{2}}_{\eqqcolon\gamma_{2}}\\
 & \quad+2\underbrace{\left\langle \mathcal{A}\left(\text{\ensuremath{\text{\ensuremath{\text{\ensuremath{\bm{h}^{\star}\bm{u}^{\mathsf{H}}}}+\bm{v}\bm{x}^{\star\mathsf{H}}}}}}\right),\mathcal{A}\left(\ensuremath{\text{\ensuremath{\bm{\Delta_{h}}}}\bm{u}^{\mathsf{H}}}+\text{\ensuremath{\bm{v}\bm{\Delta_{x}}^{\mathsf{H}}}}\right)\right\rangle -\left\langle \ensuremath{\text{\ensuremath{\text{\ensuremath{\bm{h}^{\star}\bm{u}^{\mathsf{H}}}}+\bm{v}\bm{x}^{\star\mathsf{H}}}}},\ensuremath{\text{\ensuremath{\bm{\Delta_{h}}}}\bm{u}^{\mathsf{H}}}+\text{\ensuremath{\bm{v}\bm{\Delta_{x}}^{\mathsf{H}}}}\right\rangle }_{\eqqcolon\gamma_{3}},
\end{align*}
thereby motivating us to cope with these terms separately. 
\begin{itemize}
\item Regarding $\gamma_{1}$, it is easily seen that
\[
\left|\gamma_{1}\right|\leq\left\Vert \mathcal{P}_{T}\mathcal{A}^{*}\mathcal{A}\mathcal{P}_{T}-\mathcal{P}_{T}\right\Vert \cdot\left\Vert \text{\ensuremath{\bm{h}^{\star}\bm{u}^{\mathsf{H}}}}+\bm{v}\bm{x}^{\star\mathsf{H}}\right\Vert _{\mathrm{F}}^{2}\leq\frac{1}{100}\left(\left\Vert \bm{u}\right\Vert ^{2}+\left\Vert \bm{v}\right\Vert ^{2}\right),
\]
where the last inequality is obtained by invoking \citet[Lemma 5.12]{li2019rapid}.
\item When it comes to $\gamma_{2}$, we observe that
\[
\gamma_{2}\geq-\left\Vert \text{\ensuremath{\text{\ensuremath{\bm{\Delta_{h}}}}\bm{u}^{\mathsf{H}}}+\text{\ensuremath{\bm{v}\bm{\Delta_{x}}^{\mathsf{H}}}}}\right\Vert _{\text{F}}^{2}\geq-\frac{1}{100}\left(\left\Vert \bm{u}\right\Vert _{2}^{2}+\left\Vert \bm{v}\right\Vert _{2}^{2}\right)
\]
under our constraints on the sizes of $\bm{\Delta_{h}}$ and $\bm{\Delta_{x}}$. 
\item The term $\gamma_{3}$ can be further decomposed into four terms,
which we control separately. 
\begin{enumerate}
\item First of all, observe that\textcolor{black}{
\begin{align}
 & \left|\left\langle \mathcal{A}\left(\text{\ensuremath{\text{\ensuremath{\bm{v}\bm{x}^{\star\mathsf{H}}}}}}\right),\mathcal{A}\left(\text{\ensuremath{\bm{v}\bm{\Delta_{x}}^{\mathsf{H}}}}\right)\right\rangle -\left\langle \ensuremath{\text{\ensuremath{\bm{v}\bm{x}^{\star\mathsf{H}}}}},\text{\ensuremath{\bm{v}\bm{\Delta_{x}}^{\mathsf{H}}}}\right\rangle \right|\nonumber \\
 & \quad\quad\leq\left|\left\langle \mathcal{A}\left(\text{\ensuremath{\text{\ensuremath{\bm{v}\bm{x}^{\star\mathsf{H}}}}}}\right),\mathcal{A}\left(\text{\ensuremath{\bm{v}\bm{\Delta_{x}}^{\mathsf{H}}}}\right)\right\rangle \right|+\left|\left\langle \ensuremath{\text{\ensuremath{\bm{v}\bm{x}^{\star\mathsf{H}}}}},\text{\ensuremath{\bm{v}\bm{\Delta_{x}}^{\mathsf{H}}}}\right\rangle \right|\nonumber \\
 & \quad\quad\overset{\text{(i)}}{\leq}\left\Vert \mathcal{A}\left(\bm{v}\bm{x}^{\star\mathsf{H}}\right)\right\Vert _{2}\left\Vert \mathcal{A}\left(\text{\ensuremath{\bm{v}\bm{\Delta_{x}}^{\mathsf{H}}}}\right)\right\Vert _{2}+\left\Vert \bm{x}^{\star}\right\Vert _{2}\left\Vert \bm{\Delta}_{\bm{x}}^{\mathsf{H}}\right\Vert _{2}\left\Vert \bm{v}\right\Vert _{2}^{2}\nonumber \\
 & \quad\quad\overset{\text{(ii)}}{\leq}\sqrt{\sum_{j=1}^{m}\left|\bm{b}_{j}^{\mathsf{H}}\bm{v}\right|^{2}\left|\bm{x}^{\star\mathsf{H}}\bm{a}_{j}\right|^{2}}\sqrt{\sum_{j=1}^{m}\left|\bm{b}_{j}^{\mathsf{H}}\bm{v}\right|^{2}\left|\bm{\Delta_{x}}^{\mathsf{H}}\bm{a}_{j}\right|^{2}}+\frac{1}{200}\left\Vert \bm{v}\right\Vert _{2}^{2}\nonumber \\
 & \quad\quad\overset{(\text{iii})}{\leq}\sqrt{\left\Vert \bm{v}\right\Vert _{2}^{2}\max_{1\leq j\leq m}\left|\bm{x}^{\star\mathsf{H}}\bm{a}_{j}\right|^{2}}\cdot\sqrt{\left\Vert \bm{v}\right\Vert _{2}^{2}\max_{1\leq j\leq m}\left|\bm{\Delta_{x}}^{\mathsf{H}}\bm{a}_{j}\right|^{2}}+\frac{1}{200}\left\Vert \bm{v}\right\Vert _{2}^{2}\nonumber \\
 & \quad\quad\overset{\text{(iv)}}{\leq}20\sqrt{\log m}\cdot C\sqrt{\log m}\left(\lambda+\sigma\sqrt{K\log m}\right)\left\Vert \bm{v}\right\Vert _{2}^{2}+\frac{1}{200}\left\Vert \bm{v}\right\Vert _{2}^{2}\nonumber \\
 & \quad\quad\leq\frac{1}{100}\left\Vert \bm{v}\right\Vert _{2}^{2},\label{eq:lem-inj-gamma3}
\end{align}
where the (i) and (ii) follow from the Cauchy-Schwarz inequality and
\eqref{eq:hx-properties} that $\left\Vert \bm{\Delta}_{\bm{x}}^{\mathsf{H}}\right\Vert _{2}\lesssim\lambda+\sigma\sqrt{K\log m}\leq1/200$;
(iii) comes from the fact that $\sum_{j=1}^{m}\bm{b}_{j}\bm{b}_{j}^{\mathsf{H}}=\bm{I}_{K}$
and thus $\sum_{j=1}^{m}\left|\bm{b}_{j}^{\mathsf{H}}\bm{v}\right|^{2}=\sum_{j=1}^{m}\bm{v}^{\mathsf{H}}\bm{b}_{j}\bm{b}_{j}^{\mathsf{H}}\bm{v}=\bm{v}^{\mathsf{H}}\bm{v}=\Vert\bm{v}\Vert_{2}^{2}$;
(iv) is due to Lemma \ref{lemma:useful} and \eqref{eq:hx-properties2};
and the last inequality holds true as long as $\sigma\sqrt{K\log^{3}m}\ll1$.}
\item \textcolor{black}{Similarly, we can demonstrate that
\begin{align*}
 & \left|\left\langle \mathcal{A}\left(\text{\ensuremath{\text{\ensuremath{\text{\ensuremath{\bm{h}^{\star}\bm{u}^{\mathsf{H}}}}}}}}\right),\mathcal{A}\left(\text{\ensuremath{\bm{v}\bm{\Delta_{x}}^{\mathsf{H}}}}\right)\right\rangle -\left\langle \ensuremath{\text{\ensuremath{\text{\ensuremath{\bm{h}^{\star}\bm{u}^{\mathsf{H}}}}}}},\text{\ensuremath{\bm{v}\bm{\Delta_{x}}^{\mathsf{H}}}}\right\rangle \right|\\
 & \quad\quad\overset{(\text{i})}{\leq}\sqrt{\sum_{j=1}^{m}\left|\bm{b}_{j}^{\mathsf{H}}\bm{h}^{\star}\right|^{2}\left|\bm{u}^{\mathsf{H}}\bm{a}_{j}\right|^{2}}\sqrt{\sum_{j=1}^{m}\left|\bm{b}_{j}^{\mathsf{H}}\bm{v}\right|^{2}\left|\bm{\Delta_{x}}^{\mathsf{H}}\bm{a}_{j}\right|^{2}}+\frac{1}{200}\left\Vert \bm{u}\right\Vert _{2}\left\Vert \bm{v}\right\Vert _{2}\\
 & \quad\quad\overset{(\text{ii})}{\leq}\sqrt{\sum_{j=1}^{m}\left|\bm{b}_{j}^{\mathsf{H}}\bm{h}^{\star}\right|^{2}\left|\bm{u}^{\mathsf{H}}\bm{a}_{j}\right|^{2}}\cdot C\sqrt{\log m}\left(\lambda+\sigma\sqrt{K\log m}\right)\left\Vert \bm{v}\right\Vert _{2}+\frac{1}{200}\left\Vert \bm{u}\right\Vert _{2}\left\Vert \bm{v}\right\Vert _{2}\\
 & \quad\quad\leq\frac{1}{100}\left\Vert \bm{u}\right\Vert _{2}\left\Vert \bm{v}\right\Vert _{2},
\end{align*}
where (i) holds for the same reason as Step (ii) in \eqref{eq:lem-inj-gamma3};
(ii) arises due to the identity $\sum_{j=1}^{m}\left|\bm{b}_{j}^{\mathsf{H}}\bm{v}\right|^{2}=\Vert\bm{v}\Vert_{2}^{2}$
and \eqref{eq:hx-properties2}; and the last inequality relies on
the following claim.}\begin{claim}\label{claim:lemma-inj-hoeffding}With
probability exceeding $1-O\left(m^{-100}\right)$, the following inequality
\begin{equation}
\left|\sum_{j=1}^{m}\left|\bm{b}_{j}^{\mathsf{H}}\bm{h}^{\star}\right|^{2}\left|\bm{u}^{\mathsf{H}}\bm{a}_{j}\right|^{2}-\left\Vert \bm{u}\right\Vert _{2}^{2}\right|\lesssim\sqrt{\frac{\mu^{2}K\log m}{m}}\left\Vert \bm{u}\right\Vert _{2}^{2}\label{eq:lemma7hoeffding}
\end{equation}
holds uniformly for any $\bm{u}$. \end{claim}\begin{proof}See Appendix
\ref{subsec:Proof-of-Claim-inj-123}.\end{proof}
\item \textcolor{black}{The next term we shall control is}
\begin{align*}
\left\langle \mathcal{A}\left(\text{\ensuremath{\text{\ensuremath{\text{\ensuremath{\bm{h}^{\star}\bm{u}^{\mathsf{H}}}}}}}}\right),\mathcal{A}\left(\ensuremath{\text{\ensuremath{\bm{\Delta_{h}}}}\bm{u}^{\mathsf{H}}}\right)\right\rangle -\left\langle \ensuremath{\text{\ensuremath{\text{\ensuremath{\bm{h}^{\star}\bm{u}^{\mathsf{H}}}}}}},\ensuremath{\text{\ensuremath{\bm{\Delta_{h}}}}\bm{u}^{\mathsf{H}}}\right\rangle  & =\sum_{j=1}^{m}\left(\bm{b}_{j}^{\mathsf{H}}\bm{h}^{\star}\right)\left(\bm{b}_{j}^{\mathsf{H}}\text{\ensuremath{\bm{\Delta_{h}}}}\right)\left(\left|\bm{a}_{j}^{\mathsf{H}}\bm{u}\right|^{2}-\left\Vert \bm{u}\right\Vert _{2}^{2}\right).
\end{align*}
By virtue of the Bernstein inequality \citet[Theorem 2.8.2]{vershynin2018high},
we have 
\begin{align*}
 & \mathbb{P}\left(\left|\left\langle \mathcal{A}\left(\text{\ensuremath{\text{\ensuremath{\text{\ensuremath{\bm{h}^{\star}\bm{u}^{\mathsf{H}}}}}}}}\right),\mathcal{A}\left(\ensuremath{\text{\ensuremath{\bm{\Delta_{h}}}}\bm{u}^{\mathsf{H}}}\right)\right\rangle -\left\langle \ensuremath{\text{\ensuremath{\text{\ensuremath{\bm{h}^{\star}\bm{u}^{\mathsf{H}}}}}}},\ensuremath{\text{\ensuremath{\bm{\Delta_{h}}}}\bm{u}^{\mathsf{H}}}\right\rangle \right|\geq\tau\left\Vert \bm{u}\right\Vert _{2}^{2}\right)\\
 & \qquad\leq2\max\left\{ \exp\left(-\frac{\tau^{2}}{4\left\Vert \bm{B}\text{\ensuremath{\bm{\Delta_{h}}}}\right\Vert _{\infty}^{2}}\right),\exp\left(-\frac{\tau}{4\left\Vert \bm{B}\text{\ensuremath{\bm{\Delta_{h}}}}\right\Vert _{\infty}\left\Vert \bm{B}\bm{h}^{\star}\right\Vert _{\infty}}\right)\right\} 
\end{align*}
for any $\tau\geq0$. Let us choose $\tau$ to be
\begin{align*}
\tau & =2\left\Vert \bm{B}\text{\ensuremath{\bm{\Delta_{h}}}}\right\Vert _{\infty}\sqrt{2K\log m}+8\left\Vert \bm{B}\text{\ensuremath{\bm{\Delta_{h}}}}\right\Vert _{\infty}\left\Vert \bm{B}\bm{h}^{\star}\right\Vert _{\infty}K\log m.
\end{align*}
In view of \eqref{eq:hx-properties2} and \eqref{eq:incoherence-condition},
this quantity is bounded above by
\begin{align*}
\tau & \lesssim2\sigma\sqrt{2K\log m}+8\sigma\frac{\mu}{\sqrt{m}}K\log m\leq\frac{1}{100}.
\end{align*}
It then follows that
\begin{align}
\mathbb{P}\left(\left|\left\langle \mathcal{A}\left(\bm{h}^{\star}\bm{u}^{\mathsf{H}}\right),\mathcal{A}\left(\text{\ensuremath{\bm{\Delta_{h}}}}\bm{u}^{\mathsf{H}}\right)\right\rangle -\left\langle \bm{h}^{\star}\bm{u}^{\mathsf{H}},\ensuremath{\text{\ensuremath{\bm{\Delta_{h}}}}\bm{u}^{\mathsf{H}}}\right\rangle \right|\geq\frac{1}{100}\left\Vert \bm{u}\right\Vert _{2}^{2}\right) & \text{\ensuremath{\leq}}2\exp\left(-2K\log m\right).\label{eq:lem7-gamma3-3}
\end{align}
Additionally, define $r\coloneqq\lambda+\sigma\sqrt{K\log m}$, and
let $\mathcal{N}_{\bm{h}}$ be an $\varepsilon_{1}$-net of $\mathcal{B}_{\bm{h}}\left(\frac{C_{5}}{1-\rho}\eta r\right)\coloneqq\left\{ \bm{h}:\left\Vert \bm{h}-\bm{h}^{\star}\right\Vert _{2}\leq\frac{C_{5}}{1-\rho}\eta r\right\} $
and $\mathcal{N}_{0}$ an $\varepsilon_{2}$-net of the unit sphere
$\mathcal{S}^{K-1}=\left\{ \bm{u}\in\mathbb{C}^{K}:\left\Vert \bm{u}\right\Vert _{2}=1\right\} $.
Let $\varepsilon_{1}=r/\left(m\log m\right)$ and $\varepsilon_{2}=1/\left(m\log m\right)$.
In view of \citet[Corollary 4.2.13]{vershynin2018high}, it is seen
that
\[
\left|\mathcal{N}_{\bm{h}}\right|\leq\left(1+\frac{2C_{5}\eta r}{\left(1-\rho\right)\varepsilon_{1}}\right)^{2K}\quad\mathrm{and}\quad\left|\mathcal{N}_{0}\right|\leq\left(1+\frac{2}{\varepsilon_{2}}\right)^{2K}.
\]
Taking the union bound indicates that with probability at least 
\[
1-\left(1+\frac{2C_{5}\eta r}{\left(1-\rho\right)\varepsilon_{1}}\right)^{2K}\left(1+\frac{2}{\varepsilon_{2}}\right)^{4K}\cdot2e^{-2K\log m}\geq1-O\left(m^{-100}\right),
\]
the following inequality
\[
\left|\left\langle \mathcal{A}\left(\bm{h}^{\star}\bm{u}^{\mathsf{H}}\right),\mathcal{A}\left(\text{\ensuremath{\bm{\Delta_{h}}}}\bm{u}^{\mathsf{H}}\right)\right\rangle -\left\langle \bm{h}^{\star}\bm{u}^{\mathsf{H}},\ensuremath{\text{\ensuremath{\bm{\Delta_{h}}}}\bm{u}^{\mathsf{H}}}\right\rangle \right|\geq\frac{1}{100}\left\Vert \bm{u}\right\Vert _{2}^{2}
\]
holds uniformly for all $(\bm{h},\bm{u})\in\mathcal{N}_{\bm{h}}\times\mathcal{N}_{0}$.
As a result, for any $(\bm{h},\bm{u})\in\mathcal{N}_{\bm{h}}\times\mathcal{N}_{0}$,
there holds 
\begin{align*}
\left|\left\langle \mathcal{A}\left(\bm{h}^{\star}\bm{u}^{\mathsf{H}}\right),\mathcal{A}\left(\text{\ensuremath{\bm{\Delta_{h}}}}\bm{u}^{\mathsf{H}}\right)\right\rangle -\left\langle \bm{h}^{\star}\bm{u}^{\mathsf{H}},\ensuremath{\text{\ensuremath{\bm{\Delta_{h}}}}\bm{u}^{\mathsf{H}}}\right\rangle \right|\geq\frac{1}{100}\left\Vert \bm{u}\right\Vert _{2}^{2}
\end{align*}
with probability exceeding $1-O\left(m^{-100}\right)$. Furthermore,
if we let 
\[
F\left(\bm{h},\bm{u}\right)\coloneqq\left\langle \mathcal{A}\left(\bm{h}^{\star}\bm{u}^{\mathsf{H}}\right),\mathcal{A}\left(\text{\ensuremath{\bm{\Delta_{h}}}}\bm{u}^{\mathsf{H}}\right)\right\rangle -\left\langle \bm{h}^{\star}\bm{u}^{\mathsf{H}},\ensuremath{\text{\ensuremath{\bm{\Delta_{h}}}}\bm{u}^{\mathsf{H}}}\right\rangle ,
\]
then for any $\bm{h}\in\mathcal{B}_{\bm{h}}\left(\frac{C_{5}}{1-\rho}\eta r\right)$
and $\bm{u}\in\mathcal{S}^{K-1}$, we can find a point $(\bm{h}_{0},\bm{u}_{0})\in\mathcal{N}_{\bm{h}}\times\mathcal{N}_{0}$
satisfying $\left\Vert \bm{h}-\bm{h}_{0}\right\Vert _{2}\leq\varepsilon_{1}$
and $\left\Vert \bm{u}-\bm{u}_{0}\right\Vert _{2}\leq\varepsilon_{2}$.
Consequently, one can deduce that
\begin{align*}
 & \left|F\left(\bm{h},\bm{u}\right)-F\left(\bm{h}_{0},\bm{u}_{0}\right)\right|\\
 & \quad\leq\left|\left\langle \mathcal{A}\left(\bm{h}^{\star}\left(\bm{u}-\bm{u}_{0}\right)^{\mathsf{H}}\right),\mathcal{A}\left(\left(\bm{h}-\bm{h}^{\star}\right)\bm{u}^{\mathsf{H}}\right)\right\rangle -\left\langle \bm{h}^{\star}\left(\bm{u}-\bm{u}_{0}\right)^{\mathsf{H}},\left(\bm{h}-\bm{h}^{\star}\right)\bm{u}^{\mathsf{H}}\right\rangle \right|\\
 & \quad\quad+\left|\left\langle \mathcal{A}\left(\bm{h}^{\star}\bm{u}_{0}^{\mathsf{H}}\right),\mathcal{A}\left(\left(\bm{h}-\bm{h}_{0}\right)\bm{u}^{\mathsf{H}}\right)\right\rangle -\left\langle \bm{h}^{\star}\bm{u}_{0}^{\mathsf{H}},\left(\bm{h}-\bm{h}_{0}\right)\bm{u}^{\mathsf{H}}\right\rangle \right|\\
 & \quad\quad+\left|\left\langle \mathcal{A}\left(\bm{h}^{\star}\bm{u}_{0}^{\mathsf{H}}\right),\mathcal{A}\left(\left(\bm{h}-\bm{h}_{0}\right)\left(\bm{u}-\bm{u}_{0}\right)^{\mathsf{H}}\right)\right\rangle -\left\langle \bm{h}^{\star}\bm{u}_{0}^{\mathsf{H}},\left(\bm{h}-\bm{h}_{0}\right)\left(\bm{u}-\bm{u}_{0}\right)^{\mathsf{H}}\right\rangle \right|\\
 & \quad\leq\left(\left\Vert \mathcal{A}\right\Vert ^{2}+1\right)\left\Vert \bm{h}^{\star}\right\Vert \left\Vert \bm{u}\right\Vert _{2}\left\Vert \bm{u}-\bm{u}_{0}\right\Vert _{2}\left\Vert \bm{h}-\bm{h}^{\star}\right\Vert _{2}+\left(\left\Vert \mathcal{A}\right\Vert ^{2}+1\right)\left\Vert \bm{h}^{\star}\right\Vert _{2}\left\Vert \bm{u}\right\Vert _{2}\left\Vert \bm{u}_{0}\right\Vert _{2}\left\Vert \bm{h}-\bm{h}_{0}\right\Vert _{2}\\
 & \quad\quad+\left(\left\Vert \mathcal{A}\right\Vert ^{2}+1\right)\left\Vert \bm{h}^{\star}\right\Vert _{2}\left\Vert \bm{u}_{0}\right\Vert _{2}\left\Vert \bm{u}-\bm{u}_{0}\right\Vert _{2}\left\Vert \bm{h}-\bm{h}_{0}\right\Vert _{2}\\
 & \quad\leq\left(2K\log K+10\log m+1\right)\left(\frac{C_{5}}{1-\rho}\eta r\epsilon_{2}+\epsilon_{1}+\epsilon_{1}\epsilon_{2}\right)\\
 & \quad\leq\frac{1}{100}\left\Vert \bm{u}\right\Vert _{2}^{2}
\end{align*}
as long as $m\gg K$, where the above bound on $\Vert\mathcal{A}\Vert$
relies on Lemma \ref{lemma:normbound}. Hence, with probability exceeding
$1-O\left(m^{-10}\right)$ we have
\begin{align*}
\left|\left\langle \mathcal{A}\left(\text{\ensuremath{\text{\ensuremath{\text{\ensuremath{\bm{h}^{\star}\bm{u}^{\mathsf{H}}}}}}}}\right),\mathcal{A}\left(\ensuremath{\text{\ensuremath{\bm{\Delta_{h}}}}\bm{u}^{\mathsf{H}}}\right)\right\rangle -\left\langle \ensuremath{\text{\ensuremath{\text{\ensuremath{\bm{h}^{\star}\bm{u}^{\mathsf{H}}}}}}},\ensuremath{\text{\ensuremath{\bm{\Delta_{h}}}}\bm{u}^{\mathsf{H}}}\right\rangle \right| & \leq\left|F\left(\bm{h},\bm{u}\right)-F\left(\bm{h}_{0},\bm{u}_{0}\right)\right|+\left|F\left(\bm{h}_{0},\bm{u}_{0}\right)\right|\\
 & \leq\frac{1}{100}\left\Vert \bm{u}\right\Vert _{2}^{2}+\frac{1}{100}\left\Vert \bm{u}\right\Vert _{2}^{2}\leq\frac{1}{50}\left\Vert \bm{u}\right\Vert _{2}^{2},
\end{align*}
which holds uniformly over all $\bm{h}\in\mathcal{B}_{\bm{h}}\left(\frac{C_{5}}{1-\rho}\eta r\right)$
and $\bm{u}\in\mathcal{S}^{K-1}$.
\item \textcolor{black}{The bound on $\left\langle \mathcal{A}\left(\text{\ensuremath{\text{\ensuremath{\bm{v}\bm{x}^{\star\mathsf{H}}}}}}\right),\mathcal{A}\left(\ensuremath{\text{\ensuremath{\bm{\Delta_{h}}}}\bm{u}^{\mathsf{H}}}\right)\right\rangle -\left\langle \ensuremath{\text{\ensuremath{\bm{v}\bm{x}^{\star\mathsf{H}}}}},\ensuremath{\text{\ensuremath{\bm{\Delta_{h}}}}\bm{u}^{\mathsf{H}}}\right\rangle $
can be obtained in a similar manner; we thus omit it here for simplicity. }
\item \textcolor{black}{The above bounds on four terms taken collectively
demonstrate that
\[
\left|\gamma_{3}\right|\leq\frac{1}{100}\left\Vert \bm{v}\right\Vert _{2}^{2}+\frac{1}{100}\left\Vert \bm{u}\right\Vert _{2}\left\Vert \bm{v}\right\Vert _{2}+\frac{1}{50}\left\Vert \bm{u}\right\Vert _{2}^{2}+\frac{1}{100}\left\Vert \bm{u}\right\Vert _{2}\left\Vert \bm{v}\right\Vert _{2}\leq\frac{1}{25}\left(\left\Vert \bm{u}\right\Vert _{2}^{2}+\left\Vert \bm{v}\right\Vert _{2}^{2}\right).
\]
}
\end{enumerate}
\end{itemize}
\textcolor{black}{Combining the above results, we can continue the
relation \eqref{eq:AZ-alpha12} to conclude that
\begin{align*}
\left\Vert \mathcal{A}\left(\boldsymbol{Z}\right)\right\Vert _{2}^{2} & =\alpha_{2}+\alpha_{1}\\
 & \geq\left\Vert \boldsymbol{h}^{\star}\boldsymbol{u}^{\mathsf{H}}\right\Vert _{\text{F}}^{2}+\left\Vert \boldsymbol{v}\boldsymbol{x}^{\mathsf{\star H}}\right\Vert _{\text{F}}^{2}-\frac{1}{50}\left(\left\Vert \boldsymbol{u}\right\Vert _{2}^{2}+\left\Vert \boldsymbol{v}\right\Vert _{2}^{2}\right)-\left|\gamma_{1}\right|-\frac{1}{100}\left(\left\Vert \bm{u}\right\Vert _{2}^{2}+\left\Vert \bm{v}\right\Vert _{2}^{2}\right)-\frac{1}{25}\left(\left\Vert \bm{u}\right\Vert _{2}^{2}+\left\Vert \bm{v}\right\Vert _{2}^{2}\right)\\
 & \geq\frac{1}{2}\left(\left\Vert \bm{u}\right\Vert _{2}^{2}+\left\Vert \bm{v}\right\Vert _{2}^{2}\right)
\end{align*}
as claimed. }

\subsubsection{Proof of Claim \ref{claim:lemma-inj-hoeffding}\label{subsec:Proof-of-Claim-inj-123}}

We start by defining 
\[
\eta\coloneqq\sum_{j=1}^{m}\left|\bm{b}_{j}^{\mathsf{H}}\bm{h}^{\star}\right|^{2}\left(\left|\bm{a}_{j}^{\mathsf{H}}\bm{u}\right|^{2}-\left\Vert \bm{u}\right\Vert _{2}^{2}\right),
\]
which is the sum of sub-exponential variables with zero mean $\mathbb{E}\left[\left|\bm{a}_{j}^{\mathsf{H}}\bm{u}\right|^{2}-\left\Vert \bm{u}\right\Vert _{2}^{2}\right]=0$.

\paragraph{Concentration.} In view of the Bernstein inequality (cf.~\citet[Theorem 2.8.2]{vershynin2018high}),
we have 
\begin{align*}
 & \mathbb{P}\left(\left|\sum_{j=1}^{m}\left|\bm{b}_{j}^{\mathsf{H}}\bm{h}^{\star}\right|^{2}\left(\left|\bm{a}_{j}^{\mathsf{H}}\bm{u}\right|^{2}-\left\Vert \bm{u}\right\Vert _{2}^{2}\right)\right|\geq\tau\left\Vert \bm{u}\right\Vert _{2}^{2}\right)\\
 & \qquad\leq2\max\left\{ \exp\left(-\frac{\tau^{2}}{4\left\Vert \bm{B}\bm{h}^{\star}\right\Vert _{\infty}^{2}\left\Vert \bm{u}\right\Vert _{2}^{2}}\right),\exp\left(-\frac{\tau}{4\left\Vert \bm{B}\text{\ensuremath{\bm{\Delta_{h}}}}\right\Vert _{\infty}^{2}\left\Vert \bm{u}\right\Vert _{2}}\right)\right\} 
\end{align*}
for any $\tau\geq0$. Set 
\[
\tau=4\left\Vert \bm{B}\bm{h}^{\star}\right\Vert _{\infty}\left\Vert \bm{u}\right\Vert _{2}\sqrt{2K\log m}+16\left\Vert \bm{B}\bm{h}^{\star}\right\Vert _{\infty}^{2}\left\Vert \bm{u}\right\Vert _{2}K\log m,
\]
then there holds 
\begin{equation}
\mathbb{P}\left(\left|\sum_{j=1}^{m}\left|\bm{b}_{j}^{\mathsf{H}}\bm{h}^{\star}\right|^{2}\left(\left|\bm{a}_{j}^{\mathsf{H}}\bm{u}\right|^{2}-\left\Vert \bm{u}\right\Vert _{2}^{2}\right)\right|\geq\tau\left\Vert \bm{u}\right\Vert _{2}^{2}\right)\leq2\exp\left(-4K\text{\ensuremath{\log m}}\right).\label{eq:lemma7gamma1}
\end{equation}

\paragraph{Union bound.}
Next, define $\mathcal{N}_{0}$ to be an $\epsilon_{0}$-net of the
unit sphere $\mathcal{S}^{K-1}\coloneqq\left\{ \bm{u}\in\mathbb{C}^{K}:\left\Vert \bm{u}\right\Vert _{2}=1\right\} $,
which can be chosen to obey \citet[Corollary 4.2.13]{vershynin2018high}
\[
\left|\mathcal{N}_{0}\right|\leq\left(1+\frac{2}{\epsilon_{0}}\right)^{2K}.
\]
By taking the union bound over $\mathcal{N}_{0}$, we reach
\[
\left|\sum_{j=1}^{m}\left|\bm{b}_{j}^{\mathsf{H}}\bm{h}^{\star}\right|^{2}\left(\left|\bm{a}_{j}^{\mathsf{H}}\bm{u}\right|^{2}-\left\Vert \bm{u}\right\Vert _{2}^{2}\right)\right|\geq4\left\Vert \bm{B}\bm{h}^{\star}\right\Vert _{\infty}\sqrt{2K\log m}+16\left\Vert \bm{B}\bm{h}^{\star}\right\Vert _{\infty}^{2}K\log m,\qquad\forall\bm{u}\in\mathcal{N}_{0}
\]
with probability at least 
\[
1-\left(1+\frac{2}{\epsilon_{0}}\right)^{2K}e^{-4K\log m}\geq1-O\left(m^{-10}\right).
\]

\paragraph{Approximation.}
Our goal is then to extend the above concentration result to cover
all $\bm{h}\in\mathcal{B}_{\bm{h}}$, $\bm{u}\in\mathcal{S}^{K-1}$
simultaneously, towards which we invoke the standard epsilon-net argument.
For any $\bm{u}\in\mathcal{S}^{K-1}$, let $\bm{u}_{0}\in\mathcal{N}_{0}$
be a point satisfying $\left\Vert \bm{u}-\bm{u}_{0}\right\Vert _{2}\leq\epsilon_{0}$.
Then straightforward calculation gives
\begin{align*}
 & \left|\left(\sum_{j=1}^{m}\left|\bm{b}_{j}^{\mathsf{H}}\bm{h}^{\star}\right|^{2}\left(\left|\bm{a}_{j}^{\mathsf{H}}\bm{u}\right|^{2}-\left\Vert \bm{u}\right\Vert _{2}^{2}\right)\right)-\left(\sum_{j=1}^{m}\left|\bm{b}_{j}^{\mathsf{H}}\bm{h}^{\star}\right|^{2}\left(\left|\bm{a}_{j}^{\mathsf{H}}\bm{u}_{0}\right|^{2}-\left\Vert \bm{u}_{0}\right\Vert _{2}^{2}\right)\right)\right|\\
 & \qquad\overset{(\text{i})}{=}\left|\sum_{j=1}^{m}\left|\bm{b}_{j}^{\mathsf{H}}\bm{h}^{\star}\right|^{2}\left|\bm{a}_{j}^{\mathsf{H}}\bm{u}\right|^{2}-\left\Vert \bm{h}^{\star}\right\Vert _{2}^{2}\left\Vert \bm{u}\right\Vert _{2}^{2}-\sum_{j=1}^{m}\left|\bm{b}_{j}^{\mathsf{H}}\bm{h}^{\star}\right|^{2}\left|\bm{a}_{j}^{\mathsf{H}}\bm{u}_{0}\right|^{2}+\left\Vert \bm{h}^{\star}\right\Vert _{2}^{2}\left\Vert \bm{u}_{0}\right\Vert _{2}^{2}\right|\\
 & \qquad=\left|\left\Vert \mathcal{A}\left(\bm{h}^{\star}\text{\ensuremath{\bm{u}^{\mathsf{H}}}}\right)\right\Vert _{2}^{2}-\left\Vert \mathcal{A}\left(\bm{h}^{\star}\bm{u}_{0}^{\mathsf{H}}\right)\right\Vert _{2}^{2}+\left\Vert \bm{u}_{0}\right\Vert _{2}^{2}-\left\Vert \bm{u}\right\Vert _{2}^{2}\right|\\
 & \qquad\overset{(\text{ii})}{\leq}\left|\left\Vert \mathcal{A}\left(\bm{h}^{\star}\text{\ensuremath{\bm{u}^{\mathsf{H}}}}\right)\right\Vert _{2}^{2}-\left\Vert \mathcal{A}\left(\bm{h}^{\star}\bm{u}_{0}^{\mathsf{H}}\right)\right\Vert _{2}^{2}\right|+\left\Vert \bm{u}_{0}-\bm{u}\right\Vert _{2}\left(\left\Vert \bm{u}_{0}\right\Vert _{2}+\left\Vert \bm{u}\right\Vert _{2}\right)\\
 & \qquad\text{\ensuremath{\overset{(\text{iii})}{\leq}\left|\left(\left\Vert \mathcal{A}\left(\bm{h}^{\star}\text{\ensuremath{\bm{u}^{\mathsf{H}}}}\right)\right\Vert _{2}+\left\Vert \mathcal{A}\left(\bm{h}^{\star}\bm{u}_{0}^{\mathsf{H}}\right)\right\Vert _{2}\right)\left\Vert \mathcal{A}\left(\bm{h}^{\star}\text{\ensuremath{\bm{u}^{\mathsf{H}}}}\right)-\mathcal{A}\left(\bm{h}^{\star}\bm{u}_{0}^{\mathsf{H}}\right)\right\Vert _{2}\right|}}+\epsilon_{0}\\
 & \qquad\lesssim\left\Vert \mathcal{A}\right\Vert ^{2}\left(\left\Vert \bm{h}^{\star}\right\Vert _{2}\left\Vert \bm{u}\right\Vert _{2}+\left\Vert \bm{h}^{\star}\right\Vert _{2}\left\Vert \bm{u}_{0}\right\Vert _{2}\right)\left\Vert \bm{h}^{\star}\right\Vert _{2}\left\Vert \bm{u}-\bm{u}_{0}\right\Vert _{2}+\epsilon_{0}\\
 & \qquad\overset{(\text{iv})}{\leq}\left(4K\log K+20\log m+1\right)\epsilon_{0},
\end{align*}
where (i) comes from $\sum_{j=1}^{m}\left|\bm{b}_{j}^{\mathsf{H}}\bm{h}^{\star}\right|^{2}=\left\Vert \bm{h}^{\star}\right\Vert _{2}^{2}$;
(ii) and (iii) are due to triangle inequality; (iv) follows from the
following bound

\begin{equation}
\left\Vert \mathcal{A}\right\Vert \leq\sqrt{2K\log K+10\log m},\label{eq:a-normbound}
\end{equation}
which holds with probability at least $1-O\left(m^{-10}\right)$ according
to Lemma \ref{lemma:normbound}. Letting $\epsilon_{0}=r/\left(m\log m\right)$
with $r=\lambda+\sigma\sqrt{K\log m}$, we note it satisfies
\[
1-\left(1+\frac{2}{\epsilon_{0}}\right)^{2K}e^{-4K\log m}\geq1-O\left(m^{-10}\right).
\]

\paragraph{Putting all this together.}
 Therefore, we conclude that: with probability at least $1-O\left(m^{-10}\right)$,
one has
\begin{align*}
\left|\text{\ensuremath{\eta}}\right| & \leq4\left\Vert \bm{B}\bm{h}^{\star}\right\Vert _{\infty}\sqrt{2K\log m}+16\left\Vert \bm{B}\bm{h}^{\star}\right\Vert _{\infty}^{2}K\log m+\left(4K\log K+20\log m+1\right)\epsilon_{0}\\
 & \lesssim\sqrt{\frac{\mu^{2}K\log m}{m}}
\end{align*}
uniformly for all $\bm{h}\in\mathcal{B}_{\bm{h}}$ and $\bm{u}\in\mathcal{S}^{K-1}$,
with the proviso that $m\geq C\mu^{2}K\log m$. Here, the second inequality
arises from \eqref{eq:incoherence-condition}.

\section{Analysis: Nonconvex formulation under Gaussian design\label{appendix:gaussian-ncvx}}

We consider the loss function 
\begin{equation}
\underset{\ensuremath{\boldsymbol{Z}\in\mathbb{C}^{K\times K}}}{\mathrm{minimize}}\quad f\left(\bm{h},\bm{x}\right)=\sum_{j=1}^{m}\left|\boldsymbol{b}_{j}^{\mathsf{H}}\bm{h}\bm{x}^{\mathsf{H}}\boldsymbol{a}_{j}-y_{j}\right|^{2}+\lambda\left\Vert \boldsymbol{h}\right\Vert _{2}^{2}+\lambda\left\Vert \boldsymbol{x}\right\Vert _{2}^{2}.\label{eq:gaussian-ncvx-loss}
\end{equation}
The main idea similar to the one presented in Appendix \ref{appendix:noncvx}, although the proof for Gaussian design is easier due to the presence of more randomness.
We shall also assume $\Vert\bm{h}^{\star}\Vert_{2}=\Vert\bm{x}^{\star}\Vert_{2}=1$
for the sake of simplicity and adopt the same notation as (\ref{eq:defn-alignment-alphat})-(\ref{eq:defn-hhat-xhat}).
The main part of the analysis lies in demonstrating the following set
of hypotheses by induction:
\begin{subequations}\label{sec:gaussian-hypotheses-ncvx}
\begin{align}
\mathsf{dist}\left(\boldsymbol{z}^{t},\boldsymbol{z}^{\star}\right) & \leq\big\|\widehat{\boldsymbol{z}}^{t-1/2}-\boldsymbol{z}^{\star}\big\|_{2}\leq\rho\mathsf{dist}\left(\bm{z}^{t-1},\bm{z}^{\star}\right)+C_{11}\eta\left(\lambda+\sigma\sqrt{mK\log m}\right) \label{eq:gaussian-hypothesis-dist}\\
\mathsf{dist}\big(\boldsymbol{z}^{t,\left(l\right)},\widetilde{\bm{z}}^{t}\big) & \leq C_{12}\left(\frac{\sqrt{K\log^{3}m}}{m}+\frac{\sigma\sqrt{K\log^{2}m}}{m}\right) \label{eq:gaussian-hypothesis-loo}\\
\max_{1\leq l\leq m}\left|\boldsymbol{a}_{l}^{\mathsf{H}}\left(\widetilde{\bm{x}}^{t}-\boldsymbol{x}^{\star}\right)\right| & \leq C_{13}\left(\frac{\sqrt{mK\log^{3}m}}{m}+\frac{\sigma\sqrt{mK\log^{2}m}}{m}\right) \label{eq:gaussian-hypothesis-incoherence1}\\
\max_{1\leq l\leq m}\big|\boldsymbol{b}_{l}^{\mathsf{H}}\big(\widetilde{\bm{h}}^{t}-\bm{h}^{\star}\big)\big| & \leq C_{13}\left(\frac{\sqrt{mK\log^{3}m}}{m}+\frac{\sigma\sqrt{mK\log^{2}m}}{m}\right) \label{eq:gaussian-hypothesis-incoherence2}
\end{align}
\end{subequations}
for some constants $C_{11}, C_{12}, C_{13}>0$. 
Additionally, to complete the induction argument for the base case, we are in need of the following
results of initialization: \begin{subequations}\label{sec:gaussian-hypotheses-ncvx-initialization}
\begin{align}
\text{dist}\left(\bm{z}^{0},\bm{z}^{\star}\right) & \lesssim\frac{\sqrt{mK\log^{2}m}}{m}+\frac{\sigma\sqrt{mK\log m}}{m},\label{eq:gaussian-hypothesis-dist-initialization}\\
\mathsf{dist}\big(\boldsymbol{z}^{0,\left(l\right)},\widetilde{\bm{z}}^{0}\big) & \leq C_{13}\left(\frac{\sqrt{K\log^{3}m}}{m}+\frac{\sigma\sqrt{K\log^{2}m}}{m}\right),\label{eq:gaussian-hypothesis-loo-1}\\
\max_{1\leq l\leq m}\left|\boldsymbol{a}_{l}^{\mathsf{H}}\left(\widetilde{\bm{x}}^{0}-\boldsymbol{x}^{\star}\right)\right| & \leq C_{12}\left(\frac{\sqrt{mK\log^{3}m}}{m}+\frac{\sigma\sqrt{mK\log^{2}m}}{m}\right),\label{eq:gaussian-hypothesis-incoherence1-1}\\
\max_{1\leq l\leq m}\big|\boldsymbol{b}_{l}^{\mathsf{H}}\big(\widetilde{\bm{h}}^{0}-\bm{h}^{\star}\big)\big| & \leq C_{13}\left(\frac{\sqrt{mK\log^{3}m}}{m}+\frac{\sigma\sqrt{mK\log^{2}m}}{m}\right).\label{eq:gaussian-hypothesis-incoherence2-1}
\end{align}
\end{subequations}

\subsection{Induction analysis}

Before embarking on the analysis, we state below a useful lemma which is direct
consequence of the hypotheses (\ref{sec:gaussian-hypotheses-ncvx})
and (\ref{sec:gaussian-hypotheses-ncvx-initialization}).

\begin{lemma}\label{lem:gaussian-consequence}Instate the notation
and assumptions in Theorem \ref{theorem:gaussian-cvx}. For $t\geq0$,
suppose that the hypotheses (\ref{sec:gaussian-hypotheses-ncvx})
and (\ref{subeq:Additional-induction-hypotheses}) hold in the first
$t$ iterations. Then there exist some constants $C,C'>0$ such that
for any $1\le l\leq m$, \begin{subequations}
\begin{align}
\mathsf{dist}\left(\boldsymbol{z}^{t},\boldsymbol{z}^{\star}\right) & \leq C\left(\frac{\sqrt{mK\log^{2}m}}{m}+\frac{\lambda+\sigma\sqrt{mK\log m}}{m}\right),\label{eq:gaussian-dist-bound}\\
\left\Vert \bm{h}^{t}\big(\bm{x}^{t}\big)^{\mathsf{H}}-\boldsymbol{h}^{\star}\boldsymbol{x}^{\mathsf{\star H}}\right\Vert  & \leq C'\left(\frac{\sqrt{mK\log^{2}m}}{m}+\frac{\lambda+\sigma\sqrt{mK\log m}}{m}\right),\label{eq:gaussian-hx-normbound}\\
\big\|\widetilde{\boldsymbol{z}}^{t,\left(l\right)}-\boldsymbol{z}^{\star}\big\|_{2} & \leq2C\left(\frac{\sqrt{mK\log^{2}m}}{m}+\frac{\lambda+\sigma\sqrt{mK\log m}}{m}\right),\label{eq:gaussian-conseq}\\
\frac{1}{2}\leq\left\Vert \widetilde{\bm{x}}^{t}\right\Vert _{2}\leq\frac{3}{2}, & \qquad\frac{1}{2}\le\big\|\widetilde{\bm{h}}^{t}\big\|_{2}\leq\frac{3}{2},\label{eq:gaussian-tilde-hx}\\
\frac{1}{2}\le\big\|\widetilde{\bm{x}}^{t,\left(l\right)}\big\|_{2}\leq\frac{3}{2}, & \qquad\frac{1}{2}\le\big\|\widetilde{\bm{h}}^{t,\left(l\right)}\big\|_{2}\leq\frac{3}{2},\label{eq:gaussian-tilde-hx-loo}\\
\frac{1}{2}\le\big\|\widehat{\bm{x}}^{t,\left(l\right)}\big\|_{2}\leq\frac{3}{2}, & \qquad\frac{1}{2}\le\big\|\widehat{\bm{h}}^{t,\left(l\right)}\big\|_{2}\leq\frac{3}{2}.\label{eq:gaussian-hat-hx-loo}
\end{align}
In addition, if $t>0$, then one also has
\begin{align}
\big\|\widehat{\bm{z}}^{t-1/2}-\boldsymbol{z}^{\star}\big\|_{2} & \leq C\left(\frac{\sqrt{mK\log^{2}m}}{m}+\frac{\lambda+\sigma\sqrt{K\log m}}{m}\right).\label{eq:gaussian-conseq-2}
\end{align}
\end{subequations}\end{lemma}This lemma can be proved in the same manner 
as Lemma \ref{lem:consequence} and hence we omit the proof here for brevity. 

\subsubsection{Characterizing local geometry}

Similar to the nonconvex analysis of blind deconvolution, our first
step is to establish some kind of restricted strong convexity and
smoothness as described in the following lemma. The proof can be found in
 Appendix \ref{subsec:Proof-of-Lemma-gaussian-geometry}. 

\begin{lemma}\label{lemma:gaussian-geometry}Let $\delta:=c/\log^{2}m$
for some sufficiently small constant $c>0$. Suppose that $m\geq CK\log^{6}m$
for some sufficiently large constant $C>0$ and that $\sigma\sqrt{K\log^{3}m/m}\leq c_{1}$
for some sufficiently small constant $c_{1}>0$. Then with probability
$1-O\left(m^{-10}+e^{-K}\log m\right)$, one has
\begin{align*}
\boldsymbol{u}^{\mathsf{H}}\left[\bm{D}\nabla^{2}f\left(\boldsymbol{z}\right)+\nabla^{2}f\left(\boldsymbol{z}\right)\bm{D}\right]\boldsymbol{u} & \geq\frac{m}{4}\left\Vert \boldsymbol{u}\right\Vert _{2}^{2}\quad\text{and}\\
\left\Vert \nabla^{2}f\left(\boldsymbol{z}\right)\right\Vert  & \leq3m
\end{align*}
simultaneously for all points
\[
\boldsymbol{z}=\left[\begin{array}{c}
\boldsymbol{h}\\
\boldsymbol{x}
\end{array}\right],\quad\boldsymbol{u}=\left[\begin{array}{c}
\boldsymbol{h}_{1}-\boldsymbol{h}_{2}\\
\boldsymbol{x}_{1}-\boldsymbol{x}_{2}\\
\overline{\boldsymbol{h}_{1}-\boldsymbol{h}_{2}}\\
\overline{\boldsymbol{x}_{1}-\boldsymbol{x}_{2}}
\end{array}\right]\quad\text{and}\quad\bm{D}=\left[\begin{array}{cccc}
\gamma_{1}\bm{I}_{K}\\
 & \gamma_{2}\bm{I}_{K}\\
 &  & \gamma_{1}\bm{I}_{K}\\
 &  &  & \gamma_{2}\bm{I}_{K}
\end{array}\right]
\]
obeying the following properties:
\begin{itemize}
\item $\boldsymbol{z}$ satisfies \begin{subequations}\label{subeq:gaussian-assumptions-geometry}
\begin{align}
\max\left\{ \left\Vert \boldsymbol{h}-\boldsymbol{h}^{\star}\right\Vert _{2},\left\Vert \boldsymbol{x}-\boldsymbol{x}^{\star}\right\Vert _{2}\right\}  & \leq\delta,\\
\max_{1\leq j\leq m}\left\{ \left|\boldsymbol{a}_{j}^{\mathsf{H}}\left(\boldsymbol{x}-\boldsymbol{x}^{\star}\right)\right|,\left|\boldsymbol{b}_{j}^{\mathsf{H}}\left(\boldsymbol{h}-\boldsymbol{h}^{\star}\right)\right|\right\}  & \leq C_{13}\tfrac{1}{\log^{3/2}m},
\end{align}
\item $\bm{z}_{1}:=\left(\boldsymbol{h}_{1},\boldsymbol{x}_{1}\right)$
is aligned with $\bm{z}_{2}:=\left(\boldsymbol{h}_{2},\boldsymbol{x}_{2}\right)$
in the sense that $\|\bm{z}_{1}-\bm{z}_{2}\|_{2}=\mathsf{dist}(\bm{z}_{1},\bm{z}_{2})$;
in addition, they satisfy 
\[
\max\left\{ \left\Vert \boldsymbol{h}_{1}-\boldsymbol{h}^{\star}\right\Vert _{2},\left\Vert \boldsymbol{h}_{2}-\boldsymbol{h}^{\star}\right\Vert _{2},\left\Vert \boldsymbol{x}_{1}-\boldsymbol{x}^{\star}\right\Vert _{2},\left\Vert \boldsymbol{x}_{2}-\boldsymbol{x}^{\star}\right\Vert _{2}\right\} \leq\delta;
\]
\item $\gamma_{1},\gamma_{2}\in\mathbb{R}$ and obey
\[
\max\left\{ \left|\gamma_{1}-1\right|,\left|\gamma_{2}-1\right|\right\} \leq\delta.
\]
\end{subequations}
\end{itemize}
\end{lemma}

\subsubsection{$\ell_{2}$ error contraction}

Next, by employing the established restricted strong convexity and
smoothness in Lemma \ref{lemma:gaussian-geometry}, we can prove the
hypothesis \eqref{eq:gaussian-hypothesis-dist} holds inductively. Our
result is this: \begin{lemma}Set $\lambda=C_{\lambda}\sigma\sqrt{mK\log m}$
for some sufficiently large constant $C_{\lambda}>0$ and the stepsize
$\eta=c_{\eta}/m$ for some sufficiently small constant $c_{\eta}>0$. Suppose the
sample complexity satisfies $m\geq CK\log^{3}m$ for some sufficiently
large constant $C>0$. Then if the hypotheses \eqref{sec:gaussian-hypotheses-ncvx} hold true at $t$th
iteration, we have for some constant $C_{11}>0$,
\[
\mathsf{dist}\left(\bm{z}^{t+1},\bm{z}^{\star}\right)\le\left(1-c_{\eta}/16\right)\mathsf{dist}\left(\bm{z}^{t},\bm{z}^{\star}\right)+C_{11}\eta\left(\lambda+\sigma\sqrt{mK\log m}\right),
\]
holds with probability exceeding $1-O(m^{-100})$.
\end{lemma}
\begin{proof}The proof is the same as the analysis for Lemma \ref{lemma:distance}
with the help of Lemma \ref{lemma:gaussian-geometry} and thus omitted
here for simplicity.  
\end{proof}

Before moving on to the next step, we provide a corollary to guarantee that
the alignment parameters does not change much between adjacent iterates. 

\begin{corollary}\label{corollary:gaussian-alpha}Instate the notation
and assumptions in Theorem \ref{theorem:gaussian-cvx}. For an integer
$t>0$, suppose that the hypotheses (\ref{sec:hypotheses-ncvx}) and
(\ref{subeq:Additional-induction-hypotheses}) hold in the first $t-1$
iterations. Then there exists some constant $C>0$ such that for any
$1\le l\leq m$, one has\begin{subequations}
\begin{align}
\left|\left|\alpha^{t}\right|-1\right| & \lesssim\mathsf{dist}\left(\widetilde{\bm{z}}^{t},\bm{z}^{\star}\right)\lesssim\frac{\sqrt{mK\log^{2}m}}{m}+\frac{\sigma\sqrt{mK\log m}}{m},\label{eq:alpha-asymp1-2}\\
\left|\frac{\alpha^{t-1/2}}{\alpha^{t-1}}-1\right| & \lesssim c_{\eta}\left(\frac{\sqrt{mK\log^{2}m}}{m}+\frac{\sigma\sqrt{mK\log m}}{m}\right),\label{eq:alpharatio-round1-2}\\
\left|\left|\alpha_{\mathrm{mutual}}^{t,\left(l\right)}\right|-1\right| & \lesssim\big\|\widehat{\bm{z}}^{t,\left(l\right)}-\bm{z}^{\star}\big\|_{2}\lesssim\frac{\sqrt{mK\log^{2}m}}{m}+\frac{\sigma\sqrt{mK\log m}}{m},\label{eq:alphaloo-asymp1-1}\\
\frac{1}{2}\leq\left\Vert \bm{x}^{t}\right\Vert _{2}\leq\frac{3}{2}, & \qquad\frac{1}{2}\leq\left\Vert \bm{h}^{t}\right\Vert _{2}\leq\frac{3}{2},\label{eq:hx-asymp1-1}\\
\frac{1}{2}\leq\big\|\bm{x}^{t,\left(l\right)}\big\|_{2}\leq\frac{3}{2}, & \qquad\frac{1}{2}\leq\big\|\bm{h}^{t,\left(l\right)}\big\|_{2}\leq\frac{3}{2}\label{eq:hxloo-asymp1-1}
\end{align}
\end{subequations}with probability at least $1-O\left(m^{-100}+e^{-CK}\log m\right)$.
\end{corollary}This corollary can be proved in the same way as Corollary
\ref{corollary:alpha} and hence we omit it here for simplicity. 

\subsubsection{Leave-one-out proximity}

The next step is to control the discrepancy between the leave-one-out
sequence and the original sequence. The formal statement is given
in the lemma below. 

\begin{lemma}\label{lemma:gaussian-loo}
Suppose the sample size obeys
$m\geq CK\log^{3}m$ for some large enough constant $C>0$. If the
hypotheses (\ref{sec:gaussian-hypotheses-ncvx}) hold true for the
$t$th iteration, then with probability exceeding $1-O(m^{-10})$,
we have
\begin{align}
\max_{1\leq l\leq m}\mathsf{dist}\left(\bm{z}^{t+1,\left(l\right)},\widetilde{\bm{z}}^{t+1}\right) & \le C_{12}\left(\frac{\sqrt{K\log^{3}m}}{m}+\frac{\sigma\sqrt{K}\log m}{m}\right),\label{eq:gaussian-loo-1}\\
\max_{1\leq l\leq m}\left\Vert \widetilde{\bm{z}}^{t+1,\left(l\right)}-\widetilde{\bm{z}}^{t+1}\right\Vert _{2} & \lesssim C_{12}\left(\frac{\sqrt{K\log^{3}m}}{m}+\frac{\sigma\sqrt{K}\log m}{m}\right).\label{eq:gaussian-loo-2}
\end{align}

\end{lemma}The proof can be found in Appendix \ref{subsec:Proof-of-Lemma-gaussian-loo}. 

\subsubsection{Establishing incoherence}

Then we proceed to prove the incoherence hypotheses, i.e. (\ref{eq:gaussian-hypothesis-incoherence1})
and (\ref{eq:gaussian-hypothesis-incoherence2}). They are much easier
to handle than the Fourier designs. We actually only need
to prove the incoherence of $\bm{a}_{l}$ and $\bm{x}^{t+1}$. Then
the other follows immediately by the symmetry between $\{\bm{a}_{j}\}_{j=1}^{m}$
and $\{\bm{b}_{j}\}_{j=1}^{m}$ under Assumption \ref{assumptions:models-gausssian}. Similar
to (\ref{eq:incoherencea-proof}), the triangle inequality and Cauchy-Schwarz
inequality yield

\begin{align}
\left|\boldsymbol{a}_{l}^{\mathsf{H}}\big(\widetilde{\boldsymbol{x}}^{t+1}-\boldsymbol{x}^{\star}\big)\right| & \leq\left|\boldsymbol{a}_{l}^{\mathsf{H}}\big(\widetilde{\boldsymbol{x}}^{t+1}-\widetilde{\boldsymbol{x}}^{t+1,\left(l\right)}\big)\right|+\left|\boldsymbol{a}_{l}^{\mathsf{H}}\big(\widetilde{\boldsymbol{x}}^{t+1,\left(l\right)}-\boldsymbol{x}^{\star}\big)\right|\nonumber \\
 & \leq\left\Vert \boldsymbol{a}_{l}\right\Vert _{2}\big\|\widetilde{\boldsymbol{x}}^{t+1}-\widetilde{\boldsymbol{x}}^{t+1,\left(l\right)}\big\|_{2}+\left|\boldsymbol{a}_{l}^{\mathsf{H}}\big(\widetilde{\boldsymbol{x}}^{t+1,\left(l\right)}-\boldsymbol{x}^{\star}\big)\right|\nonumber \\
 & \leq10\sqrt{K}\cdot C_{12}\left(\frac{\sqrt{K\log^{3}m}}{m}+\frac{\sigma\sqrt{K\log^{2}m}}{m}\right)\nonumber \\
 & \quad\quad+20\sqrt{\log m}\cdot2C_{11}\left(\frac{\sqrt{mK\log^{2}m}}{m}+\frac{\sigma\sqrt{mK\log m}}{m}\right)\nonumber \\
 & \leq C_{13}\left(\frac{\sqrt{mK\log^{3}m}}{m}+\frac{\sigma\sqrt{mK\log^{2}m}}{m}\right),\label{eq:incoherencea-proof-1}
\end{align}
where the penultimate inequality follows from (\ref{eq:useful1}),
(\ref{eq:useful2}) and (\ref{eq:gaussian-loo-2}). This establishes
the hypothesis (\ref{eq:gaussian-hypothesis-incoherence1}) for the
$(t+1)$-th iteration. 

The incoherence of $\bm{b}_{l}$ and $\bm{h}^{t+1}$ (as stated in
the hypothesis (\ref{eq:gaussian-hypothesis-incoherence2})) follows
from the symmetry between $\{\bm{a}_{j}\}_{j=1}^{m}$ and $\{\bm{b}_{j}\}_{j=1}^{m}$.
We summarize the results in the following lemma. 

\begin{lemma}\label{lemma:gaussian-incoherence}Suppose the sample
complexity obeys $m\geq CK\log m$ for some sufficiently large constant
$C>0$ and $\lambda=C_{\lambda}\sigma\sqrt{mK\log m}$ for some absolute
constant $C_{\lambda}>0$. If the hypotheses (\ref{eq:gaussian-hypothesis-dist})-(\ref{eq:gaussian-hypothesis-incoherence2})
hold for the $t$th iteration, then with probability exceeding $1-O\left(m^{-100}\right)$
for some constant $C_{13}>0$, one has 
\begin{align*}
\max_{1\leq l\leq m}\big|\boldsymbol{a}_{l}^{\mathsf{H}}\left(\widetilde{\boldsymbol{x}}^{t+1}-\bm{x}^{\star}\right)\big| & \le C_{13}\left(\frac{\sqrt{mK\log^{3}m}}{m}+\frac{\sigma\sqrt{mK\log^{2}m}}{m}\right),\\
\max_{1\leq l\leq m}\big|\boldsymbol{b}_{l}^{\mathsf{H}}\left(\widetilde{\boldsymbol{h}}^{t+1}-\bm{h}^{\star}\right)\big| & \le C_{13}\left(\frac{\sqrt{mK\log^{3}m}}{m}+\frac{\sigma\sqrt{mK\log^{2}m}}{m}\right),
\end{align*}
as long as $C_{13}>0$ is some sufficiently large constant and $\eta>0$
is taken to be some sufficiently small constant.\end{lemma}

\subsubsection{The base case: Spectral initialization}

The last step of the proof is to establish the induction hypotheses
for the base case. The following three lemmas justify (\ref{eq:gaussian-hypothesis-dist-initialization})-(\ref{eq:gaussian-hypothesis-incoherence2-1})
respectively. 

\begin{lemma}\label{lemma:gaussian-ncvx-initialization}Suppose the
sample size satisfies $m\geq CK\log^{5}m$ for some large enough constant
$C>0$. Then with probability exceeding $1-O(m^{-10})$, one has
\begin{align*}
\mathsf{dist}\left(\bm{z}^{0},\bm{z}^{\star}\right) & \lesssim\sqrt{\frac{K\log^2 m}{m}}+\sigma \sqrt{\frac{K\log m}{m}},\\
\mathsf{dist}\left(\bm{z}^{0,\left(l\right)},\bm{z}^{\star}\right) & \lesssim\sqrt{\frac{K\log^2 m}{m}}+\sigma \sqrt{\frac{K\log m}{m}},\qquad1\leq l\leq m,
\end{align*}
and $\left|\left|\alpha_{0}\right|-1\right|\leq1/4$.

\end{lemma}

\begin{proof}With the aid of Lemma \ref{lemma:gaussian-M}, the proof
is essentially identical to \citet[Eqn (94)]{ma2017implicit} and thus
omitted here for brevity. 
\end{proof}

\begin{lemma}\label{lemma:gaussian-loo-initialization}Suppose $m\geq CK\log^{5}m$
for some sufficiently large constant $C_{12}>0$. Then with probability
at least $1-O(m^{-1})$, one has
\[
\max_{1\leq l\leq m}\mathsf{dist}\left(\bm{z}^{0,\left(l\right)},\widetilde{\bm{z}}^{0}\right)\leq\frac{C_{12}\sqrt{K\log^{3}m}}{m}.
\]

	\end{lemma} \begin{proof} The proof of this lemma is deferred to Appendix \ref{subsec:Proof-of-Lemma-gaussian-loo-initialization}. \end{proof}

\begin{lemma}Suppose that $m\geq CK\log^{6}m$ for some large enough
constant $C>0$. Then with probability at least $1-O(m^{-1})$, we
have
\begin{align*}
\max_{1\leq j\leq m}\left|\bm{a}_{j}^{\mathsf{H}}\left(\widetilde{\bm{x}}^{0}-\bm{x}^{\star}\right)\right| & \leq C_{13}\left(\sqrt{\frac{K\log^3 m}{m}}+\sigma\sqrt{\frac{K\log^2m}{m}}\right),\\
\max_{1\leq j\leq m}\left|\bm{b}_{j}^{\mathsf{H}}\left(\widetilde{\bm{h}}^{0}-\bm{h}^{\star}\right)\right| & \leq C_{13}\left(\sqrt{\frac{K\log^3 m}{m}}+\sigma\sqrt{\frac{K\log^2m}{m}}\right).
\end{align*}

\end{lemma}\begin{proof}The first inequality can be established
by the same derivation as for \citet[Lemma 21]{ma2017implicit}, which is omitted
here for brevity. The second inequality follows immediately since $\{\bm{a}_{j}\}_{j=1}^{m}$
and $\{\bm{b}_{j}\}_{j=1}^{m}$ have the same distributions. \end{proof}

\subsection{Proof of Lemma \ref{lemma:gaussian-geometry}\label{subsec:Proof-of-Lemma-gaussian-geometry}}

To begin with, we decompose $\nabla^{2}f(\bm{z})$ as follows
\[
	\nabla^{2}f\left(\bm{z}\right)=\lambda\bm{I}_{4K}+\mathbb{E}\left[\nabla^{2}f_{\mathsf{reg}\text{-}\mathsf{free}}\left(\bm{z}^{\star}\right)\right]+\left(\nabla^{2}f\left(\bm{z}\right)-\mathbb{E}\left[\nabla^{2}f_{\mathsf{reg}\text{-}\mathsf{free}}\left(\bm{z}^{\star}\right)\right]-\lambda\bm{I}_{4K}\right),
\]
where 
\[
f_{\mathsf{reg}\text{-}\mathsf{free}}\left(\bm{z}\right)=\sum_{j=1}^{m}\left|\boldsymbol{b}_{j}^{\mathsf{H}}\bm{h}\bm{x}^{\mathsf{H}}\boldsymbol{a}_{j}-y_{j}\right|^{2}.
\]
The following two lemmas allow us to control the two terms on the right-hand side of the above identity separately. 
\begin{lemma}\label{lemma:hessian-population}Instate the notation
and conditions of Lemma \ref{lemma:gaussian-geometry}. One has
\[
\left\Vert \mathbb{E}\left[\nabla^{2}f_{\mathsf{reg}\text{-}\mathsf{free}}\left(\bm{z}^{\star}\right)\right]\right\Vert =2m\qquad\text{and}\qquad\bm{u}^{\mathsf{H}}\left[\bm{D}\mathbb{E}\left[\nabla^{2}f_{\mathsf{reg}\text{-}\mathsf{free}}\left(\bm{z}^{\star}\right)\right]+\mathbb{E}\left[\nabla^{2}f_{\mathsf{reg}\text{-}\mathsf{free}}\left(\bm{z}^{\star}\right)\right]\bm{D}\right]\bm{u}\geq m\left\Vert \bm{u}\right\Vert _{2}^{2}.
\]
\end{lemma} 

\begin{proof}Note that the expression of $\mathbb{E}[\nabla^{2}f_{\mathsf{reg}\text{-}\mathsf{free}}(\bm{z}^{\star})]/m$
is the same as that of $\nabla^{2}F(\bm{z}^{\star})$ in \citet[Lemma 26]{ma2017implicit}.
Hence the proof there can be straightforwardly adapted to our case and thus omitted
here.
\end{proof}

\begin{lemma}\label{lemma:hessian-deviation}Suppose the sample size
obeys $m\geq CK\log^{3}m$ for some large enough constant $C>0$.
Then with probability at least $1-O(m^{-10})$, one has 
\[
\left\Vert \nabla^{2}f\left(\bm{z}\right)-\mathbb{E}\left[\nabla^{2}f\left(\bm{z}^{\star}\right)\right]\right\Vert \leq\frac{1}{4}m 
\]
holds uniformly for all $\bm{z}$ satisfying \eqref{subeq:gaussian-assumptions-geometry}. 
\end{lemma}
\begin{proof}See Appendix \ref{subsec:Proof-of-Lemma-hessian-deviation}.
\end{proof}
With these two lemmas in hand, we have, for any $(\bm{h},\bm{x})$
obeying (\ref{subeq:gaussian-assumptions-geometry}), that
\begin{align*}
\left\Vert \nabla^{2}f\left(\bm{z}\right)\right\Vert  & \leq\left\Vert \mathbb{E}\left[\nabla^{2}f\left(\bm{z}^{\star}\right)\right]\right\Vert +\left\Vert \nabla^{2}f\left(\bm{z}\right)-\mathbb{E}\left[\nabla^{2}f\left(\bm{z}^{\star}\right)\right]\right\Vert \\
 & \leq\left\Vert \mathbb{E}\left[\nabla^{2}f_{\mathsf{reg}\text{-}\mathsf{free}}\left(\bm{z}^{\star}\right)\right]\right\Vert +\lambda+\left\Vert \nabla^{2}f\left(\bm{z}\right)-\mathbb{E}\left[\nabla^{2}f\left(\bm{z}^{\star}\right)\right]\right\Vert \\
 & \leq2m+\lambda+\frac{1}{4}m\\
 & \leq3m.
\end{align*}
Furthermore, it is readily seen that
\begin{align*}
 & \boldsymbol{u}^{\mathsf{H}}\left[\bm{D}\nabla^{2}f\left(\boldsymbol{z}\right)+\nabla^{2}f\left(\boldsymbol{z}\right)\bm{D}\right]\boldsymbol{u}\\
&\quad= \boldsymbol{u}^{\mathsf{H}}\left\{ \bm{D}\mathbb{E}\left[\nabla^{2}f_{\mathsf{reg}\text{-}\mathsf{free}}\left(\bm{z}^{\star}\right)\right]+\mathbb{E}\left[\nabla^{2}f_{\mathsf{reg}\text{-}\mathsf{free}}\left(\bm{z}^{\star}\right)\right]\bm{D}\right\} \boldsymbol{u}+2\lambda\boldsymbol{u}^{\mathsf{H}}\bm{D}\boldsymbol{u}\\
 & \quad\qquad+\boldsymbol{u}^{\mathsf{H}}\bm{D}\left\{ \nabla^{2}f\left(\boldsymbol{z}\right)-\mathbb{E}\left[\nabla^{2}f\left(\bm{z}^{\star}\right)\right]\right\} \boldsymbol{u}+\boldsymbol{u}^{\mathsf{H}}\left\{ \nabla^{2}f\left(\boldsymbol{z}\right)-\mathbb{E}\left[\nabla^{2}f\left(\bm{z}^{\star}\right)\right]\right\} \bm{D}\boldsymbol{u}\\
&\quad\overset{(\text{i})}{\geq}  m\left\Vert \boldsymbol{u}\right\Vert _{2}^{2}+2\lambda\left(1-\delta\right)\left\Vert \boldsymbol{u}\right\Vert _{2}^{2}-2\left\Vert \bm{D}\right\Vert \left\Vert \nabla^{2}f\left(\boldsymbol{z}\right)-\mathbb{E}\left[\nabla^{2}f\left(\bm{z}^{\star}\right)\right]\right\Vert \left\Vert \boldsymbol{u}\right\Vert _{2}^{2}\\
&\quad\overset{(\text{ii})}{\geq} m\left\Vert \boldsymbol{u}\right\Vert _{2}^{2}+2\lambda\left(1-\delta\right)\left\Vert \boldsymbol{u}\right\Vert _{2}^{2}-2\left(1+\delta\right)\cdot\frac{1}{4}m\left\Vert \boldsymbol{u}\right\Vert _{2}^{2}\\
&\quad\overset{(\text{iii})}{\geq}  \frac{1}{4}m\left\Vert \boldsymbol{u}\right\Vert _{2}^{2},
\end{align*}
where (i) is due to Lemma \ref{lemma:hessian-population} and the
fact that $\boldsymbol{u}^{\mathsf{H}}\bm{D}\boldsymbol{u}\geq(1-\delta)\Vert\bm{u}\Vert_{2}^{2}$;
(ii) relies on the bound $\Vert\bm{D}\Vert\leq1+\delta$ and Lemma \ref{lemma:hessian-deviation};
and (iii) holds as long as $\delta\leq 1/4$. We have thus finished the proof
for the desired  smoothness and restricted strong convexity conditions. 

\subsubsection{Proof of Lemma \ref{lemma:hessian-deviation}\label{subsec:Proof-of-Lemma-hessian-deviation}}

The idea of the proof is similar to that of \citet[Lemma 27]{ma2017implicit}
except that the design of $\{\bm{b}_{j}\}_{j=1}^{m}$ is different.
By triangle inequality, we can upper bound the quantity of interest
as

\begin{equation}
\left\Vert \nabla^{2}f\left(\bm{z}\right)-\mathbb{E}\left[\nabla^{2}f\left(\bm{z}^{\star}\right)\right]\right\Vert \leq2\alpha_{1}+2\alpha_{2}+4\alpha_{3}+4\alpha_{4},\label{eq:gaussian-alphaall}
\end{equation}
where 
\begin{align*}
\alpha_{1} & =\left\Vert \sum_{j=1}^{m}\left|\bm{a}_{j}^{\mathsf{H}}\bm{x}\right|^{2}\bm{b}_{j}\bm{b}_{j}^{\mathsf{H}}-m\bm{I}_{K}\right\Vert ,\qquad\alpha_{2}=\left\Vert \sum_{j=1}^{m}\left|\bm{b}_{j}^{\mathsf{H}}\bm{h}\right|^{2}\bm{a}_{j}\bm{a}_{j}^{\mathsf{H}}-m\bm{I}_{K}\right\Vert ,\\
\alpha_{3} & =\left\Vert \sum_{j=1}^{m}\left(\bm{b}_{j}^{\mathsf{H}}\bm{h}\bm{x}^{\mathsf{H}}\bm{a}_{j}-y_{j}\right)\bm{b}_{j}\bm{a}_{j}^{\mathsf{H}}\right\Vert ,\qquad\alpha_{4}=\left\Vert \sum_{j=1}^{m}\bm{b}_{j}\bm{b}_{j}^{\mathsf{H}}\bm{h}\left(\bm{a}_{j}\bm{a}_{j}^{\mathsf{H}}\bm{x}\right)^{\mathsf{H}}-m\bm{h}^{\star}\bm{x}^{\star\mathsf{H}}\right\Vert .
\end{align*}
We will control these four terms separately as follows.

\paragraph{Controlling $\alpha_1$.}

 In terms of $\alpha_{1}$, by the triangle inequality, one has 
\[
\alpha_{1}\leq\underbrace{\left\Vert \sum_{j=1}^{m}\left(\left|\bm{a}_{j}^{\mathsf{H}}\bm{x}\right|^{2}-\left|\bm{a}_{j}^{\mathsf{H}}\bm{x}^{\star}\right|^{2}\right)\bm{b}_{j}\bm{b}_{j}^{\mathsf{H}}\right\Vert }_{\eqqcolon\gamma_{1}}+\underbrace{\left\Vert \sum_{j=1}^{m}\left|\bm{a}_{j}^{\mathsf{H}}\bm{x}^{\star}\right|^{2}\bm{b}_{j}\bm{b}_{j}^{\mathsf{H}}-m\bm{I}_{K}\right\Vert }_{\eqqcolon\gamma_{2}}.
\]

\begin{enumerate}
\item Regarding $\gamma_{1}$, we have 
\begin{align}
\gamma_{1} & \leq\left\Vert \sum_{j=1}^{m}\left|\left|\bm{a}_{j}^{\mathsf{H}}\bm{x}\right|^{2}-\left|\bm{a}_{j}^{\mathsf{H}}\bm{x}^{\star}\right|^{2}\right|\bm{b}_{j}\bm{b}_{j}^{\mathsf{H}}\right\Vert \nonumber \\
 & \leq\max_{1\leq j\leq m}\left|\left|\bm{a}_{j}^{\mathsf{H}}\bm{x}\right|^{2}-\left|\bm{a}_{j}^{\mathsf{H}}\bm{x}^{\star}\right|^{2}\right|\left\Vert \sum_{j=1}^{m}\bm{b}_{j}\bm{b}_{j}^{\mathsf{H}}\right\Vert \nonumber \\
 & \leq\max_{1\leq j\leq m}\left(\left|\bm{a}_{j}^{\mathsf{H}}\left(\bm{x}-\bm{x}^{\star}\right)\right|^{2}+2\left|\bm{a}_{j}^{\mathsf{H}}\left(\bm{x}-\bm{x}^{\star}\right)\right|\left|\bm{a}_{j}^{\mathsf{H}}\bm{x}^{\star}\right|\right)\cdot\left\Vert \sum_{j=1}^{m}\bm{b}_{j}\bm{b}_{j}^{\mathsf{H}}\right\Vert .\label{eq:gaussian-gamma1-decompose}
\end{align}
It is first seen that
\begin{align*}
 & \max_{1\leq j\leq m}\left(\left|\bm{a}_{j}^{\mathsf{H}}\left(\bm{x}-\bm{x}^{\star}\right)\right|^{2}+2\left|\bm{a}_{j}^{\mathsf{H}}\left(\bm{x}-\bm{x}^{\star}\right)\right|\left|\bm{a}_{j}^{\mathsf{H}}\bm{x}^{\star}\right|\right)\\
&\quad\leq  \left(C_{13}\tfrac{1}{\log^{3/2}m}\right)^{2}+2\cdot C_{13}\tfrac{1}{\log^{3/2}m}\cdot20\sqrt{\log m}\\
&\quad \lesssim  C_{13}\frac{1}{\log m}.
\end{align*}
When it comes to $\Vert\sum_{j=1}^{m}\bm{b}_{j}\bm{b}_{j}^{\mathsf{H}}\Vert$,
one has 
\begin{equation}
\left\Vert \sum_{j=1}^{m}\bm{b}_{j}\bm{b}_{j}^{\mathsf{H}}\right\Vert \leq\left\Vert \sum_{j=1}^{m}\left(\bm{b}_{j}\bm{b}_{j}^{\mathsf{H}}-\bm{I}_{K}\right)\right\Vert +m.\label{eq:gaussian-gamma1}
\end{equation}
We intend to invoke the matrix Bernstein inequality \citet[Proposition 2]{koltchinskii2011nuclear}
to control $\Vert\sum_{j=1}^{m}(\bm{b}_{j}\bm{b}_{j}^{\mathsf{H}}-\bm{I}_{K})\Vert$.
Observe that 
\[
B_{\bm{Z}}:=\Big\|\big\|\bm{b}_{j}\bm{b}_{j}^{\mathsf{H}}-\bm{I}_{K}\big\|\Big\|_{\psi_{1}}=\Big\|\max\left\{ \left|\left\Vert \bm{b}_{j}\right\Vert _{2}^{2}-1\right|,1\right\} \Big\|_{\psi_{1}}\leq\Big\|\|\boldsymbol{b}_{j}\|_{2}\Big\|_{\psi_{2}}^{2}+1\lesssim K..
\]
Here, we have used $\big\|\Vert\bm{b}_{j}\Vert_{2}\big\|_{\psi_{2}}\lesssim\sqrt{K}$
(cf.~\citet[Theorem 3.1.1]{vershynin2018high}). In addition, simple
calculation yields 
\begin{align*}
\left\Vert \sum_{j=1}^{m}\mathbb{\mathbb{E}}\big[\left(\bm{b}_{j}\bm{b}_{j}^{\mathsf{H}}-\bm{I}_{K}\right)\left(\bm{b}_{j}\bm{b}_{j}^{\mathsf{H}}-\bm{I}_{K}\right)^{\mathsf{H}}\big]\right\Vert  & =\left\Vert \sum_{j=1}^{m}\mathbb{\mathbb{E}}\Big[\bm{b}_{j}\bm{b}_{j}^{\mathsf{H}}\bm{b}_{j}\bm{b}_{j}^{\mathsf{H}}-\bm{I}_{K}\Big]\right\Vert =\left(K+1\right)m,
\end{align*}
and
\begin{align*}
\left\Vert \sum_{j=1}^{m}\mathbb{\mathbb{E}}\big[\left(\bm{b}_{j}\bm{b}_{j}^{\mathsf{H}}-\bm{I}_{K}\right)^{\mathsf{H}}\left(\bm{b}_{j}\bm{b}_{j}^{\mathsf{H}}-\bm{I}_{K}\right)\big]\right\Vert  & =\left\Vert \sum_{j=1}^{m}\mathbb{\mathbb{E}}\big[\left(\bm{b}_{j}\bm{b}_{j}^{\mathsf{H}}-\bm{I}_{K}\right)\left(\bm{b}_{j}\bm{b}_{j}^{\mathsf{H}}-\bm{I}_{K}\right)^{\mathsf{H}}\big]\right\Vert =\left(K+1\right)m.
\end{align*}
As a result, by setting
\begin{align*}
\sigma_{\bm{Z}} & :=\max\left\{ \left\Vert \sum_{j=1}^{m}\mathbb{\mathbb{E}}\big[\left(\bm{b}_{j}\bm{b}_{j}^{\mathsf{H}}-\bm{I}_{K}\right)\left(\bm{b}_{j}\bm{b}_{j}^{\mathsf{H}}-\bm{I}_{K}\right)^{\mathsf{H}}\big]\right\Vert ^{1/2},\left\Vert \sum_{j=1}^{m}\mathbb{\mathbb{E}}\big[\left(\bm{b}_{j}\bm{b}_{j}^{\mathsf{H}}-\bm{I}_{K}\right)^{\mathsf{H}}\left(\bm{b}_{j}\bm{b}_{j}^{\mathsf{H}}-\bm{I}_{K}\right)\big]\right\Vert ^{1/2}\right\} \\
& =\sqrt{\left(K+1\right)m},
\end{align*}
we are ready to apply the matrix Bernstein inequality \citet[Proposition 2]{koltchinskii2011nuclear}
to derive
\begin{equation}
\left\Vert \sum_{j=1}^{m}\left(\bm{b}_{j}\bm{b}_{j}^{\mathsf{H}}-\bm{I}_{K}\right)\right\Vert \lesssim\sigma_{\bm{Z}}\sqrt{\log m}+B_{\bm{Z}}\log\left(\frac{B_{\bm{Z}}\sqrt{m}}{\sigma_{\bm{Z}}}\right)\log m\lesssim\sqrt{mK\log m}\label{eq:gaussian-lemmanoise1}
\end{equation}
with high probability. 
Substitution of (\ref{eq:gaussian-lemmanoise1})  into (\ref{eq:gaussian-gamma1})
yields
\begin{equation}
\left\Vert \sum_{j=1}^{m}\bm{b}_{j}\bm{b}_{j}^{\mathsf{H}}\right\Vert \leq2m,\label{eq:gaussian-sumbbH}
\end{equation}
as long as $m\gg K\log m$. Plugging this inequality into (\ref{eq:gaussian-gamma1-decompose})
gives 
\begin{equation}
\gamma_{1}\lesssim C_{13}\frac{m}{\log m}.\label{eq:gaussian-alpha1-1}
\end{equation}
\item The second term $\gamma_{2}$ can be further decomposed as follows
\begin{equation}
\gamma_{2}\leq\left\Vert \sum_{j=1}^{m}\left(\left|\bm{a}_{j}^{\mathsf{H}}\bm{x}^{\star}\right|^{2}-1\right)\bm{b}_{j}\bm{b}_{j}^{\mathsf{H}}\right\Vert +\left\Vert \sum_{j=1}^{m}\bm{b}_{j}\bm{b}_{j}^{\mathsf{H}}-m\bm{I}_{K}\right\Vert .\label{eq:gaussian-gamma2-decompose}
\end{equation}
The second term on the right-hand side of (\ref{eq:gaussian-gamma2-decompose})
has already been considered in (\ref{eq:gaussian-lemmanoise1}). We
are therefore left to control the first term. Let 
\[
\bm{W}_{j}\coloneqq\left(\left|\bm{a}_{j}^{\mathsf{H}}\bm{x}^{\star}\right|^{2}\ind_{\left\{ \left|\bm{a}_{j}^{\mathsf{H}}\bm{x}^{\star}\right|\leq20\sqrt{\log m}\right\} }-\mathbb{E}\left[\left|\bm{a}_{j}^{\mathsf{H}}\bm{x}^{\star}\right|^{2}\ind_{\left\{ \left|\bm{a}_{j}^{\mathsf{H}}\bm{x}^{\star}\right|\leq20\sqrt{\log m}\right\} }\right]\right)\bm{b}_{j}\bm{b}_{j}^{\mathsf{H}}.
\]
We make the observation that
\begin{align}
 & \left\Vert \sum_{j=1}^{m}\left(\left|\bm{a}_{j}^{\mathsf{H}}\bm{x}^{\star}\right|^{2}-1\right)\bm{b}_{j}\bm{b}_{j}^{\mathsf{H}}\right\Vert \nonumber \\
 & \quad\leq\left\Vert \sum_{j=1}^{m}\left(\left|\bm{a}_{j}^{\mathsf{H}}\bm{x}^{\star}\right|^{2}\ind_{\left\{ \left|\bm{a}_{j}^{\mathsf{H}}\bm{x}^{\star}\right|\leq20\sqrt{\log m}\right\} }-1\right)\bm{b}_{j}\bm{b}_{j}^{\mathsf{H}}\right\Vert +\left\Vert \sum_{j=1}^{m}\left|\bm{a}_{j}^{\mathsf{H}}\bm{x}^{\star}\right|^{2}\ind_{\left\{ \left|\bm{a}_{j}^{\mathsf{H}}\bm{x}^{\star}\right|>20\sqrt{\log m}\right\} }\bm{b}_{j}\bm{b}_{j}^{\mathsf{H}}\right\Vert . \label{eq:gaussian-gamma2-decompose2}
\end{align}
Regarding the second term of \eqref{eq:gaussian-gamma2-decompose2}, due to (\ref{eq:useful2}) we have 
\[
\left\Vert \sum_{j=1}^{m}\left|\bm{a}_{j}^{\mathsf{H}}\bm{x}^{\star}\right|^{2}\ind_{\left\{ \left|\bm{a}_{j}^{\mathsf{H}}\bm{x}^{\star}\right|>20\sqrt{\log m}\right\} }\bm{b}_{j}\bm{b}_{j}^{\mathsf{H}}\right\Vert =0
\]
holds with probability over $1-O(m^{-100})$. For the first term of \eqref{eq:gaussian-gamma2-decompose2},
one can derive
\begin{align*}
 & \left\Vert \sum_{j=1}^{m}\left(\left|\bm{a}_{j}^{\mathsf{H}}\bm{x}^{\star}\right|^{2}\ind_{\left\{ \left|\bm{a}_{j}^{\mathsf{H}}\bm{x}^{\star}\right|\leq20\sqrt{\log m}\right\} }-1\right)\bm{b}_{j}\bm{b}_{j}^{\mathsf{H}}\right\Vert \\
&\quad\leq  \left\Vert \sum_{j=1}^{m}\bm{W}_{j}\right\Vert +\left\Vert \sum_{j=1}^{m}\left(\mathbb{E}\left[\left|\bm{a}_{j}^{\mathsf{H}}\bm{x}^{\star}\right|^{2}\ind_{\left\{ \left|\bm{a}_{j}^{\mathsf{H}}\bm{x}^{\star}\right|\leq20\sqrt{\log m}\right\} }\right]-1\right)\bm{b}_{j}\bm{b}_{j}^{\mathsf{H}}\right\Vert \\
&\quad\leq  \left\Vert \sum_{j=1}^{m}\bm{W}_{j}\right\Vert +\max_{1\leq j\leq m}\left|\mathbb{E}\left[\left|\bm{a}_{j}^{\mathsf{H}}\bm{x}^{\star}\right|^{2}\ind_{\left\{ \left|\bm{a}_{j}^{\mathsf{H}}\bm{x}^{\star}\right|\leq20\sqrt{\log m}\right\} }\right]-1\right|\cdot\left\Vert \sum_{j=1}^{m}\bm{b}_{j}\bm{b}_{j}^{\mathsf{H}}\right\Vert \\
&\quad = \left\Vert \sum_{j=1}^{m}\bm{W}_{j}\right\Vert +\max_{1\leq j\leq m}\mathbb{E}\left[\left|\bm{a}_{j}^{\mathsf{H}}\bm{x}^{\star}\right|^{2}\ind_{\left\{ \left|\bm{a}_{j}^{\mathsf{H}}\bm{x}^{\star}\right|>20\sqrt{\log m}\right\} }\right]\cdot\left\Vert \sum_{j=1}^{m}\bm{b}_{j}\bm{b}_{j}^{\mathsf{H}}\right\Vert,
\end{align*}
where the first inequality holds due to the triangle inequality. 
Invoking the Cauchy-Schwartz inequality yields
\begin{align*}
\mathbb{E}\left[\left|\bm{a}_{j}^{\mathsf{H}}\bm{x}^{\star}\right|^{2}\ind_{\left\{ \left|\bm{a}_{j}^{\mathsf{H}}\bm{x}^{\star}\right|>20\sqrt{\log m}\right\} }\right] & \leq\sqrt{\mathbb{E}\left[\left|\bm{a}_{j}^{\mathsf{H}}\bm{x}^{\star}\right|^{4}\right]\cdot\mathbb{P}\left(\left|\bm{a}_{j}^{\mathsf{H}}\bm{x}^{\star}\right|>20\sqrt{\log m}\right)}\\
 & \leq O\left(m^{-100}\right),
\end{align*}
which taken collectively with (\ref{eq:gaussian-sumbbH}) gives 
\begin{equation}
\max_{1\leq j\leq m}\mathbb{E}\left[\left|\bm{a}_{j}^{\mathsf{H}}\bm{x}^{\star}\right|^{2}\ind_{\left\{ \left|\bm{a}_{j}^{\mathsf{H}}\bm{x}^{\star}\right|>20\sqrt{\log m}\right\} }\right]\cdot\left\Vert \sum_{j=1}^{m}\bm{b}_{j}\bm{b}_{j}^{\mathsf{H}}\right\Vert \leq O\left(m^{-100}\right)\cdot2m=O\left(m^{-98}\right).\label{eq:gaussian-lemma-gamma2}
\end{equation}
We can then invoke the matrix Bernstein inequality \citet[Proposition 2]{koltchinskii2011nuclear}
to control $\Vert\sum_{j=1}^{m}\bm{W}_{j}\Vert$. To this end, note that 
\[
B_{\bm{Z}}:=\Big\|\big\|\bm{W}_{j}\big\|\Big\|_{\psi_{1}}\leq\left(20\sqrt{\log m}\right)^{2}\cdot\Big\|\|\boldsymbol{b}_{j}\|_{2}\Big\|_{\psi_{2}}^{2}\lesssim K\log m,
\]
where we have used $\big\|\Vert\bm{b}_{j}\Vert_{2}\big\|_{\psi_{2}}\lesssim\sqrt{K}$
(cf.~\citet[Theorem 3.1.1]{vershynin2018high}). In addition, simple
calculation yields 
\begin{align*}
\left\Vert \sum_{j=1}^{m}\mathbb{\mathbb{E}}\big[\bm{W}_{j}\bm{W}_{j}^{\mathsf{H}}\big]\right\Vert  & =\left\Vert \sum_{j=1}^{m}\text{Var}\left(\left|\bm{a}_{j}^{\mathsf{H}}\bm{x}^{\star}\right|^{2}\ind_{\left\{ \left|\bm{a}_{j}^{\mathsf{H}}\bm{x}^{\star}\right|\leq20\sqrt{\log m}\right\} }\right)\mathbb{\mathbb{E}}\Big[\bm{b}_{j}\bm{b}_{j}^{\mathsf{H}}\bm{b}_{j}\bm{b}_{j}^{\mathsf{H}}\Big]\right\Vert \leq3\left(K+2\right)m,
\end{align*}
and
\begin{align*}
\left\Vert \sum_{j=1}^{m}\mathbb{\mathbb{E}}\big[\bm{W}_{j}^{\mathsf{H}}\bm{W}_{j}\big]\right\Vert  & =\left\Vert \sum_{j=1}^{m}\mathbb{\mathbb{E}}\big[\bm{W}_{j}\bm{W}_{j}^{\mathsf{H}}\big]\right\Vert \leq3\left(K+2\right)m.
\end{align*}
As a result, by setting
\[
\sigma_{\bm{Z}}:=\max\left\{ \left\Vert \sum_{j=1}^{m}\mathbb{\mathbb{E}}\big[\bm{W}_{j}\bm{W}_{j}^{\mathsf{H}}\big]\right\Vert ^{1/2},\left\Vert \sum_{j=1}^{m}\mathbb{\mathbb{E}}\big[\bm{W}_{j}^{\mathsf{H}}\bm{W}_{j}\big]\right\Vert ^{1/2}\right\} \leq\sqrt{3\left(K+2\right)m},
\]
we can apply the matrix Bernstein inequality \citet[Proposition 2]{koltchinskii2011nuclear}
to derive
\begin{equation}
\left\Vert \sum_{j=1}^{m}\bm{W}_{j}\right\Vert \lesssim\sigma_{\bm{Z}}\sqrt{\log m}+B_{\bm{Z}}\log\left(\frac{B_{\bm{Z}}\sqrt{m}}{\sigma_{\bm{Z}}}\right)\log m\lesssim\sqrt{mK\log m}\label{eq:gaussian-lemmanoise2}
\end{equation}
with high probability, where the last inequality holds as long as $m\gg K\log^{5}m$. Plugging
(\ref{eq:gaussian-lemma-gamma2}) and (\ref{eq:gaussian-lemmanoise2})
into (\ref{eq:gaussian-gamma2-decompose2}) gives
\begin{equation}
\left\Vert \sum_{j=1}^{m}\left(\left|\bm{a}_{j}^{\mathsf{H}}\bm{x}^{\star}\right|^{2}-1\right)\bm{b}_{j}\bm{b}_{j}^{\mathsf{H}}\right\Vert \lesssim\sqrt{mK\log m}.\label{eq:gaussian-gamma2-2}
\end{equation}
Substitution of (\ref{eq:gaussian-lemmanoise1}) and (\ref{eq:gaussian-gamma2-2})
into (\ref{eq:gaussian-gamma2-decompose}) yields
\begin{equation}
\gamma_{2}\lesssim\sqrt{mK\log m}.\label{eq:gaussian-alpha1-2}
\end{equation}
\end{enumerate}
As a consequence, taking (\ref{eq:gaussian-alpha1-1}) and (\ref{eq:gaussian-alpha1-2})
collectively yields
\begin{equation}
\alpha_{1}\lesssim\frac{m}{\log m}+\sqrt{mK\log m}.\label{eq:gaussian-alpha1}
\end{equation}

\paragraph{Controlling $\alpha_2$.}

Regarding $\alpha_{2}$, since the roles played by $\{\bm{a}_{j}\}_{j=1}^{m}$
and $\{\bm{b}_{j}\}_{j=1}^{m}$ are symmetric in this problem, it
is easily seen that $\alpha_{2}$ admits the same bound as that of
$\alpha_{1}$.

\paragraph{Controlling $\alpha_3$.}
When it comes to the third term $\alpha_3$, one makes the observation that
\begin{equation}
\alpha_{3}\leq\left\Vert \sum_{j=1}^{m}\bm{b}_{j}\bm{b}_{j}^{\mathsf{H}}\left(\bm{h}\bm{x}^{\mathsf{H}}-\bm{h}^{\star}\bm{x}^{\star\mathsf{H}}\right)\bm{a}_{j}\bm{a}_{j}^{\mathsf{H}}\right\Vert +\left\Vert \sum_{j=1}^{m}\xi_{j}\bm{b}_{j}\bm{a}_{j}^{\mathsf{H}}\right\Vert .\label{eq:gaussian-alpha3-decompose}
\end{equation}
The second term on the right-hand side of this relation has already been bounded by Lemma \ref{lemma:gaussian-noise}.
Regarding the first term on the right-hand side of \eqref{eq:gaussian-alpha3-decompose}, one can further decompose
\begin{align}
 & \left\Vert \sum_{j=1}^{m}\bm{b}_{j}\bm{b}_{j}^{\mathsf{H}}\left(\bm{h}\bm{x}^{\mathsf{H}}-\bm{h}^{\star}\bm{x}^{\star\mathsf{H}}\right)\bm{a}_{j}\bm{a}_{j}^{\mathsf{H}}\right\Vert \nonumber \\
 &\quad\leq\left\Vert m\left(\bm{h}\bm{x}^{\mathsf{H}}-\bm{h}^{\star}\bm{x}^{\star\mathsf{H}}\right)\right\Vert +\left\Vert \sum_{j=1}^{m}\bm{b}_{j}\bm{b}_{j}^{\mathsf{H}}\left(\bm{h}\bm{x}^{\mathsf{H}}-\bm{h}^{\star}\bm{x}^{\star\mathsf{H}}\right)\bm{a}_{j}\bm{a}_{j}^{\mathsf{H}}-m\left(\bm{h}\bm{x}^{\mathsf{H}}-\bm{h}^{\star}\bm{x}^{\star\mathsf{H}}\right)\right\Vert \nonumber \\
 &\quad\leq\left\Vert m\left(\bm{h}\bm{x}^{\mathsf{H}}-\bm{h}^{\star}\bm{x}^{\star\mathsf{H}}\right)\right\Vert +\left\Vert \sum_{j=1}^{m}\bm{b}_{j}\bm{b}_{j}^{\mathsf{H}}\bm{h}^{\star}\left(\bm{x}-\bm{x}^{\star}\right)^{\mathsf{H}}\bm{a}_{j}\bm{a}_{j}^{\mathsf{H}}-m\bm{h}^{\star}\left(\bm{x}-\bm{x}^{\star}\right)^{\mathsf{H}}\right\Vert \nonumber \\
 & \quad\qquad+\left\Vert \sum_{j=1}^{m}\bm{b}_{j}\bm{b}_{j}^{\mathsf{H}}\left(\bm{h}-\bm{h}^{\star}\right)\bm{x}^{\star\mathsf{H}}\bm{a}_{j}\bm{a}_{j}^{\mathsf{H}}-m\left(\bm{h}-\bm{h}^{\star}\right)\bm{x}^{\star\mathsf{H}}\right\Vert \nonumber \\
 & \quad\qquad+\left\Vert \sum_{j=1}^{m}\bm{b}_{j}\bm{b}_{j}^{\mathsf{H}}\left(\bm{h}-\bm{h}^{\star}\right)\left(\bm{x}-\bm{x}^{\star}\right)^{\mathsf{H}}\bm{a}_{j}\bm{a}_{j}^{\mathsf{H}}-m\left(\bm{h}-\bm{h}^{\star}\right)\left(\bm{x}-\bm{x}^{\star}\right)^{\mathsf{H}}\right\Vert .\label{eq:gaussian-alpha3-decompose2}
\end{align}
To bound the last three terms of \eqref{eq:gaussian-alpha3-decompose2}, we resort to the following two lemmas, whose
proofs can be found in Appendix \ref{subsec:Proof-of-Lemma-geometry-alpha3-1}
and Appendix  \ref{subsec:Proof-of-Lemma-geometry-alpha3-2}. 
\begin{lemma}\label{lemma:geometry-alpha3-1}With
probability at least $1-O(m^{-100}+me^{-CK})$ for some constant
$C>0$, one has
\begin{equation}
\left\Vert \sum_{j=1}^{m}\bm{b}_{j}\bm{b}_{j}^{\mathsf{H}}\bm{h}^{\star}\left(\bm{x}-\bm{x}^{\star}\right)^{\mathsf{H}}\bm{a}_{j}\bm{a}_{j}^{\mathsf{H}}-m\bm{h}^{\star}\left(\bm{x}-\bm{x}^{\star}\right)^{\mathsf{H}}\right\Vert \leq2\delta m \label{eq:gaussian-geometry-alpha3-1}
\end{equation}
holds uniformly for any $\bm{x}$ satisfying (\ref{subeq:gaussian-assumptions-geometry}).
\end{lemma}\begin{lemma}\label{lemma:geometry-alpha3-2}With probability
at least $1-2\exp(-CK\log m)$ for some constant $C>0$, one has
\begin{equation}
\left\Vert \sum_{j=1}^{m}\bm{b}_{j}\bm{b}_{j}^{\mathsf{H}}\left(\bm{h}-\bm{h}^{\star}\right)\left(\bm{x}-\bm{x}^{\star}\right)^{\mathsf{H}}\bm{a}_{j}\bm{a}_{j}^{\mathsf{H}}-m\left(\bm{h}-\bm{h}^{\star}\right)\left(\bm{x}-\bm{x}^{\star}\right)^{\mathsf{H}}\right\Vert \leq \delta^{2}m+4C'\sqrt{mK}
	\label{eq:gaussian-geometry-alpha3-2}
\end{equation}
holds uniformly for any $(\bm{h},\bm{x})$ obeying (\ref{subeq:gaussian-assumptions-geometry})
for some sufficiently large constant $C'>0$. \end{lemma}By the
symmetry between $\{\bm{a}_{j}\}_{j=1}^{m}$ and $\{\bm{b}_{j}\}_{j=1}^{m}$
and Lemma \ref{lemma:geometry-alpha3-1}, one arrives at 
\begin{equation}
\sup_{\bm{x}\in\mathcal{S}}\left\Vert \sum_{j=1}^{m}\bm{b}_{j}\bm{b}_{j}^{\mathsf{H}}\left(\bm{h}-\bm{h}^{\star}\right)\bm{x}^{\star}{}^{\mathsf{H}}\bm{a}_{j}\bm{a}_{j}^{\mathsf{H}}-m\left(\bm{h}-\bm{h}^{\star}\right)\bm{x}^{\star}{}^{\mathsf{H}}\right\Vert \leq2\delta m \label{eq:gaussian-geometry-alpha3-3}
\end{equation}
 with probability over $1-2\exp(-CK\log m)$. Plugging (\ref{eq:gaussian-geometry-alpha3-1}),
(\ref{eq:gaussian-geometry-alpha3-2}) and (\ref{eq:gaussian-geometry-alpha3-3})
into (\ref{eq:gaussian-alpha3-decompose2}) yields
\begin{equation}
\left\Vert \sum_{j=1}^{m}\bm{b}_{j}\bm{b}_{j}^{\mathsf{H}}\left(\bm{h}\bm{x}^{\mathsf{H}}-\bm{h}^{\star}\bm{x}^{\star\mathsf{H}}\right)\bm{a}_{j}\bm{a}_{j}^{\mathsf{H}}\right\Vert \leq6\delta m.\label{eq:gaussian-alpha3-core}
\end{equation}
Substitution of (\ref{eq:gaussian-alpha3-core}) and (\ref{eq:gaussian-noise})
into (\ref{eq:gaussian-alpha3-decompose}) thus gives
\begin{equation}
\alpha_{3}\leq6\delta m+C\sigma\sqrt{mK\log m} \label{eq:gaussian-alpha3}
\end{equation}
for some large enough constant $C>0$.

\paragraph{Controlling $\alpha_4$.} With regards to the last term $\alpha_4$, we have
\[
\alpha_{4}\leq\underbrace{\left\Vert \sum_{j=1}^{m}\bm{b}_{j}\bm{b}_{j}^{\mathsf{H}}\left(\bm{h}\bm{x}^{\mathsf{H}}-\bm{h}^{\star}\bm{x}^{\star\mathsf{H}}\right)\bm{a}_{j}\bm{a}_{j}^{\mathsf{H}}\right\Vert }_{\eqqcolon\theta_{1}}+\underbrace{\left\Vert \sum_{j=1}^{m}\bm{b}_{j}\bm{b}_{j}^{\mathsf{H}}\bm{h}^{\star}\bm{x}^{\star\mathsf{H}}\bm{a}_{j}\bm{a}_{j}^{\mathsf{H}}-m\bm{h}^{\star}\bm{x}^{\star\mathsf{H}}\right\Vert }_{\eqqcolon\theta_{2}}.
\]
These two terms have already been bounded by (\ref{eq:gaussian-alpha3-core})
and (\ref{eq:gaussian-M-second-term}) respectively. Combining this
inequality with (\ref{eq:gaussian-alpha3-core}) gives
\begin{equation}
\alpha_{4}\leq6\delta m+4C_{t}\sqrt{mK}\log m.\label{eq:gaussian-alpha4}
\end{equation}

\paragraph{Putting all this together.}
Finally, by plugging (\ref{eq:gaussian-alpha1}), (\ref{eq:gaussian-alpha3})
and (\ref{eq:gaussian-alpha4}) into (\ref{eq:gaussian-alphaall}),
we arrive at 
\[
\left\Vert \nabla^{2}f\left(\bm{z}\right)-\nabla^{2}F\left(\bm{z}^{\star}\right)\right\Vert \lesssim\sigma\sqrt{mK\log m}+\frac{m}{\log m}\leq\frac{1}{4}m 
\]
holds with probability exceeding $1-O(m^{-10})$.

\subsubsection{Proof of Lemma \ref{lemma:geometry-alpha3-1}\label{subsec:Proof-of-Lemma-geometry-alpha3-1}}

Consider the event 
\begin{equation}
\mathcal{E}\coloneqq\left\{ \max_{1\leq j\leq m}\left|\bm{b}_{j}^{\mathsf{H}}\bm{h}^{\star}\right|\leq20\sqrt{\log m},\max_{1\leq j\leq m}\left\Vert \bm{a}_{j}\right\Vert _{2}\leq10\sqrt{K}\right\} .\label{eq:event-E}
\end{equation}
(\ref{eq:useful1}) and (\ref{eq:useful2}) suggest that event $\mathcal{E}$
holds with probability at least $1-O(m^{-100}+me^{-CK})$. The proof
thereafter will be developed on this event. 

Due to the assumptions (\ref{subeq:gaussian-assumptions-geometry}), we have --- for any given unit
vectors $\bm{u}$, $\bm{v}\in\mathbb{C}^{K}$ --- that
\[
\sum_{j=1}^{m}\bm{u}^{\mathsf{H}}\bm{b}_{j}\bm{b}_{j}^{\mathsf{H}}\bm{h}^{\star}\left(\bm{x}-\bm{x}^{\star}\right)^{\mathsf{H}}\bm{a}_{j}\bm{a}_{j}^{\mathsf{H}}\bm{v}=\sum_{j=1}^{m}\underbrace{\bm{u}^{\mathsf{H}}\bm{b}_{j}\bm{b}_{j}^{\mathsf{H}}\bm{h}^{\star}\left(\bm{x}-\bm{x}^{\star}\right)^{\mathsf{H}}\bm{a}_{j}\bm{a}_{j}^{\mathsf{H}}\bm{v}\ind_{\left\{ \left|\left(\bm{x}-\bm{x}^{\star}\right)^{\mathsf{H}}\bm{a}_{j}\right|\leq20C_{13}\frac{1}{\log^{3/2}m}\right\} }}_{\eqqcolon X_{j}}.
\]
In what follows, we shall first establish concentration inequalities for this quantity for a given point $(\bm{u},\bm{v})$, 
and then establish uniform concentration that holds for simultaneously for all points of interest.

\paragraph{Concentration.} 
Consider any fixed unit vectors $\bm{u}$ and $\bm{v})$. 
We seek to invoke the Bernstein inequality \citet[Theorem 2.8.2]{vershynin2018high} to control $\sum_{j=1}^{m}(X_j-\mathbb{E}[X_j])$. 
We observe that 
\begin{align*}
\left\Vert X_{j}-\mathbb{E}\left[X_{j}\right]\right\Vert _{\psi_{1}} & \leq C\left\Vert X_{j}\right\Vert _{\psi_{1}}\leq C\left|\bm{b}_{j}^{\mathsf{H}}\bm{h}^{\star}\left(\bm{x}-\bm{x}^{\star}\right)^{\mathsf{H}}\bm{a}_{j}\ind_{\left\{ \left|\left(\bm{x}-\bm{x}^{\star}\right)^{\mathsf{H}}\bm{a}_{j}\right|\leq20C_{13}\frac{1}{\log^{3/2}m}\right\} }\right|\left\Vert \bm{u}^{\mathsf{H}}\bm{b}_{j}\right\Vert _{\psi_{2}}\left\Vert \bm{a}_{j}^{\mathsf{H}}\bm{v}\right\Vert _{\psi_{2}}\\
 & =C\left|\bm{b}_{j}^{\mathsf{H}}\bm{h}^{\star}\right|\cdot\left|\left(\bm{x}-\bm{x}^{\star}\right)^{\mathsf{H}}\bm{a}_{j}\ind_{\left\{ \left|\left(\bm{x}-\bm{x}^{\star}\right)^{\mathsf{H}}\bm{a}_{j}\right|\leq20C_{13}\frac{1}{\log^{3/2}m}\right\} }\right|\\
 & \leq400CC_{13}\frac{1}{\log m},
\end{align*}
where the first inequality comes from the fact that $\Vert X-\mathbb{E}[X]\Vert_{\psi_{1}}\leq C\Vert X\Vert_{\psi_{1}}$
(cf.~\citet[Section 2.7]{vershynin2018high}) and the last inequality
is due to the event $\mathcal{E}$. Hence, the Bernstein inequality \citet[Theorem 2.8.2]{vershynin2018high}
reveals that
\[
\mathbb{P}\left(\left|\sum_{j=1}^{m}\left(X_{j}-\mathbb{E}\left[X_{j}\right]\right)\right|\geq t\right)\leq2\exp\left(-c\min\left(\frac{t^{2}\log^{2}m}{m},t\log m\right)\right).
\]
Letting $t=C_{t}\sqrt{mK}$ for some large enough constant $C_{t}>0$,
we obtain 
\begin{equation}
\left|\sum_{j=1}^{m}\left(X_{j}-\mathbb{E}\left[X_{j}\right]\right)\right|\leq C_{t}\sqrt{mK}, \label{eq:lemma28-1}
\end{equation}
 with probability exceeding $1-2\exp(-cC_{t}^{2}K\log m)$. 

\paragraph{Union bound over epsilon-nets.}
Next, we intend to show that \eqref{eq:lemma28-1} holds uniformly for any unit vectors $\bm{u}$ and $\bm{v}$. Define $\mathcal{N}_{\bm{x}}$ to be an $\epsilon_{1}$-net of
$\mathcal{B}_{\bm{x}}(\delta)\coloneqq\{\bm{x}:\Vert\bm{x}-\bm{x}^{\star}\Vert_{2}\leq\delta\}$
and $\mathcal{N}_{0}$ an $\epsilon_{2}$-net of the unit sphere $\mathcal{S}^{K-1}$.
In view of \citet[Corollary 4.2.13]{vershynin2018high}, we can choose these nets to guarantee that 
\[
\left|\mathcal{N}_{\bm{x}}\right|\leq\left(1+\frac{2\delta}{\epsilon_{1}}\right)^{4K}\qquad\text{and}\qquad\left|\mathcal{N}_{0}\right|\leq\left(1+\frac{2}{\epsilon_{2}}\right)^{2K}.
\]
Taking these collectively with the union bound reveals that $\eqref{eq:lemma28-1}$
holds uniformly for all $\bm{x}\in\mathcal{N}_{\bm{x}}$ and $\bm{u}$,
$\bm{v}\in\mathcal{N}_{0}$ with probability exceeding
\[
1-\left(1+\frac{2\delta}{\epsilon_{1}}\right)^{4K}\left(1+\frac{2}{\epsilon_{2}}\right)^{4K}\cdot2\exp\left(-cC_{t}^{2}K\log m\right)\geq1-2\exp\left(-CK\log m\right).
\]

\paragraph{Approximation.}
We then turn to the following quantity
\[
g\left(\bm{u},\bm{v},\bm{x}\right)\coloneqq\sum_{j=1}^{m}\left[\bm{u}^{\mathsf{H}}\bm{b}_{j}\bm{b}_{j}^{\mathsf{H}}\bm{h}^{\star}\left(\bm{x}-\bm{x}^{\star}\right)^{\mathsf{H}}\bm{a}_{j}\bm{a}_{j}^{\mathsf{H}}\ind_{\left\{ \left|\left(\bm{x}-\bm{x}^{\star}\right)^{\mathsf{H}}\bm{a}_{j}\right|\leq20C_{13}\frac{1}{\log^{3/2}m}\right\} }\bm{v}-m\bm{h}^{\star}\left(\bm{x}-\bm{x}^{\star}\right)^{\mathsf{H}}\right].
\]
For any $\bm{x}$ satisfying the assumptions (\ref{subeq:gaussian-assumptions-geometry})
and any $\bm{u}$, $\bm{v}\in\mathcal{S}^{K-1}$, one can choose $\bm{x}_{0}\in\mathcal{N}_{\bm{x}}$,
$\bm{u}_{0}\in\mathcal{N}_{0}$ and $\bm{v}_{0}\in\mathcal{N}_{0}$
satisfying $\Vert\bm{x}-\bm{x}_{0}\Vert_{2}\leq\epsilon_{1}$ and
$\max\{\Vert\bm{u}-\bm{u}_{0}\Vert_{2},\Vert\bm{v}-\bm{v}_{0}\Vert_{2}\}\leq\epsilon_{2}$.
Set $\epsilon_{1}=\delta/K$ and $\epsilon_{2}=1/4$. The triangle
inequality gives
\begin{align}
 & \left|g\left(\bm{u},\bm{v},\bm{x}\right)-g\left(\bm{u}_{0},\bm{v}_{0},\bm{x}_{0}\right)\right| \nonumber\\
 & \quad\leq\left|g\left(\bm{u},\bm{v},\bm{x}\right)-g\left(\bm{u}_{0},\bm{v},\bm{x}\right)\right|+\left|g\left(\bm{u}_{0},\bm{v},\bm{x}\right)-g\left(\bm{u}_{0},\bm{v}_{0},\bm{x}\right)\right| \nonumber\\
 & \quad\qquad+\left|g\left(\bm{u}_{0},\bm{v}_{0},\bm{x}\right)-g\left(\bm{u}_{0},\bm{v}_{0},\bm{x}_{0}\right)\right| \nonumber\\
 & \quad\leq2\left\Vert \sum_{j=1}^{m}\bm{b}_{j}\bm{b}_{j}^{\mathsf{H}}\bm{h}^{\star}\left(\bm{x}-\bm{x}^{\star}\right)^{\mathsf{H}}\bm{a}_{j}\bm{a}_{j}^{\mathsf{H}}\ind_{\left\{ \left|\left(\bm{x}-\bm{x}^{\star}\right)^{\mathsf{H}}\bm{a}_{j}\right|\leq20C_{13}\frac{1}{\log^{3/2}m}\right\} }-m\bm{h}^{\star}\left(\bm{x}-\bm{x}^{\star}\right)^{\mathsf{H}}\right\Vert \epsilon_{2} \nonumber\\
 & \quad\qquad+\left|g\left(\bm{u}_{0},\bm{v}_{0},\bm{x}\right)-g\left(\bm{u}_{0},\bm{v}_{0},\bm{x}_{0}\right)\right|.\label{eq:lemma28-2}
\end{align}
To simplify the second term above, we notice that on event $\mathcal{E}$ (cf.~\eqref{eq:event-E}),
\begin{align}
\left|\left(\bm{x}-\bm{x}_{0}\right)^{\mathsf{H}}\bm{a}_{j}\right| & \leq\max_{1\leq j\leq m}\left\Vert \bm{a}_{j}\right\Vert _{2}\cdot\left\Vert \bm{x}-\bm{x}_{0}\right\Vert _{2}\leq10\sqrt{K}\cdot\epsilon_{1}\leq2C_{13}\frac{1}{\log^{3/2}m},\label{eq:gaussian-s-condition}
\end{align}
and hence
\begin{align}
\left|\left(\bm{x}_{0}-\bm{x}^{\star}\right)^{\mathsf{H}}\bm{a}_{j}\right| & \leq\left|\left(\bm{x}-\bm{x}^{\star}\right)^{\mathsf{H}}\bm{a}_{j}\right|+\left|\left(\bm{x}-\bm{x}_{0}\right)^{\mathsf{H}}\bm{a}_{j}\right|\nonumber\\
 & \leq4C_{13}\frac{1}{\log^{3/2}m}.\label{eq:lemma28-6}
\end{align}
As a result, one has the following identity
\begin{align}
\ind_{\left\{ \left|\left(\bm{x}-\bm{x}_{0}\right)^{\mathsf{H}}\bm{a}_{j}\right|\leq20C_{13}\frac{1}{\log^{3/2}m}\right\} }=\ind_{\left\{ \left|\left(\bm{x}_{0}-\bm{x}^{\star}\right)^{\mathsf{H}}\bm{a}_{j}\right|\leq20C_{13}\frac{1}{\log^{3/2}m}\right\} }=\ind_{\left\{ \left|\left(\bm{x}-\bm{x}^{\star}\right)^{\mathsf{H}}\bm{a}_{j}\right|\leq20C_{13}\frac{1}{\log^{3/2}m}\right\} }=1. \label{eq:lemma28-allone}
\end{align}
It then follows that 
\begin{align}
 & \left|g\left(\bm{u}_{0},\bm{v}_{0},\bm{x}\right)-g\left(\bm{u}_{0},\bm{v}_{0},\bm{x}_{0}\right)\right| \nonumber\\
 & \quad=\left\Vert \sum_{j=1}^{m}\bm{b}_{j}\bm{b}_{j}^{\mathsf{H}}\bm{h}^{\star}\left(\bm{x}-\bm{x}_{0}\right)^{\mathsf{H}}\bm{a}_{j}\bm{a}_{j}^{\mathsf{H}}\ind_{\left\{ \left|\left(\bm{x}-\bm{x}_{0}\right)^{\mathsf{H}}\bm{a}_{j}\right|\leq20C_{13}\frac{1}{\log^{3/2}m}\right\} }-m\bm{h}^{\star}\left(\bm{x}-\bm{x}_{0}\right)^{\mathsf{H}}\right\Vert .\label{eq:lemma28-3}
\end{align}
Plugging \eqref{eq:lemma28-3} into \eqref{eq:lemma28-2} yields
\begin{align}
& \left|g\left(\bm{u},\bm{v},\bm{x}\right)-g\left(\bm{u}_{0},\bm{v}_{0},\bm{x}_{0}\right)\right| \nonumber\\
& \quad\leq2\left\Vert \sum_{j=1}^{m}\bm{b}_{j}\bm{b}_{j}^{\mathsf{H}}\bm{h}^{\star}\left(\bm{x}-\bm{x}^{\star}\right)^{\mathsf{H}}\bm{a}_{j}\bm{a}_{j}^{\mathsf{H}}\ind_{\left\{ \left|\left(\bm{x}-\bm{x}^{\star}\right)^{\mathsf{H}}\bm{a}_{j}\right|\leq20C_{13}\frac{1}{\log^{3/2}m}\right\} }-m\bm{h}^{\star}\left(\bm{x}-\bm{x}^{\star}\right)^{\mathsf{H}}\right\Vert \epsilon_{2} \nonumber\\
& \quad\qquad+\left\Vert \sum_{j=1}^{m}\bm{b}_{j}\bm{b}_{j}^{\mathsf{H}}\bm{h}^{\star}\left(\bm{x}-\bm{x}_{0}\right)^{\mathsf{H}}\bm{a}_{j}\bm{a}_{j}^{\mathsf{H}}\ind_{\left\{ \left|\left(\bm{x}-\bm{x}_{0}\right)^{\mathsf{H}}\bm{a}_{j}\right|\leq20C_{13}\frac{1}{\log^{3/2}m}\right\} }-m\bm{h}^{\star}\left(\bm{x}-\bm{x}_{0}\right)^{\mathsf{H}}\right\Vert .\label{eq:lemma28-4}
\end{align}

Next, we look at $g(\bm{u}_{0},\bm{v}_{0},\bm{x}_{0})$, and notice that \eqref{eq:lemma28-1} holds for \[X_j = \bm{u}_0^{\mathsf{H}}\bm{b}_{j}\bm{b}_{j}^{\mathsf{H}}\bm{h}^{\star}\left(\bm{x}_0-\bm{x}^{\star}\right)^{\mathsf{H}}\bm{a}_{j}\bm{a}_{j}^{\mathsf{H}}\bm{v}_0\ind_{\left\{ \left|\left(\bm{x}_0-\bm{x}^{\star}\right)^{\mathsf{H}}\bm{a}_{j}\right|\leq20C_{13}\frac{1}{\log^{3/2}m}\right\} },\] due to $\bm{x}_{0}\in\mathcal{N}_{\bm{x}}$,
$\bm{u}_{0}\in\mathcal{N}_{0}$ and $\bm{v}_{0}\in\mathcal{N}_{0}$. By virtue of the triangle inequality, one has
\begin{align}
 & \left|g\left(\bm{u}_{0},\bm{v}_{0},\bm{x}_{0}\right)\right|\nonumber\\
 & \quad\leq\left|\sum_{j=1}^{m}\left(X_{j}-\mathbb{E}\left[X_{j}\right]\right)\right|+\left|\sum_{j=1}^{m}\left(\mathbb{E}\left[X_{j}\right]-m\bm{u}_0^{\mathsf{H}}\bm{h}^{\star}\left(\bm{x}_0-\bm{x}^{\star}\right)^{\mathsf{H}}\bm{v}_0\right)\right|\nonumber\\
 &\quad \leq C_{t}\sqrt{mK}+\left|\sum_{j=1}^{m}\mathbb{E}\left[\bm{u}_0^{\mathsf{H}}\bm{b}_{j}\bm{b}_{j}^{\mathsf{H}}\bm{h}^{\star}\left(\bm{x}_0-\bm{x}^{\star}\right)^{\mathsf{H}}\bm{a}_{j}\bm{a}_{j}^{\mathsf{H}}\bm{v}\ind_{\left\{ \left|\left(\bm{x}_0-\bm{x}^{\star}\right)^{\mathsf{H}}\bm{a}_{j}\right|>20C_{13}\frac{1}{\log^{3/2}m}\right\} }\right]\right|\nonumber\\
 &\quad \leq C_{t}\sqrt{mK},\label{eq:lemma28-5}
\end{align}
where
\begin{align*}
 & \left|\mathbb{E}\left[\bm{u}_0^{\mathsf{H}}\bm{b}_{j}\bm{b}_{j}^{\mathsf{H}}\bm{h}^{\star}\left(\bm{x}_0-\bm{x}^{\star}\right)^{\mathsf{H}}\bm{a}_{j}\bm{a}_{j}^{\mathsf{H}}\bm{v}_0\ind_{\left\{ \left|\left(\bm{x}_0-\bm{x}^{\star}\right)^{\mathsf{H}}\bm{a}_{j}\right|>20C_{13}\frac{1}{\log^{3/2}m}\right\} }\right]\right|=0, \\
\end{align*}
which is a consequence of \eqref{eq:lemma28-allone}.

\paragraph{Putting all this together.}
Let us define
\[
\mathcal{S}\coloneqq\left\{ \bm{x}:\left|\left(\bm{x}-\bm{x}^{\star}\right)^{\mathsf{H}}\bm{a}_{j}\right|\leq20C_{13}\frac{1}{\log^{3/2}m},\left\Vert \bm{x}-\bm{x}^{\star}\right\Vert _{2}\leq\delta\right\} .
\]
Taking \eqref{eq:lemma28-4} and \eqref{eq:lemma28-5} collectively gives rise to
\begin{align}
\left|g\left(\bm{u},\bm{v},\bm{x}\right)\right| & \leq\left|g\left(\bm{u}_{0},\bm{v}_{0},\bm{x}_{0}\right)\right|+\left|g\left(\bm{u},\bm{v},\bm{x}\right)-g\left(\bm{u}_{0},\bm{v}_{0},\bm{x}_{0}\right)\right|\nonumber \\
 & \leq C_{t}\sqrt{mK}+2\left\Vert \sum_{j=1}^{m}\bm{b}_{j}\bm{b}_{j}^{\mathsf{H}}\bm{h}^{\star}\left(\bm{x}-\bm{x}^{\star}\right)^{\mathsf{H}}\bm{a}_{j}\bm{a}_{j}^{\mathsf{H}}\ind_{\left\{ \left|\left(\bm{x}-\bm{x}^{\star}\right)^{\mathsf{H}}\bm{a}_{j}\right|\leq20C_{13}\frac{1}{\log^{3/2}m}\right\} }-m\bm{h}^{\star}\left(\bm{x}-\bm{x}^{\star}\right)^{\mathsf{H}}\right\Vert \epsilon_{2}\nonumber \\
 & \qquad+\left\Vert \sum_{j=1}^{m}\bm{b}_{j}\bm{b}_{j}^{\mathsf{H}}\bm{h}^{\star}\left(\bm{x}-\bm{x}_{0}\right)^{\mathsf{H}}\bm{a}_{j}\bm{a}_{j}^{\mathsf{H}}\ind_{\left\{ \left|\left(\bm{x}-\bm{x}_{0}\right)^{\mathsf{H}}\bm{a}_{j}\right|\leq20C_{13}\frac{1}{\log^{3/2}m}\right\} }-m\bm{h}^{\star}\left(\bm{x}-\bm{x}_{0}\right)^{\mathsf{H}}\right\Vert .\label{eq:gaussian-g-1}
\end{align}

A key observation is that $\bm{x}'\coloneqq5(\bm{x}-\bm{x}_{0})+\bm{x}^{\star}\in\mathcal{S}$
by $\Vert\bm{x}-\bm{x}_{0}\Vert_{2}\leq\epsilon_{1}$
and (\ref{eq:lemma28-6}). Hence,  the last term in \eqref{eq:gaussian-g-1} satisfies
\begin{align*}
 & \left\Vert \sum_{j=1}^{m}\bm{b}_{j}\bm{b}_{j}^{\mathsf{H}}\bm{h}^{\star}\left(\bm{x}-\bm{x}_{0}\right)^{\mathsf{H}}\bm{a}_{j}\bm{a}_{j}^{\mathsf{H}}\ind_{\left\{ \left|\left(\bm{x}-\bm{x}_{0}\right)^{\mathsf{H}}\bm{a}_{j}\right|\leq20C_{13}\frac{1}{\log^{3/2}m}\right\} }-m\bm{h}^{\star}\left(\bm{x}-\bm{x}_{0}\right)^{\mathsf{H}}\right\Vert \\
  & \quad=\frac15\left\Vert \sum_{j=1}^{m}\bm{b}_{j}\bm{b}_{j}^{\mathsf{H}}\bm{h}^{\star}\left(\bm{x}'-\bm{x}_{\star}\right)^{\mathsf{H}}\bm{a}_{j}\bm{a}_{j}^{\mathsf{H}}\ind_{\left\{ \left|\left(\bm{x}'-\bm{x}_{\star}\right)^{\mathsf{H}}\bm{a}_{j}\right|\leq20C_{13}\frac{1}{\log^{3/2}m}\right\} }-m\bm{h}^{\star}\left(\bm{x}'-\bm{x}_{0}\right)^{\mathsf{H}}\right\Vert \\
 & \quad\leq\frac{1}{5}\sup_{\widetilde{\bm{x}}\in\mathcal{S}}\left\Vert \sum_{j=1}^{m}\bm{b}_{j}\bm{b}_{j}^{\mathsf{H}}\bm{h}^{\star}\left(\widetilde{\bm{x}}-\bm{x}^{\star}\right)^{\mathsf{H}}\bm{a}_{j}\bm{a}_{j}^{\mathsf{H}}\ind_{\left\{ \left|\left(\widetilde{\bm{x}}-\bm{x}^{\star}\right)^{\mathsf{H}}\bm{a}_{j}\right|\leq20C_{13}\frac{1}{\log^{3/2}m}\right\} }-m\bm{h}^{\star}\left(\widetilde{\bm{x}}-\bm{x}^{\star}\right)^{\mathsf{H}}\right\Vert ,
\end{align*}
where the first equality comes from \eqref{eq:gaussian-s-condition}.
Plugging this inequality into (\ref{eq:gaussian-g-1}), taking the
maximum over $\bm{u}$ and $\bm{v}$ on the left-hand side of (\ref{eq:gaussian-g-1})
and rearranging terms yield
\begin{align*}
 & \left(1-2\epsilon_{2}\right)\left\Vert \sum_{j=1}^{m}\bm{b}_{j}\bm{b}_{j}^{\mathsf{H}}\bm{h}^{\star}\left(\bm{x}-\bm{x}^{\star}\right)^{\mathsf{H}}\bm{a}_{j}\bm{a}_{j}^{\mathsf{H}}\ind_{\left\{ \left|\left(\bm{x}-\bm{x}^{\star}\right)^{\mathsf{H}}\bm{a}_{j}\right|\leq20C_{13}\frac{1}{\log^{3/2}m}\right\} }-m\bm{h}^{\star}\left(\bm{x}-\bm{x}^{\star}\right)^{\mathsf{H}}\right\Vert \\
&\quad\leq  2C_{t}\sqrt{mK}+\frac{1}{5}\sup_{\widetilde{\bm{x}}\in\mathcal{S}}\left\Vert \sum_{j=1}^{m}\bm{b}_{j}\bm{b}_{j}^{\mathsf{H}}\bm{h}^{\star}\left(\widetilde{\bm{x}}-\bm{x}^{\star}\right)^{\mathsf{H}}\bm{a}_{j}\bm{a}_{j}^{\mathsf{H}}\ind_{\left\{ \left|\left(\widetilde{\bm{x}}-\bm{x}^{\star}\right)^{\mathsf{H}}\bm{a}_{j}\right|\leq20C_{13}\frac{1}{\log^{3/2}m}\right\} }-m\bm{h}^{\star}\left(\widetilde{\bm{x}}-\bm{x}^{\star}\right)^{\mathsf{H}}\right\Vert .
\end{align*}
Further, taking the maximum over $\bm{x}\in\mathcal{S}$ on the left-hand side of the above inequality gives 
\[
\left(1-2\epsilon_{2}-\frac{1}{5}\right)\sup_{\bm{x}\in\mathcal{S}}\left\Vert \sum_{j=1}^{m}\bm{b}_{j}\bm{b}_{j}^{\mathsf{H}}\bm{h}^{\star}\left(\bm{x}-\bm{x}^{\star}\right)^{\mathsf{H}}\bm{a}_{j}\bm{a}_{j}^{\mathsf{H}}\ind_{\left\{ \left|\left(\bm{x}-\bm{x}^{\star}\right)^{\mathsf{H}}\bm{a}_{j}\right|\leq20C_{13}\frac{1}{\log^{3/2}m}\right\} }-m\bm{h}^{\star}\left(\bm{x}-\bm{x}^{\star}\right)^{\mathsf{H}}\right\Vert \leq2C_{t}\sqrt{mK},
\]
and, consequently, 
\begin{align*}
& \sup_{\bm{x}\in\mathcal{S}}\left\Vert \sum_{j=1}^{m}\bm{b}_{j}\bm{b}_{j}^{\mathsf{H}}\bm{h}^{\star}\left(\bm{x}-\bm{x}^{\star}\right)^{\mathsf{H}}\bm{a}_{j}\bm{a}_{j}^{\mathsf{H}}\ind_{\left\{ \left|\left(\bm{x}-\bm{x}^{\star}\right)^{\mathsf{H}}\bm{a}_{j}\right|\leq20C_{13}\frac{1}{\log^{3/2}m}\right\} }\right\Vert  \\
&\quad= \sup_{\bm{x}\in\mathcal{S}}\left\Vert \sum_{j=1}^{m}\bm{b}_{j}\bm{b}_{j}^{\mathsf{H}}\bm{h}^{\star}\left(\bm{x}-\bm{x}^{\star}\right)^{\mathsf{H}}\bm{a}_{j}\bm{a}_{j}^{\mathsf{H}}\right\Vert  \\
&\quad \leq m\left\Vert \bm{h}^{\star}\left(\bm{x}-\bm{x}^{\star}\right)^{\mathsf{H}}\right\Vert +4C_{t}\sqrt{mK}\\
 &\quad \leq2\delta m,
\end{align*}
as long as $m\gg K\log^{4}m$.

\subsubsection{Proof of Lemma \ref{lemma:geometry-alpha3-2}\label{subsec:Proof-of-Lemma-geometry-alpha3-2}}
Similar to proof of Lemma \ref{lemma:geometry-alpha3-1}, we consider the event 
\begin{equation}
\mathcal{E}\coloneqq\left\{ \max_{1\leq j\leq m}\left|\bm{b}_{j}^{\mathsf{H}}\bm{h}^{\star}\right|\leq20\sqrt{\log m},\max_{1\leq j\leq m}\left\Vert \bm{a}_{j}\right\Vert _{2}\leq10\sqrt{K}\right\}, \label{eq:event-E-1}
\end{equation}
which
holds with probability at least $1-O(m^{-100}+me^{-CK})$. The proof
thereafter will be developed on this event. For any fixed unit vectors
$\bm{u}$, $\bm{v}\in\mathbb{C}^{K}$ and $(\bm{h},\bm{x})$ obeying
the assumptions (\ref{subeq:gaussian-assumptions-geometry}), one
has 
\begin{align*}
 & \sum_{j=1}^{m}\bm{u}^{\mathsf{H}}\bm{b}_{j}\bm{b}_{j}^{\mathsf{H}}\left(\bm{h}-\bm{h}^{\star}\right)\left(\bm{x}-\bm{x}^{\star}\right)^{\mathsf{H}}\bm{a}_{j}\bm{a}_{j}^{\mathsf{H}}\bm{v}\\
&\quad =  \sum_{j=1}^{m}\underbrace{\bm{u}^{\mathsf{H}}\bm{b}_{j}\bm{b}_{j}^{\mathsf{H}}\left(\bm{h}-\bm{h}^{\star}\right)\left(\bm{x}-\bm{x}^{\star}\right)^{\mathsf{H}}\bm{a}_{j}\bm{a}_{j}^{\mathsf{H}}\bm{v}\ind_{\left\{ \max\left\{ \left|\bm{b}_{j}^{\mathsf{H}}\left(\bm{h}-\bm{h}^{\star}\right)\right|,\left|\left(\bm{x}-\bm{x}^{\star}\right)^{\mathsf{H}}\bm{a}_{j}\right|\right\} \leq20C_{13}\frac{1}{\log^{3/2}m}\right\} }}_{\eqqcolon W_{j}}.
\end{align*}

\paragraph{Concentration.}
Consider any fixed  vectors
$\bm{u}$, $\bm{v}$ and $(\bm{h},\bm{x})$. 
We seek to invoke the Bernstein inequality \citet[Theorem 2.8.2]{vershynin2018high} to control $\sum_{j=1}^m W_j$. We observe that 
\begin{align*}
&\left\Vert W_{j}-\mathbb{E}\left[W_{j}\right]\right\Vert _{\psi_{1}} \leq C\left\Vert W_{j}\right\Vert _{\psi_{1}}\\
 &\quad \leq C\left|\bm{b}_{j}^{\mathsf{H}}\left(\bm{h}-\bm{h}^{\star}\right)\left(\bm{x}-\bm{x}^{\star}\right)^{\mathsf{H}}\bm{a}_{j}\ind_{\left\{ \max\left\{ \left|\bm{b}_{j}^{\mathsf{H}}\left(\bm{h}-\bm{h}^{\star}\right)\right|,\left|\left(\bm{x}-\bm{x}^{\star}\right)^{\mathsf{H}}\bm{a}_{j}\right|\right\} \leq20C_{13}\frac{1}{\log^{3/2}m}\right\} }\right|\left\Vert \bm{u}^{\mathsf{H}}\bm{b}_{j}\right\Vert _{\psi_{2}}\left\Vert \bm{a}_{j}^{\mathsf{H}}\bm{v}\right\Vert _{\psi_{2}}\\
 &\quad =C\left|\bm{b}_{j}^{\mathsf{H}}\left(\bm{h}-\bm{h}^{\star}\right)\left(\bm{x}-\bm{x}^{\star}\right)^{\mathsf{H}}\bm{a}_{j}\ind_{\left\{ \max\left\{ \left|\bm{b}_{j}^{\mathsf{H}}\left(\bm{h}-\bm{h}^{\star}\right)\right|,\left|\left(\bm{x}-\bm{x}^{\star}\right)^{\mathsf{H}}\bm{a}_{j}\right|\right\} \leq20C_{13}\frac{1}{\log^{3/2}m}\right\} }\right|\\
 &\quad \leq400CC_{13}^{2}\frac{1}{\log^{3}m},
\end{align*}
where the first inequality comes from the fact that $\Vert X-\mathbb{E}[X]\Vert_{\psi_{1}}\leq C\Vert X\Vert_{\psi_{1}}$
(cf.~\citet[Section 2.7]{vershynin2018high}), the second one is due to \citet[Lemma 2.7.7]{vershynin2018high} and the last inequality
is due to the event $\mathcal{E}$. 
Hence, the Bernstein inequality \citet[Theorem 2.8.2]{vershynin2018high}
reveals that
\[
\mathbb{P}\left(\left|\sum_{j=1}^{m}\left(W_{j}-\mathbb{E}\left[W_{j}\right]\right)\right|\geq t\right)\leq2\exp\left(-c\min\left(\frac{t^{2}\log^{6}m}{m},t\log^{3}m\right)\right).
\]
Letting $t=C_{t}\sqrt{mK}$ for some large enough constant $C_{t}>0$,
we obtain that 
\begin{equation}
	\left|\sum_{j=1}^{m}\left(W_{j}-\mathbb{E}\left[W_{j}\right]\right)\right|\leq C_{t}\sqrt{mK}, \label{eq:lemma29-1}
\end{equation}
holds with probability exceeding $1-2\exp(-cC_{t}^{2}K\log m)$. 

\paragraph{Union bound.}
Next,
we define $\mathcal{N}_{\bm{z}}$ to be an $\epsilon_{1}$-net of
$\mathcal{B}_{\bm{z}}(\delta)\coloneqq\{(\bm{h},\bm{x}):\max\{\Vert\bm{h}-\bm{h}^{\star}\Vert_{2},\Vert\bm{x}-\bm{x}^{\star}\Vert_{2}\}\leq\delta\}$
and $\mathcal{N}_{0}$ an $\epsilon_{2}$-net of the unit sphere $\mathcal{S}^{K-1}$.
In view of \citet[Corollary 4.2.13]{vershynin2018high}, we have 
\[
\left|\mathcal{N}_{\bm{z}}\right|\leq\left(1+\frac{2\delta}{\epsilon_{1}}\right)^{4K}\qquad\text{and}\qquad\left|\mathcal{N}_{0}\right|\leq\left(1+\frac{2}{\epsilon_{2}}\right)^{2K}.
\]
Taking this collectively with the union bound yields that \eqref{eq:lemma29-1} holds uniformly for any $(\bm{h},\bm{x})\in\mathcal{N}_{\bm{z}}$
and $\bm{u}$, $\bm{v}\in\mathcal{N}_{0}$ with probability over
\[
1-\left(1+\frac{2\delta}{\epsilon_{1}}\right)^{4K}\left(1+\frac{2}{\epsilon_{2}}\right)^{4K}\cdot2\exp\left(-CK\log m\right)\geq1-2\exp(-CK\log m).
\]

\paragraph{Approximation.}
Define 
\[
\bm{H}_{j}\left(\bm{h},\bm{x}\right)\coloneqq\bm{b}_{j}\bm{b}_{j}^{\mathsf{H}}\left(\bm{h}-\bm{h}^{\star}\right)\left(\bm{x}-\bm{x}^{\star}\right)^{\mathsf{H}}\bm{a}_{j}\bm{a}_{j}^{\mathsf{H}}\ind_{\left\{ \max\left\{ \left|\bm{b}_{j}^{\mathsf{H}}\left(\bm{h}-\bm{h}^{\star}\right)\right|,\left|\left(\bm{x}-\bm{x}^{\star}\right)^{\mathsf{H}}\bm{a}_{j}\right|\right\} \leq20C_{13}\frac{1}{\log^{3/2}m}\right\} }.
\]
For any $(\bm{h},\bm{x})$ satisfying the assumptions (\ref{subeq:gaussian-assumptions-geometry})
and any $\bm{u}$, $\bm{v}\in\mathcal{S}^{K-1}$, one can choose $(\bm{h}_{0},\bm{x}_{0})\in\mathcal{N}_{\bm{z}}$,
$\bm{u}_{0}\in\mathcal{N}_{0}$ and $\bm{v}_{0}\in\mathcal{N}_{0}$
satisfying $\max\{\Vert\bm{h}-\bm{h}_{0}\Vert_{2},\Vert\bm{x}-\bm{x}_{0}\Vert_{2}\}\leq\epsilon_{1}$
and $\max\{\Vert\bm{u}-\bm{u}_{0}\Vert_{2},\Vert\bm{v}-\bm{v}_{0}\Vert_{2}\}\leq\epsilon_{2}$.
Let
\[
g\left(\bm{u},\bm{v},\bm{h},\bm{x}\right)\coloneqq\sum_{j=1}^{m}\bm{u}^{\mathsf{H}}\bm{H}_{j}\left(\bm{h},\bm{x}\right)\bm{v}-m\left(\bm{h}-\bm{h}^{\star}\right)\left(\bm{x}-\bm{x}^{\star}\right)^{\mathsf{H}}.
\]
Set $\epsilon_{1}=\delta/K$ and $\epsilon_{2}=1/4$. In view of the triangle
inequality, one has 
\begin{align}
 & \left|g\left(\bm{u},\bm{v},\bm{h},\bm{x}\right)-g\left(\bm{u}_{0},\bm{v}_{0},\bm{h}_{0},\bm{x}_{0}\right)\right|\nonumber \\
 & \quad\leq\left|g\left(\bm{u},\bm{v},\bm{h},\bm{x}\right)-g\left(\bm{u}_{0},\bm{v},\bm{h},\bm{x}\right)\right|+\left|g\left(\bm{u}_{0},\bm{v},\bm{h},\bm{x}\right)-g\left(\bm{u}_{0},\bm{v}_{0},\bm{h},\bm{x}\right)\right|\nonumber \\
 & \quad\qquad+\left|g\left(\bm{u}_{0},\bm{v}_{0},\bm{h},\bm{x}\right)-g\left(\bm{u}_{0},\bm{v}_{0},\bm{h}_{0},\bm{x}\right)\right|+\left|g\left(\bm{u}_{0},\bm{v}_{0},\bm{h}_{0},\bm{x}\right)-g\left(\bm{u}_{0},\bm{v}_{0},\bm{h}_{0},\bm{x}_{0}\right)\right|\nonumber \\
 &\quad \leq2\epsilon_{2}\left\Vert \sum_{j=1}^{m}\bm{H}_{j}\left(\bm{h},\bm{x}\right)-m\left(\bm{h}-\bm{h}^{\star}\right)\left(\bm{x}-\bm{x}^{\star}\right)^{\mathsf{H}}\right\Vert \nonumber \\
 & \quad\qquad+\left|g\left(\bm{u}_{0},\bm{v}_{0},\bm{h},\bm{x}\right)-g\left(\bm{u}_{0},\bm{v}_{0},\bm{h}_{0},\bm{x}\right)\right|+\left|g\left(\bm{u}_{0},\bm{v}_{0},\bm{h}_{0},\bm{x}\right)-g\left(\bm{u}_{0},\bm{v}_{0},\bm{h}_{0},\bm{x}_{0}\right)\right|.\label{eq:geometry-eqn-1}
\end{align}
To simplify the last two terms, we observe that 
\begin{equation}
\left|\left(\bm{x}-\bm{x}_{0}\right)^{\mathsf{H}}\bm{a}_{j}\right|\leq\max_{1\leq j\leq m}\left\Vert \bm{a}_{j}\right\Vert _{2}\cdot\left\Vert \bm{x}-\bm{x}_{0}\right\Vert _{2}\leq10\sqrt{K}\epsilon_{1}\leq C_{13}\frac{1}{\log^{3/2}m},\label{eq:gaussian-s-condition-1}
\end{equation}
and furthermore, 
\begin{align*}
\left|\left(\bm{x}_{0}-\bm{x}^{\star}\right)^{\mathsf{H}}\bm{a}_{j}\right| & \leq\left|\left(\bm{x}-\bm{x}^{\star}\right)^{\mathsf{H}}\bm{a}_{j}\right|+\left|\left(\bm{x}-\bm{x}_{0}\right)^{\mathsf{H}}\bm{a}_{j}\right|\\
 & \leq\left|\left(\bm{x}-\bm{x}^{\star}\right)^{\mathsf{H}}\bm{a}_{j}\right|+C_{13}\frac{1}{\log^{3/2}m}\\
 & \leq3C_{13}\frac{1}{\log^{3/2}m}.
\end{align*}
Similarly the same bounds also hold for $|\bm{b}_{j}^{\mathsf{H}}(\bm{h}_{0}-\bm{h}^{\star})|$.
It follows that
\begin{align}
\ind_{\left\{ \left|\left(\bm{x}-\bm{x}_{0}\right)^{\mathsf{H}}\bm{a}_{j}\right|\leq20C_{13}\frac{1}{\log^{3/2}m}\right\} }=\ind_{\left\{ \left|\left(\bm{x}_{0}-\bm{x}^{\star}\right)^{\mathsf{H}}\bm{a}_{j}\right|\leq20C_{13}\frac{1}{\log^{3/2}m}\right\} }=\ind_{\left\{ \left|\left(\bm{x}-\bm{x}^{\star}\right)^{\mathsf{H}}\bm{a}_{j}\right|\leq20C_{13}\frac{1}{\log^{3/2}m}\right\} }=1,\\ \label{eq:lemma29-2}
\ind_{\left\{ \left|\bm{b}_{j}^{\mathsf{H}}\left(\bm{h}-\bm{h}_{0}\right)\right|\leq20C_{13}\frac{1}{\log^{3/2}m}\right\} }=\ind_{\left\{ \left|\bm{b}_{j}^{\mathsf{H}}\left(\bm{h}_{0}-\bm{h}^{\star}\right)\right|\leq20C_{13}\frac{1}{\log^{3/2}m}\right\} }=\ind_{\left\{ \left|\bm{b}_{j}^{\mathsf{H}}\left(\bm{h}-\bm{h}^{\star}\right)\right|\leq20C_{13}\frac{1}{\log^{3/2}m}\right\} }=1.
\end{align}
Then, we can bound the last two term in (\ref{eq:geometry-eqn-1})
as follows
\begin{align*}
 & \left|g\left(\bm{u}_{0},\bm{v}_{0},\bm{h},\bm{x}\right)-g\left(\bm{u}_{0},\bm{v}_{0},\bm{h}_{0},\bm{x}\right)\right|+\left|g\left(\bm{u}_{0},\bm{v}_{0},\bm{h}_{0},\bm{x}\right)-g\left(\bm{u}_{0},\bm{v}_{0},\bm{h}_{0},\bm{x}_{0}\right)\right|\\
 & \quad\leq\left\Vert \sum_{j=1}^{m}\bm{H}_{j}\left(\bm{h}-\bm{h}_{0}+\bm{h}^{\star},\bm{x}\right)-m\left(\bm{h}-\bm{h}_{0}\right)\left(\bm{x}-\bm{x}^{\star}\right)^{\mathsf{H}}\right\Vert \\
 & \quad\qquad+\left\Vert \sum_{j=1}^{m}\bm{H}_{j}\left(\bm{h},\bm{x}-\bm{x}_{0}+\bm{x}^{\star}\right)-m\left(\bm{h}_0-\bm{h}^{\star}\right)\left(\bm{x}-\bm{x}_{0}\right)^{\mathsf{H}}\right\Vert .
\end{align*}
Considering $g(\bm{u}_{0},\bm{v}_{0},\bm{h}_{0},\bm{x}_{0})$, one
has 
\begin{align}
 & \left|g\left(\bm{u}_{0},\bm{v}_{0},\bm{h}_{0},\bm{x}_{0}\right)\right|\nonumber\\
 & \quad\leq\left|\sum_{j=1}^{m}\left(W_{j}-\mathbb{E}\left[W_{j}\right]\right)\right|+\left|\sum_{j=1}^{m}\left(\mathbb{E}\left[W_{j}\right]-m\bm{u}^{\mathsf{H}}\left(\bm{h}_{0}-\bm{h}^{\star}\right)\left(\bm{x}_{0}-\bm{x}^{\star}\right)^{\mathsf{H}}\bm{v}\right)\right|\nonumber\\
 & \quad\leq C_{t}\sqrt{mK}+\left|\sum_{j=1}^{m}\mathbb{E}\left[\bm{u}^{\mathsf{H}}\bm{b}_{j}\bm{b}_{j}^{\mathsf{H}}\left(\bm{h}_0-\bm{h}^{\star}\right)\left(\bm{x}_0-\bm{x}^{\star}\right)^{\mathsf{H}}\bm{a}_{j}\bm{a}_{j}^{\mathsf{H}}\bm{v}\ind_{\left\{ \max\left\{ \left|\bm{b}_{j}^{\mathsf{H}}\left(\bm{h}_0-\bm{h}^{\star}\right)\right|,\left|\left(\bm{x}_0-\bm{x}^{\star}\right)^{\mathsf{H}}\bm{a}_{j}\right|\right\} >20C_{13}\frac{1}{\log^{3/2}m}\right\} }\right]\right|\nonumber\\
 & \quad= C_{t}\sqrt{mK}, \label{eq:lemma29-3}
\end{align}
where the first inequality is due to triangle inequality; the second comes from \eqref{eq:lemma29-1} and the last is because of \eqref{eq:lemma29-2}. 

\paragraph{Putting all this together.}
Let 
\[
\mathcal{S}'\coloneqq\left\{ \left(\bm{h},\bm{x}\right):\max\left\{\left|\left(\bm{x}-\bm{x}^{\star}\right)^{\mathsf{H}}\bm{a}_{j}\right|,\left|\left(\bm{h}-\bm{h}^{\star}\right)^{\mathsf{H}}\bm{b}_{j}\right|\right\}\leq20C_{13}\frac{1}{\log^{3/2}m},\max\left\{ \left\Vert \bm{h}-\bm{h}^{\star}\right\Vert _{2},\left\Vert \bm{x}-\bm{x}^{\star}\right\Vert _{2}\right\} \leq\delta\right\} .
\]
It is easy to check that $(\bm{h},5(\bm{x}-\bm{x}_{0})+\bm{x}^{\star})\in\mathcal{S}$
by using the facts that $\Vert\bm{x}-\bm{x}_{0}\Vert_{2}\leq\epsilon_{1}$
and (\ref{eq:gaussian-s-condition-1}). Hence, we have
\begin{align*}
 & \left\Vert \sum_{j=1}^{m}\bm{H}_{j}\left(\bm{h},\bm{x}-\bm{x}_{0}+\bm{x}^{\star}\right)-m\left(\bm{h}-\bm{h}^{\star}\right)\left(\bm{x}-\bm{x}_{0}\right)^{\mathsf{H}}\right\Vert \\
 &\quad \leq\frac{1}{5}\sup_{\left(\bm{h},\bm{x}\right)\in\mathcal{S}'}\left\Vert \sum_{j=1}^{m}\bm{H}_{j}\left(\bm{h},\bm{x}\right)-m\left(\bm{h}-\bm{h}^{\star}\right)\left(\bm{x}-\bm{x}^{\star}\right)^{\mathsf{H}}\right\Vert .
\end{align*}
Similarly, one has $(5(\bm{h}-\bm{h}_{0})+\bm{h}^{\star},\bm{x})\in\mathcal{S}$
and therefore,
\begin{align*}
 & \left\Vert \sum_{j=1}^{m}\bm{H}_{j}\left(\bm{h}-\bm{h}_{0}+\bm{h}^{\star},\bm{x}\right)-m\left(\bm{h}-\bm{h}_{0}\right)\left(\bm{x}-\bm{x}^{\star}\right)^{\mathsf{H}}\right\Vert \\
 &\quad \leq\frac{1}{5}\sup_{\left(\bm{h},\bm{x}\right)\in\mathcal{S}'}\left\Vert \sum_{j=1}^{m}\bm{H}_{j}\left(\bm{h},\bm{x}\right)-m\left(\bm{h}-\bm{h}^{\star}\right)\left(\bm{x}-\bm{x}^{\star}\right)^{\mathsf{H}}\right\Vert .
\end{align*}
Hence, combining the above two inequalities with (\ref{eq:geometry-eqn-1})
and (\ref{eq:lemma29-3}) reveals that 
\begin{align*}
\left|g\left(\bm{u},\bm{v},\bm{h},\bm{x}\right)\right| & \leq\left|g\left(\bm{u}_{0},\bm{v}_{0},\bm{h}_{0},\bm{x}_{0}\right)\right|+\left|g\left(\bm{u},\bm{v},\bm{h},\bm{x}\right)-g\left(\bm{u}_{0},\bm{v}_{0},\bm{h}_{0},\bm{x}_{0}\right)\right|\\
 & \leq C_{t}\sqrt{mK}+2\epsilon_{2}\left\Vert \sum_{j=1}^{m}\bm{H}_{j}\left(\bm{h},\bm{x}\right)-m\left(\bm{h}-\bm{h}^{\star}\right)\left(\bm{x}-\bm{x}^{\star}\right)^{\mathsf{H}}\right\Vert \\
 & \qquad+\frac{2}{5}\sup_{\left(\bm{h},\bm{x}\right)\in\mathcal{S}'}\left\Vert \sum_{j=1}^{m}\bm{H}_{j}\left(\bm{h},\bm{x}\right)-m\left(\bm{h}-\bm{h}^{\star}\right)\left(\bm{x}-\bm{x}^{\star}\right)^{\mathsf{H}}\right\Vert .
\end{align*}
Taking the maximum over $\bm{u}$ and $\bm{v}$ on the left-hand side of the above inequality
and rearranging terms yield
\begin{align*}
	&\left(1-2\epsilon_{2}\right)\left\Vert \sum_{j=1}^{m}\bm{H}_{j}\left(\bm{h},\bm{x}\right)-m\left(\bm{h}-\bm{h}^{\star}\right)\left(\bm{x}-\bm{x}^{\star}\right)^{\mathsf{H}}\right\Vert \\
	&\quad \leq C_{t}\sqrt{mK}+\frac{2}{5}\sup_{\left(\bm{h},\bm{x}\right)\in\mathcal{S}'}\left\Vert \sum_{j=1}^{m}\bm{H}_{j}\left(\bm{h},\bm{x}\right)-m\left(\bm{h}-\bm{h}^{\star}\right)\left(\bm{x}-\bm{x}^{\star}\right)^{\mathsf{H}}\right\Vert .
\end{align*}
Further taking the maximum over $(\bm{h},\bm{x})$ on $\mathcal{S}'$ gives
\[
\left(1-2\epsilon_{2}-\frac{2}{5}\right)\sup_{\left(\bm{h},\bm{x}\right)\in\mathcal{S}'}\left\Vert \sum_{j=1}^{m}\bm{H}_{j}\left(\bm{h},\bm{x}\right)-m\left(\bm{h}-\bm{h}^{\star}\right)\left(\bm{x}-\bm{x}^{\star}\right)^{\mathsf{H}}\right\Vert \leq C_{t}\sqrt{mK},
\]
and then rearranging terms yields
\begin{align*}
\sup_{\left(\bm{h},\bm{x}\right)\in\mathcal{S}'}\left\Vert \sum_{j=1}^{m}\bm{H}_{j}\left(\bm{h},\bm{x}\right)\right\Vert  
& \leq\sup_{\left(\bm{h},\bm{x}\right)\in\mathcal{S}'}\left\Vert m\left(\bm{h}-\bm{h}^{\star}\right)\left(\bm{x}-\bm{x}^{\star}\right)^{\mathsf{H}}\right\Vert +4C_{t}\sqrt{mK}\\
 & \leq\delta^{2}m+4C_{t}\sqrt{mK}.
\end{align*}
Recognizing that
\begin{align*}
\sup_{\left(\bm{h},\bm{x}\right)\in\mathcal{S}'}\left\Vert \sum_{j=1}^{m}\bm{H}_{j}\left(\bm{h},\bm{x}\right)\right\Vert  &=\sup_{\left(\bm{h},\bm{x}\right)\in\mathcal{S}'}\left\Vert\bm{b}_{j}\bm{b}_{j}^{\mathsf{H}}\left(\bm{h}-\bm{h}^{\star}\right)\left(\bm{x}-\bm{x}^{\star}\right)^{\mathsf{H}}\bm{a}_{j}\bm{a}_{j}^{\mathsf{H}} \right\Vert
\end{align*} and that the set of all $(\bm{h},\bm{x})$ obeying (\ref{subeq:gaussian-assumptions-geometry})
is a subset of $\mathcal{S}'$, we have established the desired result. 

\subsection{Proof of Lemma \ref{lemma:gaussian-loo}\label{subsec:Proof-of-Lemma-gaussian-loo}}

The proof is very much the same as that of Lemma \ref{lemma:proximity},
except that the contraction coefficient in the expression $\bm{\nu}_{1}$
in (\ref{eq:prox-nu12}) is $(1-c_{\eta})$ rather than $(1-\eta)$
and the bound on $\bm{\nu}_{3}$ is different. In what follows, we shall only describe
how to bound $\bm{\nu}_{3}$ here, for the sake of brevity.

The proof proceeds by bounding
$\bm{\nu}_{3}$ via the four terms as indicated by (\ref{eq:lemma-proximity-nu3-decompose}), which we discuss as follows.
\begin{enumerate}
\item For the first term $\nu_{31}$, one has
\begin{align}
\nu_{31} & \leq\left|\boldsymbol{b}_{l}^{\mathsf{H}}\widehat{\boldsymbol{h}}^{t,(l)}\widehat{\boldsymbol{x}}^{t,(l)\mathsf{H}}\boldsymbol{a}_{l}-\boldsymbol{b}_{l}^{\mathsf{H}}\boldsymbol{h}^{\star}\boldsymbol{x}^{\star\mathsf{H}}\boldsymbol{a}_{l}\right|\left\Vert \boldsymbol{b}_{l}\right\Vert _{2}\left|\boldsymbol{a}_{l}^{\mathsf{H}}\widehat{\boldsymbol{x}}^{t,(l)}\right|\nonumber \\
 & \leq\left|\boldsymbol{b}_{l}^{\mathsf{H}}\widehat{\boldsymbol{h}}^{t,(l)}\widehat{\boldsymbol{x}}^{t,(l)\mathsf{H}}\boldsymbol{a}_{l}-\boldsymbol{b}_{l}^{\mathsf{H}}\boldsymbol{h}^{\star}\boldsymbol{x}^{\star\mathsf{H}}\boldsymbol{a}_{l}\right|\cdot10\sqrt{K}\cdot20\sqrt{\log m}\cdot\big\|\widehat{\boldsymbol{x}}^{t,(l)}\big\|_{2}\nonumber \\
 & \leq400\sqrt{K\log m}\left|\boldsymbol{b}_{l}^{\mathsf{H}}\widehat{\boldsymbol{h}}^{t,(l)}\widehat{\boldsymbol{x}}^{t,(l)\mathsf{H}}\boldsymbol{a}_{l}-\boldsymbol{b}_{l}^{\mathsf{H}}\boldsymbol{h}^{\star}\boldsymbol{x}^{\star\mathsf{H}}\boldsymbol{a}_{l}\right|,\label{eq:nu31-1}
\end{align}
where the penultimate inequality follows from (\ref{eq:useful1})
and (\ref{eq:useful2}); the last inequality is due to (\ref{eq:gaussian-hat-hx-loo}). 
\item Regarding $\nu_{32}$, one has
\begin{align}
\nu_{32} & \leq\left|\boldsymbol{b}_{l}^{\mathsf{H}}\widehat{\boldsymbol{h}}^{t,(l)}\widehat{\boldsymbol{x}}^{t,(l)\mathsf{H}}\boldsymbol{a}_{l}-\boldsymbol{b}_{l}^{\mathsf{H}}\boldsymbol{h}^{\star}\boldsymbol{x}^{\star\mathsf{H}}\boldsymbol{a}_{l}\right|\left\Vert \bm{a}_{l}\right\Vert _{2}\left|\bm{b}_{l}^{\mathsf{H}}\widehat{\bm{h}}^{t,\left(l\right)}\right|\nonumber \\
 & \leq\left|\boldsymbol{b}_{l}^{\mathsf{H}}\widehat{\boldsymbol{h}}^{t,(l)}\widehat{\boldsymbol{x}}^{t,(l)\mathsf{H}}\boldsymbol{a}_{l}-\boldsymbol{b}_{l}^{\mathsf{H}}\boldsymbol{h}^{\star}\boldsymbol{x}^{\star\mathsf{H}}\boldsymbol{a}_{l}\right|\cdot10\sqrt{K}\cdot20\sqrt{\log m}\left\Vert \widehat{\bm{h}}^{t,\left(l\right)}\right\Vert _{2}\nonumber \\
 & \leq400\sqrt{K\log m}\left|\boldsymbol{b}_{l}^{\mathsf{H}}\widehat{\boldsymbol{h}}^{t,(l)}\widehat{\boldsymbol{x}}^{t,(l)\mathsf{H}}\boldsymbol{a}_{l}-\boldsymbol{b}_{l}^{\mathsf{H}}\boldsymbol{h}^{\star}\boldsymbol{x}^{\star\mathsf{H}}\boldsymbol{a}_{l}\right|,\label{eq:nu32-1}
\end{align}
where the second line follows from (\ref{eq:useful1}) and (\ref{eq:useful2});
the last inequality is due to (\ref{eq:gaussian-hat-hx-loo}). Further
for some sufficiently large constant $C>0$, there holds
\begin{align}
\left|\bm{b}_{l}^{\mathsf{H}}\big(\widehat{\boldsymbol{h}}^{t,(l)}-\bm{h}^{\star}\big)\right| & \le20\sqrt{\log m}\left\Vert \widehat{\boldsymbol{h}}^{t,(l)}-\bm{h}^{\star}\right\Vert _{2}\nonumber \\
 & \leq20\sqrt{\log m}\left(\left\Vert \widehat{\boldsymbol{h}}^{t,(l)}-\widetilde{\bm{h}}^{t}\right\Vert _{2}+\left\Vert \widetilde{\bm{h}}^{t}-\bm{h}^{\star}\right\Vert _{2}\right)\nonumber \\
 & \leq C\left(\frac{\sqrt{mK\log^{3}m}}{m}+\frac{\sigma\sqrt{K\log^{2}m}}{m}\right),\label{eq:gaussian-eqn-nu3-1}
\end{align}
where the last inequality comes from (\ref{eq:gaussian-hypothesis-loo})
and (\ref{eq:gaussian-dist-bound}). Similarly we can see this bound
also holds for $|(\widehat{\boldsymbol{x}}^{t,(l)}-\bm{x}^{\star})^{\mathsf{H}}\boldsymbol{a}_{l}|$.
Therefore, 
\begin{align}
 & \left|\boldsymbol{b}_{l}^{\mathsf{H}}\widehat{\boldsymbol{h}}^{t,(l)}\widehat{\boldsymbol{x}}^{t,(l)\mathsf{H}}\boldsymbol{a}_{l}-\boldsymbol{b}_{l}^{\mathsf{H}}\boldsymbol{h}^{\star}\boldsymbol{x}^{\star\mathsf{H}}\boldsymbol{a}_{l}\right|\nonumber \\
 & \quad\leq\left|\boldsymbol{b}_{l}^{\mathsf{H}}\widehat{\boldsymbol{h}}^{t,(l)}\left(\widehat{\boldsymbol{x}}^{t,(l)}-\bm{x}^{\star}\right)^{\mathsf{H}}\boldsymbol{a}_{l}\right|+\left|\boldsymbol{b}_{l}^{\mathsf{H}}\big(\widehat{\boldsymbol{h}}^{t,(l)}-\bm{h}^{\star}\big)\bm{x}^{\star\mathsf{H}}\boldsymbol{a}_{l}\right|\nonumber \\
 & \quad\leq\left(\left|\boldsymbol{b}_{l}^{\mathsf{H}}\big(\widehat{\boldsymbol{h}}^{t,(l)}-\bm{h}^{\star}\big)\right|+\left|\boldsymbol{b}_{l}^{\mathsf{H}}\bm{h}^{\star}\right|\right)\cdot\left|\left(\widehat{\boldsymbol{x}}^{t,(l)}-\bm{x}^{\star}\right)^{\mathsf{H}}\boldsymbol{a}_{l}\right|+\left|\bm{b}_{l}^{\mathsf{H}}\big(\widehat{\boldsymbol{h}}^{t,(l)}-\bm{h}^{\star}\big)\right|\cdot\left|\bm{x}^{\star\mathsf{H}}\boldsymbol{a}_{l}\right|\nonumber \\
 & \quad\leq\left(C\left(\frac{\sqrt{mK\log^{3}m}}{m}+\frac{\sigma\sqrt{mK\log^{2}m}}{m}\right)+20\sqrt{\log m}\right)\cdot C\left(\frac{\sqrt{mK\log^{3}m}}{m}+\frac{\sigma\sqrt{K\log^{2}m}}{m}\right)\nonumber \\
 & \quad\qquad+C\left(\frac{\sqrt{mK\log^{3}m}}{m}+\frac{\sigma\sqrt{K\log^{2}m}}{m}\right)\cdot20\sqrt{\log m}\\
 & \quad\lesssim\frac{\sqrt{mK\log^{4}m}}{m}+\frac{\sigma\sqrt{mK\log^{3}m}}{m},\label{eq:diff-loo-1}
\end{align}
where the penultimate inequality follows from (\ref{eq:useful1})
and (\ref{eq:gaussian-eqn-nu3-1}). Substituting (\ref{eq:diff-loo-1})
into (\ref{eq:nu31-1}) and (\ref{eq:nu32-1}), we reach\begin{subequations}
\begin{align}
\nu_{31}+\nu_{32} & \lesssim\sqrt{K\log m}\cdot\left(\frac{\sqrt{mK\log^{4}m}}{m}+\frac{\sigma\sqrt{mK\log^{3}m}}{m}\right)\nonumber \\
 & \leq\frac{K\sqrt{m\log^{5}m}}{m}+\frac{\sigma K\sqrt{m\log^{4}m}}{m}.\label{eq:lem-loo-2-1}
\end{align}

\item Regarding $\nu_{33}$ and $\nu_{34}$, it can be seen that
\begin{align}
\nu_{33}=\left\Vert \xi_{l}\boldsymbol{b}_{l}\boldsymbol{a}_{l}^{\mathsf{H}}\widehat{\boldsymbol{x}}^{t,(l)}\right\Vert _{2} & \leq\left|\xi_{l}\right|\left\Vert \boldsymbol{b}_{l}\right\Vert _{2}\left|\boldsymbol{a}_{l}^{\mathsf{H}}\widehat{\boldsymbol{x}}^{t,(l)}\right|\overset{\text{(i)}}{\lesssim}\sigma\sqrt{K}\big\|\widehat{\boldsymbol{x}}^{t,(l)}\big\|_{2}\log m\leq2\sigma\sqrt{K}\log m,\label{eq:gaussian-loo-3}\\
\nu_{34}=\left\Vert \overline{\xi_{l}}\bm{a}_{l}\bm{b}_{l}^{\mathsf{H}}\widehat{\bm{h}}^{t,\left(l\right)}\right\Vert _{2} & \leq\left|\xi_{l}\right|\left\Vert \bm{a}_{l}\right\Vert _{2}\left|\bm{b}_{l}^{\mathsf{H}}\widehat{\bm{h}}^{t,\left(l\right)}\right|\overset{\text{(i)}}{\lesssim}\sigma\sqrt{K}\big\|\widehat{\boldsymbol{h}}^{t,(l)}\big\|_{2}\log m\leq2\sigma\sqrt{K}\log m,\label{eq:gaussian-loo-4}
\end{align}
\end{subequations}where (i) holds by (\ref{eq:useful1}), (\ref{eq:useful2})
and the independence between $\xi_{l},\bm{a}_{l}$, $\bm{b}_{l}$
and $\widehat{\boldsymbol{x}}^{t,(l)}$. 

Consequently, by \eqref{eq:nu31-1} and (\ref{eq:lem-loo-2-1})-(\ref{eq:gaussian-loo-4})
we have
\begin{equation}
\left\Vert \boldsymbol{\nu}_{3}\right\Vert _{2}\lesssim\frac{K\sqrt{m\log^{5}m}}{m}+\sigma\sqrt{K}\log m,\label{eq:prox-nu3-1-1-1}
\end{equation}
as long as $m\gg K\log^{2}m$. Then the proof follows the same line
of idea as Appendix \ref{subsec:Proof-of-Lemmaproximity},  resulting in a similar inequality as (\ref{eq:lemmaproximity-3}) as follows:
\begin{align*}
\mathsf{dist}\big(\boldsymbol{z}^{t+1,\left(l\right)},\widetilde{\boldsymbol{z}}^{t+1}\big) & \leq\left(1-c_{\eta}\right)\mathsf{dist}\big(\boldsymbol{z}^{t,\left(l\right)},\widetilde{\boldsymbol{z}}^{t}\big)+\eta C\left(\frac{K\sqrt{m\log^{5}m}}{m}+\sigma\sqrt{K}\log m\right)\\
 & \leq C\left(\frac{\sqrt{K\log^{3}m}}{m}+\frac{\sigma\sqrt{K\log^{2}m}}{m}\right),
\end{align*}
 provided that $\eta=c_{\eta}/m$ with $c_{\eta}>0$ being some sufficiently
small constant. The proof for (\ref{eq:gaussian-loo-2}) follows from the same argument
leading to (\ref{eq:lemproxlast}) and is thus omitted here for simplicity. 
\end{enumerate}

\subsection{Proof of Lemma \ref{lemma:gaussian-loo-initialization}\label{subsec:Proof-of-Lemma-gaussian-loo-initialization}}

Recall the definition of $\bm{M}$ and $\bm{M}^{(l)}$ under the Gaussian
design:
\[
\bm{M}\coloneqq\frac{1}{m}\sum_{j=1}^{m}y_{j}\bm{b}_{j}\bm{a}_{j}^{\mathsf{H}},\qquad\text{and}\qquad\bm{M}^{(l)}\coloneqq\frac{1}{m}\sum_{j\neq l}y_{j}\bm{b}_{j}\bm{a}_{j}^{\mathsf{H}}.
\]
 Applying Wedin's sin$\Theta$ theorem \citet[Theorem 2.1]{dopico2000note}
gives that for some universal constant $C'>0$, there holds
\[
\min_{\alpha\in\mathbb{C},\left|\alpha\right|=1}\left\{ \left\Vert \alpha\check{\bm{h}}^{0}-\check{\bm{h}}^{0,\left(l\right)}\right\Vert _{2}+\left\Vert \alpha\check{\bm{x}}^{0}-\check{\bm{x}}^{0,\left(l\right)}\right\Vert _{2}\right\} \leq C'\frac{\left\Vert \left(\bm{M}-\bm{M}^{\left(l\right)}\right)\check{\bm{x}}^{0,\left(l\right)}\right\Vert _{2}+\left\Vert \check{\bm{h}}^{0,\left(l\right)\mathsf{H}}\left(\bm{M}-\bm{M}^{\left(l\right)}\right)\right\Vert _{2}}{\sigma_{1}\left(\bm{M}^{\left(l\right)}\right)-\sigma_{2}\left(\bm{M}\right)}.
\]
By invoking Weyl's inequality, we obtain
\begin{align*}
\sigma_{1}\left(\bm{M}^{\left(l\right)}\right)-\sigma_{2}\left(\bm{M}\right) & \geq\sigma_{1}\left(\mathbb{E}\left[\bm{M}^{\left(l\right)}\right]\right)-\left\Vert \bm{M}^{\left(l\right)}-\mathbb{E}\left[\bm{M}^{\left(l\right)}\right]\right\Vert -\sigma_{2}\left(\mathbb{E}\left[\bm{M}\right]\right)-\left\Vert \bm{M}-\mathbb{E}\left[\bm{M}\right]\right\Vert \\
 & \overset{\text{(i)}}{\geq}\frac{3}{4}-\left\Vert \bm{M}^{\left(l\right)}-\mathbb{E}\left[\bm{M}^{\left(l\right)}\right]\right\Vert -\left\Vert \bm{M}-\mathbb{E}\left[\bm{M}\right]\right\Vert \overset{\text{(ii)}}{\geq}\frac{1}{2},
\end{align*}
where (i) is due to the facts that 
\begin{align*}
\sigma_{1}\left(\mathbb{E}\left[\bm{M}^{\left(l\right)}\right]\right)=\sigma_{1}\left(\frac{m-1}{m}\bm{h}^{\star}\bm{x}^{\star\mathsf{H}}\right)\geq\frac{3}{4},\qquad\text{and}\qquad\sigma_{2}\left(\mathbb{E}\left[\bm{M}\right]\right)=\sigma_{2}\left(\bm{h}^{\star}\bm{x}^{\star\mathsf{H}}\right)=0,
\end{align*}
and (ii) comes from Lemma \ref{lemma:gaussian-M}. Hence, one has
\begin{equation}
\min_{\alpha\in\mathbb{C},\left|\alpha\right|=1}\left\{ \left\Vert \alpha\check{\bm{h}}^{0}-\check{\bm{h}}^{0,\left(l\right)}\right\Vert _{2}+\left\Vert \alpha\check{\bm{x}}^{0}-\check{\bm{x}}^{0,\left(l\right)}\right\Vert _{2}\right\} \leq2C'\left(\left\Vert \left(\bm{M}-\bm{M}^{\left(l\right)}\right)\check{\bm{x}}^{0,\left(l\right)}\right\Vert _{2}+\left\Vert \check{\bm{h}}^{0,\left(l\right)\mathsf{H}}\left(\bm{M}-\bm{M}^{\left(l\right)}\right)\right\Vert _{2}\right).\label{eq:gaussian-base-norm-1}
\end{equation}
We are left with bounding the two terms on the right-hand side of (\ref{eq:gaussian-base-norm-1}). 
\begin{itemize}
\item Regarding the first term on the right-hand side of (\ref{eq:gaussian-base-norm-1}), we have 
\begin{align*}
\left\Vert \left(\bm{M}-\bm{M}^{\left(l\right)}\right)\check{\bm{x}}^{0,\left(l\right)}\right\Vert _{2} & =\left\Vert \frac{1}{m}\bm{b}_{l}\left(\bm{b}_{l}^{\mathsf{H}}\bm{h}^{\star}\bm{x}^{\star\mathsf{H}}\bm{a}_{l}+\xi_{l}\right)\bm{a}_{l}^{\mathsf{H}}\check{\bm{x}}^{0,\left(l\right)}\right\Vert _{2}\\
 & \leq\left\Vert \frac{1}{m}\bm{b}_{l}\bm{b}_{l}^{\mathsf{H}}\bm{h}^{\star}\bm{x}^{\star\mathsf{H}}\bm{a}_{l}\bm{a}_{l}^{\mathsf{H}}\check{\bm{x}}^{0,\left(l\right)}\right\Vert _{2}+\left\Vert \frac{1}{m}\xi_{l}\bm{b}_{l}\bm{a}_{l}^{\mathsf{H}}\check{\bm{x}}^{0,\left(l\right)}\right\Vert _{2}\\
 & =\frac{1}{m}\left\Vert \bm{b}_{l}\right\Vert _{2}\left|\bm{b}_{l}^{\mathsf{H}}\bm{h}^{\star}\right|\left|\bm{x}^{\star\mathsf{H}}\bm{a}_{l}\right|\left|\bm{a}_{l}^{\mathsf{H}}\check{\bm{x}}^{0,\left(l\right)}\right|+\frac{1}{m}\left|\xi_{l}\right|\left|\bm{a}_{l}^{\mathsf{H}}\check{\bm{x}}^{0,\left(l\right)}\right|\left\Vert \bm{b}_{l}\right\Vert _{2}\\
 & \leq\frac{1}{m}\cdot10\sqrt{K}\cdot\left(20\sqrt{\log m}\right)^{2}\cdot20\sqrt{\log m}+\frac{1}{m}\cdot20\sigma\sqrt{\log m}\cdot20\sqrt{\log m}\cdot10\sqrt{K}\\
 & \lesssim\frac{\sqrt{K\log^{3}m}}{m}+\frac{\sigma\sqrt{K\log^{2}m}}{m},
\end{align*}
where the second inequality is due to the triangle inequality; the
penultimate inequality comes from (\ref{eq:useful1}), (\ref{eq:useful2})
and the fact that with probability exceeding $1-O(m^{-100})$,
\[
\max_{1\leq l\leq m}\left|\bm{a}_{l}^{\mathsf{H}}\check{\bm{x}}^{0,\left(l\right)}\right|\leq20\sqrt{\log m},
\]
due to the independence between $\check{\bm{x}}^{0,\left(l\right)}$
and $\bm{a}_{l}$.
\item The second term on the right-hand side of (\ref{eq:gaussian-base-norm-1}) can be bounded in a similar fashion as follows
\begin{align*}
\left\Vert \check{\bm{h}}^{0,\left(l\right)\mathsf{H}}\left(\bm{M}-\bm{M}^{\left(l\right)}\right)\right\Vert _{2} & =\frac{1}{m}\left\Vert \check{\bm{h}}^{0,\left(l\right)\mathsf{H}}\bm{b}_{l}\left(\bm{b}_{l}^{\mathsf{H}}\bm{h}^{\star}\bm{x}^{\star\mathsf{H}}\bm{a}_{l}+\xi_{l}\right)\bm{a}_{l}^{\mathsf{H}}\right\Vert _{2}\\
 & \leq\frac{1}{m}\left\Vert \check{\bm{h}}^{0,\left(l\right)\mathsf{H}}\bm{b}_{l}\bm{b}_{l}^{\mathsf{H}}\bm{h}^{\star}\bm{x}^{\star\mathsf{H}}\bm{a}_{l}\bm{a}_{l}^{\mathsf{H}}\right\Vert _{2}+\frac{1}{m}\left\Vert \xi_{l}\check{\bm{h}}^{0,\left(l\right)\mathsf{H}}\bm{b}_{l}\bm{a}_{l}^{\mathsf{H}}\right\Vert _{2}\\
 & =\frac{1}{m}\cdot\left|\check{\bm{h}}^{0,\left(l\right)\mathsf{H}}\bm{b}_{l}\right|\left|\bm{b}_{l}^{\mathsf{H}}\bm{h}^{\star}\right|\left|\bm{x}^{\star\mathsf{H}}\bm{a}_{l}\right|\left\Vert \bm{a}_{l}^{\mathsf{H}}\right\Vert _{2}+\frac{1}{m}\left|\xi_{l}\right|\left|\check{\bm{h}}^{0,\left(l\right)\mathsf{H}}\bm{b}_{l}\right|\left\Vert \bm{a}_{l}^{\mathsf{H}}\right\Vert _{2}\\
 & \leq\frac{1}{m}\cdot20\sqrt{\log m}\cdot\left(20\sqrt{\log m}\right)^{2}\cdot10\sqrt{K}+\frac{1}{m}\cdot20\sigma\sqrt{\log m}\cdot20\sqrt{\log m}\cdot10\sqrt{K}\\
 & \lesssim\frac{\sqrt{K\log^{3}m}}{m}+\frac{\sigma\sqrt{K\log^{2}m}}{m},
\end{align*}
where the penultimate inequality comes from (\ref{eq:useful1}), (\ref{eq:useful2})
and the fact that 
\[
\max_{1\leq l\leq m}\left|\check{\bm{h}}^{0,\left(l\right)\mathsf{H}}\bm{b}_{l}\right|\leq20\sqrt{\log m}
\]
		holds with probability exceeding $1-O(m^{-100})$ (due to the independence between $\check{\bm{h}}^{0,\left(l\right)}$
		and $\bm{b}_{l}$).
\end{itemize}
Plugging the above two bounds into (\ref{eq:gaussian-base-norm-1})
leads to
\[
\min_{\alpha\in\mathbb{C},\left|\alpha\right|=1}\left\{ \left\Vert \alpha\check{\bm{h}}^{0}-\check{\bm{h}}^{0,\left(l\right)}\right\Vert _{2}+\left\Vert \alpha\check{\bm{x}}^{0}-\check{\bm{x}}^{0,\left(l\right)}\right\Vert _{2}\right\} \leq\widetilde{C}\left(\frac{\sqrt{K\log^{3}m}}{m}+\frac{\sigma\sqrt{K\log^{2}m}}{m}\right),
\]
for some universal constant $\widetilde{C}>0$. To convert this bound
into the desired version, we can employ the same argument connecting \citet[Eqn (240)]{ma2017implicit}
to \citet[Eqn (245)]{ma2017implicit}. The details are omitted here
for brevity.

\section{Analysis under Gaussian design: connections between convex and nonconvex solutions\label{appendix:Proof-of-Theorem-gaussian-cvx}}

\subsection{Preliminaries}\label{subsec:gaussian-preliminary}
Here, we state below a few elementary technical lemmas that prove useful in the proof. 
%
To begin with, we show that the operator $\mathcal{A}$ is well-controlled
in this case, whose counterpart in the Fourier design is Lemma \ref{lemma:normbound}.

\begin{lemma}\label{lemma:gaussian-operator}For the operator $\mathcal{A}$ defined
under the Gaussian setting, we have, with probability at least $1-O(m^{-10})$, that
\[
\left\Vert \mathcal{A}\right\Vert \leq10\sqrt{mK\log m}.
\]
\end{lemma}

\begin{proof}Denote 
\[
\bm{A}\coloneqq\left[\begin{array}{c}
\bm{a}_{1}^{\top}\\
\vdots\\
\bm{a}_{m}^{\top}
\end{array}\right]\in\mathbb{C}^{m\times K},\qquad\bm{B}\coloneqq\left[\begin{array}{c}
\bm{b}_{1}^{\top}\\
\vdots\\
\bm{b}_{m}^{\top}
\end{array}\right]\in\mathbb{C}^{m\times K}.
\]
We can rewrite $\mathcal{A}$ in matrix form as follows
\[
\mathcal{A}\left(\bm{Z}\right)=\left\{ \bm{b}_{j}^{\mathsf{H}}\bm{Z}\bm{a}_{j}\right\} _{j=1}^{m}=\left[\begin{array}{cccc}
\text{diag}\left(\bm{A}_{:,1}\right)\bm{B} & \text{diag}\left(\bm{A}_{:,2}\right)\bm{B} & \cdots & \text{diag}\left(\bm{A}_{:,K}\right)\bm{B}\end{array}\right]\text{vec}\left(\bm{Z}\right).
\]
This allows one to express and obtain
\begin{align*}
\left\Vert \mathcal{A}\right\Vert ^{2} & =\left\Vert \left[\begin{array}{cccc}
\text{diag}\left(\bm{A}_{:,1}\right)\bm{B} & \text{diag}\left(\bm{A}_{:,2}\right)\bm{B} & \cdots & \text{diag}\left(\bm{A}_{:,K}\right)\bm{B}\end{array}\right]\right\Vert ^{2}\\
 & \leq\left\Vert \bm{B}\right\Vert ^{2}\cdot\sum_{i=1}^{K}\left\Vert \text{diag}\left(\bm{A}_{:,i}\right)\right\Vert ^{2}\\
 & \leq\left\Vert \sum_{j=1}^{m}\bm{b}_{j}\bm{b}_{j}^{\mathsf{H}}\right\Vert \cdot K\max_{1\leq i\leq K}\max_{1\leq j\leq m}\left|\bm{A}_{i,j}\right|^{2}\\
 & \leq2m\cdot K\cdot20\log m
\end{align*}
with probability at least $1-O(m^{-100})$.
\end{proof}

Next, the following lemma corresponds to Lemma \ref{lem:strengthen}
under the Fourier design. Its proof is deferred to Appendix \ref{subsec:Proof-of-Lemma-gaussian-projection}. 

\begin{lemma}\label{lemma:gaussian-projection}Suppose that $T$
is the tangent space of $\bm{h}\bm{x}^{\mathsf{H}}$ with $\Vert\bm{h}\Vert_{2}=\Vert\bm{x}\Vert_{2}=1$
and $m\geq CK\log^{2}m$ for some sufficiently large constant $C>0$.
Then there exists some sufficiently large constant $C'>0$
such that 
\[
\left\Vert \mathcal{P}_{T}\mathcal{A}^{*}\mathcal{A}\mathcal{P}_{T}-m\mathcal{P}_{T}\right\Vert \leq C'\sqrt{mK\log m}
\]
 with probability exceeding $1-O(m^{-10})$.
\end{lemma}

\subsection{Proof of Theorem \ref{theorem:gaussian-cvx}}
In this section, we proceed to prove Theorem \ref{theorem:gaussian-cvx}
by connecting the convex minimizer with nonconvex iterates, in the
same vein as in the Fourier design case (cf.~Appendix \ref{appendix:cvx}). 
To begin with, a lemma
stating the results of Algorithm \ref{alg:gd-bd-ncvx-2} under the
Gaussian design is listed below. 
\begin{lemma}Take $\lambda=C_{\lambda}\sigma\sqrt{mK\log m}$ for
	some sufficiently large constant $C_{\lambda}>0$. Suppose that Assumption \ref{assumptions:models-gausssian} holds. Assume the number
	of measurements obeys $m\geq CK\log^{6}m$ for some sufficiently
	large constant $C>0$ and the noise satisfies $\sigma\sqrt{K\log^5/ m}\leq c$
	for some sufficiently small constant $c>0$. Let stepsize $\eta$ be $c_{\eta}/m$ for some sufficiently small constant $c_{\eta}>0$. Then, with probability
	at least $1-O\left(m^{-100}+me^{-K}\right)$,
	the iterates $\left\{ \boldsymbol{h}^{t},\boldsymbol{x}^{t}\right\} _{0<t\leq t_{0}}$
	of Algorithm \ref{alg:gd-bd-ncvx-2} satisfy \begin{subequations}
		\begin{align}
		\mathsf{dist}\left(\boldsymbol{z}^{t},\boldsymbol{z}^{\star}\right) & \leq\rho\mathsf{dist}\left(\bm{z}^{t-1},\bm{z}^{\star}\right)+C_{11}\eta\left(\lambda+\sigma\sqrt{mK\log m}\right),\label{eq:gaussian-cvx-dist}\\
		\mathsf{dist}\big(\boldsymbol{z}^{t,\left(l\right)},\widetilde{\bm{z}}^{t}\big) & \leq C_{12}\frac{\sigma\sqrt{K\log^{2}m}}{m},\label{eq:gaussian-cvx-loo}\\
		\max_{1\leq l\leq m}\left\Vert \widetilde{\bm{z}}^{t,\left(l\right)}-\widetilde{\bm{z}}^{t}\right\Vert _{2} & \lesssim C_{12}\frac{\sigma\sqrt{K\log^{2}m}}{m},\label{eq:gaussian-cvx-loo-2}\\
		\max_{1\leq l\leq m}\left|\boldsymbol{a}_{l}^{\mathsf{H}}\left(\widetilde{\bm{x}}^{t}-\boldsymbol{x}^{\star}\right)\right| & \leq C_{13}\frac{\sigma\sqrt{mK\log^{2}m}}{m},\label{eq:gaussian-cvx-incoh-a}\\
		\max_{1\leq l\leq m}\big|\boldsymbol{b}_{l}^{\mathsf{H}}\left(\widetilde{\bm{h}}^{t}-\bm{h}^{\star}\right)\big| & \leq C_{13}\frac{\sigma\sqrt{mK\log^{2}m}}{m} \label{eq:gaussian-cvx-incoh-b}
		\end{align}
		for any $0<t\leq t_{0}$, where $\rho=1-c_{\rho}c_{\eta}$
		for some small constant $c_{\rho}>0$, and we take $t_{0}=m^{20}$.
		Here, $C_{11}$, $C_{12}$ and $C_{13}$ are positive constants.
		Additionally, one has 
		\begin{equation}
		\min_{0\leq t\leq t_{0}}\left\Vert \nabla f\left(\boldsymbol{h}^{t},\boldsymbol{x}^{t}\right)\right\Vert _{2}\leq\frac{\lambda}{m^{10}}.\label{eq:gaussian-cvx-small-gradient}
		\end{equation}
		
	\end{subequations}
	
\end{lemma}

(\ref{eq:gaussian-cvx-dist})-(\ref{eq:gaussian-cvx-incoh-b}) can
be seen as direct consequences from our analysis in Appendix \ref{appendix:gaussian-ncvx},
while (\ref{eq:gaussian-cvx-small-gradient}) can be derived by following
the proof in Appendix \ref{subsec:Proof-of-the-small-gradient}.
Hence, we do not repeat the proof here for brevity. 

Similar to Conditions \ref{condition:fourier-lambda} and \ref{condition:fourier-inj}, we single out two critical conditions on the operators under Assumption \ref{assumptions:models-gausssian}. The first condition below requires the regularization parameter $\lambda$ to be large enough, so as to dominate a certain form of noise and the deviation of $\mathcal{T}\left(\boldsymbol{h}\boldsymbol{x}^{\mathsf{H}}-\boldsymbol{h}^{\star}\boldsymbol{x}^{\star\mathsf{H}}\right)$ from its mean $m\left(\boldsymbol{h}\boldsymbol{x}^{\mathsf{H}}-\boldsymbol{h}^{\star}\boldsymbol{x}^{\star\mathsf{H}}\right)$. 
\begin{condition}\label{condition:gaussian-lambda}
	The regularization parameter $\lambda$ satisfies
	\begin{enumerate}
		\item $\left\Vert \mathcal{T}\left(\boldsymbol{h}\boldsymbol{x}^{\mathsf{H}}-\boldsymbol{h}^{\star}\boldsymbol{x}^{\star\mathsf{H}}\right)-m\left(\boldsymbol{h}\boldsymbol{x}^{\mathsf{H}}-\boldsymbol{h}^{\star}\boldsymbol{x}^{\star\mathsf{H}}\right)\right\Vert <\lambda/8.$
		\item $\left\Vert \mathcal{A}^{*}\left(\bm{\xi}\right)\right\Vert \leq c\lambda$ for some small constant $c>0$.
	\end{enumerate}
\end{condition}
The second condition is concerned with the injectivity property of $\mathcal{A}$.
\begin{condition}\label{condition:gaussian-inj}
	Let $T$ be the tangent space of $\boldsymbol{h}\boldsymbol{x}^{\mathsf{H}}$. Then for all $\boldsymbol{Z}\in T$, one has
	$$\left\Vert \mathcal{A}\left(\boldsymbol{Z}\right)\right\Vert _{2}^{2}\geq\frac{m}{16}\left\Vert \boldsymbol{Z}\right\Vert _{\mathrm{F}}^{2}.$$
\end{condition}

Armed with these two conditions, the following lemma reveals how an approximate nonconvex optimizer can serve as a proxy of the convex minimizer. The proof of this lemma can be developed in the same manner as in Appendix \ref{subsec:Proof-of-Lemmaproximity}; the details are omitted here for brevity.
\begin{lemma}\label{lemma:gaussian-proximity-cvx-ncvx}Suppose that $\left(\boldsymbol{h},\boldsymbol{x}\right)$
	obeys\begin{subequations}\label{eq:gaussian-hx-properties}
		\begin{equation}
			\left\Vert \nabla f\left(\boldsymbol{h},\boldsymbol{x}\right)\right\Vert _{2}\leq C\frac{\lambda}{m^{10}}
			\label{eq:gaussian-small-gradient}
		\end{equation}
	\end{subequations}for some constants $C>0$. Then under Conditions \ref{condition:gaussian-lambda} and \ref{condition:gaussian-inj}, 
	any minimizer $\boldsymbol{Z}_{\mathsf{cvx}}$ of the convex problem
	(\ref{eq:objcvx}) satisfies
	\[
	\left\Vert \boldsymbol{h}\boldsymbol{x}^{\mathsf{H}}-\boldsymbol{Z}_{\mathsf{cvx}}\right\Vert _{\mathrm{F}}\lesssim\left\Vert \nabla f\left(\boldsymbol{h},\boldsymbol{x}\right)\right\Vert _{2}.
	\]
\end{lemma}
Consequently, the conclusions in Theorem \ref{theorem:gaussian-cvx} can be easily derived from Lemma \ref{lemma:gaussian-proximity-cvx-ncvx} by similar calculations as proof of Theorem \ref{theorem:cvx} in Appendix \ref{sec:Proof-outline-for-theorem-cvx}, and thus omitted here for brevity. 

It remains to demonstrate that Conditions \ref{condition:gaussian-lambda} and \ref{condition:gaussian-inj} hold with high probability under the sample size and noise level conditions  \eqref{eq:gaussian-assumptions-sample-noise}. We start with the first point in Condition \ref{condition:gaussian-lambda}.  Its proof can be directly adaptated from the proof in Appendix  \ref{subsec:Proof-of-Lemma8innoisy}, and thus omitted here for simplicity.
\begin{lemma}\label{lemma:T-uniform-mean-gaussian}Suppose that the
	sample complexity satisfies $m\geq CK\log^{4}m$ for some sufficiently
	large constant $C>0$. Take $\lambda=C_{\lambda}\sigma\sqrt{mK\log m}$
	for some large enough constant $C_{\lambda}>0$. Then with probability
	at least $1-O\left(m^{-10}+me^{-CK}\right)$, we have
	\[
	\left\Vert \mathcal{T}\left(\boldsymbol{h}\boldsymbol{x}^{\mathsf{H}}-\boldsymbol{h}^{\star}\boldsymbol{x}^{\star\mathsf{H}}\right)-m\left(\boldsymbol{h}\boldsymbol{x}^{\mathsf{H}}-\boldsymbol{h}^{\star}\boldsymbol{x}^{\star\mathsf{H}}\right)\right\Vert <\lambda/8 
	\]
	simultaneously for any $\left(\boldsymbol{h},\boldsymbol{x}\right)$
	obeying (\ref{eq:hx-properties}) and (\ref{eq:hx-properties2}).
\end{lemma}
The next lemma corresponds to the second point in Condition \ref{condition:gaussian-lambda}. 
\begin{lemma}\label{lemma:gaussian-noise}
	Suppose that Assumption \ref{assumptions:models-gausssian} holds and $m\geq CK\log^{5}m$ for some sufficiently large constant
	$C>0$. Then one has 
	\begin{equation}
	\left\Vert \mathcal{A}^{*}\left(\bm{\xi}\right)\right\Vert \lesssim\sigma\sqrt{mK\log m} \label{eq:gaussian-noise}
	\end{equation}
	holds with probability exceeding $1-O(m^{-10})$. 
	
\end{lemma}\begin{proof}See Appendix \ref{subsec:Proof-of-Lemmagaussian-noise}.
\end{proof}
Turning attention to Condition \ref{condition:gaussian-inj}, we have the following lemma. 
\begin{lemma}\label{lemma:gaussian-inj}Suppose that the sample complexity
	satisfies $m\geq CK\log m$ for some sufficiently large constant $C>0$.
	Then with probability at least $1-O\left(m^{-10}\right)$,
	\[
	\left\Vert \mathcal{A}\left(\boldsymbol{Z}\right)\right\Vert _{2}^{2}\geq\frac{m}{16}\left\Vert \boldsymbol{Z}\right\Vert _{\mathrm{F}}^{2},\quad\forall\boldsymbol{Z}\in T
	\]
	holds simultaneously for all $T$ for which the associated point $\left(\boldsymbol{h},\boldsymbol{x}\right)$
	obeys (\ref{eq:hx-properties}) and (\ref{eq:hx-properties2}). Here,
	$T$ denotes the tangent space of $\boldsymbol{h}\boldsymbol{x}^{\mathsf{H}}$.\end{lemma}
The proof is a direct adaptation from Appendix \ref{subsec:Proof-of-Lemma-inj} and thus omitted for brevity.

\subsection{Proof of Lemma \ref{lemma:gaussian-projection}\label{subsec:Proof-of-Lemma-gaussian-projection}}
The framework and notation adopted here are similar to \citet[Section 5.2]{ahmed2013blind}. To facilitate the proof, we introduce an operator for $\bm{x}_{1},\bm{x}_{2},\bm{y}_{1},\bm{y}_{2}\in\mathbb{C}^{K}$ as follows:
\[
\bm{x}_{1}\bm{y}_{1}^{\mathsf{H}}\otimes\bm{x}_{2}\bm{y}_{2}^{\mathsf{H}}\coloneqq\left\{ \overline{y_{1i}}y_{1k}\bm{x}_{1}\bm{x}_{2}^{\mathsf{H}}\right\} _{i,k}\in\mathbb{C}^{K^{2}\times K^{2}}.
\]
Denote by $\bm{v}_{j}=\left\langle \bm{h},\bm{b}_{j}\right\rangle \bm{a}_{j}$
and $\bm{u}_{j}=\left\langle \bm{x},\bm{a}_{j}\right\rangle \left(\bm{I}_{K}-\bm{h}\bm{h}^{\mathsf{H}}\right)\bm{b}_{j}$.
Then we can rewrite the operator $\mathcal{P}_{T}\mathcal{A}^{*}\mathcal{A}\mathcal{P}_{T}:\mathbb{C}^{K\times K}\rightarrow\mathbb{C}^{K\times K}$ as the following matrix
\[
\mathcal{Q}\coloneqq\sum_{j=1}^{m}\left(\bm{h}\bm{v}_{j}^{\mathsf{H}}\otimes\bm{h}\bm{v}_{j}^{\mathsf{H}}+\bm{h}\bm{v}_{j}^{\mathsf{H}}\otimes\bm{u}_{j}\bm{x}^{\mathsf{H}}+\bm{u}_{j}\bm{x}^{\mathsf{H}}\otimes\bm{h}\bm{v}_{j}^{\mathsf{H}}+\bm{u}_{j}\bm{x}^{\mathsf{H}}\otimes\bm{u}_{j}\bm{x}^{\mathsf{H}}\right)\in\mathbb{C}^{K^{2}\times K^{2}},
\]
which satisfies
\[\text{vec}\left(\mathcal{P}_{T}\mathcal{A}^{*}\mathcal{A}\mathcal{P}_{T}(\bm X)\right)=\mathcal{Q}\text{vec}\left(\bm{X}\right) \]
for any $\bm{X}\in\mathbb{C}^{K\times K}$.
This implies that 
\begin{align}
& \left\Vert \mathcal{P}_{T}\mathcal{A}^{*}\mathcal{A}\mathcal{P}_{T}-m\mathcal{P}_{T}\right\Vert \notag\\
& \quad=  \left\Vert \mathcal{P}_{T}\mathcal{A}^{*}\mathcal{A}\mathcal{P}_{T}-\mathbb{E}\left[\mathcal{P}_{T}\mathcal{A}^{*}\mathcal{A}\mathcal{P}_{T}\right]\right\Vert= \left\Vert \mathcal{Q}-\mathbb{E}\left[\mathcal{Q}\right]\right\Vert \notag\\
&\quad\leq\underbrace{\left\Vert\sum_{j=1}^{m}\left(\bm{h}\bm{v}_{j}^{\mathsf{H}}\otimes\bm{h}\bm{v}_{j}^{\mathsf{H}}-\mathbb{E}\left[\bm{h}\bm{v}_{j}^{\mathsf{H}}\otimes\bm{h}\bm{v}_{j}^{\mathsf{H}}\right]\right)\right\Vert }_{\beta_1} +\underbrace{\left\Vert\sum_{j=1}^{m}\left(\bm{h}\bm{v}_{j}^{\mathsf{H}}\otimes\bm{u}_{j}\bm{x}^{\mathsf{H}}-\mathbb{E}\left[\bm{h}\bm{v}_{j}^{\mathsf{H}}\otimes\bm{u}_{j}\bm{x}^{\mathsf{H}}\right]\right)\right\Vert}_{\beta_2} \nonumber \\
&\qquad+\underbrace{\left\Vert\sum_{j=1}^{m}\left(\bm{u}_{j}\bm{x}^{\mathsf{H}}\otimes\bm{h}\bm{v}_{j}^{\mathsf{H}}-\bm{u}_{j}\bm{x}^{\mathsf{H}}\otimes\bm{h}\bm{v}_{j}^{\mathsf{H}}\right)\right\Vert}_{\beta_3}+ +\underbrace{\left\Vert\sum_{j=1}^{m}\left(\bm{u}_{j}\bm{x}^{\mathsf{H}}\otimes\bm{u}_{j}\bm{x}^{\mathsf{H}}-\mathbb{E}\left[\bm{u}_{j}\bm{x}^{\mathsf{H}}\otimes\bm{u}_{j}\bm{x}^{\mathsf{H}}\right]\right)\right\Vert}_{\beta_4} \label{eq:lemma33-1} 
\end{align}
In the sequel, we consider the four terms on the right-hand side of \eqref{eq:lemma33-1} separately. 

\paragraph{Controlling $\beta_1$.}
Regarding the first term $\beta_1$, we denote
\[
\bm{Z}_{j}\coloneqq\bm{h}\bm{v}_{j}^{\mathsf{H}}\otimes\bm{h}\bm{v}_{j}^{\mathsf{H}}-\mathbb{E}\left[\bm{h}\bm{v}_{j}^{\mathsf{H}}\otimes\bm{h}\bm{v}_{j}^{\mathsf{H}}\right].
\]
Then one has 
\begin{align*}
\left\Vert \left\Vert \bm{Z}_{j}\right\Vert \right\Vert _{\psi_{1}} & =\left\Vert \left\Vert \left\{ \left(\left|\left\langle \bm{h},\bm{b}_{j}\right\rangle \right|^{2}\overline{a_{ji}}a_{jk}-\delta_{ik}\right)\bm{h}\bm{h}^{\mathsf{H}}\right\} _{i,k}\right\Vert \right\Vert _{\psi_{1}}\\
 & \leq\left\Vert \bm{h}\bm{h}^{\mathsf{H}}\right\Vert \cdot\left\Vert \left\Vert \left|\left\langle \bm{h},\bm{b}_{j}\right\rangle \right|^{2}\bm{a}_{j}\bm{a}_{j}^{\mathsf{H}}-\bm{I}\right\Vert \right\Vert _{\psi_{1}}\\
 & \overset{(\text{i})}{\leq}\left\Vert \max\left\{ \left|\left\langle \bm{h},\bm{b}_{j}\right\rangle \right|^{2}\cdot\left\Vert \bm{a}_{j}\right\Vert _{2}^{2},1\right\} \right\Vert _{\psi_{1}}\\
 & \leq\left|\left\langle \bm{h},\bm{b}_{j}\right\rangle \right|^{2}\cdot\Big\|\|\boldsymbol{a}_{j}\|_{2}\Big\|_{\psi_{2}}^{2}+1\\
 & \overset{(\text{ii})}{\leq}CK\log m,
\end{align*}
where (i) is due to the fact that $\Vert\bm{h}\bm{h}^{\mathsf{H}}\Vert=\Vert\bm{h}\Vert_{2}^{2}=1$;
(ii) uses (\ref{eq:useful1}) and $\big\|\Vert\bm{b}_{j}\Vert_{2}\big\|_{\psi_{2}}\lesssim\sqrt{K}$
(cf.~\citet[Theorem 3.1.1]{vershynin2018high}). To compute the variance
term $\mathbb{E}[\bm{Z}_{j}^{\mathsf{H}}\bm{Z}_{j}]$ and $\mathbb{E}[\bm{Z}_{j}\bm{Z}_{j}^{\mathsf{H}}]$,
we express the operation of $\bm{Z}_{j}$ on a matrix $\bm{X}$ as
\[
\bm{Z}_{j}\left(\bm{X}\right)=\left|\left\langle \bm{h},\bm{b}_{j}\right\rangle \right|^{2}\bm{h}\bm{h}^{\mathsf{H}}\bm{X}\bm{a}_{j}\bm{a}_{j}^{\mathsf{H}}-\left\Vert \bm{h}\right\Vert _{2}^{2}\bm{h}\bm{h}^{\mathsf{H}}\bm{X},
\]
and hence
\[
\bm{Z}_{j}^{\mathsf{H}}\bm{Z}_{j}\left(\bm{X}\right)=\left|\left\langle \bm{h},\bm{b}_{j}\right\rangle \right|^{4}\left\Vert \bm{h}\right\Vert _{2}^{2}\left\Vert \bm{a}_{j}\right\Vert _{2}^{2}\bm{h}\bm{h}^{\mathsf{H}}\bm{X}\bm{a}_{j}\bm{a}_{j}^{\mathsf{H}}-2\left|\left\langle \bm{h},\bm{b}_{j}\right\rangle \right|^{2}\left\Vert \bm{h}\right\Vert _{2}^{4}\bm{h}\bm{h}^{\mathsf{H}}\bm{X}\bm{a}_{j}\bm{a}_{j}^{\mathsf{H}}+\left\Vert \bm{h}\right\Vert _{2}^{6}\bm{h}\bm{h}^{\mathsf{H}}\bm{X}.
\]
Then one has
\[
\mathbb{E}\left[\bm{Z}_{j}^{\mathsf{H}}\bm{Z}_{j}\left(\bm{X}\right)\right]=3\left(K+2\right)\left\Vert \bm{h}\right\Vert _{2}^{4}\bm{h}\bm{h}^{\mathsf{H}}\bm{X}-2\left\Vert \bm{h}\right\Vert _{2}^{6}\bm{h}\bm{h}^{\mathsf{H}}\bm{X}+\left\Vert \bm{h}\right\Vert _{2}^{6}\bm{h}\bm{h}^{\mathsf{H}}\bm{X}=\left(3K+5\right)\left\Vert \bm{h}\right\Vert _{2}^{4}\bm{h}\bm{h}^{\mathsf{H}}\bm{X}.
\]
Similarly, one can derive that 
\[
\mathbb{E}\left[\bm{Z}_{j}\bm{Z}_{j}^{\mathsf{H}}\left(\bm{X}\right)\right]=\mathbb{E}\left[\bm{Z}_{j}^{\mathsf{H}}\bm{Z}_{j}\left(\bm{X}\right)\right]=\left(3K+5\right)\left\Vert \bm{h}\right\Vert _{2}^{4}\bm{h}\bm{h}^{\mathsf{H}}\bm{X}, 
\]
thus indicating that 
\[
\sigma_{\bm{Z}}:=\max\left\{ \left\Vert \sum_{j=1}^{m}\mathbb{\mathbb{E}}\big[\bm{Z}_{j}^{\mathsf{H}}\bm{Z}_{j}\big]\right\Vert ^{1/2},\left\Vert \sum_{j=1}^{m}\mathbb{\mathbb{E}}\big[\bm{Z}_{j}\bm{Z}_{j}^{\mathsf{H}}\big]\right\Vert ^{1/2}\right\} \leq\sqrt{\left(3K+5\right)m\left\Vert \bm{h}\right\Vert _{2}^{6}}.
\]
By the matrix Bernstein inequality \citet[Proposition 2]{koltchinskii2011nuclear},
one has 
\begin{equation}
\left\Vert \sum_{j=1}^{m}\bm{Z}_{j}\right\Vert \lesssim\sigma_{\bm{Z}}\sqrt{\log m}+B_{\bm{Z}}\log\left(\frac{B_{\bm{Z}}\sqrt{m}}{\sigma_{\bm{Z}}}\right)\log m\lesssim\sqrt{mK\log m}\label{eq:gaussian-projection-1}
\end{equation}
with high probability. 

\paragraph{Controlling $\beta_2$.}
When it comes to the second term $\beta_2$, we first set 
\[
\bm{H}_{j}\coloneqq\bm{h}\bm{v}_{j}^{\mathsf{H}}\otimes\bm{u}_{j}\bm{x}^{\mathsf{H}}-\mathbb{E}\left[\bm{h}\bm{v}_{j}^{\mathsf{H}}\otimes\bm{u}_{j}\bm{x}^{\mathsf{H}}\right],
\]
which satisfies 
\begin{align*}
\left\Vert \bm{H}_{j}\right\Vert  & =\left\Vert \left\{ \overline{\left\langle \bm{h},\bm{b}_{j}\right\rangle }\overline{\left\langle \bm{x},\bm{a}_{j}\right\rangle }\overline{a_{ji}}x_{k}\bm{h}\bm{b}_{j}^{\mathsf{H}}\left(\bm{I}_{K}-\bm{h}\bm{h}^{\mathsf{H}}\right)\right\} _{i,k}\right\Vert \\
 & \leq\left\Vert \left\langle \bm{h},\bm{b}_{j}\right\rangle \left\langle \bm{x},\bm{a}_{j}\right\rangle \bm{a}_{j}\bm{x}^{\mathsf{H}}\right\Vert \cdot\left\Vert \bm{h}\bm{b}_{j}^{\mathsf{H}}\left(\bm{I}_{K}-\bm{h}\bm{h}^{\mathsf{H}}\right)\right\Vert \\
 & \leq\left|\left\langle \bm{h},\bm{b}_{j}\right\rangle \right|\cdot\left\Vert \left\langle \bm{x},\bm{a}_{j}\right\rangle \bm{a}_{j}\right\Vert \cdot\left\Vert \bm{x}\right\Vert _{2}\cdot\left\Vert \bm{h}\right\Vert _{2}\cdot\left\Vert \bm{b}_{j}\right\Vert _{2}\cdot\left\Vert \bm{I}_{K}-\bm{h}\bm{h}^{\mathsf{H}}\right\Vert \\
 & \leq\left|\left\langle \bm{h},\bm{b}_{j}\right\rangle \right|\cdot\left|\left\langle \bm{x},\bm{a}_{j}\right\rangle \right|\cdot\left\Vert \bm{a}_{j}\right\Vert _{2}\cdot\left\Vert \bm{x}\right\Vert _{2}\cdot\left\Vert \bm{h}\right\Vert _{2}\cdot\left\Vert \bm{b}_{j}\right\Vert _{2}\cdot\left\Vert \bm{I}_{K}-\bm{h}\bm{h}^{\mathsf{H}}\right\Vert 
\end{align*}
By employing $\Vert\bm{h}\Vert_{2}=\Vert\bm{x}\Vert_{2}=1$, (\ref{eq:useful1})
and $\big\|\Vert\bm{a}_{j}\Vert_{2}\big\|_{\psi_{2}}=\big\|\Vert\bm{b}_{j}\Vert_{2}\big\|_{\psi_{2}}\lesssim\sqrt{K}$
(cf.~\citet[Theorem 3.1.1]{vershynin2018high}), we obtain
\begin{align*}
\left\Vert \left\Vert \bm{H}_{j}\right\Vert \right\Vert _{\psi_{1}} & \leq CK\log m.
\end{align*}

Next, let us consider the operation of $\bm{H}_{j}$ and $\bm{H}_{j}^{\mathsf{H}}$
on $\bm{X}$, which obeys
\begin{align*}
\bm{H}_{j}\left(\bm{X}\right) & =\overline{\left\langle \bm{h},\bm{b}_{j}\right\rangle }\overline{\left\langle \bm{x},\bm{a}_{j}\right\rangle }\bm{h}\bm{b}_{j}^{\mathsf{H}}\left(\bm{I}_{K}-\bm{h}\bm{h}^{\mathsf{H}}\right)\bm{X}\bm{x}\bm{a}_{j}^{\mathsf{H}},\\
\bm{H}_{j}^{\mathsf{H}}\left(\bm{X}\right) & =\left\langle \bm{h},\bm{b}_{j}\right\rangle \left\langle \bm{x},\bm{a}_{j}\right\rangle \left(\bm{I}_{K}-\bm{h}\bm{h}^{\mathsf{H}}\right)\bm{b}_{j}\bm{h}^{\mathsf{H}}\bm{X}\bm{a}_{j}\bm{x}^{\mathsf{H}}.
\end{align*}
Consequently, one can deduce that
\begin{align*}
\bm{H}_{j}\bm{H}_{j}^{\mathsf{H}}\left(\bm{X}\right) & =\left|\left\langle \bm{h},\bm{b}_{j}\right\rangle \right|^{2}\left|\left\langle \bm{x},\bm{a}_{j}\right\rangle \right|^{2}\left\Vert \bm{x}\right\Vert _{2}^{2}\bm{h}\bm{b}_{j}^{\mathsf{H}}\left(\bm{I}_{K}-\bm{h}\bm{h}^{\mathsf{H}}\right)\bm{b}_{j}\bm{h}^{\mathsf{H}}\bm{X}\bm{a}_{j}\bm{a}_{j}^{\mathsf{H}},
\end{align*}
and 
\begin{align*}
\bm{H}_{j}^{\mathsf{H}}\bm{H}_{j}\left(\bm{X}\right) & =\left|\left\langle \bm{h},\bm{b}_{j}\right\rangle \right|^{2}\left|\left\langle \bm{x},\bm{a}_{j}\right\rangle \right|^{2}\left\Vert \bm{a}_{j}\right\Vert _{2}^{2}\left\Vert \bm{h}\right\Vert _{2}^{2}\left(\bm{I}_{K}-\bm{h}\bm{h}^{\mathsf{H}}\right)\bm{b}_{j}\bm{b}_{j}^{\mathsf{H}}\left(\bm{I}_{K}-\bm{h}\bm{h}^{\mathsf{H}}\right)\bm{X}\bm{x}\bm{x}^{\mathsf{H}}.
\end{align*}
It follows that their expectations are 
\begin{align*}
\mathbb{E}\left[\bm{H}_{j}\bm{H}_{j}^{\mathsf{H}}\left(\bm{X}\right)\right] & =\left[\left(K+2\right)\left\Vert \bm{h}\right\Vert _{2}^{2}-3\left\Vert \bm{h}\right\Vert _{2}^{4}\right]\bm{h}\bm{h}^{\mathsf{H}}\bm{X}\left(2\bm{x}\bm{x}^{\mathsf{H}}+\left\Vert \bm{x}\right\Vert _{2}^{2}\bm{I}_{K}\right),
\end{align*}
and
\begin{align*}
\mathbb{E}\left[\bm{H}_{j}^{\mathsf{H}}\bm{H}_{j}\left(\bm{X}\right)\right] & =\left(K+2\right)\left\Vert \bm{h}\right\Vert _{2}^{2}\left\Vert \bm{x}\right\Vert _{2}^{2}\left(\bm{I}_{K}-\bm{h}\bm{h}^{\mathsf{H}}\right)\left(2\bm{h}\bm{h}^{\mathsf{H}}+\left\Vert \bm{h}\right\Vert _{2}^{2}\bm{I}_{K}\right)\left(\bm{I}_{K}-\bm{h}\bm{h}^{\mathsf{H}}\right)\bm{X}\bm{x}\bm{x}^{\mathsf{H}}\\
 & =\left(K+2\right)\left(\bm{I}_{K}-\bm{h}\bm{h}^{\mathsf{H}}\right)\bm{X}\bm{x}\bm{x}^{\mathsf{H}}.
\end{align*}
Hence, we have 
\[
\sigma_{\bm{Z}}:=\max\left\{ \left\Vert \sum_{j=1}^{m}\mathbb{\mathbb{E}}\big[\bm{H}_{j}^{\mathsf{H}}\bm{H}_{j}\big]\right\Vert ^{1/2},\left\Vert \sum_{j=1}^{m}\mathbb{\mathbb{E}}\big[\bm{H}_{j}\bm{H}_{j}^{\mathsf{H}}\big]\right\Vert ^{1/2}\right\} \leq\sqrt{3mK}.
\]
By the matrix Bernstein inequality \citet[Proposition 2]{koltchinskii2011nuclear},
one has 
\begin{equation}
\left\Vert \sum_{j=1}^{m}\bm{H}_{j}\right\Vert \lesssim\sigma_{\bm{Z}}\sqrt{\log m}+B_{\bm{Z}}\log\left(\frac{B_{\bm{Z}}\sqrt{m}}{\sigma_{\bm{Z}}}\right)\log m\lesssim\sqrt{mK\log m}.\label{eq:gaussian-projection-2}
\end{equation}

\paragraph{Controlling $\beta_3$.}
When being written in matrix form, one has $\bm{u}_{j}\bm{x}^{\mathsf{H}}\otimes\bm{h}\bm{v}_{j}^{\mathsf{H}}-\mathbb{E}[\bm{u}_{j}\bm{x}^{\mathsf{H}}\otimes\bm{h}\bm{v}_{j}^{\mathsf{H}}]$
is the conjugate transpose of $\bm{h}\bm{v}_{j}^{\mathsf{H}}\otimes\bm{u}_{j}\bm{x}^{\mathsf{H}}-\mathbb{E}[\bm{h}\bm{v}_{j}^{\mathsf{H}}\otimes\bm{u}_{j}\bm{x}^{\mathsf{H}}]$,
so that their norms are the same and (\ref{eq:gaussian-projection-2})
holds for $\bm{u}_{j}\bm{x}^{\mathsf{H}}\otimes\bm{h}\bm{v}_{j}^{\mathsf{H}}-\mathbb{E}[\bm{u}_{j}\bm{x}^{\mathsf{H}}\otimes\bm{h}\bm{v}_{j}^{\mathsf{H}}]$
as well.

\paragraph{Controlling $\beta_4$.}
For the last term $\beta_4$, we denote
\[
\bm{W}_{j}\coloneqq\bm{u}_{j}\bm{x}^{\mathsf{H}}\otimes\bm{u}_{j}\bm{x}^{\mathsf{H}}-\mathbb{E}\left[\bm{u}_{j}\bm{x}^{\mathsf{H}}\otimes\bm{u}_{j}\bm{x}^{\mathsf{H}}\right],
\]
which satisfies 
\begin{align*}
\left\Vert \bm{W}_{j}\right\Vert  & =\left\Vert \left\{ \overline{x_{i}}x_{k}\left|\left\langle \bm{x},\bm{a}_{j}\right\rangle \right|^{2}\left(\bm{I}_{K}-\bm{h}\bm{h}^{\mathsf{H}}\right)\bm{b}_{j}\bm{b}_{j}^{\mathsf{H}}\left(\bm{I}_{K}-\bm{h}\bm{h}^{\mathsf{H}}\right)-\overline{x_{i}}x_{k}\left\Vert \bm{x}\right\Vert _{2}^{2}\left(\bm{I}_{K}-\bm{h}\bm{h}^{\mathsf{H}}\right)\right\} _{i,k}\right\Vert \\
 & \overset{(\text{i})}{\leq}\left\Vert \left\{ \overline{x_{i}}x_{k}\left|\left\langle \bm{x},\bm{a}_{j}\right\rangle \right|^{2}\left(\bm{I}_{K}-\bm{h}\bm{h}^{\mathsf{H}}\right)\bm{b}_{j}\bm{b}_{j}^{\mathsf{H}}\left(\bm{I}_{K}-\bm{h}\bm{h}^{\mathsf{H}}\right)\right\} _{i,k}\right\Vert +\left\Vert \left\{ \overline{x_{i}}x_{k}\left\Vert \bm{x}\right\Vert _{2}^{2}\left(\bm{I}_{K}-\bm{h}\bm{h}^{\mathsf{H}}\right)\right\} _{i,k}\right\Vert \\
 & \leq\left\Vert \bm{I}_{K}-\bm{h}\bm{h}^{\mathsf{H}}\right\Vert ^{2}\left\Vert \bm{b}_{j}\bm{b}_{j}^{\mathsf{H}}\right\Vert \left|\left\langle \bm{x},\bm{a}_{j}\right\rangle \right|^{2}\left\Vert \bm{x}\right\Vert _{2}^{2}+\left\Vert \bm{I}_{K}-\bm{h}\bm{h}^{\mathsf{H}}\right\Vert \left\Vert \bm{x}\bm{x}^{\mathsf{H}}\right\Vert \left\Vert \bm{x}\right\Vert _{2}^{2}\\
 & \overset{(\text{ii})}{\leq}\left\Vert \bm{b}_{j}\right\Vert _{2}^{2}\left|\left\langle \bm{x},\bm{a}_{j}\right\rangle \right|^{2}+1. 
\end{align*}
Here, (i) is due to the triangle inequality, and (ii) applies $\Vert\bm{h}\Vert_{2}=\Vert\bm{x}\Vert_{2}=1$
and the fact that $\Vert\bm{I}_{K}-\bm{h}\bm{h}^{\mathsf{H}}\Vert\le1$.
It then follows that
\begin{align*}
\left\Vert \left\Vert \bm{W}_{j}\right\Vert \right\Vert _{\psi_{1}} & \leq\max_{1\leq j\leq m}\left|\left\langle \bm{x},\bm{a}_{j}\right\rangle \right|^{2}\cdot\left\Vert \left\Vert \bm{b}_{j}\right\Vert _{2}\right\Vert _{\psi_{2}}^{2}+1
 & \leq CK\log m,
\end{align*}
where the second inequality uses (\ref{eq:useful1}) and $\big\|\Vert\bm{b}_{j}\Vert_{2}\big\|_{\psi_{2}}\lesssim\sqrt{K}$
(cf.~\citet[Theorem 3.1.1]{vershynin2018high}). To calculate the
variance term, one observes that
\[
\bm{W}_{j}\left(\bm{X}\right)=\bm{W}_{j}^{\mathsf{H}}\left(\bm{X}\right)=\left|\left\langle \bm{h},\bm{b}_{j}\right\rangle \right|^{2}\bm{h}\bm{h}^{\mathsf{H}}\bm{X}\bm{a}_{j}\bm{a}_{j}^{\mathsf{H}}-\left\Vert \bm{h}\right\Vert _{2}^{2}\bm{h}\bm{h}^{\mathsf{H}}\bm{X},
\]
which gives
\[
\bm{W}_{j}^{\mathsf{H}}\bm{W}_{j}\left(\bm{X}\right)=\left|\left\langle \bm{h},\bm{b}_{j}\right\rangle \right|^{4}\left\Vert \bm{h}\right\Vert _{2}^{2}\left\Vert \bm{a}_{j}\right\Vert _{2}^{2}\bm{h}\bm{h}^{\mathsf{H}}\bm{X}\bm{a}_{j}\bm{a}_{j}^{\mathsf{H}}-2\left|\left\langle \bm{h},\bm{b}_{j}\right\rangle \right|^{2}\left\Vert \bm{h}\right\Vert _{2}^{4}\bm{h}\bm{h}^{\mathsf{H}}\bm{X}\bm{a}_{j}\bm{a}_{j}^{\mathsf{H}}+\left\Vert \bm{h}\right\Vert _{2}^{6}\bm{h}\bm{h}^{\mathsf{H}}\bm{X}.
\]
It is then seen that
\[
\mathbb{E}\left[\bm{W}_{j}^{\mathsf{H}}\bm{W}_{j}\left(\bm{X}\right)\right]=3\left(K+2\right)\left\Vert \bm{h}\right\Vert _{2}^{4}\bm{h}\bm{h}^{\mathsf{H}}\bm{X}-2\left\Vert \bm{h}\right\Vert _{2}^{6}\bm{h}\bm{h}^{\mathsf{H}}\bm{X}+\left\Vert \bm{h}\right\Vert _{2}^{6}\bm{h}\bm{h}^{\mathsf{H}}\bm{X}=\left(3K+5\right)\bm{h}\bm{h}^{\mathsf{H}}\bm{X} 
\]
and 
\[
\mathbb{E}\left[\bm{W}_{j}^{\mathsf{H}}\bm{W}_{j}\left(\bm{X}\right)\right]=\mathbb{E}\left[\bm{W}_{j}\bm{W}_{j}^{\mathsf{H}}\left(\bm{X}\right)\right]=\left(3K+5\right)\bm{h}\bm{h}^{\mathsf{H}}\bm{X}.
\]
Therefore, one has 
\[
\sigma_{\bm{Z}}:=\max\left\{ \left\Vert \sum_{j=1}^{m}\mathbb{\mathbb{E}}\big[\bm{W}_{j}^{\mathsf{H}}\bm{W}_{j}\big]\right\Vert ^{1/2},\left\Vert \sum_{j=1}^{m}\mathbb{\mathbb{E}}\big[\bm{W}_{j}\bm{W}_{j}^{\mathsf{H}}\big]\right\Vert ^{1/2}\right\} \leq\sqrt{\left(3K+5\right)m}.
\]
By the matrix Bernstein inequality \citet[Proposition 2]{koltchinskii2011nuclear},
one has 
\begin{equation}
\left\Vert \sum_{j=1}^{m}\bm{W}_{j}\right\Vert \lesssim\sigma_{\bm{Z}}\sqrt{\log m}+B_{\bm{Z}}\log\left(\frac{B_{\bm{Z}}\sqrt{m}}{\sigma_{\bm{Z}}}\right)\log m\lesssim\sqrt{mK\log m}.\label{eq:gaussian-projection-4}
\end{equation}

\paragraph{Putting all this together.}
Plugging (\ref{eq:gaussian-projection-1}), (\ref{eq:gaussian-projection-2})
and (\ref{eq:gaussian-projection-4}) into \eqref{eq:lemma33-1} yields that with probability at least $1-O(m^{-10})$, 
\[
\left\Vert \mathcal{P}_{T}\mathcal{A}^{*}\mathcal{A}\mathcal{P}_{T}-m\mathcal{P}_{T}\right\Vert \leq C\sqrt{mK\log m}
\]
holds for some large enough
constant $C>0$.

\section{Proof of Theorem \ref{thm:LB}\label{appendix:LB}}

The proof of this lower bound is rather standard, and hence we only
provide a proof sketch here. First of all, it suffices to consider
the case where $\bm{h},\bm{x}\in\mathbb{R}^{K}$. We assume that $\bm{h}^{\star}\sim\mathcal{N}(\bm{0},\bm{I}_{K})$
and suppose that there is an oracle informing us of $\bm{h}^{\star}$,
which reduces the problem to estimating $\bm{x}^{\star}$ from linear
measurements
\[
\bm{y}=\widetilde{\bm{A}}\bm{x}^{\star}+\bm{\xi},
\]
where $\widetilde{\bm{A}}\coloneqq[\widetilde{\bm{a}}_{1},\widetilde{\bm{a}}_{2},\cdots,\widetilde{\bm{a}}_{m}]^{\mathsf{H}}$
with $\widetilde{\bm{a}}_{j}=\overline{\bm{b}_{j}^{\mathsf{H}}\bm{h}^{\star}}\bm{a}_{j}$.
Denoting by $\widetilde{\bm{A}}_{\mathsf{re}}$ and $\widetilde{\bm{A}}_{\mathsf{im}}$
the real and the imaginary part of $\widetilde{\bm{A}}$, respectively,
the standard minimax risk results for linear regression (e.g.~\citet[Lemma 3.11]{candes2011tight})
gives
\begin{align}
\inf_{\widehat{\bm{x}}}\sup_{\bm{x}^{\star}\in\mathbb{C}^{K}}\mathbb{E}\left[\left\Vert \widehat{\bm{x}}-\bm{x}^{\star}\right\Vert _{2}^{2}\big|\bm{A}\right] & =\frac{1}{2}\sigma^{2}\left(\mathsf{tr}\left[\big(\widetilde{\bm{A}}_{\mathsf{re}}^{\top}\widetilde{\bm{A}}_{\mathsf{re}}\big)^{-1}\right]+\mathsf{tr}\left[\big(\widetilde{\bm{A}}_{\mathsf{im}}^{\top}\widetilde{\bm{A}}_{\mathsf{im}}\big)^{-1}\right]\right)\nonumber \\
 & \geq K\sigma^{2}/\max\left\{ \big\|\widetilde{\bm{A}}_{\mathsf{re}}\big\|^{2},\big\|\widetilde{\bm{A}}_{\mathsf{im}}\big\|^{2}\right\} ,\label{eq:LB-123}
\end{align}
where the infimum is over all estimator $\widehat{\bm{x}}$. It is
known from standard Gaussian concentration results that, with high
probability,
\[
\max\left\{ \big\|\widetilde{\bm{A}}_{\mathsf{re}}\big\|,\big\|\widetilde{\bm{A}}_{\mathsf{im}}\big\|\right\} \leq\left\{ \max_{1\leq j\leq m}\left|\bm{b}_{j}^{\mathsf{H}}\bm{h}^{\star}\right|\right\} \left\Vert \bm{A}\right\Vert \lesssim\sqrt{\frac{K}{m}\log m}\cdot\sqrt{m}\asymp\sqrt{K\log m},
\]
which together with \eqref{eq:LB-123} gives
\[
\inf_{\widehat{\bm{x}}}\sup_{\bm{x}^{\star}\in\mathbb{C}^{K}}\mathbb{E}\left[\left\Vert \widehat{\bm{x}}-\bm{x}^{\star}\right\Vert _{2}^{2}\big|\bm{A}\right]\gtrsim\sigma^{2}/\log m.
\]
In turn, this oracle lower bound implies that, with high probability,
\begin{align*}
\inf_{\widehat{\bm{Z}}}\sup_{\bm{Z}^{\star}\in\mathcal{M}^{\star}}\mathbb{E}\left[\big\|\widehat{\bm{Z}}-\bm{Z}^{\star}\big\|_{\mathrm{F}}^{2}\mid\bm{A}\right] & \gtrsim\inf_{\widehat{\bm{x}}}\sup_{\bm{x}^{\star}\in\mathbb{C}^{K}}\mathbb{E}\left[\big\|\bm{h}^{\star}\widehat{\bm{x}}^{\mathsf{H}}-\bm{h}^{\star}\bm{x}^{\star\mathsf{H}}\big\|_{\mathrm{F}}^{2}\mid\bm{A}\right]\asymp\inf_{\widehat{\bm{x}}}\sup_{\bm{x}^{\star}\in\mathbb{C}^{K}}\mathbb{E}\left[\left\Vert \widehat{\bm{x}}-\bm{x}^{\star}\right\Vert _{2}^{2}\left\Vert \bm{h}^{\star}\right\Vert _{2}^{2}\mid\bm{A}\right]\\
 & \gtrsim\sigma^{2}K/\log m.
\end{align*}
Similarly, for the second case, we assume that $\bm{h}^{\star}$ is
a unit vector and there is an oracle informing us of $\bm{h}^{\star}$.
Then we again relates the problem to estimating $\bm{x}^{\star}$
from linear measurements
\[
\bm{y}=\check{\bm{A}}\bm{x}^{\star}+\bm{\xi},
\]
where $\check{\bm{A}}\coloneqq[\check{\bm{a}}_{1},\check{\bm{a}}_{2},\cdots,\check{\bm{a}}_{m}]^{\mathsf{H}}$
with $\check{\bm{a}}_{j}=\overline{\bm{b}_{j}^{\mathsf{H}}\bm{h}^{\star}}\bm{a}_{j}$.
Denoting by $\check{\bm{A}}_{\mathsf{re}}$ and $\check{\bm{A}}_{\mathsf{im}}$
the real and the imaginary part of $\check{\bm{A}}$, respectively.
Similar to \eqref{eq:LB-123}, one has 
\begin{align}
\inf_{\widehat{\bm{x}}}\sup_{\bm{x}^{\star}\in\mathbb{C}^{K}}\mathbb{E}\left[\left\Vert \widehat{\bm{x}}-\bm{x}^{\star}\right\Vert _{2}^{2}\big|\bm{A},\bm{B}\right] & =\frac{1}{2}\sigma^{2}\left(\mathsf{tr}\left[\big(\check{\bm{A}}_{\mathsf{re}}^{\top}\check{\bm{A}}_{\mathsf{re}}\big)^{-1}\right]+\mathsf{tr}\left[\big(\check{\bm{A}}_{\mathsf{im}}^{\top}\check{\bm{A}}_{\mathsf{im}}\big)^{-1}\right]\right)\nonumber \\
 & \geq K\sigma^{2}/\max\left\{ \big\|\check{\bm{A}}_{\mathsf{re}}\big\|^{2},\big\|\check{\bm{A}}_{\mathsf{im}}\big\|^{2}\right\} ,\label{eq:gaussian-LB-123}
\end{align}
by the standard minimax risk results for linear regression (e.g.~\citet[Lemma 3.11]{candes2011tight}).
From standard Gaussian concentration, we have, with high probability,
\[
\max\left\{ \big\|\check{\bm{A}}_{\mathsf{re}}\big\|,\big\|\check{\bm{A}}_{\mathsf{im}}\big\|\right\} \leq\left\{ \max_{1\leq j\leq m}\left|\bm{b}_{j}^{\mathsf{H}}\bm{h}^{\star}\right|\right\} \left\Vert \bm{A}\right\Vert \lesssim\sqrt{\log m}\cdot\sqrt{m}\asymp\sqrt{m\log m},
\]
which taken collectively with \eqref{eq:gaussian-LB-123} gives
\[
\inf_{\widehat{\bm{x}}}\sup_{\bm{x}^{\star}\in\mathbb{C}^{K}}\mathbb{E}\left[\left\Vert \widehat{\bm{x}}-\bm{x}^{\star}\right\Vert _{2}^{2}\big|\bm{A},\bm{B}\right]\gtrsim\frac{\sigma^{2}K}{m\log m}.
\]
Hence, this oracle lower bound implies that, 
\begin{align*}
\inf_{\widehat{\bm{Z}}}\sup_{\bm{Z}^{\star}\in\mathcal{M}^{\star}}\mathbb{E}\left[\big\|\widehat{\bm{Z}}-\bm{Z}^{\star}\big\|_{\mathrm{F}}^{2}\mid\bm{A},\bm{B}\right] & \gtrsim\inf_{\widehat{\bm{x}}}\sup_{\bm{x}^{\star}\in\mathbb{C}^{K}}\mathbb{E}\left[\big\|\bm{h}^{\star}\widehat{\bm{x}}^{\mathsf{H}}-\bm{h}^{\star}\bm{x}^{\star\mathsf{H}}\big\|_{\mathrm{F}}^{2}\mid\bm{A},\bm{B}\right]\\
&\asymp\inf_{\widehat{\bm{x}}}\sup_{\bm{x}^{\star}\in\mathbb{C}^{K}}\mathbb{E}\left[\left\Vert \widehat{\bm{x}}-\bm{x}^{\star}\right\Vert _{2}^{2}\left\Vert \bm{h}^{\star}\right\Vert _{2}^{2}\mid\bm{A},\bm{B}\right]\\
 & \gtrsim\frac{\sigma^{2}K}{m\log m},
\end{align*}
with high probability.

\section{Auxiliary lemmas\label{appendix:auxiliary-lemmas}}

In this section, we collect several auxiliary lemmas that are useful
for the proofs of our main theorems. 

\begin{lemma}\label{lemma:useful}Consider any fixed vector $\bm{x}$
independent of $\{\bm{a}_{l}\}_{1\leq l\leq m}$. Then with probability
at least $1-O\left(m^{-100}\right)$, we have
\begin{equation}
\max_{1\leq l\leq m}\left|\bm{a}_{l}^{\mathsf{H}}\bm{x}\right|\leq20\sqrt{\log m}\left\Vert \bm{x}\right\Vert _{2}.\label{eq:useful1}
\end{equation}
Additionally, there exists some constant $C>0$ such that with probability
at least $1-O\left(me^{-CK}\right)$, we have
\begin{equation}
\max_{1\leq l\leq m}\left\Vert \bm{a}_{l}\right\Vert _{2}\leq10\sqrt{K}.\label{eq:useful2}
\end{equation}
\end{lemma}\begin{proof}The first result follows from standard Gaussian
concentration inequalities as well as the union bound. The second
claim results from \citet[Theorem 3.1.1]{vershynin2018high}.\end{proof}


\begin{lemma}\label{lem:strengthen}Fix an arbitrarily small constant
$\epsilon>0$. Suppose that Assumption \ref{assumptions:models} holds and $m\geq C\mu^{2}K\log^{2}m/\epsilon^{2}$
for some sufficiently large constant $C>0$. Then one has 
\[
\left\Vert \mathcal{P}_{T}\mathcal{A}^{*}\mathcal{A}\mathcal{P}_{T}-\mathcal{P}_{T}\right\Vert \leq\epsilon,
\]
with probability exceeding $1-O(m^{-10})$.\end{lemma}\begin{proof}This
has been established in \citet[Section 5.2]{ahmed2013blind}. \end{proof}

\begin{lemma}\label{lemma:gaussian-M}Under Assumption \ref{assumptions:models-gausssian},
one has 
\[
\left\Vert \frac{1}{m}\sum_{j=1}^{m}y_{j}\bm{b}_{j}\bm{a}_{j}^{\mathsf{H}}-\bm{h}^{\star}\bm{x}^{\star\mathsf{H}}\right\Vert \lesssim\frac{\sqrt{mK\log^{2}m}}{m}+\frac{\sigma\sqrt{mK\log m}}{m},
\]
holds with probability over $1-O(m^{-10})$, as long as $m>CK\log^{5}m$
for some large enough constant $C>0$.\end{lemma}\begin{proof}See
Appendix \ref{subsec:Proof-of-Lemma-gaussian-M}. \end{proof}
\subsection{Proof of Lemma \ref{lemma:gaussian-noise}\label{subsec:Proof-of-Lemmagaussian-noise}}

By the definition of $\mathcal{A}^{*}$, we have
\begin{align*}
\mathcal{A}^{*}\left(\bm{\xi}\right) & =\sum_{j=1}^{m}\underbrace{\xi_{j}\bm{b}_{j}\bm{a}_{j}^{\mathsf{H}}\ind_{\left\{ \left|\xi_{j}\right|\leq C\sigma\log m\right\} }}_{\eqqcolon\bm{X}_{j}}+\sum_{j=1}^{m}\xi_{j}\bm{b}_{j}\bm{a}_{j}^{\mathsf{H}}\ind_{\left\{ \left|\xi_{j}\right|>C\sigma\log m\right\} }.
\end{align*}
Since 
\begin{align*}
\mathbb{P}\left(\min_{1\leq j\leq m}\left|\xi_{j}\right|>C\sigma\log m\right) & \leq\sum_{j=1}^{m}\mathbb{P}\left(\left|\xi_{j}\right|>C\sigma\log m\right)\\
 & \leq O\left(m^{-100}\right),
\end{align*}
 for sufficiently large constant $C>0$, we have with probability
exceeding $1-O(m^{-10})$, that 
\begin{equation}
\left\Vert \mathcal{A}^{*}\left(\bm{\xi}\right)\right\Vert =\left\Vert \sum_{j=1}^{m}\bm{X}_{j}\right\Vert .\label{eq:gaussian-noise-2}
\end{equation}
To bound $\Vert\sum_{j=1}^{m}\bm{X}_{j}\Vert$, we proceed by applying
the matrix Bernstein inequality \citet[Proposition 2]{koltchinskii2011nuclear}.
One has
\begin{align*}
B_{\bm{Z}} & \coloneqq\left\Vert \left\Vert \xi_{j}\bm{b}_{j}\bm{a}_{j}^{\mathsf{H}}\ind_{\left\{ \left|\xi_{j}\right|\leq C\sigma\log m\right\} }\right\Vert \right\Vert _{\psi_{1}}\\
 & =\left\Vert \left|\xi_{j}\ind_{\left\{ \left|\xi_{j}\right|\leq C\sigma\log m\right\} }\right|\left\Vert \bm{b}_{j}\right\Vert _{2}\left\Vert \bm{a}_{j}\right\Vert _{2}\right\Vert _{\psi_{1}}\\
 & \overset{(\text{i})}{\leq}C\sigma\log m\left\Vert \left\Vert \bm{b}_{j}\right\Vert _{2}\right\Vert _{\psi_{2}}\left\Vert \left\Vert \bm{a}_{j}\right\Vert _{2}\right\Vert _{\psi_{2}}\\
 & \overset{(\text{ii})}{\lesssim}C\sigma K\log m,
\end{align*}
where (i) uses \citet[Lemma 2.7.7]{vershynin2018high} and (ii) is
due to the facts that $\Vert\Vert\bm{a}_{j}\Vert_{2}\Vert_{\psi_{2}}\lesssim\sqrt{K}$
and $\Vert\Vert\bm{b}_{j}\Vert_{2}\Vert_{\psi_{2}}\lesssim\sqrt{K}$
(cf.~\citet[Theorem 3.1.1]{vershynin2018high}). Next, we turn to
control the variance term. One has 
\begin{align*}
\left\Vert \sum_{j=1}^{m}\mathbb{E}\left[\bm{X}_{j}\bm{X}_{j}^{\mathsf{H}}\right]\right\Vert  & =\left\Vert \sum_{j=1}^{m}\mathbb{E}\left[\left|\xi_{j}\right|^{2}\bm{b}_{j}\bm{a}_{j}^{\mathsf{H}}\bm{a}_{j}\bm{b}_{j}^{\mathsf{H}}\ind_{\left\{ \left|\xi_{j}\right|\leq C\sigma\log m\right\} }\right]\right\Vert \\
 & =\left\Vert \sum_{j=1}^{m}\mathbb{E}\left[\left|\xi_{j}\right|^{2}\ind_{\left\{ \left|\xi_{j}\right|\leq C\sigma\log m\right\} }\right]\mathbb{E}\left[\bm{b}_{j}\bm{b}_{j}^{\mathsf{H}}\right]\mathbb{E}\left[\bm{a}_{j}^{\mathsf{H}}\bm{a}_{j}\right]\right\Vert \\
 & \leq\sigma^{2}mK.
\end{align*}
Since $\{\bm{a}_{j}\}_{j=1}^{m}$ have the same distribution as $\{\bm{b}_{j}\}_{j=1}^{m}$,
$\Vert\sum_{j=1}^{m}\mathbb{E}[\bm{X}_{j}^{\mathsf{H}}\bm{X}_{j}]\Vert$
can be controlled in the same way as above. Then, we have
\[
\sigma_{\bm{Z}}\coloneqq\max\left\{ \left\Vert \sum_{j=1}^{m}\mathbb{E}\left[\bm{X}_{j}\bm{X}_{j}^{\mathsf{H}}\right]\right\Vert ^{1/2},\left\Vert \sum_{j=1}^{m}\mathbb{E}\left[\bm{X}_{j}^{\mathsf{H}}\bm{X}_{j}\right]\right\Vert ^{1/2}\right\} \leq\sigma\sqrt{mK}.
\]
Now we are ready to invoke \citet[Proposition 2]{koltchinskii2011nuclear}
to derive that with probability over $1-O(m^{-20})$, there holds
\begin{equation}
\left\Vert \sum_{j=1}^{m}\bm{X}_{j}\right\Vert \lesssim\sigma_{\bm{Z}}\sqrt{\log m}+B_{\bm{Z}}\log\left(\frac{B_{\bm{Z}}\sqrt{m}}{\sigma_{\bm{Z}}}\right)\log m\lesssim\sigma\sqrt{mK\log m},\label{eq:gaussian-noise-1}
\end{equation}
where the last inequality holds as long as $m\gg K\log^{5}m$. Taking
(\ref{eq:gaussian-noise-1}) collectively with (\ref{eq:gaussian-noise-2}),
one has

\[
\left\Vert \mathcal{A}^{*}\left(\bm{\xi}\right)\right\Vert =\left\Vert \sum_{j=1}^{m}\bm{X}_{j}\right\Vert \lesssim\sigma\sqrt{mK\log m},
\]
holds with probability exceeding $1-O(m^{-10})$.

\subsection{Proof of Lemma \ref{lemma:gaussian-M}\label{subsec:Proof-of-Lemma-gaussian-M}}

Denote by $\bm{M}=\frac{1}{m}\sum_{j=1}^{m}y_{j}\bm{b}_{j}\bm{a}_{j}^{\mathsf{H}}$.
Then we have
\begin{align}
\left\Vert \bm{M}-\mathbb{E}\left[\bm{M}\right]\right\Vert  & =\left\Vert \frac{1}{m}\sum_{j=1}^{m}y_{j}\bm{b}_{j}\bm{a}_{j}^{\mathsf{H}}-\bm{h}^{\star}\bm{x}^{\star\mathsf{H}}\right\Vert \nonumber \\
 & \leq\frac{1}{m}\left\Vert \sum_{j=1}^{m}\bm{b}_{j}\bm{b}_{j}^{\mathsf{H}}\bm{h}^{\star}\bm{x}^{\star\mathsf{H}}\bm{a}_{j}\bm{a}_{j}^{\mathsf{H}}-m\bm{h}^{\star}\bm{x}^{\star\mathsf{H}}\right\Vert +\frac{1}{m}\left\Vert \sum_{j=1}^{m}\xi_{j}\bm{b}_{j}\bm{a}_{j}^{\mathsf{H}}\right\Vert .\label{eq:gaussian-base}
\end{align}
The second term can be bounded by Lemma \ref{lemma:gaussian-noise}.
We are left to control the first term. 

In view of (\ref{eq:useful2}),
one has
\begin{align}
 & \sum_{j=1}^{m}\bm{b}_{j}\bm{b}_{j}^{\mathsf{H}}\bm{h}^{\star}\bm{x}^{\star\mathsf{H}}\bm{a}_{j}\bm{a}_{j}^{\mathsf{H}}-m\bm{h}^{\star}\bm{x}^{\star\mathsf{H}}=  \sum_{j=1}^{m}\bm{b}_{j}\bm{b}_{j}^{\mathsf{H}}\bm{h}^{\star}\bm{x}^{\star\mathsf{H}}\bm{a}_{j}\bm{a}_{j}^{\mathsf{H}}\ind_{\left\{ \left|\bm{a}_{j}^{\mathsf{H}}\bm{x}^{\star}\bm{b}_{j}^{\mathsf{H}}\bm{h}^{\star}\right|\leq\left(20\sqrt{\log m}\right)^{2}\right\} }-m\bm{h}^{\star}\bm{x}^{\star\mathsf{H}},\label{eq:gaussian-M-truncate}
\end{align}
holds with probability over $1-O(m^{-100})$. 

\paragraph{Concentration.}
For any fixed unit vectors
$\bm{u}$ and $\bm{v}$, define 
\[
Z_{j}\coloneqq\bm{u}^{\mathsf{H}}\bm{b}_{j}\bm{b}_{j}^{\mathsf{H}}\bm{h}^{\star}\bm{x}^{\star\mathsf{H}}\bm{a}_{j}\bm{a}_{j}^{\mathsf{H}}\bm{v}\ind_{\left\{ \left|\bm{a}_{j}^{\mathsf{H}}\bm{x}^{\star}\bm{b}_{j}^{\mathsf{H}}\bm{h}^{\star}\right|\leq\left(20\sqrt{\log m}\right)^{2}\right\} }.
\]
Then we invoke the Bernstein inequality \citet[Theorem 2.8.2]{vershynin2018high}
to control $\Vert\sum_{j=1}^{m}(Z_{j}-\mathbb{E}[Z_{j}])\Vert$. We
have 
\[
\Big\| Z_{j}-\mathbb{E}\left[Z_{j}\right]\Big\|_{\psi_{1}}\leq C\left\Vert Z_{j}\right\Vert _{\psi_{1}}\leq400C\log m\left\Vert \bm{u}^{\mathsf{H}}\bm{b}_{j}\right\Vert _{\psi_{2}}\left\Vert \bm{a}_{j}^{\mathsf{H}}\bm{v}\right\Vert _{\psi_{2}}\lesssim\log m.
\]
Here, we have used $\Vert X-\mathbb{E}[X]\Vert_{\psi_{1}}\leq C\left\Vert X\right\Vert _{\psi_{1}}$
(cf.~\citet[Section 2.7]{vershynin2018high}). Then the Bernstein inequality
\citet[Theorem 2.8.2]{vershynin2018high} allows us to derive that
\[
\mathbb{P}\left(\left|\sum_{j=1}^{m}\left(Z_{j}-\mathbb{E}\left[Z_{j}\right]\right)\right|\geq t\right)\leq2\exp\left(-c\min\left(\frac{t^{2}}{m\log^{2}m},\frac{t}{\log m}\right)\right).
\]
Letting $t=C_{t}\sqrt{mK}\log m$ for some large enough constant $C_{t}>0$,
we obtain that 
\begin{equation}
	\left|\sum_{j=1}^{m}\left(X_{j}-\mathbb{E}\left[X_{j}\right]\right)\right|\leq C_{t}\sqrt{mK}\log m, \label{eq:lemma40-1}
\end{equation}
holds with probability exceeding $1-2\exp(-cC_{t}^{2}K)$. 

\paragraph{Union bound.}
Next, we
define $\mathcal{N}_{0}$ an $\epsilon$-net of the unit sphere $\mathcal{S}^{K-1}$.
In view of \citet[Corollary 4.2.13]{vershynin2018high}, we have 
\[
\left|\mathcal{N}_{0}\right|\leq\left(1+\frac{2}{\epsilon}\right)^{2K}.
\]
Taking this collectively with the union bound yields that \eqref{eq:lemma40-1} holds uniformly for any $\bm{x}\in\mathcal{N}_{\bm{x}}$ and $\bm{u}$,
$\bm{v}\in\mathcal{N}_{0}$ with probability over
\[
1-\left(1+\frac{2}{\epsilon}\right)^{4K}\cdot2\exp\left(-cC_{t}^{2}K\right)\geq1-2\exp\left(-CK\log m\right).
\]

\paragraph{Approximation. }
Then, for any $\bm{u}$, $\bm{v}\in\mathcal{S}^{K-1}$, one can choose $\bm{u}_{0}\in\mathcal{N}_{0}$
and $\bm{v}_{0}\in\mathcal{N}_{0}$ satisfying $\max\{\Vert\bm{u}-\bm{u}_{0}\Vert_{2},\Vert\bm{v}-\bm{v}_{0}\Vert_{2}\}\leq\epsilon_{2}$.
Let
\[
g\left(\bm{u},\bm{v}\right)\coloneqq\sum_{j=1}^{m}\left[\bm{u}^{\mathsf{H}}\bm{b}_{j}\bm{b}_{j}^{\mathsf{H}}\bm{h}^{\star}\bm{x}^{\star}{}^{\mathsf{H}}\bm{a}_{j}\bm{a}_{j}^{\mathsf{H}}\bm{v}\ind_{\left\{ \left|\bm{a}_{j}^{\mathsf{H}}\bm{x}^{\star}\bm{b}_{j}^{\mathsf{H}}\bm{h}^{\star}\right|\leq\left(20\sqrt{\log m}\right)^{2}\right\} }-m\bm{u}^{\mathsf{H}}\bm{h}^{\star}\bm{x}^{\star}{}^{\mathsf{H}}\bm{v}\right].
\]
Set $\epsilon=1/4$. By triangle inequality, one has 
\begin{align*}
\left|g\left(\bm{u},\bm{v}\right)-g\left(\bm{u}_{0},\bm{v}_{0}\right)\right| & \leq\left|g\left(\bm{u},\bm{v}\right)-g\left(\bm{u}_{0},\bm{v}\right)\right|+\left|g\left(\bm{u}_{0},\bm{v}\right)-g\left(\bm{u}_{0},\bm{v}_{0}\right)\right|\\
 & \leq2\epsilon\left\Vert \sum_{j=1}^{m}\bm{b}_{j}\bm{b}_{j}^{\mathsf{H}}\bm{h}^{\star}\bm{x}^{\star}{}^{\mathsf{H}}\bm{a}_{j}\bm{a}_{j}^{\mathsf{H}}\bm{v}\ind_{\left\{ \left|\bm{a}_{j}^{\mathsf{H}}\bm{x}^{\star}\bm{b}_{j}^{\mathsf{H}}\bm{h}^{\star}\right|\leq\left(20\sqrt{\log m}\right)^{2}\right\} }-m\bm{h}^{\star}\bm{x}^{\star}{}^{\mathsf{H}}\right\Vert .
\end{align*}

Considering $g(\bm{u}_{0},\bm{v}_{0})$, let 
\[Z_j\coloneqq\bm{u}^{\mathsf{H}}\bm{b}_{j}\bm{b}_{j}^{\mathsf{H}}\bm{h}^{\star}\bm{x}^{\star}{}^{\mathsf{H}}\bm{a}_{j}\bm{a}_{j}^{\mathsf{H}}\bm{v}\ind_{\left\{ \left|\bm{a}_{j}^{\mathsf{H}}\bm{x}^{\star}\bm{b}_{j}^{\mathsf{H}}\bm{h}^{\star}\right|\leq\left(20\sqrt{\log m}\right)^{2}\right\} }.\] One has
\begin{align*}
 & \left|g\left(\bm{u}_{0},\bm{v}_{0}\right)\right|\\
 & \quad\leq\left|\sum_{j=1}^{m}\left(Z_{j}-\mathbb{E}\left[Z_{j}\right]\right)\right|+\left|\sum_{j=1}^{m}\left(\mathbb{E}\left[Z_{j}\right]-m\bm{u}^{\mathsf{H}}\bm{h}^{\star}\bm{x}^{\star}{}^{\mathsf{H}}\bm{v}\right)\right|\\
 &\quad \leq C_{t}\sqrt{mK}\log m+\left|\sum_{j=1}^{m}\mathbb{E}\left[\bm{u}^{\mathsf{H}}\bm{b}_{j}\bm{b}_{j}^{\mathsf{H}}\bm{h}^{\star}\bm{x}^{\star}{}^{\mathsf{H}}\bm{a}_{j}\bm{a}_{j}^{\mathsf{H}}\bm{v}\ind_{\left\{ \left|\bm{a}_{j}^{\mathsf{H}}\bm{x}^{\star}\bm{b}_{j}^{\mathsf{H}}\bm{h}^{\star}\right|\leq\left(20\sqrt{\log m}\right)^{2}\right\} }\right]\right|\\
 &\quad \leq C_{t}\sqrt{mK}\log m+\sum_{j=1}^{m}\left|\mathbb{E}\left[\bm{u}^{\mathsf{H}}\bm{b}_{j}\bm{b}_{j}^{\mathsf{H}}\bm{h}^{\star}\bm{x}^{\star}{}^{\mathsf{H}}\bm{a}_{j}\bm{a}_{j}^{\mathsf{H}}\bm{v}\ind_{\left\{ \left|\bm{a}_{j}^{\mathsf{H}}\bm{x}^{\star}\bm{b}_{j}^{\mathsf{H}}\bm{h}^{\star}\right|\leq\left(20\sqrt{\log m}\right)^{2}\right\} }\right]\right| \\
 & \quad \leq 2C_{t}\sqrt{mK}\log m, 
\end{align*}
where we use \eqref{eq:lemma40-1} and 
\begin{align*}
 & \left|\mathbb{E}\left[\bm{u}^{\mathsf{H}}\bm{b}_{j}\bm{b}_{j}^{\mathsf{H}}\bm{h}^{\star}\bm{x}^{\star}{}^{\mathsf{H}}\bm{a}_{j}\bm{a}_{j}^{\mathsf{H}}\bm{v}\ind_{\left\{ \left|\bm{a}_{j}^{\mathsf{H}}\bm{x}^{\star}\bm{b}_{j}^{\mathsf{H}}\bm{h}^{\star}\right|\leq\left(20\sqrt{\log m}\right)^{2}\right\} }\right]\right|\\
&\quad\leq  \sqrt{\mathbb{E}\left[\left(\bm{u}^{\mathsf{H}}\bm{b}_{j}\bm{b}_{j}^{\mathsf{H}}\bm{h}^{\star}\bm{x}^{\star}{}^{\mathsf{H}}\bm{a}_{j}\bm{a}_{j}^{\mathsf{H}}\bm{v}\right)^{2}\right]\mathbb{P}\left(\left|\bm{a}_{j}^{\mathsf{H}}\bm{x}^{\star}\bm{b}_{j}^{\mathsf{H}}\bm{h}^{\star}\right|\leq\left(20\sqrt{\log m}\right)^{2}\right)}\\
&\quad\leq  O\left(m^{-100}\right).
\end{align*}
Hence we have 
\[
\left|g\left(\bm{u}_{0},\bm{v}_{0}\right)\right|\leq2C_{t}\sqrt{mK}\log m.
\]

\paragraph{Putting all this together.}
It then follows that
\begin{align*}
\left|g\left(\bm{u},\bm{v}\right)\right| & \leq\left|g\left(\bm{u}_{0},\bm{v}_{0}\right)\right|+\left|g\left(\bm{u},\bm{v}\right)-g\left(\bm{u}_{0},\bm{v}_{0}\right)\right|\\
 & \leq2C_{t}\sqrt{mK}\log m+2\epsilon\left\Vert \sum_{j=1}^{m}\bm{b}_{j}\bm{b}_{j}^{\mathsf{H}}\bm{h}^{\star}\bm{x}^{\star}{}^{\mathsf{H}}\bm{a}_{j}\bm{a}_{j}^{\mathsf{H}}\bm{v}\ind_{\left\{ \left|\bm{a}_{j}^{\mathsf{H}}\bm{x}^{\star}\bm{b}_{j}^{\mathsf{H}}\bm{h}^{\star}\right|\leq\left(20\sqrt{\log m}\right)^{2}\right\} }-m\bm{h}^{\star}\bm{x}^{\star}{}^{\mathsf{H}}\right\Vert .
\end{align*}
Taking maximum over $\bm{u}$ and $\bm{v}$ on the left side yields
that 
\begin{align*}
\max_{\bm{u},\bm{v}\in\mathcal{S}^{K-1}}\left|g\left(\bm{u},\bm{v}\right)\right| & =\left\Vert \sum_{j=1}^{m}\bm{b}_{j}\bm{b}_{j}^{\mathsf{H}}\bm{h}^{\star}\bm{x}^{\star\mathsf{H}}\bm{a}_{j}\bm{a}_{j}^{\mathsf{H}}\ind_{\left\{ \left|\bm{a}_{j}^{\mathsf{H}}\bm{x}^{\star}\bm{b}_{j}^{\mathsf{H}}\bm{h}^{\star}\right|\leq\left(20\sqrt{\log m}\right)^{2}\right\} }-m\bm{h}^{\star}\bm{x}^{\star\mathsf{H}}\right\Vert \\
 & \leq2C_{t}\sqrt{mK}\log m+2\epsilon\left\Vert \sum_{j=1}^{m}\bm{b}_{j}\bm{b}_{j}^{\mathsf{H}}\bm{h}^{\star}\bm{x}^{\star\mathsf{H}}\bm{a}_{j}\bm{a}_{j}^{\mathsf{H}}\ind_{\left\{ \left|\bm{a}_{j}^{\mathsf{H}}\bm{x}^{\star}\bm{b}_{j}^{\mathsf{H}}\bm{h}^{\star}\right|\leq\left(20\sqrt{\log m}\right)^{2}\right\} }-m\bm{h}^{\star}\bm{x}^{\star\mathsf{H}}\right\Vert .
\end{align*}
Rearranging terms and recalling $\epsilon=1/4$ give rise to
\begin{equation}
\left\Vert \sum_{j=1}^{m}\bm{b}_{j}\bm{b}_{j}^{\mathsf{H}}\bm{h}^{\star}\bm{x}^{\star}{}^{\mathsf{H}}\bm{a}_{j}\bm{a}_{j}^{\mathsf{H}}\bm{v}\ind_{\left\{ \left|\bm{a}_{j}^{\mathsf{H}}\bm{x}^{\star}\bm{b}_{j}^{\mathsf{H}}\bm{h}^{\star}\right|\leq\left(20\sqrt{\log m}\right)^{2}\right\} }-m\bm{h}^{\star}\bm{x}^{\star}{}^{\mathsf{H}}\right\Vert \leq4C_{t}\sqrt{mK}\log m.\label{eq:gaussian-M-truncate-2}
\end{equation}
Taking (\ref{eq:gaussian-M-truncate}) with (\ref{eq:gaussian-M-truncate-2})
collectively yields that
\begin{equation}
\left\Vert \sum_{j=1}^{m}\bm{b}_{j}\bm{b}_{j}^{\mathsf{H}}\bm{h}^{\star}\bm{x}^{\star}{}^{\mathsf{H}}\bm{a}_{j}\bm{a}_{j}^{\mathsf{H}}-m\bm{h}^{\star}\bm{x}^{\star}{}^{\mathsf{H}}\right\Vert \leq4C_{t}\sqrt{mK}\log m,\label{eq:gaussian-M-second-term}
\end{equation}
holds with probability at least $1-O(\exp(-CK\log m)+m^{-100})$.
 Plugging (\ref{eq:gaussian-M-second-term}) and (\ref{eq:gaussian-noise})
into (\ref{eq:gaussian-base}) gives the desired conclusion.

\end{document}